\documentclass[letterpaper, 10 pt, conference]{ieeeconf}  

\IEEEoverridecommandlockouts                              

\usepackage{nicefrac}       
\usepackage{microtype}      
\usepackage{graphicx}
\usepackage{subcaption}
\usepackage{wrapfig}
\usepackage{algorithm}
\usepackage[ruled,vlined,algo2e]{algorithm2e}
\usepackage{algorithmic}
\let\labelindent\relax
\usepackage{enumitem}
\usepackage{amsmath}

\usepackage{amsthm}
\usepackage{amssymb}
\usepackage{booktabs,tabularx,graphicx}
\usepackage{siunitx}
\usepackage{multirow}
\usepackage{xcolor}
\usepackage[strings]{underscore}
\usepackage{siunitx,booktabs}

\newcolumntype{P}[1]{>{\centering\arraybackslash}p{#1}}
\newcolumntype{C}{>{\centering\arraybackslash}X}

\title{\LARGE \bf
Neural Lyapunov Model Predictive Control:
\\ Learning Safe Global Controllers from Sub-optimal Examples}

\author{Mayank Mittal~$^*$, Marco Gallieri~$^*$, Alessio Quaglino, Seyed Sina Mirrazavi Salehian and Jan Koutn\'{i}k 
\thanks{M. Mittal and M. Gallieri provided equal contributions.}
\thanks{M. Mittal is with ETH Z\"{u}rich, Switzerland. 
        {\tt\small mittalma@ethz.ch}} %
\thanks{All authors are associated with NAISENSE SA, Lugano, Switzerland. This work was done by M. Mittal during his internship.}
}

\newtheorem{theorem}{Theorem}

\DeclareMathOperator*{\argmin}{arg\,min}
\SetKwInput{KwInput}{Input}
\SetKwInput{KwOutput}{Output}

\newcommand{\secref}[1]{Section~\ref{#1}}

\newcommand{\eqtref}[1]{Equation~\ref{#1}}
\newcommand{\figref}[1]{Figure~\ref{#1}}
\newcommand{\tabref}[1]{Table~\ref{#1}}
\newcommand{\algref}[1]{Algorithm~\ref{#1}}

\newcommand{\addition}[1]{{#1}}

\begin{document}

\maketitle
\thispagestyle{empty}
\pagestyle{empty}

\begin{abstract}

With a growing interest in data-driven control techniques, Model Predictive Control (MPC) provides an opportunity to exploit the surplus of data reliably, particularly while taking safety and stability into account. In many real-world and industrial applications, it is typical to have an existing control strategy, for instance, execution from a human operator. The objective of this work is to improve upon this unknown, safe but suboptimal policy by learning a new controller that retains safety and stability. Learning how to be safe is achieved directly from data and from a knowledge of the system constraints. The proposed algorithm alternatively learns the terminal cost and updates the MPC parameters according to a stability metric. The terminal cost is constructed as a Lyapunov function neural network with the aim of recovering or extending the stable region of the initial demonstrator using a short prediction horizon. Theorems that characterize the stability and performance of the learned MPC in the bearing of model uncertainties and sub-optimality due to function approximation are presented. The efficacy of the proposed algorithm is demonstrated on non-linear continuous control tasks with soft constraints. The proposed approach can improve upon the initial demonstrator also in practice and achieve better stability than popular reinforcement learning baselines. 

\end{abstract}

\section{Introduction}


Control systems comprise of safety requirements that need to be considered during the controller design process. In most applications, these are in the form of state/input constraints and convergence to an equilibrium point, a specific set or a trajectory. Typically, a control strategy that violates these specifications can
lead to \emph{unsafe} behavior. While learning-based methods are promising for solving challenging non-linear control problems, the lack of interpretability and provable safety guarantees impede their use in practical control  settings~\cite{amodi2016aisafety}. Model-based reinforcement learning (RL) with planning uses a surrogate model to minimize the sum of future costs plus a learned value function terminal cost \cite{moerland2020modelbased, lowrey_plan_2018}. Approximated value functions, however, do not offer safety guarantees. In contrast, control theory focuses on these guarantees but it is limited by its assumptions. Thus, there is a gap between theory and practice. 


A Control Lyapunov Function (CLF) is necessary and sufficient for stabilization~\cite{Khalil_book}. By exploiting the expressiveness of neural networks (NNs), Lyapunov NNs have been demonstrated as a general tool to produce non-conservative stability (and safety) certificates~\cite{bobiti_samplingdriven_nodate, bobiti_sampling-based_2016} and also improve an existing controller~\cite{berkenkamp_safe_2017, gallieri_safe_2019, chang_neural_2019}. In most of these settings, the controller is parameterized through a NN as well. The flexibility provided by this choice comes at the cost of increased sample complexity, which is often expensive in real-world~\textit{safety-critical} systems. In this work, we aim to overcome this limitation by leveraging an initial set of one-step transitions from an \textit{unknown} expert demonstrator (which may be sub-optimal) and by using a learned Lyapunov function and surrogate model within a Model Predictive Control (MPC) formulation. 

Our key contribution is an algorithmic framework, \emph{Neural Lyapunov MPC} (NLMPC), that obtains a single-step horizon MPC for Lyapunov-based control of non-linear deterministic systems with constraints. 
 We extend standard MPC stability results to a discounted setup and provide a theoretical bound for its performance using an imperfect forward model and a verified Lyapunov NN. This complements \cite{lowrey_plan_2018}, which considers a perfect model. We draw links between control and RL by proving that a positive advantage function can be sufficient for stability. The proposed approach relates to Advantage Weighted Regression (AWR) \cite{peng_advantage-weighted_2019} which favours positive advantage during the policy update. In this paper, however, the same is also done for the critic by detecting and favouring stability (advantage) over Bellman optimality. 
 
 We use alternate learning to train the Lyapunov NN and a scaling factor that make the MPC terminal cost. During training, we approximate the expensive formal verification using validation and the introduction of new points from near the region boundary across iterations. 
The approach is demonstrated on a constrained torque-limited inverted pendulum and non-holonomic vehicle kinematics. Experiments show that validation can be used as a faster and effective proxy for formal verification during Lyapunov NN learning while leading to stability in practice. The stable region can in fact be larger than that from an MPC demonstrator with a longer prediction horizon. NLMPC can transfer between using an inaccurate surrogate and a nominal forward model and outperform several RL baselines in terms of stability and constraints satisfaction. 


\section{Preliminaries and Assumptions}

\paragraph{Controlled Dynamical System} Consider a discrete-time, time-invariant, deterministic system:
\begin{equation} 
    \label{eq:ode}
    x(t+1) = f(x(t), u(t)),\quad y(t) = x(t), \quad f(0,0)=0, 
\end{equation}
where $t\in \mathbb{N}$ is the timestep index, $x(t) \in \mathbb{R}^{n_x}$, $u(t) \in \mathbb{R}^{n_u}$ and $y(t) \in \mathbb{R}^{n_y}$ are, respectively, the state, control input, and measurement at timestep $t$. We assume that the states and measurements are equivalent and the origin is the equilibrium point. Further, the system (\ref{eq:ode}) is subjected to closed and bounded, convex constraints over the state and input spaces:
\begin{eqnarray}
x(t)\in\mathbb{X}\subseteq \mathbb{R}^{n_x}, \quad 
u(t)\in\mathbb{U}\subset \mathbb{R}^{n_u}, \quad \forall t>0. \label{eq:constraints}
\end{eqnarray}
The system is to be controlled by a feedback policy, $K: \mathbb{R}^{n_x} \rightarrow \mathbb{R}^{n_u}$. The policy $K$ is considered \emph{safe} if there exists an invariant set, $\mathbb{X}_s \subseteq \mathbb{X}$, for the closed-loop dynamics, inside the constraints. The set $\mathbb{X}_s$ is also referred to as the \emph{safe-set} under $K$. Namely, every trajectory for the closed-loop system that starts at some $x \in \mathbb{X}_s$ remains inside this set. If $x$ asymptotically reaches the target , $\bar{x}_T \in \mathbb{X}_s$, then $\mathbb{X}_s$ is a Region of Attraction (ROA). In practice, convergence often occurs to a small set, $\mathbb{X}_T$.

\paragraph{Lyapunov Conditions and Safety} 
We formally assess the {safety} of the closed-loop system in terms of the existence of the positively invariant-set, $\mathbb{X}_s$, inside the state constraints. This is done by means of a learned CLF, $V(x)$, given data generated under a (initially unknown) policy, $K(x)$.  

The candidate CLF needs to satisfy certain properties. First, it needs to be upper and lower bounded by strictly increasing, unbounded, positive ($\mathcal{K}_\infty$) functions~\cite{Khalil_book}. We focus on optimal control with a quadratic stage cost and assume the origin as the target state: 
 \begin{equation}
    \label{eq:stage_cost}
    \ell(x,u)= x^T Q x + u^T R u, \quad Q\succ 0,\ R\succ 0. 
\end{equation}
For above, a possible choice for $\mathcal{K}_\infty$-function is the scaled sum-of-squares of the states:
\begin{equation}
    \label{eq:lyap1}
    l_\ell \|x\|_2^2  \leq  V(x) \leq L_V \|x \|_2^2, 
  \end{equation}
where $l_\ell$ and $L_V$ are the minimum eigenvalue of $Q$ and a Lipschitz constant for $V$ respectively.   

Further for safety, the convergence to a set, $\mathbb{X}_T\subset\mathbb{X}_s$, can be verified by means of the condition, $\forall x\in \mathbb{X}_s\backslash\mathbb{X}_T$:
\begin{eqnarray}
    \label{eq:lyap2_lqr}
u =  K\left(x\right)  \Rightarrow   V\left(f\left(x, u\right)\right)-\lambda V(x)\leq 0, \    \lambda\in[0,1).
\end{eqnarray} 
This means that to have stability $V(x)$ must decrease along the closed-loop trajectory in the annulus.   

The sets $\mathbb{X}_s$, $\mathbb{X}_T$, satisfying (\ref{eq:lyap2_lqr}), are (positively) \emph{invariant}. If they are also inside constraints, i.e. $\mathbb{X}_s\subseteq \mathbb{X}$, then they are \emph{safe}.
For a valid Lyapunov function $V$, the outer safe-set can be defined as a level-set: 
\begin{equation}
    \mathbb{X}_s = \{x \in \mathbb{X}: V(x)\leq l_s\}.
\end{equation}
For further definitions, we refer the reader to ~\cite{Blanchini,Kerrigan:2000}. If condition (\ref{eq:lyap2_lqr}), holds everywhere in $\mathbb{X}_s$, then the origin is a stable equilibrium ($\mathbb{X}_T=\{0\}$). If (most likely) this holds only outside a non-empty inner set,
$\mathbb{X}_T=\{x \in \mathbb{X}: V(x)\leq l_T\}\subset\mathbb{X}_s$, with $\mathbb{X}_T\supset\{0\}$, then the system converges to a neighborhood of the origin and remains there in the future. 

\paragraph{Approach Rationale} We aim to match or enlarge the stable region of an unknown controller, $K_i(x)$. For a perfect model, $f$, and a safe set $\mathbb{X}_s^{(i)}$, there exists an  $\alpha\gg1$, such that the one-step MPC:
\begin{equation}\label{eq:ROAext}
    K(x)= \argmin_{u\in\mathbb{U},\ f(x,u)\in\mathbb{X}_s^{(i)}} \alpha V(f(x,u))+\ell(x,u), 
\end{equation}
results in a new safe set, $\mathbb{X}_s^{(i+1)}=\mathcal{C}(\mathbb{X}_s^{(i)})$, the one-step controllable set of $\mathbb{X}_s^{(i)}$ and the feasible region of (\ref{eq:ROAext}), $\mathbb{X}_s^{(i+1)}\supseteq \mathbb{X}_s^{(i)}$. 
We soften the state constraints in (\ref{eq:ROAext}) and use it recursively to estimate $\mathbb{X}_s^{(j)},\ j>i$. 
\addition{We formulate an algorithm that learns the parameter $\alpha$ as well as the safe set. We train a neural network via SGD to approximate $V$, hence the ROA estimate will not always increase through iterations. To aim for maximum ROA and minimum MPC horizon, we use cross-validation. We motivate our work by extending theoretical results on MPC stability and a sub-optimality bound for approximate $f$ and $V$. Finally, we provide an error bound on the learned ${f}$ for having stability.} 

\paragraph{Learning and Safety Verification} We wish to learn $V(x)$ and $\mathbb{X}_s$ from one-step on-policy rollouts, as well as a forward model $\hat{f}(x,u)$. 
For MPC stability, we must assume that the Lyapunov network has been formally verified, namely, that   (\ref{eq:lyap2_lqr}) is valid within the learned set. Recent work on formal Lyapunov verification \cite{bobiti_sampling-based_2016, Abate_2021, chang_neural_2019, kong2015}, respectively, grid-based and symbolic methods, can be used for systems of limited dimensionality. For large-scale systems, however, this is an open area of research. Sampling-based methods \cite{bobiti_sampling-based_2016} provide a high probability certificate which isn't enough for MPC. In this paper, we demonstrate how validation and counterexamples can be used to help increasing the ROA in practise even without the use of  verification during learning.   

\paragraph{NN-based dynamics model} In some MPC applications, it might not be possible to gather sufficient data from demonstrations in order to be able to learn a model that predicts over long sequences. One-step or few-steps dynamics learning based on NNs can suffer when the model is unrolled for longer time. For instance, errors can accumulate through the horizon due to small instabilities either from the physical system or as artifacts from short sequence learning. Although some mitigations are possible for specific architectures or longer sequences \cite{Armenio2019, Doan2019, Pathak2017, ciccone_nais-net:_2018},    
we formulate our MPC to allow for a very short horizon and unstable dynamics. Since we learn a surrogate NN forward model,  $\hat{f}(x,u)$, from one-step trajectories, we will assume this to have a locally bounded one-step-ahead prediction error, $w(t)$, where: 
 \begin{eqnarray}\label{eq:pred_error}
 w = f(x,u)-\hat{f}(x,u), & 
 \|w\|_2\leq \mu,\ \forall (x,u)\in \tilde{\mathbb{X}}\times \mathbb{U} , 
 \end{eqnarray}
for some compact set of states, $\tilde{\mathbb{X}}\supseteq\mathbb{X}$. We also assume that both $f$ and $\hat{f}$ are locally Lipschitz in this set, with constants $L_{{f}x},\  L_{{f}u}$, and $L_{\hat{f}x},\  L_{\hat{f}u}$ respectively. A conservative value of $\mu$ can be inferred from these constants as the input and state sets are bounded or it can be estimated from a test set. 

 
\section{Neural Lyapunov MPC}


In the context of MPC, a function $V$, which satisfies the Lyapunov property (\ref{eq:lyap2_lqr}) for some local controller $K_0$, is instrumental to formally guarantee stability \cite{rawlings_mayne_paper,Limon2003StableCM}.  We use this insight and build a general Lyapunov function terminal cost for our MPC, based on neural networks. We discuss the formulation of the Lyapunov network and the MPC in ~\secref{subsec:lyap_func_learning} and~\secref{subsec:lyap_mpc} respectively. In order to extend the controller's ROA while maintaining a short prediction horizon, an alternate optimization scheme is proposed to tune the MPC and re-train the Lyapunov NN. We describe this procedure in~\secref{subsec:alt_opt_mpc} and provide a pseudocode in~\algref{alg:alternateDescent}.

\subsection{Lyapunov Network and Advantage Learning}
\label{subsec:lyap_func_learning}

We use the Lyapunov function network introduced by \cite{gallieri_safe_2019}: 
\begin{equation}\label{eq:lyap_definition}
  V(x) = x^T\left(l_\ell I + V_{net}(x)^T V_{net}(x)\right)x,
\end{equation}
where $V_{net}(x)$ is a (Lipschitz) feedforward network that produces a  $n_V\times n_x$ matrix. The scalars $n_V$ and $l_\ell>0$ are hyper-parameters. It is easy to verify that (\ref{eq:lyap_definition}) satisfies the condition mentioned in (\ref{eq:lyap1}). In our algorithm, we learn the parameters of the network, $V_{net}(x)$, and a safe level, $l_s$. 
Note that equation (\ref{eq:lyap2_lqr}) allows to learn $V$ from demonstrations without explicitly knowing the current  policy.

\paragraph{Loss function} Suppose $\mathcal{D}_K$ denotes a set of state-action-transition tuples of the form $\mathcal{T}=(x,\ u,\ x^{+})$, where $x^{+}$ is the next state obtained applying the policy $u=K(x)$. The Lyapunov network is trained using the following loss:
\begin{equation}
    \min_{V_{net},\ l_s} \mathbf{E}_{\mathcal{T}\in\mathcal{D}_K} \left[\frac{\mathcal{I}_{\mathbb{X}_s}(\mathcal{T})}{\rho} J_{s}(\mathcal{T})+ J_\text{vol}(\mathcal{T})\right], 
    \label{eq:lyapunov_loss}
\end{equation}
where,
\vspace{-0.2cm}
\begin{eqnarray}
& \mathcal{I}_{\mathbb{X}_s}(\mathcal{T}) = 0.5 \left(\text{sign}\left[l_s - V(x) \right]+1\right),  & \text{Set indicator} \nonumber\\ 
 & J_{s}(\mathcal{T}) = \frac{\max\left[ \Delta V(x),\ 0 \right]}{ V(x)+\epsilon_V}, & \text{Instability} \nonumber\\
 &\Delta V(x) = V\left({x}^{+}\right)-\lambda V(x), & \nonumber \text{Disadvantage} \\
&J_\text{vol}(\mathcal{T})  = \text{sign}\big[-\Delta V(x)\big]\ \left[l_s - V(x) \right]. & \text{Discriminator} \nonumber
\end{eqnarray} 

In (\ref{eq:lyapunov_loss}), $\mathcal{I}_{\mathbb{X}_s}$ is the indicator function for the safe ROA $\mathbb{X}_s$, which is multiplied to $J_s$, a function that penalises the instability. The term $J_\text{vol}$ is a classification loss that tries to compute the correct boundary between the stable and unstable points. It is also instrumental in increasing the ROA volume. The scalars $\epsilon_V>0$, $\lambda\in[0,1)$, $0<\rho\ll1$, are hyper-parameters, where the latter trades off volume for stability (we take $\rho=10^{-3}$ as in \cite{richards_lyapunov_2018,gallieri_safe_2019}). To make sure that  $\mathbb{X}_s\subseteq\mathbb{X}$, we scale-down the learned  $l_s$ a-posteriori. 
 The loss (\ref{eq:lyapunov_loss}) extends the one proposed by \cite{richards_lyapunov_2018} in the sense that we only use one-step transitions, and safe trajectories are not explicitly labeled before training. 
 
\paragraph{RL advantage and stability} The loss (\ref{eq:lyapunov_loss}) is used as an alternative to the more common one-step value error from Reinforcement Learning (RL) \cite{sutton1998a} (maximising reward  $r$):
 \begin{eqnarray}\label{eq:td_error}
     J_\text{TD(0)}(\mathcal{T})= \mathbf{E}[(\underbrace{r(x,u)+\gamma \mathcal{V}(x^{+})}_{\text{Bootstrap target}}-\mathcal{V}(x))]^2 
 \end{eqnarray}
In RL, (\ref{eq:td_error}) is used with accumulated targets to learn the value $\mathcal{V}^\star$ given transitions $\mathcal{T}\in\mathcal{D}$ and rewards   $r(x,u)$. For a $\gamma$-discounted Markov Decision Process (MDP) with  $r(x,u)=-\ell(x,u)$, and candidate value $\mathcal{V}(x)=-V(x)$, then our loss $J_{s}(\mathcal{T})$ is a one-sided version of (\ref{eq:td_error}), and (\ref{eq:lyapunov_loss}) aims to learn an \emph{on-policy advantage} proxy, $-\Delta V(x)$, and a region $\mathbb{X}_s$ where this is \emph{positive}. Given a policy, $K_0$, the \emph{true} advantage: 
  \begin{eqnarray}\label{eq:adv}
     \mathcal{A}^\star(x,u) = \underbrace{\mathbf{E}[r(x,u)+\gamma \mathcal{V}^\star(x^{+})]}_{\mathcal{Q}^\star(x,u)}\underbrace{-\mathcal{V}^\star(x)}_{V^\star(x)} \approx -\Delta V^\star (x)
 \end{eqnarray}
 where $\mathcal{Q}$ is the action-value function. In a deterministic setting, given the assumptions on the reward being negative definite and zero at the desired target, the advantage function can be connected to stability by the following result which, to the best of our knowledge, is a novel connection to RL: 
\begin{theorem}\label{lemma:adv}
There exists $\epsilon\in(0,1),\ \bar{\gamma}\in\left(1-\epsilon,\ 1\right)$ such that: $1>\gamma\geq\bar{\gamma}$ and  $u=K(x)\Rightarrow\mathcal{A}^\star(x,u)\geq0$,  $\forall x\in\mathbb{X}_s\subseteq\mathbb{X}$ yields that  $\mathbb{X}_s$ is a \emph{safe} ROA: the state converges to  $\mathbb{X}_{\bar{\gamma}}\subseteq\mathbb{X}_s$.
\end{theorem} 
A trivial choice is $\epsilon={\ell_\ell}/{L_V}$. If $K_0$ is stabilizing, the result holds also for $\gamma=1$. From Theorem \ref{lemma:adv}, using $J_\text{TD(0)}$ for RL could provide safety, if $\mathcal{A}^\star(x,K(x)) \rightarrow 0$. In practice, our loss (\ref{eq:lyapunov_loss}) can encourage roll-out stability \emph{earlier} than Bellman error convergence.  
 Replacing $\ell$ with an equivalent contraction $\lambda$ on $V$ allows us to \emph{learn offline} from \emph{stabilizing}, \emph{sub-optimal} examples. A combination of our approach with value/advantage learning is of interest for future work. 
 
The remainder of the paper focuses on minimising losses, $\sum_{t=0}^{\infty}\gamma^t\ell(x(t),u(t))$, as this is more standard in control.

\subsection{Neural Lyapunov MPC} 
\label{subsec:lyap_mpc}
 The loss (\ref{eq:lyapunov_loss}) is used to both learn $V$ and to tune its scaling factor in the MPC loss,  $\alpha\geq1$, to provide stability. We aim to improve the ROA of the initial controller, used to collect data, by replacing it with an MPC solving the following input-limited, state soft-constrained, optimal control problem:
\begin{eqnarray}
	 \min_{\underbar{u}} \quad  & \gamma^N\alpha V(\hat{x}(N))+\sum_{i=0}^{N-1} \gamma^i \ell(\hat{x}(i), \hat{u}(i)) + \ell_{\mathbb{X}}(s(i))\label{problem} \nonumber\\
	\mathrm{s.t.} \quad &\hat{x}(i+1) = \hat{f}(\hat{x}(i), \hat{u}(i)), \\
	& \hat{x}(i)+s(i)\in\mathbb{X},\ \forall i\in[0, N], \nonumber\\
	& \ell_{\mathbb{X}}(s)=\eta_1 s^T s + \eta_2 \|s\|_1,  \ \eta_1>0,\ \eta_2 \gg 0, \nonumber\\
	& \hat{u}(i)\in\mathbb{U},\ \forall i\in[0, N-1], \nonumber\\
	&\hat{x}(0) = x(t), \nonumber 
\end{eqnarray}
where $\hat{x}(i)$ and $\hat{u}(i)$ are the predicted state and the input at $i$-steps in the future, ${s}(i)$ are slack variables, $\underbar{u} = \{ u(i) \}_{i=0}^{N-1}$, the stage cost $\ell$ is given by (\ref{eq:stage_cost}), $\gamma\in(0,1]$ is a {discount factor}, 
the function $V$ is
the Lyapunov NN from (\ref{eq:lyap_definition}), scaled by a factor $\alpha\geq1$ to provide stability, and $x(t)$ is the measured system state at the current time. The penalty  $\ell_\mathbb{X}$ is used for state constraint violation, see \cite{Kerrigan00softconstraints}. The optimal cost is denoted as $J_{\text{MPC}}^\star(x(t))$. 

Problem (\ref{problem}) is solved online given the current state $x(t)$; then, the first element of the optimal control sequence, $u^\star(0)$, provides the action for the physical system. Then, a new state is measured, and (\ref{problem}) is again solved, in a \emph{receding horizon}. 
The implementation details are given in Appendix.  

\paragraph{Stability and safety} 

We extend standard results from \cite{Limon2003StableCM,Limon2009} to the discounted case and to the $\lambda$-contractive $V$ from  (\ref{eq:lyap2_lqr}). In order to prove them, we make use of the uniform continuity of the model, the SQP solution and the terminal cost, $V$, as done by \cite{Limon2009}. 
Consider the candidate MPC ROA:
 \begin{eqnarray}\label{eq:ROA}
\nonumber     \Upsilon_{N,\gamma,\alpha} = \left\{x\in\mathbb{R}^{n_x}: J_{\text{MPC}}^\star(x)\leq \frac{1-\gamma^{N}}{1-\gamma}\ d +\gamma^N \alpha l_s\right\},  \\ \nonumber d = \inf_{x\not\in\mathbb{X}_s}\ell(x,0).
 \end{eqnarray}
The following are obtained for system (\ref{eq:ode}) in closed loop with the MPC defined by problem (\ref{problem}). Results are stated for $\mathbb{X}_T=\{0\}$. For $\mathbb{X}_T\not=\{0\}$, convergence would occur to a set instead of $0$. We assume $V$ being formally verified \cite{Abate_2021}.

\begin{theorem}\label{Theorem:stability}{\bf Stability and robustness} 
Assume that $V(x)$ satisfies (\ref{eq:lyap2_lqr}), with  $\lambda\in[0,1)$, $\mathbb{X}_T=\{0\}$. Then, given $N\geq1$, for the MPC (\ref{problem}) there exist a constant $\bar{\alpha}\geq0$, a discount factor  $\bar{\gamma}\in(0,1]$, and a model error bound $\bar{\mu}$ such that, if $\alpha\geq\bar{\alpha}$, $\mu\leq\bar{\mu}$ and $\gamma\geq\bar{\gamma}$, then, $\forall x(0)\in\mathcal{C}(\mathbb{X}_s)$:
 \vspace{-0.cm}
 \begin{enumerate}
    \item If $N=1$, $\mu=0$, then the system is asymptotically stable for any $\gamma>0$, $\forall x(0)\in\Upsilon_{N,\gamma,\alpha}\supseteq\mathbb{X}_s$.
    \item If $N>1$, $\mu=0$, then the system reaches a set $\mathbb{B}_{\gamma}$ that is included in $\mathbb{X}_s$. This set increases with decreasing discount, $\gamma$, $\forall x(0)\in\Upsilon_{N,\gamma,\alpha}$. $\gamma=1\Rightarrow \mathbb{B}_{\gamma}=\{0\}$. 
    
    \item If $\alpha V(x)$ is the optimal value function in $\mathbb{X}_s$ for the problem, $\mu=0$, and if $\mathcal{C}(\mathbb{X}_s)\neq\mathbb{X}_s$, then the system is asymptotically stable, $\forall x(0)\in\Upsilon_{N,\gamma,\alpha}\supset\mathbb{X}_s$.
    
    \item If $\mu=0$, then $\alpha\geq\bar{\alpha}$ implies that $\alpha V(x)\geq V^\star(x), \forall x\in\mathbb{X}_s$, where $V^\star$ is the optimal value for the infinite horizon problem with cost (\ref{eq:stage_cost}) subject to (\ref{eq:constraints}). 
    
    \item \addition{The MPC has a stability margin. If the MPC uses a surrogate model satisfying (\ref{eq:pred_error}), with one-step error bound $\|w\|_2^2<\bar{\mu}^2=\frac{1-\lambda}{L_V L_{\hat{f}x}^{2N}}l_s$, then the system is Input-to-State (practically) Stable (ISpS) and there exists a set $\mathbb{B}_{N, {\gamma},{\mu}}:\ x(t)\rightarrow \mathbb{B}_{N, {\gamma}, {\mu}}$, $\forall x(0)\in\beta\Upsilon_{N,\gamma,\alpha}$, $\beta\leq1$. }
\end{enumerate}
\end{theorem}


\addition{Theorem} \ref{Theorem:stability} states that for a given horizon $N$ and contraction $\lambda$, one can find a minimum scaling of the Lyapunov function $V$ and a lower bound on the discount factor $\gamma$ such that the system under the MPC has ROA $\Upsilon_{N,\gamma,\alpha}\supseteq\mathbb{X}_s$. If the model is not perfect, its error being less than  $\mu\leq\bar{\mu}$, then the ROA size decreases but the system is still \emph{safe}. In this case, a shorter horizon can be beneficial. 
The proof of the theorem follows standard MPC arguments and is in Appendix.  

\paragraph{Performance with surrogate models} \addition{In order to further motivate for the search of a $V$ giving the largest $\mathbb{X}_s$, notice that a larger $\mathbb{X}_s$ can allow for shortening the MPC horizon, possibly yielding the same ROA. Contrary to \cite{lowrey_plan_2018}, we demonstrate how model mismatch and longer horizons can decrease performance with respect to an infinite-horizon oracle with same cost and perfect model. This links RL to arguments used in nominal MPC robustness  \cite{Limon2009, gallieri_book}.}
 
Let $\mathbf{E}_\mathcal{D}[J_{V^\star}(K^\star)]$ define the expected infinite-horizon performance of the optimal policy $K^\star$, evaluated by using the expected infinite-horizon performance (value function), $V^\star$, for the stage cost (\ref{eq:stage_cost}) and subject to (\ref{eq:constraints}). Similarly, let $\mathbf{E}_{x\in\mathcal{D}}[J^\star_{\text{MPC}}(x)]$ define the MPC's expected performance with the learned $V$, when a surrogate model is used and $\mathbf{E}_{x\in\mathcal{D}}[J^\star_{\text{MPC}}(x;f)]$ when $f$ is known. 
\begin{theorem}\label{Theorem:performance}{\bf Performance}
Assume that the value function error is bounded for all $x$, namely, $\|V^\star(x)-\alpha V(x)\|_2^2\leq\epsilon$, and that the model error satisfies (\ref{eq:pred_error}), for some $\mu>0$. 
Then, for any $\delta>0$:
\begin{eqnarray}
\mathbf{E}_{x\in\mathcal{D}}[J^\star_{\text{MPC}}(x)] - \mathbf{E}_{x\in\mathcal{D}}[J^\star_{V^\star}(x)] 
\leq \\ \frac{2 \gamma^N\epsilon}{1-\gamma^N} + \left(1+\frac{1}{\delta}\right)\|Q\|_2\sum_{i=0}^{N-1}\gamma^i \left( \sum_{j=0}^{i-1} \bar{L}_{f}^j\right)^2 \mu^2 \nonumber \\
\quad + \left(1+\frac{1}{\delta}\right)\gamma^N \alpha L_V \left(\sum_{i=0}^{N-1}\bar{L}_{f}^i\right)^2 \mu^2 +\bar{\psi}(\mu) \nonumber \\
\quad +  \delta\ \mathbf{E}_{x\in\mathcal{D}}\left[J_{\text{MPC}}^\star(x; f)\right], \nonumber
\end{eqnarray}
where $\bar{L}_{f}=\min(L_{\hat{f}x},L_{{f}x})$ and $\bar{\psi}$ is a $\mathcal{K}_\infty$-function representing the constraint penalty terms. 
\end{theorem}

 Theorem \ref{Theorem:performance} is related to \cite{asadi_lipschitz_2018} for value-based RL. However, here we do not constrain the system and model to be stable, nor assume the MPC optimal cost to be Lipschitz.  Theorem \ref{Theorem:performance} shows that a discount $\gamma$ or a shorter horizon $N$ can mitigate model errors. Since $\gamma\ll 1$ can limit stability (Theorem \ref{Theorem:stability}) we opt for the shortest horizon, hence $N=1$, $\gamma=1$. Proof of Theorem \ref{Theorem:performance} is in Appendix.  
 
\paragraph{MPC auto-tuning} The stability bounds discussed in Theorem \ref{Theorem:stability} can be conservative and their computation is non-trivial. Theoretically, the bigger the $\alpha$ the larger is the ROA (the safe region) for the MPC, up to its maximum extent. Practically, for a very high $\alpha$, the MPC solver may not converge due to ill-conditioning. 
Initially, by using the tool from \cite{agrawal_differentiable_2019} within an SQP scheme, we tried to tune the parameters through gradient-based optimization of the loss~(\ref{eq:lyapunov_loss}). These attempts were not successful, as expected from the considerations in \cite{amos_differentiable_2018}. Therefore, for practical reasons, in this work we perform a grid search over the MPC parameter $\alpha$. Note that the discount factor $\gamma$ is mainly introduced for Theorem \ref{Theorem:performance} and analysed in Theorem \ref{Theorem:stability} to allow for future combination of stable MPC with value iteration.

\subsection{Learning algorithm} 
\label{subsec:alt_opt_mpc}

        \begin{algorithm}[b]
            \SetAlgoLined
            \DontPrintSemicolon
            \scriptsize
            \textbf{Inputs: }{$\mathcal{D}_{\text{demo}}$, $\hat{f}$,  $\lambda\in[0,1)$, $\{l_\ell,\epsilon_\text{ext}\}>0$, $\gamma\in(0,1]$, $N\geq1$, $\alpha_\text{list}$, $N_{ext}$, $N_V$, $\epsilon_V$, $V_{init}$, $\ell(x,u)$} 
            \\ \textbf{Outputs: }{$V_{net},\ l_s,\ {\alpha^\star}$}
            \small 
            \scriptsize
            \caption{Neural Lyapunov MPC learning}
            \vspace{-0.1cm} \hrulefill \\ \vspace{0.1cm}
            {$\mathcal{D} \gets \mathcal{D}_{\text{demo}}$ \\ 
             $V_{net} \gets V_{init}$ \\} \;
            \For{$j=0...N_V$}{ 
                    $(V_{net},\ l_s,\ \mathbb{X}_s)\gets \text{Adam step on (\ref{eq:lyapunov_loss})}$   
            }\;
            \For{$i=0...N_{\text{ext}}$}{ 
                 $\mathcal{D} \gets \mathcal{D}_{\text{demo}}\cap  (1+\epsilon_\text{ext})\mathbb{X}_s$ 
                 \\
                 \;
                \For{$\alpha \in \mathcal{\alpha}_\text{list}$}{
                    \; 
                   $\mathcal{U}_1^\star \gets \texttt{MPC}(V_{net},\hat{f}, \mathcal{D}; \alpha)$, from (\ref{problem})\\ 
                   $\mathcal{D}_{\text{MPC}}(\alpha) \gets  \texttt{one\_step\_sim}(\hat{f},\mathcal{D},\mathcal{U}_1^\star)$ 
                   \\ 
                   $\mathcal{L}(\alpha) \gets$  \text{Evaluate  (\ref{eq:lyapunov_loss}) on $\mathcal{D}_{\text{MPC}}(\alpha)$}  
                } 
                {${\alpha^\star} \gets \argmin(\mathcal{L}(\alpha))$ 
                 \\ 
                $\mathcal{D} \gets \mathcal{D}_{\text{MPC}}(\alpha^\star)$ \ \\ 
                $V_{net} \gets V_{init}$ 
                \\
                }
                \For{$j=0...N_V$}{ 
                    $(V_{net},\ l_s,\ \mathbb{X}_s)\leftarrow$ Adam step on (\ref{eq:lyapunov_loss})   
                } 
            }
            Perform formal verification \cite{bobiti_sampling-based_2016} of $V,\ \mathbb{X}_s$.  
        \label{alg:alternateDescent}
        \end{algorithm}
 
Our alternate optimization of the Lyapunov NN, $V(x)$, and the controller is similar to \cite{gallieri_safe_2019}. However, instead of training a NN policy, we tune the scaling $\alpha$ and learn $V(x)$ used by the MPC (\ref{problem}). Further, we extend their approach by using a dataset of demonstrations, $\mathcal{D}_{\text{demo}}$, instead of an explicitly defined initial policy.  These are one-step transition tuples, $(x(0),u(0),x(1))_m,\ m=1,\dots,M$, generated by a (possibly sub-optimal) stabilizing policy, $K_0$. Unlike in \cite{richards_lyapunov_2018}, our $V$ is a piece-wise quadratic, and it is learned without labels. We in fact produce our own psuedo-labels using the sign of $\Delta V(x)$ in (\ref{eq:lyapunov_loss}) in order to estimate $l_s$. The latter means that we don't require episode-terminating (long) rollouts, which aren't always available from data nor accurate when using  a surrogate. Also, there is no ambiguity on how to label.

Once the initial $V$, $\mathbb{X}_s$ are learned from the original demonstrations, we use $V$ and a learned model, $\hat{f}$, within the MPC. We propose \algref{alg:alternateDescent}, which runs multiple iterations where after each of them the tuned MPC and the surrogate model are used to generate new rollouts for training the next $V$ and $\mathbb{X}_s$. We search for the MPC parameter $\alpha$ by minimizing the loss~(\ref{eq:lyapunov_loss}), using $(1+\epsilon_\text{ext})\mathbb{X}_s$ as a new \emph{enlarged} target safe set instead of $\mathbb{X}_s$. Introducing possible \emph{counterexamples}, the use of $\epsilon_\text{ext}$ can push the safe set to extend and, together with \emph{validation}, is used as a scalable proxy for  verification. More specifically, we select $V$ and $\alpha$ using the criteria that the percentage of stable points ($\Delta V < 0$) increases and that of unstable points decreases while iterating over $j$ and $i$ when evaluated on a validation set. The best iteration is picked\footnote{Verification is needed for stability guarantees. 
The presented examples perform well even without the use of formal methods, however, we recommend the use of  verification (e.g. \cite{bobiti_sampling-based_2016})  for real-world applications.}. 


In \algref{alg:alternateDescent}, \texttt{MPC} denotes the proposed Neural Lyapunov MPC, while \texttt{one\_step\_sim} is a one-step transition of the MPC and the surrogate. To train the parameters of $V$ and the level $l_s$, Adam optimizer is used~\cite{kingma_adam:_2014}. A grid search over the MPC parameter $\alpha$ is performed. A thorough tuning of all MPC parameters is also possible, for instance, by using black-box optimisation methods. This is left for future work. 

\section{Numerical experiments}

Through our experiments, we show the following: 1) increase in the safe set for the learned controller by using our proposed alternate learning algorithm, 2) robustness of the one-step NLMPC compared to a longer horizon MPC (used as demonstrator) when surrogate model is used for predictions, and 3) effectiveness of our proposed NLMPC against the demonstrator and various RL baselines. 

\begin{table}[t!]
    \centering
    \begin{small}
    \begin{sc}
    \captionof{table}{\textbf{Inverted Pendulum learning on nominal model.} With iterations, the number of verified points increases.}
    \label{tab:lyap_results_pendulum}
    \begin{tabular}{P{0.5cm} P{1.4cm} P{1.12cm} P{1.8cm}}
        \toprule
        Iter. & Loss & Verif. & Not Verif. \\
        & $\log(1+x)$ & \% & \% \\ 
        \midrule
        1 & 3.21 & 13.25 & 0.00 \\
        2 & 1.08 & 13.54 & 0.00 \\
        \bottomrule
    \end{tabular}
    \end{sc}
    \end{small}
\end{table}

\begin{figure}[t!]
    \centering
    \begin{tabular}{P{2.2cm} P{2.2cm} P{2.2cm}}
         \includegraphics[trim={52 35 105 40}, clip, scale=0.22]{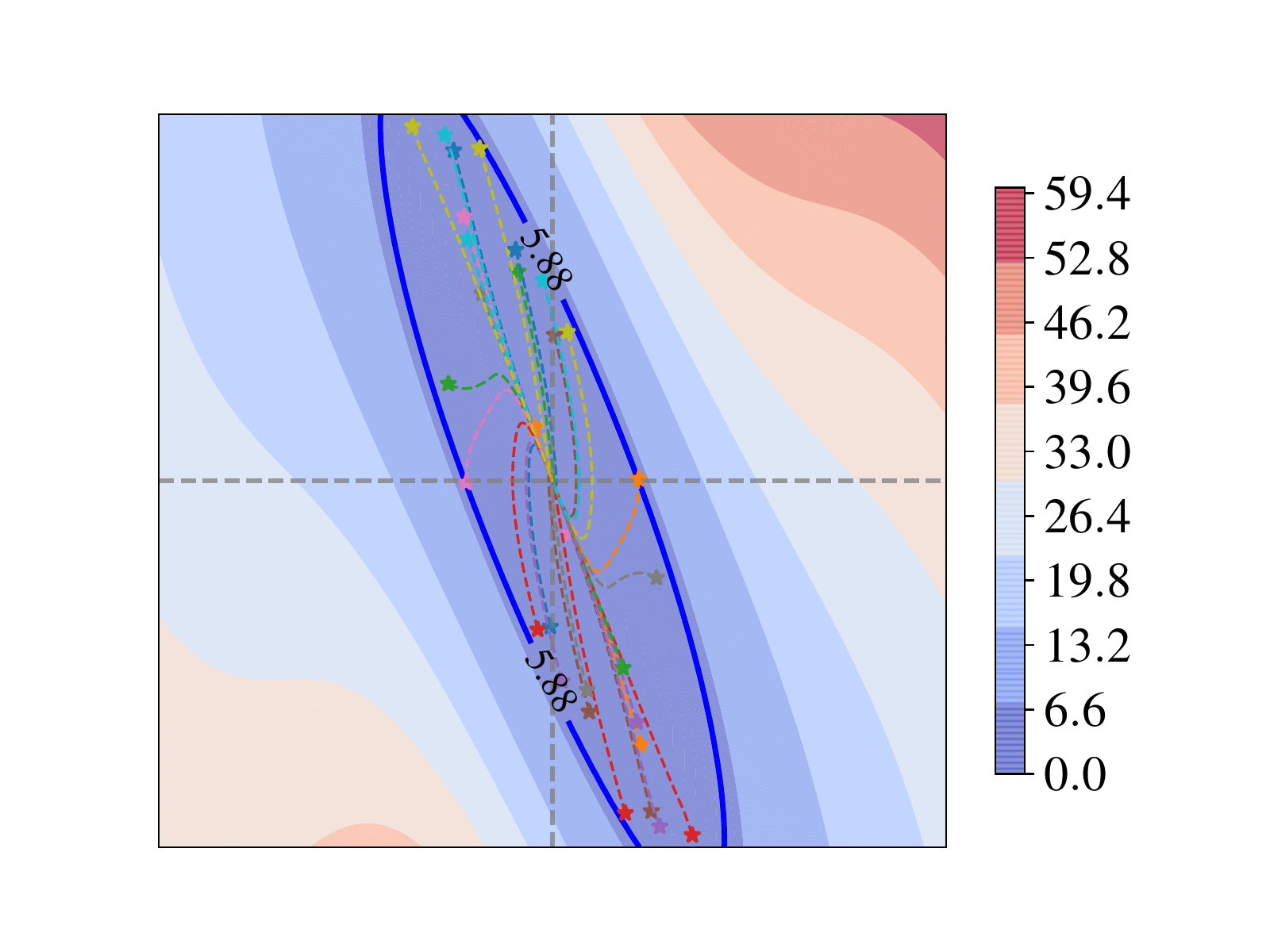} &  \includegraphics[trim={52 35 105 40}, clip, scale=0.22]{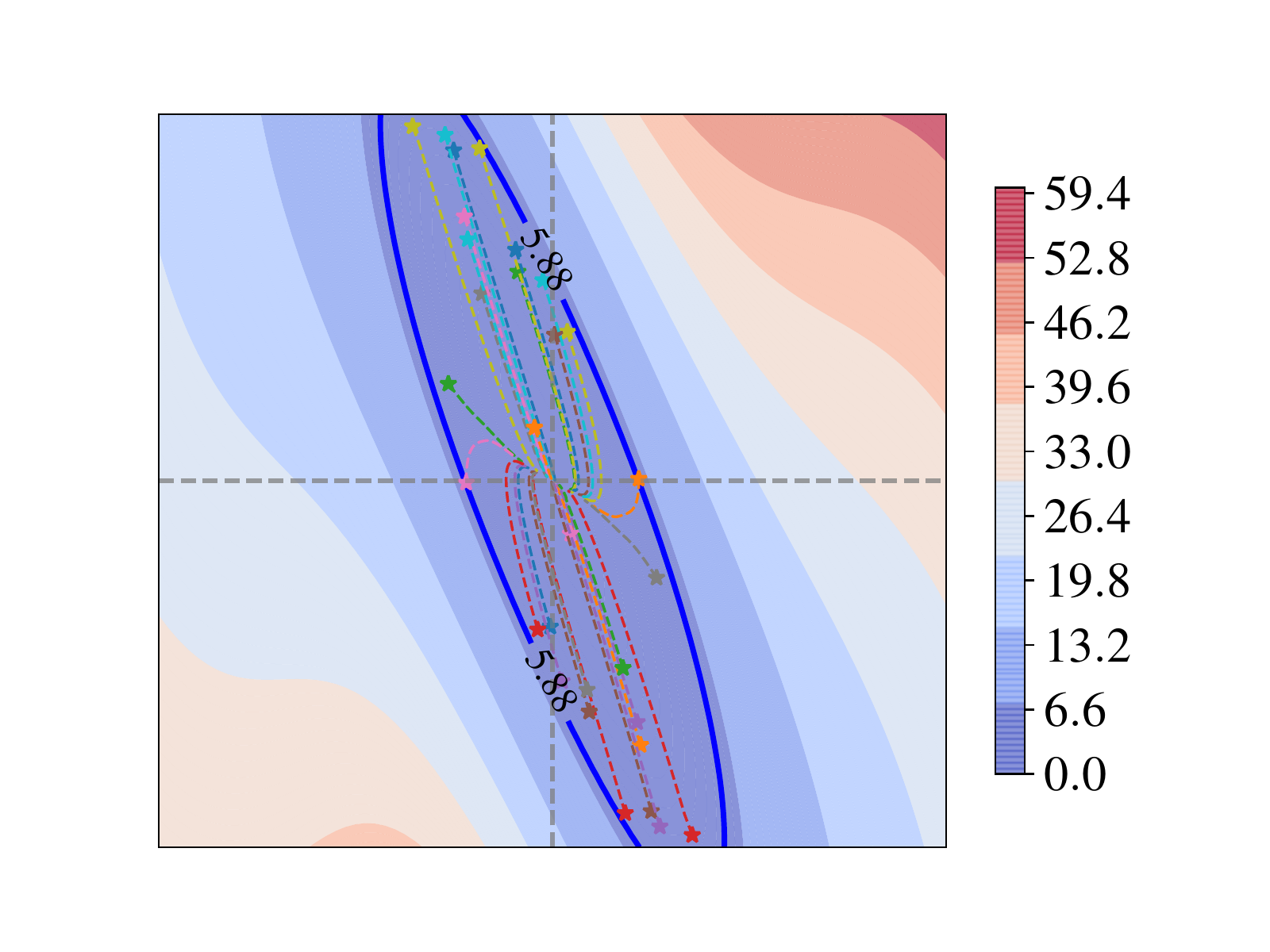} &   \includegraphics[trim={52 35 50 40}, clip, scale=0.22]{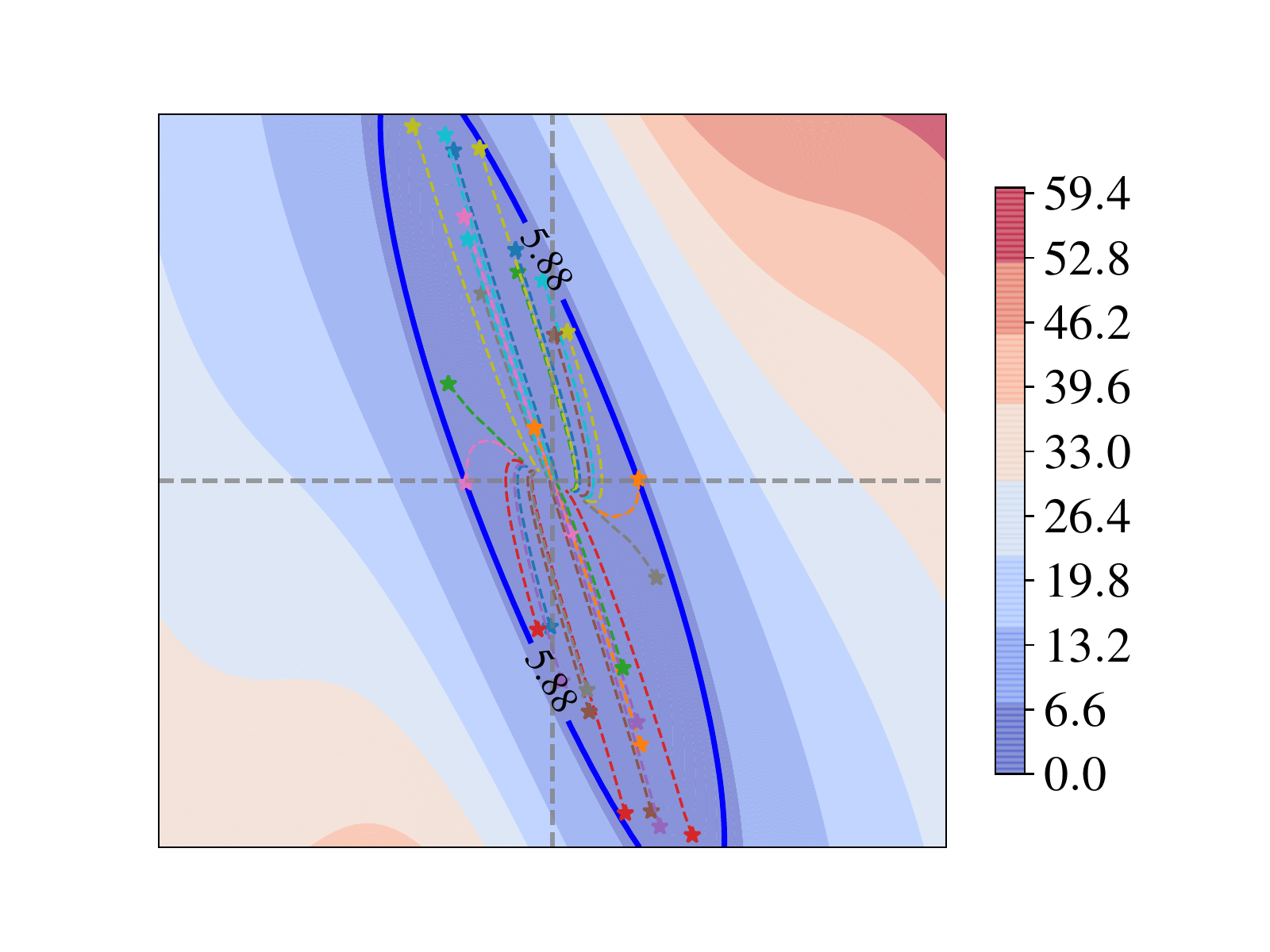}  \\
         (a) Demonstrator & (b) NLMPC (nominal) & (c) NLMPC (surrogate)
    \end{tabular}
     \caption{{\bf Inverted Pendulum: Testing learned controller on nominal system.} Lyapunov function with safe trajectories. NLMPC learns and  transfers successfully to surrogate model.}
    \label{fig:pendulum_lyap}
    \vspace{-10pt}
\end{figure}



\begin{table}[b]
    \centering
    \caption{\textbf{Constrained vehicle kinematics learning.}}
    \begin{subtable}{1\linewidth}
       \begin{center}
            \begin{small}
            \caption{Learning on nominal model}\label{tab:vehicle_model_metric}
            \begin{sc}
            \begin{tabular}{P{0.5cm} P{1.4cm} P{1.12cm} P{1.8cm}}
                \toprule
                Iter. & Loss & Verif. & Not Verif. \\
                \midrule
                1 & 1.55 & 92.20 & 4.42 \\
                2 & 0.87 & 93.17 & 4.89 \\
                3 & 0.48 & 94.87 & 3.89 \\
                \bottomrule
            \end{tabular}
            \end{sc}
            \end{small}
        \end{center}
    \end{subtable}
    
    \bigskip
    \begin{subtable}{1\linewidth}
       \begin{center}
            \begin{small}
            \caption{Learning on surrogate model}\label{tab:vehicle_surrogate_result}
            \begin{sc}
            \begin{tabular}{P{0.5cm} P{1.4cm} P{1.12cm} P{1.8cm}}
                \toprule
                Iter. & Loss & Verif. & Not Verif. \\
                \midrule
                1 & 1.84 & 91.74 & 8.26 \\
                2 & 1.43 & 92.26 & 7.74 \\
                3 & 1.65 & 91.61 & 8.39 \\
                \bottomrule
            \end{tabular}
            \end{sc}
            \end{small}
        \end{center}
    \end{subtable}
    \vskip -0.1in
\end{table}

\begin{figure*}[t]
    \captionsetup[subfigure]{justification=centering}
    \centering
    \begin{subfigure}[b]{0.19\textwidth}
        \centering
        \includegraphics[trim={52 25 105 30}, clip, scale=0.24]{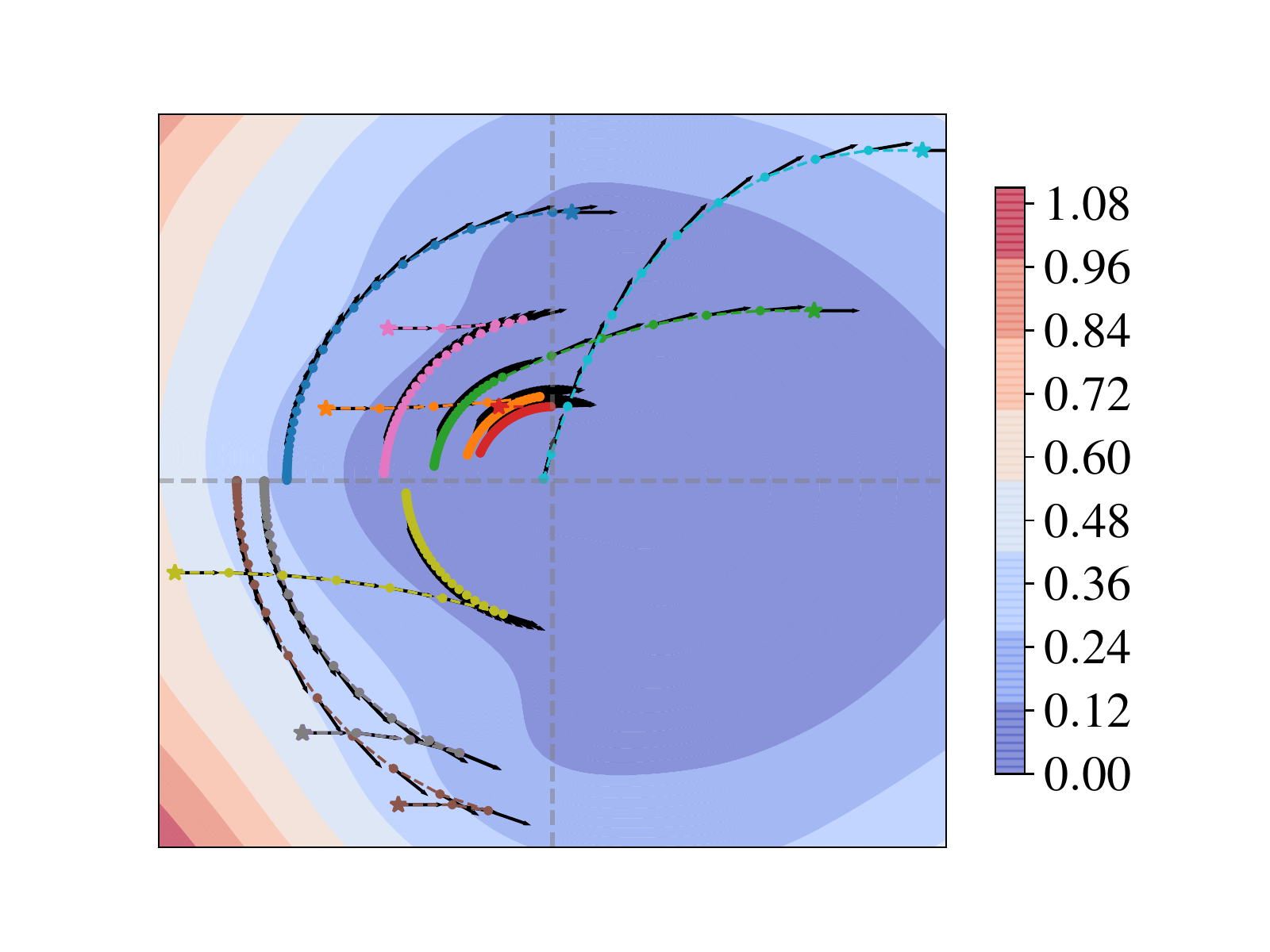}\vspace{-0.1cm}
    \end{subfigure}
    \begin{subfigure}[b]{0.19\textwidth}
        \centering
        \includegraphics[trim={52 25 105 30}, clip, scale=0.24]{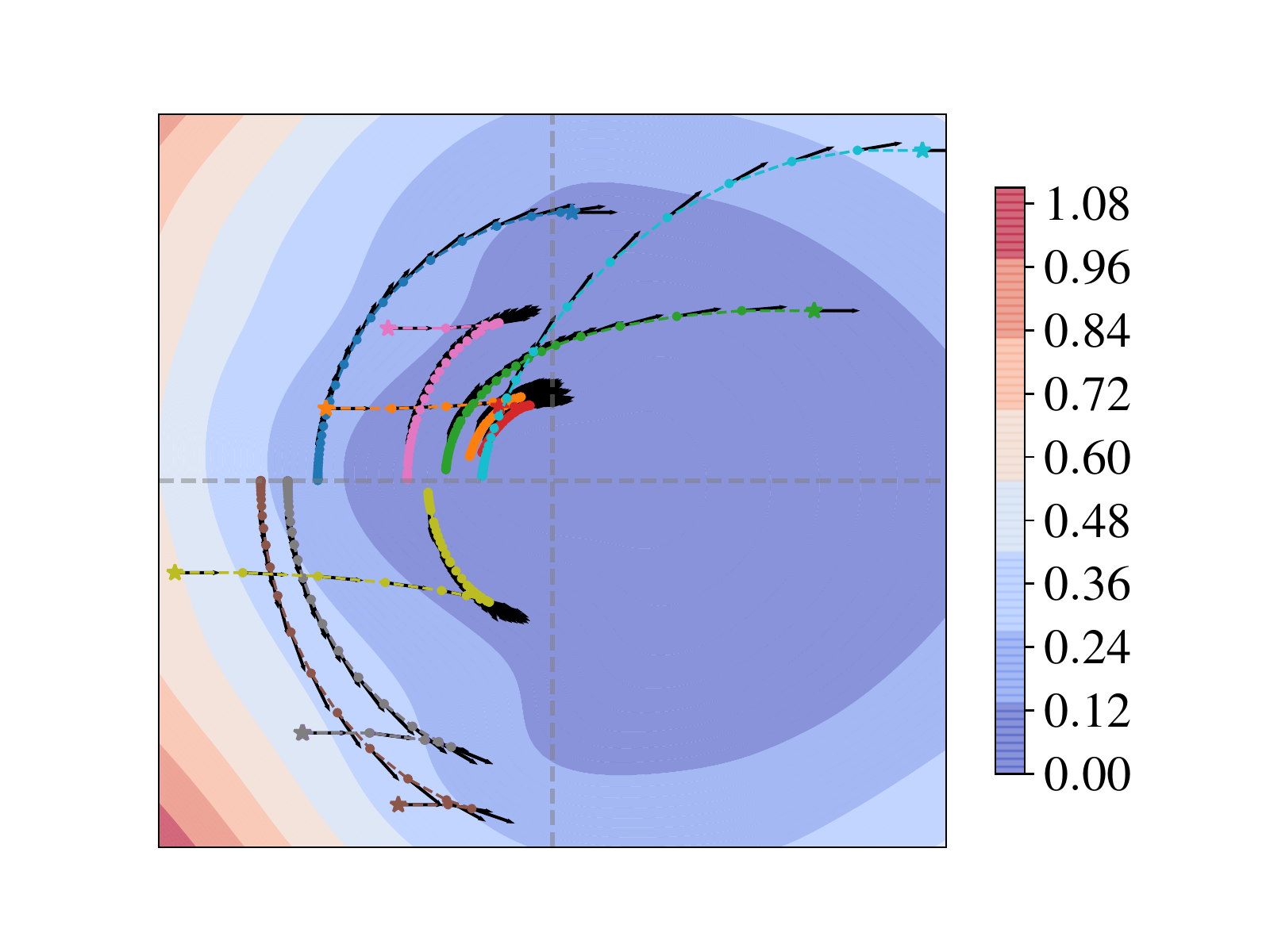}\vspace{-0.1cm}
    \end{subfigure}
    \begin{subfigure}[b]{0.19\textwidth}
        \centering
        \includegraphics[trim={52 25 105 30}, clip, scale=0.24]{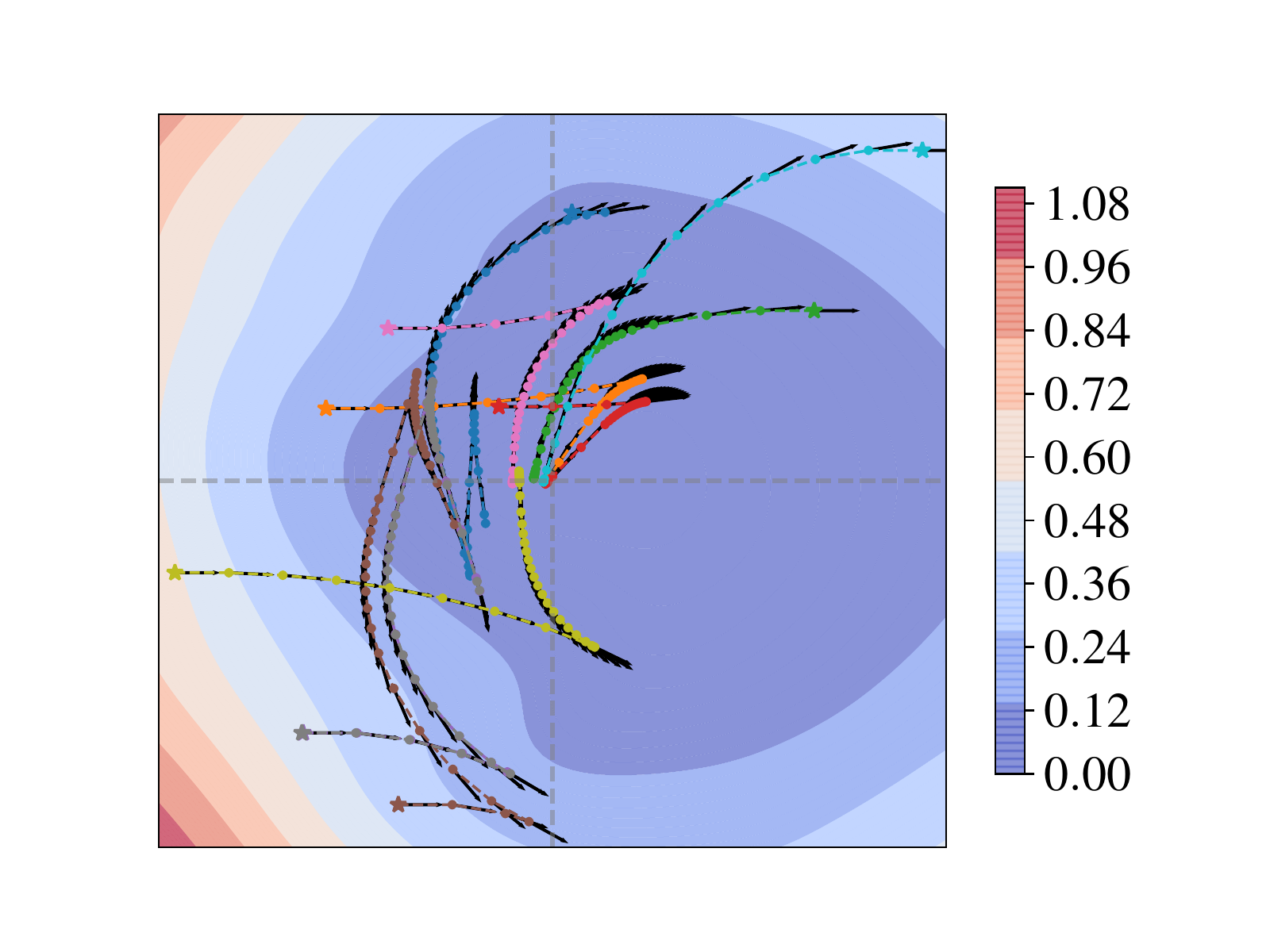}\vspace{-0.1cm}
    \end{subfigure}
    \begin{subfigure}[b]{0.19\textwidth}
        \centering
        \includegraphics[trim={52 25 105 30}, clip, scale=0.24]{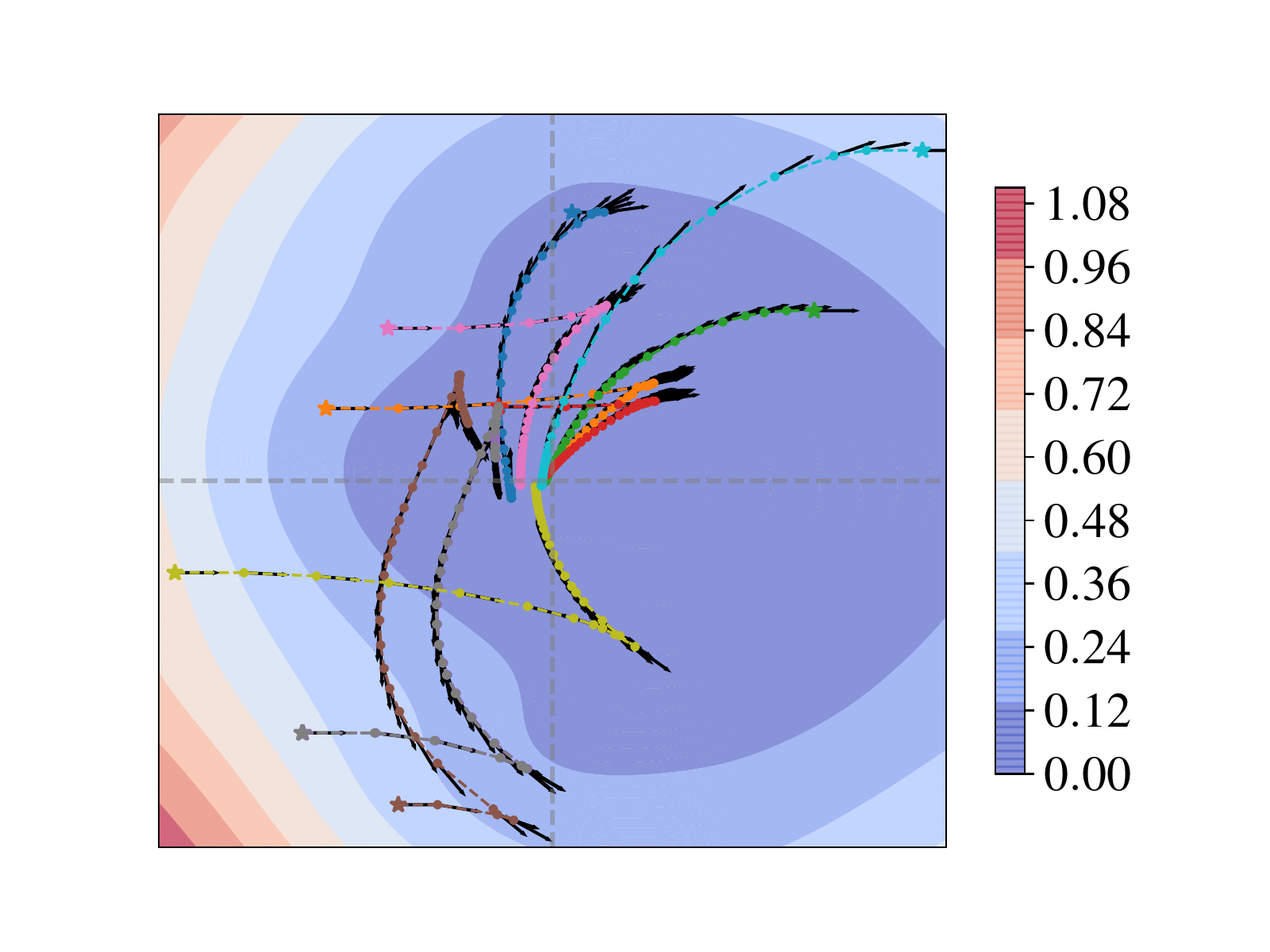}\vspace{-0.1cm}
    \end{subfigure}
    \begin{subfigure}[b]{0.19\textwidth}
        \centering
        \includegraphics[trim={52 25 50 40}, clip, scale=0.24]{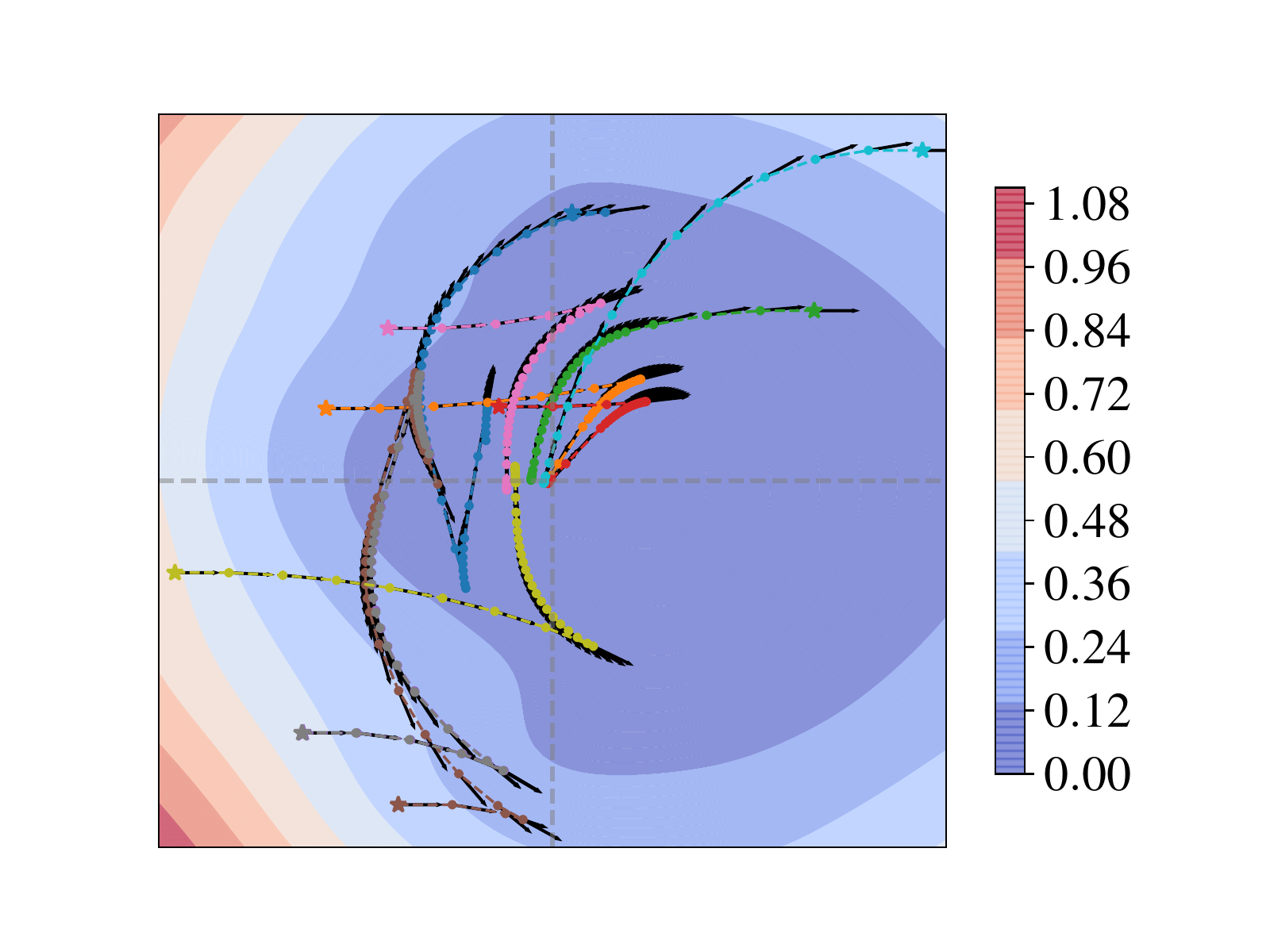}\vspace{-0.1cm}
    \end{subfigure}
    
    \medskip
    \begin{subfigure}[b]{0.19\textwidth}
    \centering
        \includegraphics[trim={15 25 30 30}, clip, scale=0.19]{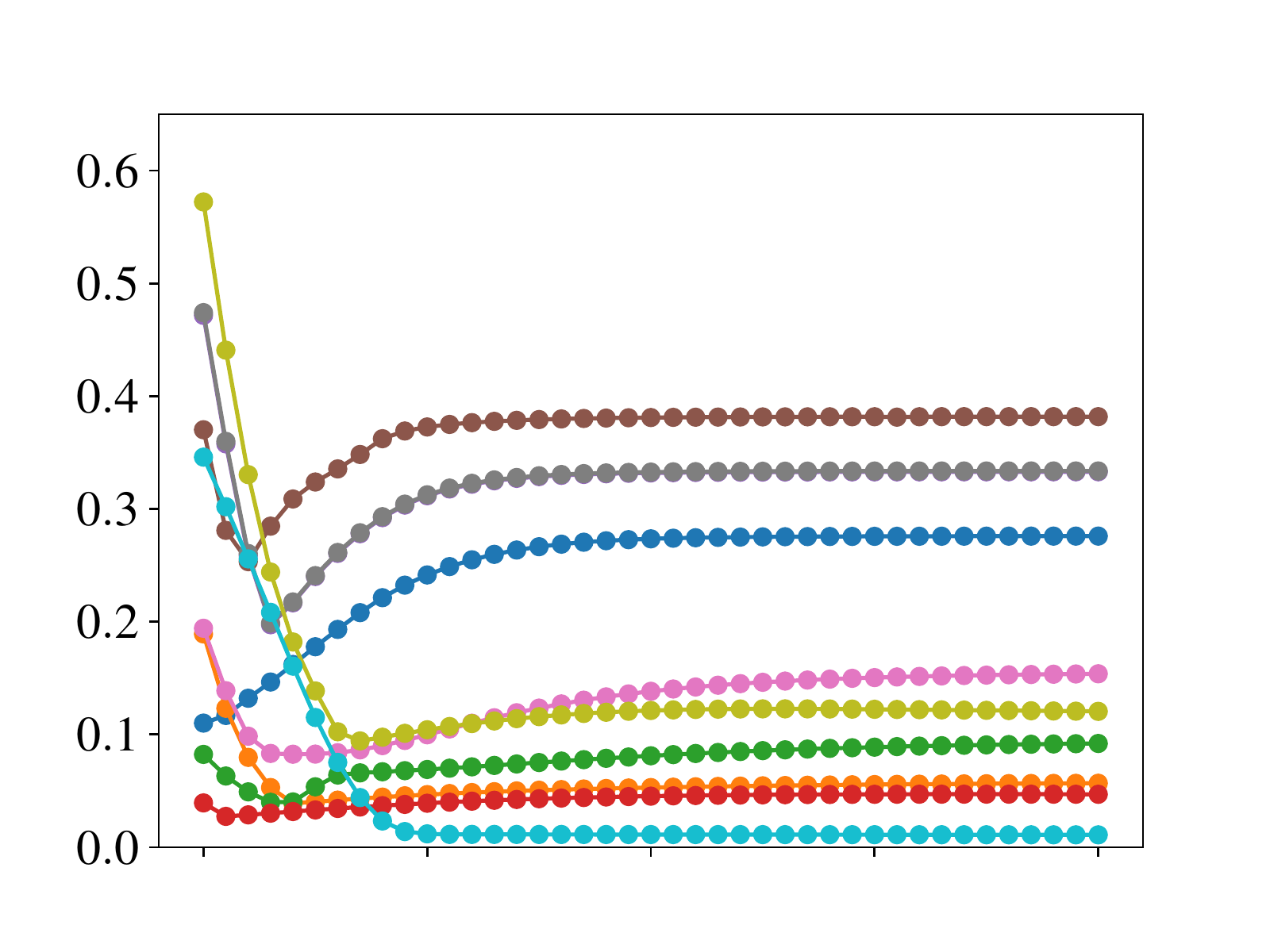}\vspace{-0.1cm}\caption{Short-horizon MPC using surrogate}
    \end{subfigure}
    \begin{subfigure}[b]{0.19\textwidth}
    \centering
        \includegraphics[trim={52 25 30 30}, clip, scale=0.19]{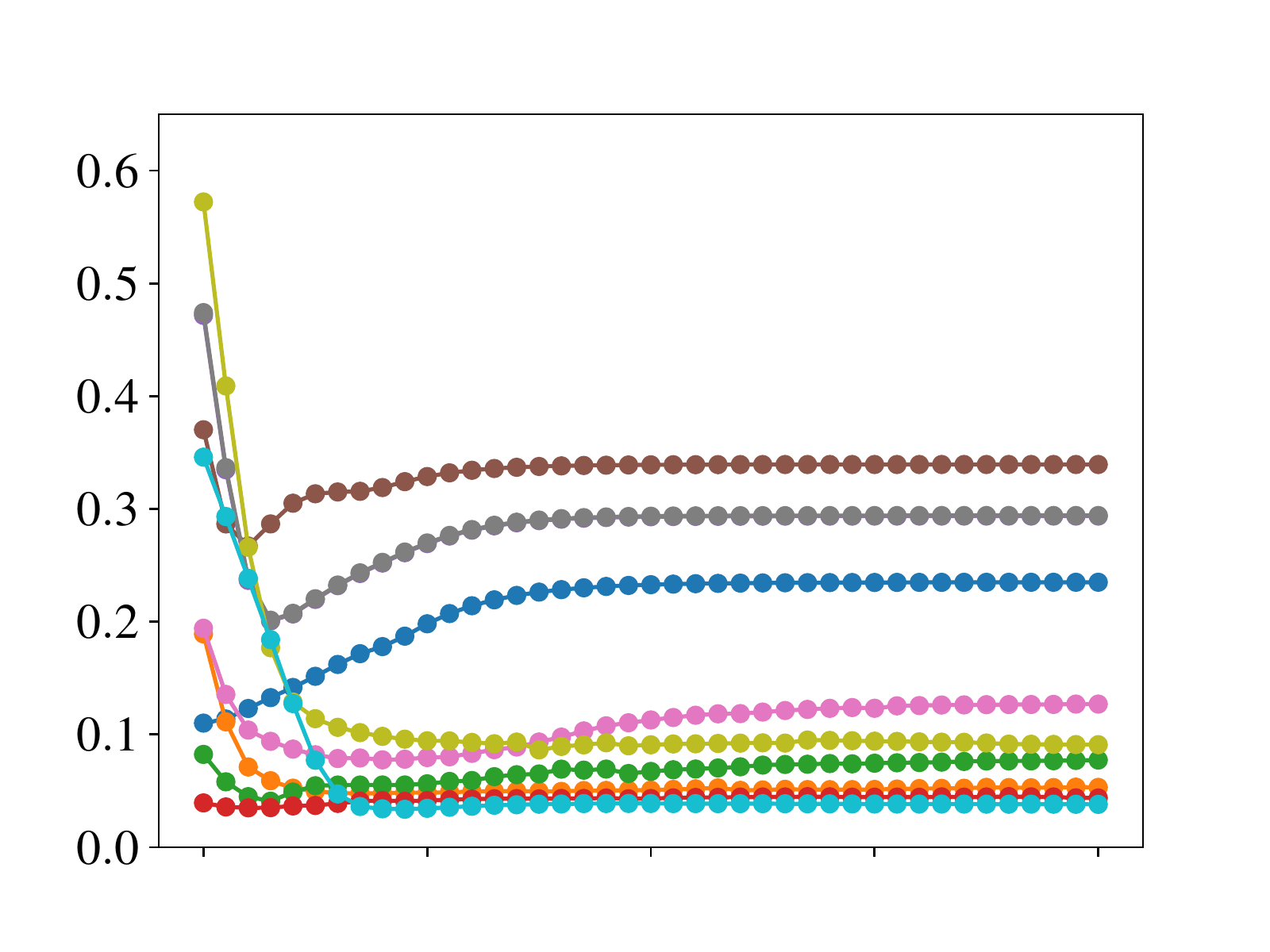}\vspace{-0.1cm}
        \caption{MPC demonstrator using surrogate}
    \end{subfigure}
        \begin{subfigure}[b]{0.19\textwidth}
    \centering
        \includegraphics[trim={52 25 30 30}, clip, scale=0.19]{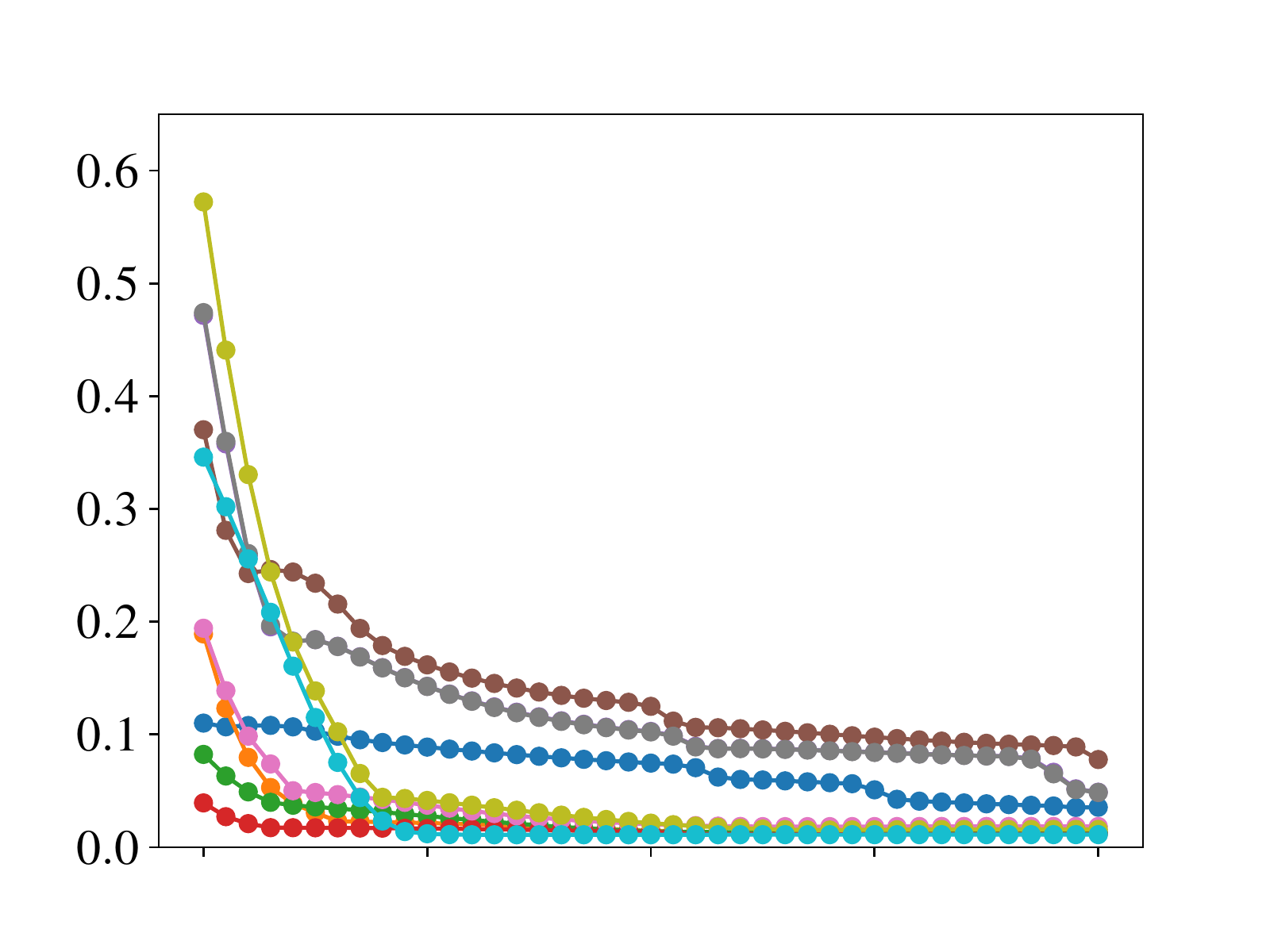}\vspace{-0.1cm} \caption{NLMPC using surrogate {\bf (ours)}}
    \end{subfigure}
    \begin{subfigure}[b]{0.19\textwidth}
    \centering
        \includegraphics[trim={52 25 30 30}, clip, scale=0.19]{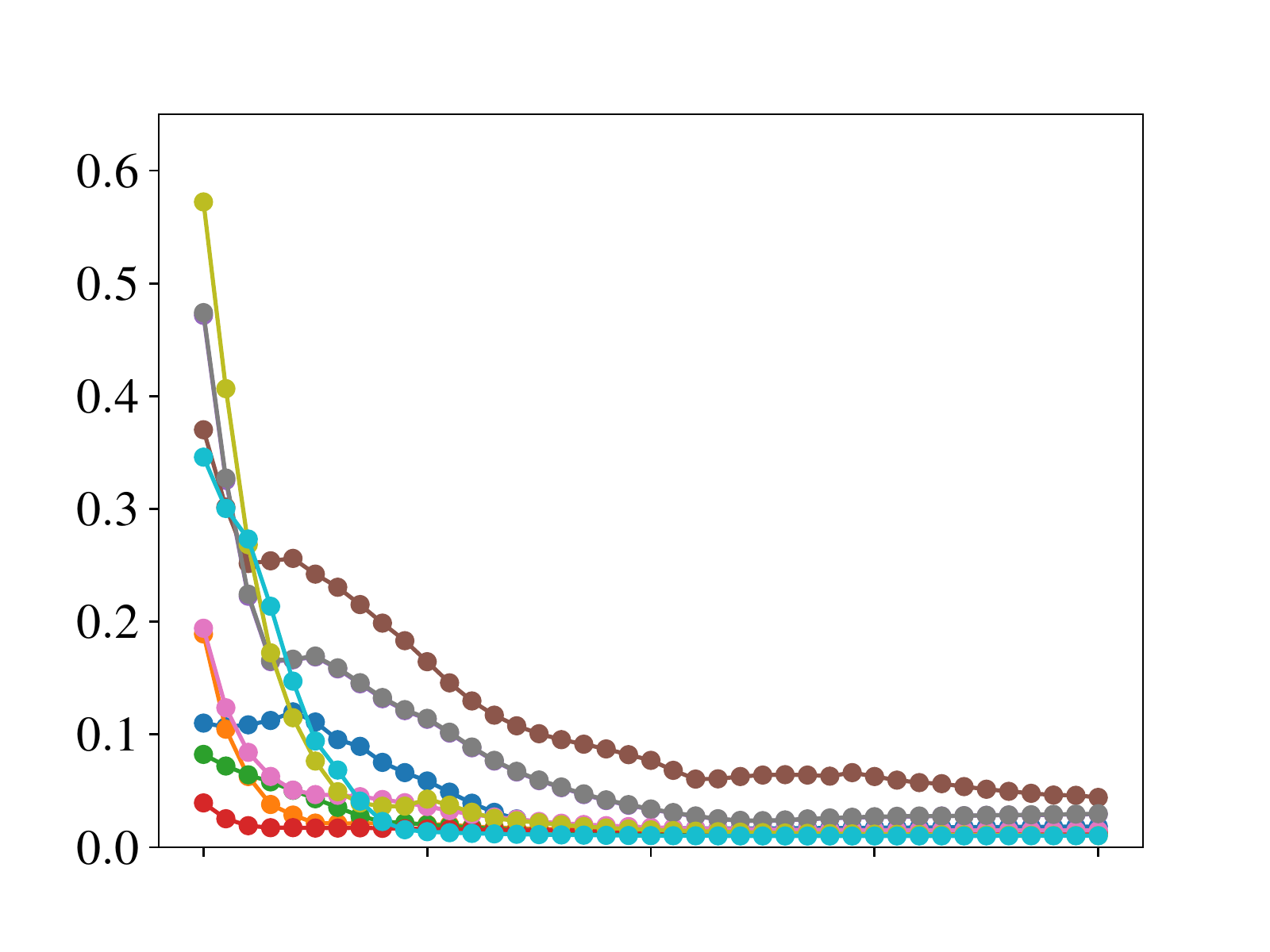}\vspace{-0.1cm}
        \caption{MPC demonstrator using nominal}
    \end{subfigure}
    \begin{subfigure}[b]{0.19\textwidth}
    \centering
        \includegraphics[trim={52 25 0 40}, clip, scale=0.19]{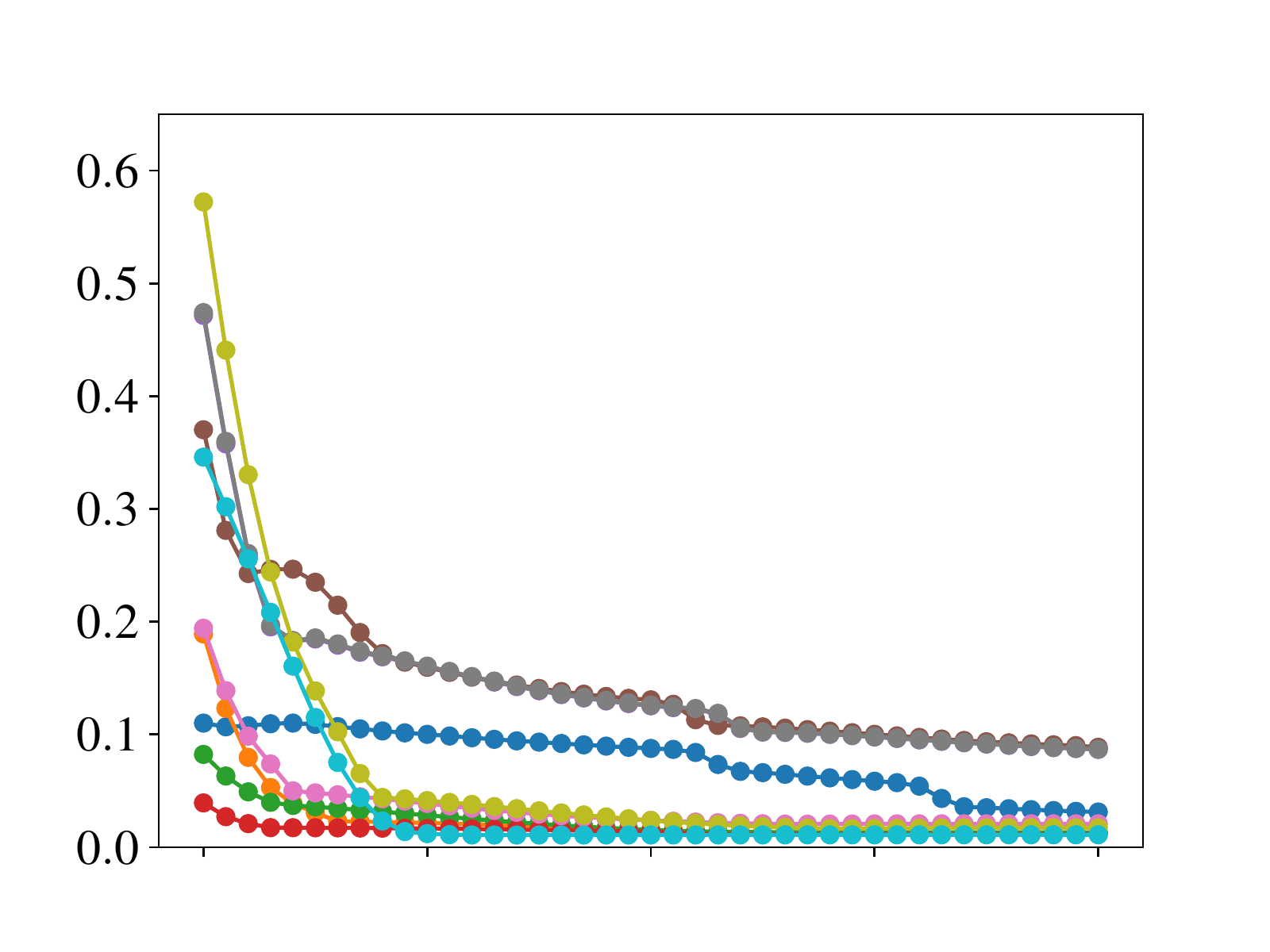}\vspace{-0.1cm}
        \caption{NLMPC using nominal {\bf (ours)}}
    \end{subfigure}
    
   
    \caption{{\bf Car kinematics: Transfer from surrogate to a nominal model.} {\bf Top}: Lyapunov function contours at $\phi=0$ with trajectories for $40$ steps. {\bf Bottom}: Lyapunov function evaluated for specific policy on several initial states (decreasing means more stable). }
    \label{fig:car_kinematics_lyap_surrogate}
\end{figure*}

\begin{figure*}[t]
    \captionsetup[subfigure]{justification=centering}
    \centering
    
    \begin{subfigure}[b]{0.32\textwidth}
    \centering
        \includegraphics[scale=0.28]{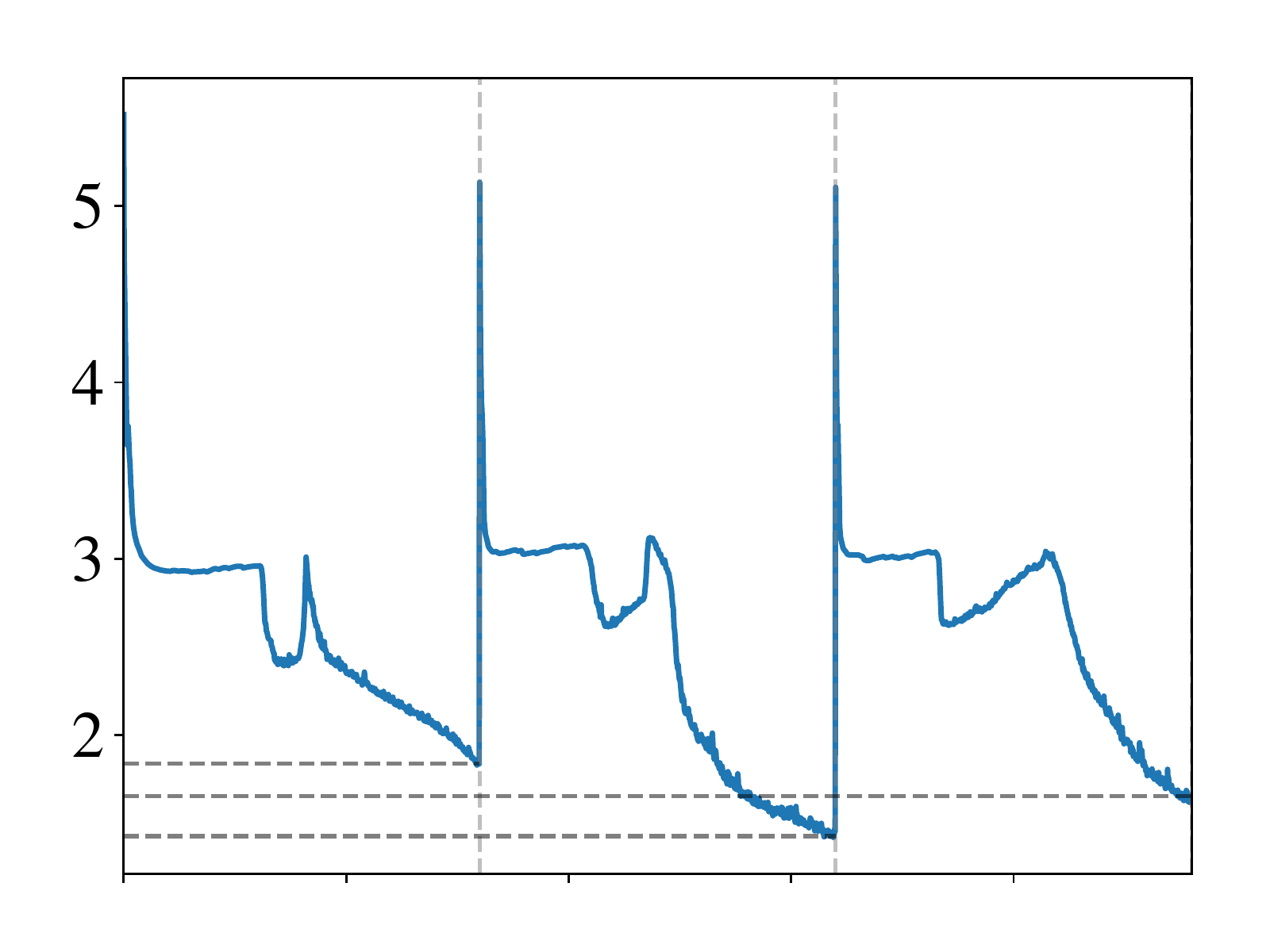}\vspace{-0.1cm}
        \caption{Lyapunov Loss \\($\log(1+x)$)}
    \end{subfigure}
    \begin{subfigure}[b]{0.32\textwidth}
        \centering
        \includegraphics[scale=0.28]{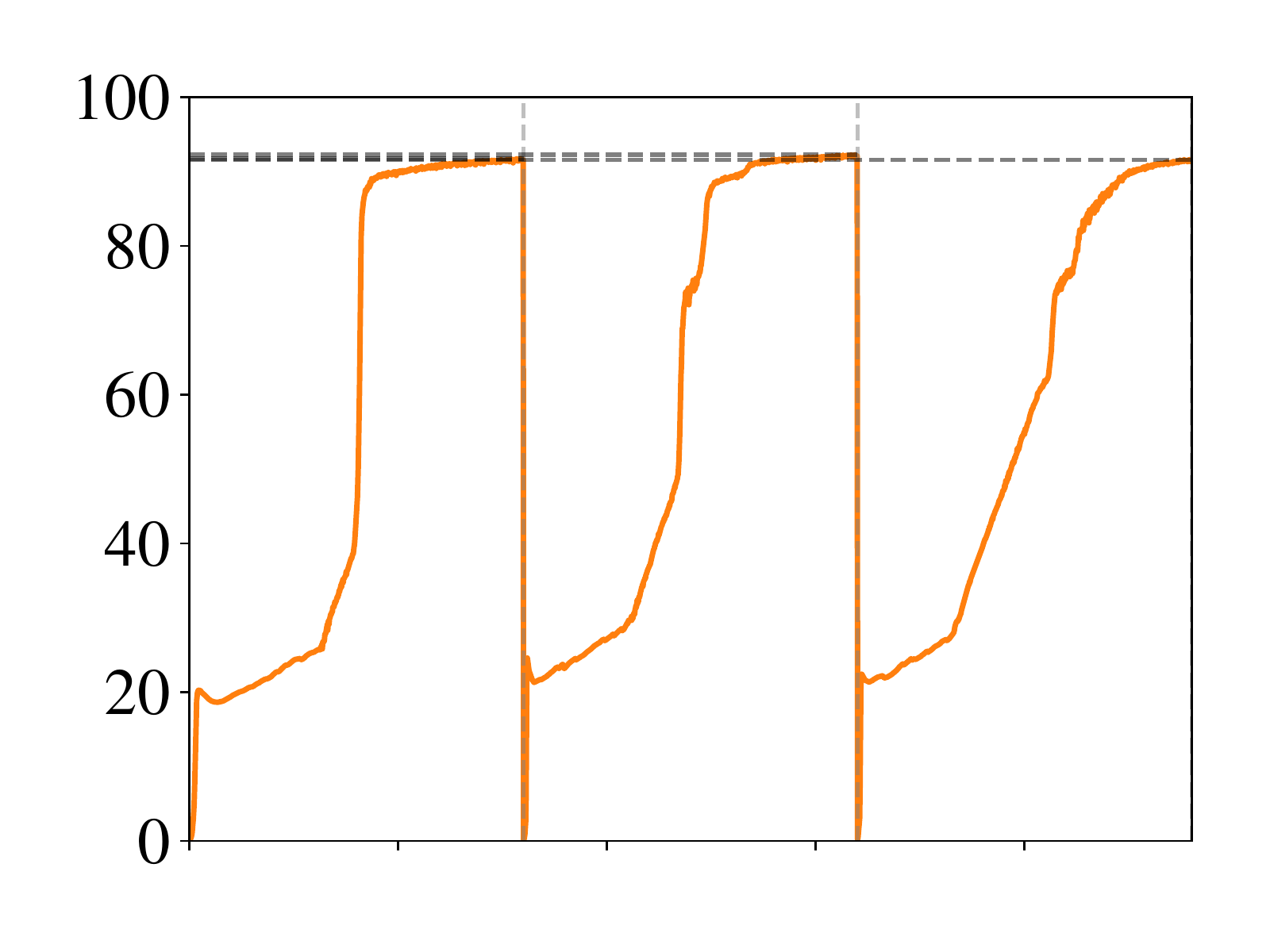}\vspace{-0.1cm}
        \caption{Verified Points \\(\%)}
    \end{subfigure}
    \begin{subfigure}[b]{0.32\textwidth}
        \centering
        \includegraphics[scale=0.28]{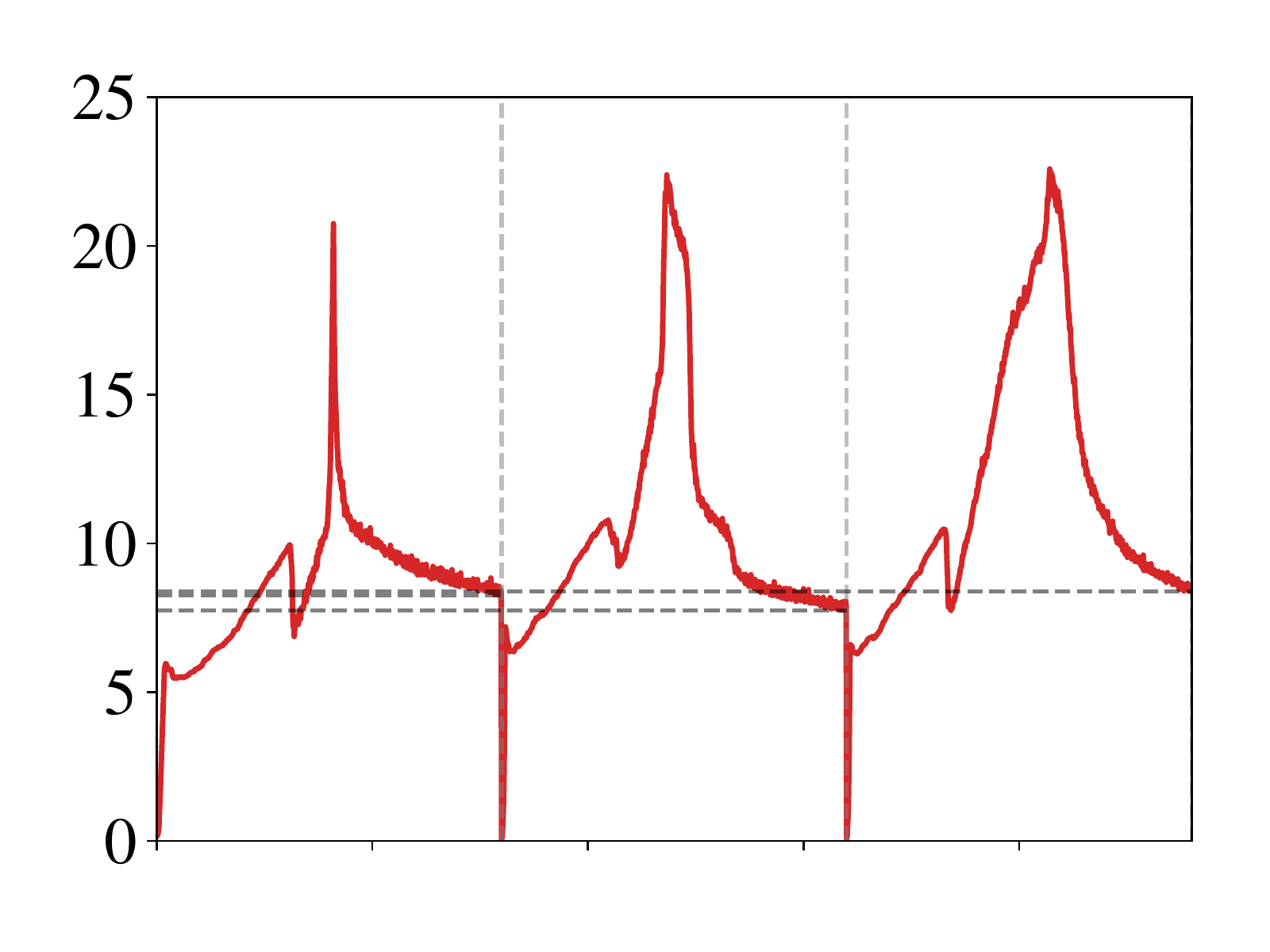}\vspace{-0.1cm}
        \caption{Not Verified Points \\(\%)}
    \end{subfigure}
    
    \medskip \medskip
    
    \begin{subfigure}[b]{0.32\textwidth}
        \centering
        \includegraphics[trim={30 25 45 20}, clip, scale=0.32]{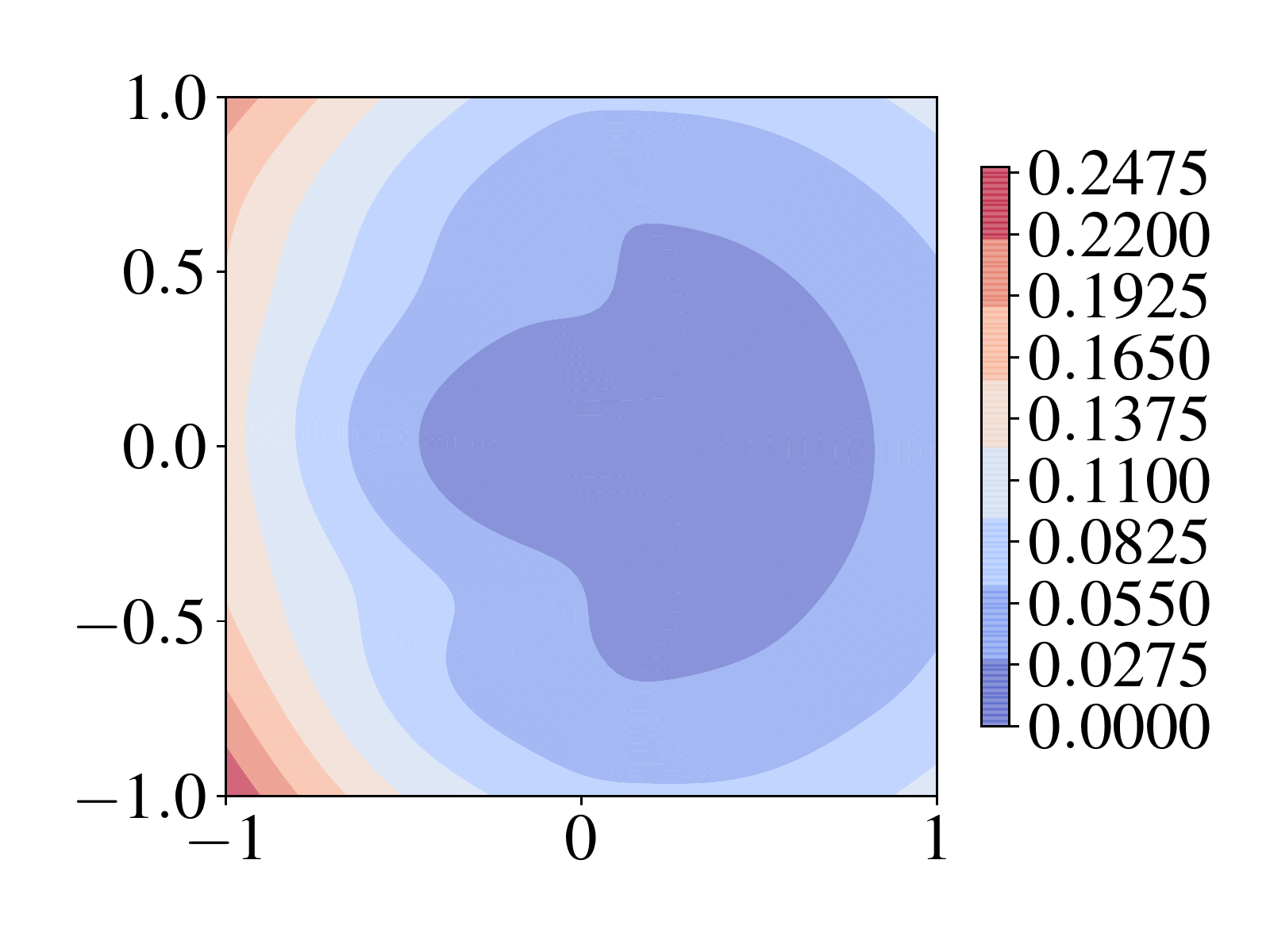}\vspace{-0.1cm}
    \end{subfigure}
    \begin{subfigure}[b]{0.32\textwidth}
        \centering
        \includegraphics[trim={30 25 0 20}, clip, scale=0.32]{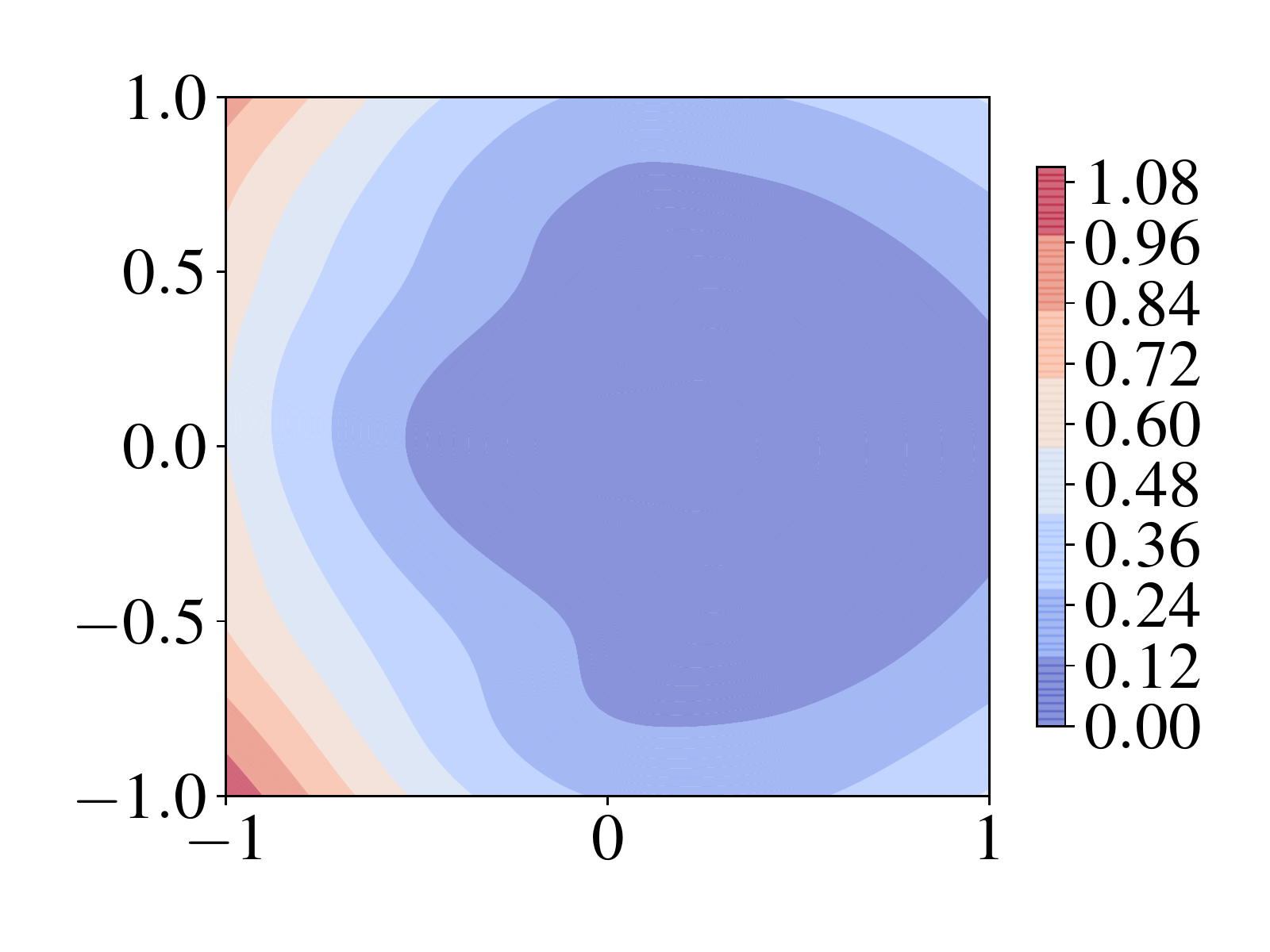}\vspace{-0.1cm}
    \end{subfigure}
    \begin{subfigure}[b]{0.32\textwidth}
        \centering
        \includegraphics[trim={30 25 35 20}, clip, scale=0.32]{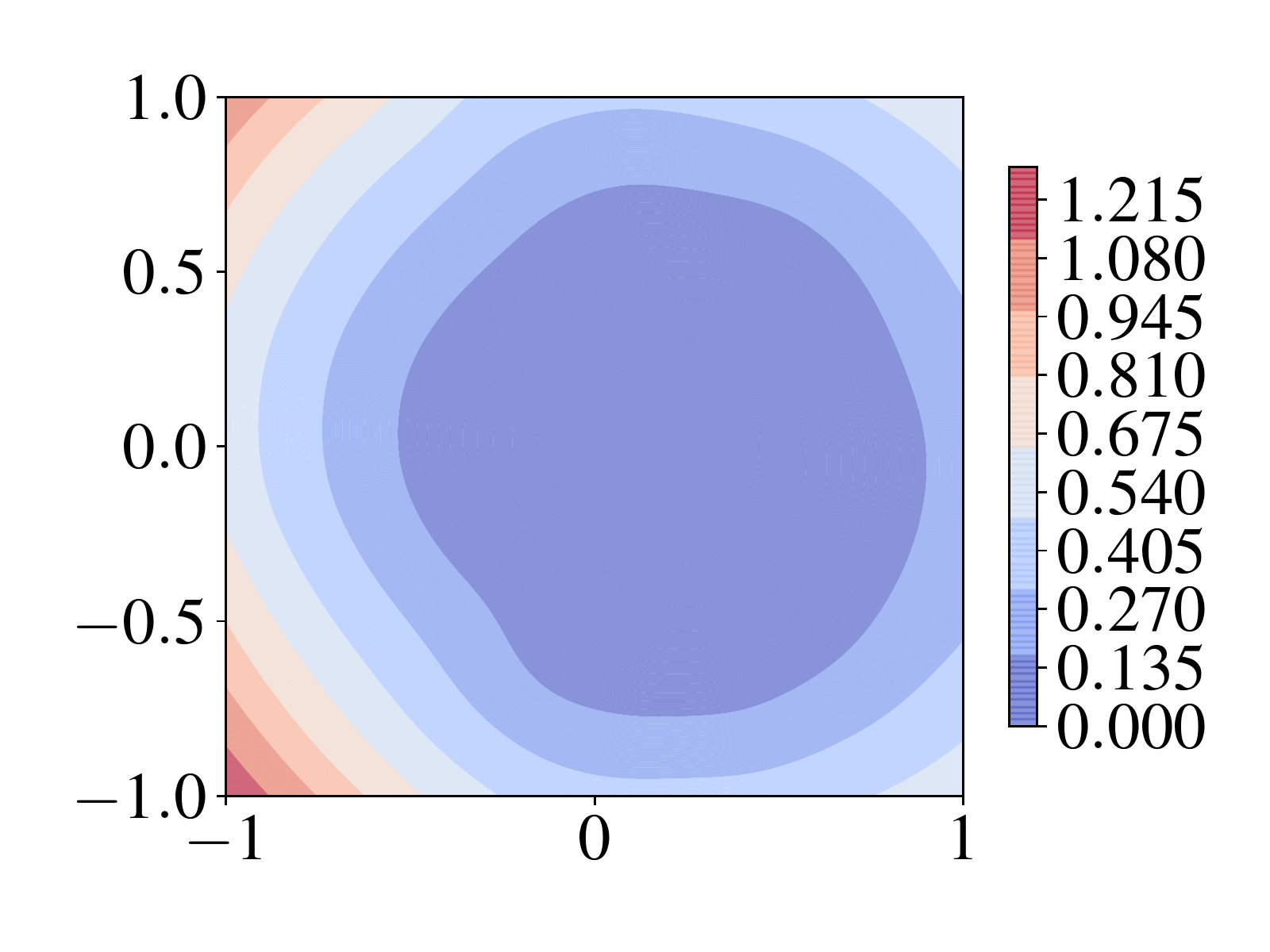}\vspace{-0.1cm}
    \end{subfigure}
    
    \medskip
    
    \begin{subfigure}[b]{0.32\textwidth}
        \centering
        \includegraphics[trim={0 20 0 0}, clip, scale=0.28]{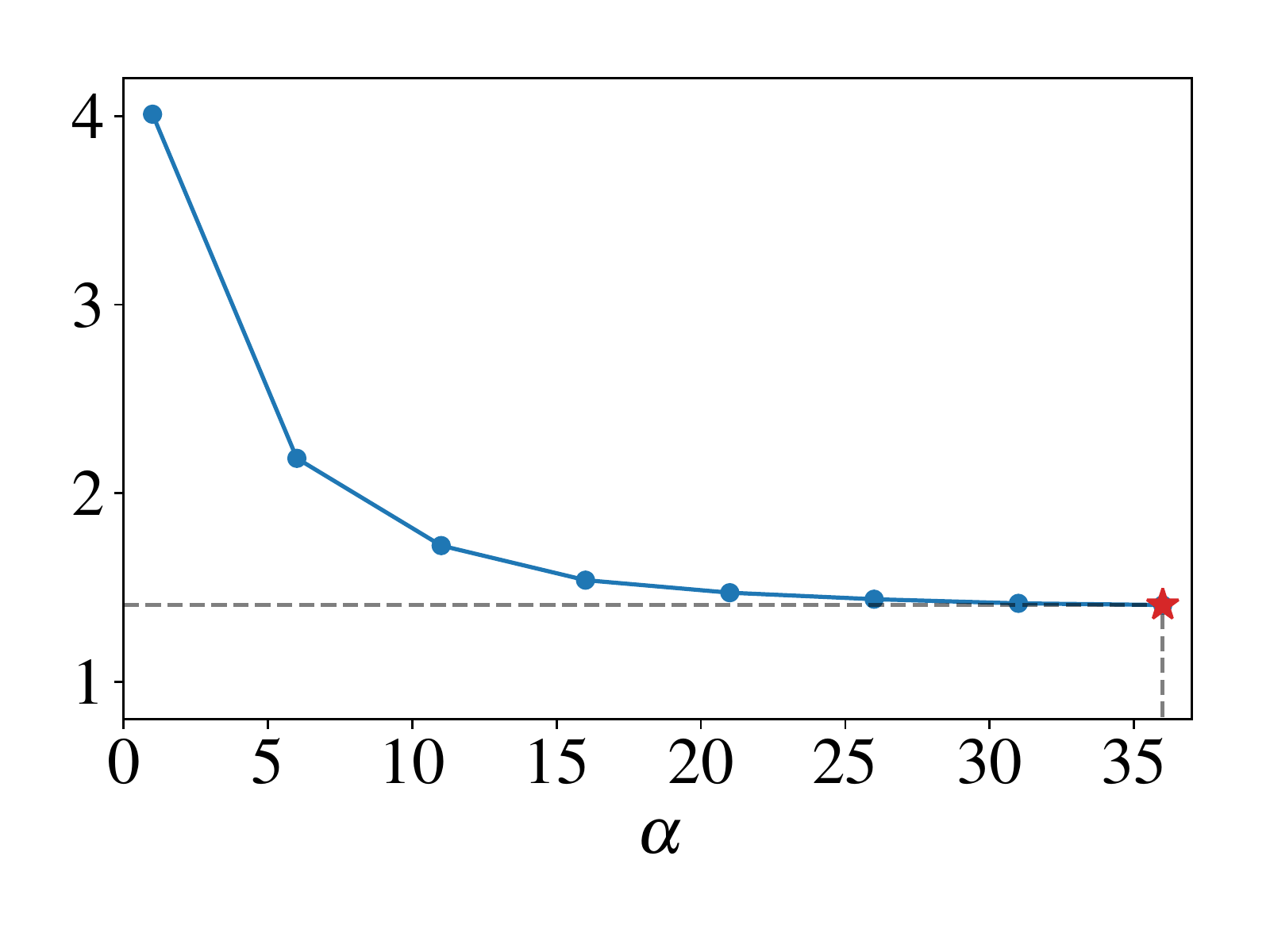}\vspace{-0.1cm}
        \caption*{Iteration 1}
    \end{subfigure}
    \begin{subfigure}[b]{0.32\textwidth}
    \centering
        \includegraphics[trim={37 20 0 0}, clip, scale=0.28]{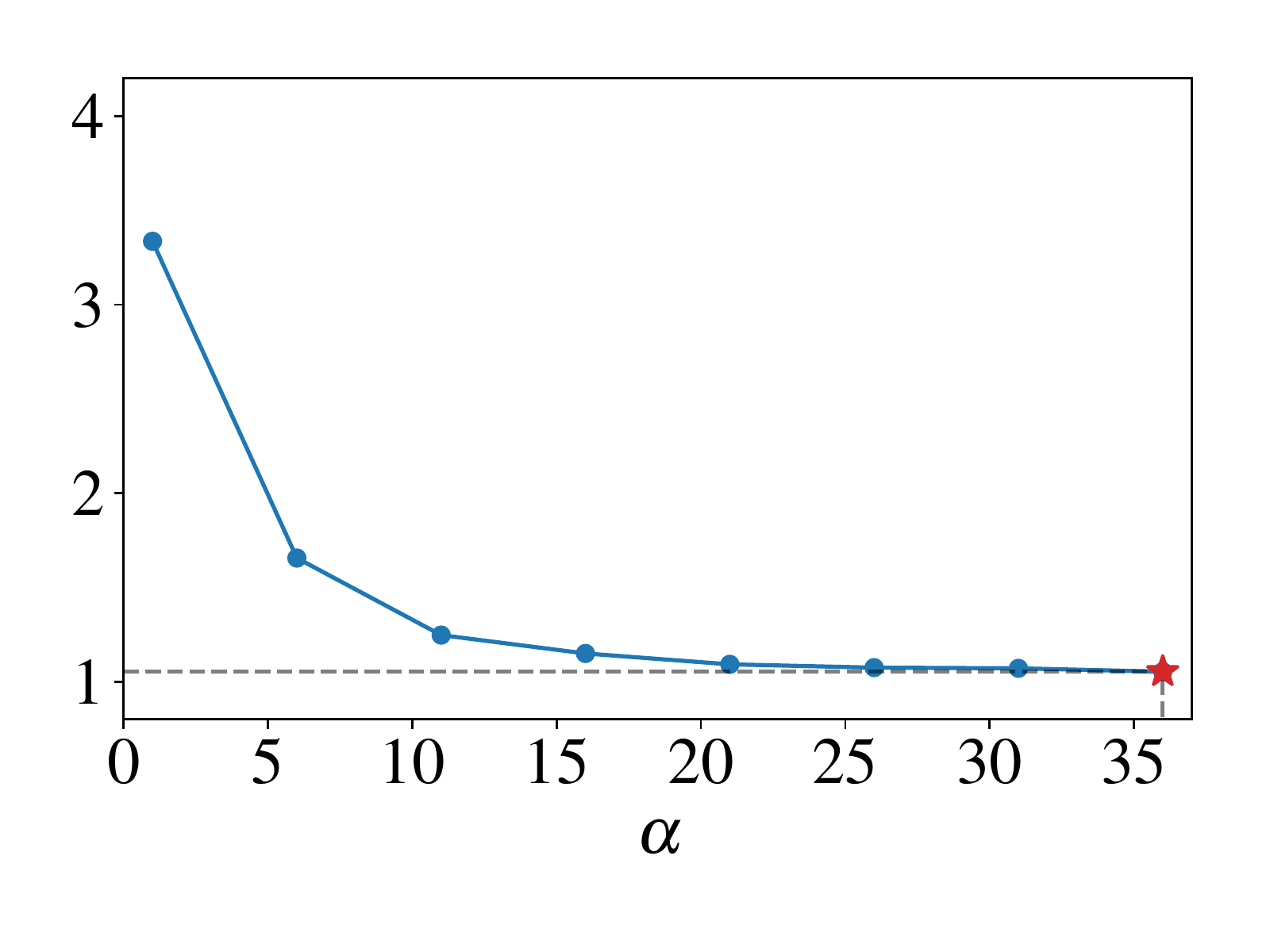}\vspace{-0.1cm}
        \caption*{Iteration 2 (best)}
    \end{subfigure}
    \begin{subfigure}[b]{0.32\textwidth}
    \centering
        \includegraphics[trim={37 20 0 0}, clip, scale=0.28]{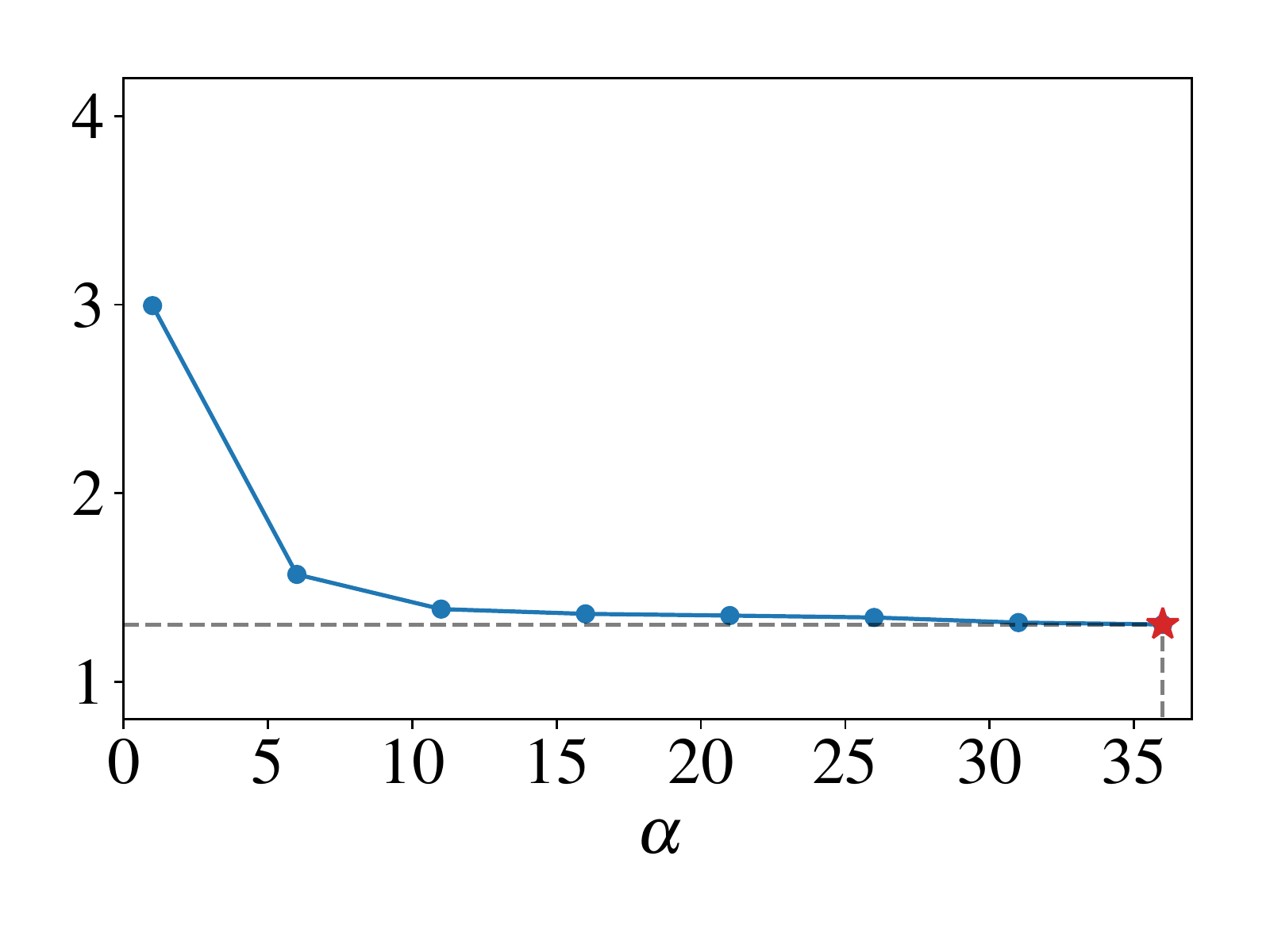}\vspace{-0.1cm}
        \caption*{Iteration 3}
    \end{subfigure}
    
    \caption{{\bf Car kinematics: Alternate learning on surrogate model.} After every $N_V=800$ epochs of Lyapunov learning, the learned Lyapunov function is used to tune the MPC parameters. \textbf{Top:} The training curves for Lyapunov function. Vertical lines separate iterations. \textbf{Middle:}  The resulting Lyapunov function $V$ at $\phi=0$ with the best performance. \textbf{Bottom:} Line-search for the MPC parameter $\alpha$ to minimize the Lyapunov loss (\protect\ref{eq:lyapunov_loss}) with $V$ as terminal cost. The loss is plotted on the y-axis in a $\log(1+x)$ scale. The point marked in red is the parameter which minimizes the loss.}
    \label{fig:car_alternate_learning_surrogate_complete}
    \vspace{-8pt}
\end{figure*}

\paragraph{Constrained inverted pendulum}
\label{subsec:pendulum}
In this task, the pendulum starts near the unstable equilibrium ($\theta=0\si{\degree}$). The goal is to stay upright. We bound the input so that the system cannot be stabilized if $|\theta|>60\si{\degree}$. 
%
We use an MPC with horizon $4$ as a demonstrator, with terminal cost, $ 500 x^T P_{\text{LQR}} x$, where $P_{\text{LQR}}$ is the LQR optimal cost matrix. This is evaluated on $10K$ equally spaced initial states to generate the dataset $\mathcal{D}_{\text{demo}}$. We train a grey-box NN model, $\hat{f}$ using $10K$ random transition tuples. More details are in Appendix. 
The learned $V$ and $\alpha$, obtained from~\algref{alg:alternateDescent}, produce a one-step MPC that stabilizes both the surrogate and the actual system. \tabref{tab:lyap_results_pendulum} shows that the loss and percentage of verified points improve across iterations. The final ROA estimate is nearly maximal and is depicted along with the safe trajectories, produced by the MPC while using predictions from the nominal and surrogate models, in~\figref{fig:pendulum_lyap}. The performance matches that of the baseline and the transfer is successful due to the accuracy of the learned model.  A full ablation is in Appendix.  

\paragraph{Constrained car kinematics}
\label{subsec:car_kinematics}

The goal is to steer the $(x,y,\theta)$ to $(0,0,0)$ with constraints $[\pm10m,\pm10m,\pm180^\circ]$. This is only possible through non-linear control. 
%
The vehicle cannot move sideways, hence policies such as LQR are not usable to generate demonstrations. Thus to create $D_{\text{demo}}$, an MPC with horizon $5$ is evaluated over $10K$ random initial states. The surrogate, $\hat{f}$ is a grey-box NN trained using $10K$ random transition tuples. More details are in Appendix.
\figref{fig:car_alternate_learning_surrogate_complete} shows the learning curves, training of the Lyapunov function over iterations and line-search for the MPC auto-tuning. \tabref{tab:vehicle_model_metric} summarises the metrics improvement across the iterations, indicating an increase in the ROA when a perfect model is used. With an imperfect model, the second iteration gives the best results, as shown in~\tabref{tab:vehicle_surrogate_result}. 

We test the transfer capability of the approach in two ways. First, we learn using the nominal model and test using the surrogate model for the MPC predictions. This is reported in Appendix for the sake of space.  
Second, the learning is performed using the surrogate model as in~\algref{alg:alternateDescent}, and the MPC is then tested on the nominal model while still using the surrogate for prediction. This is depicted in~\figref{fig:car_kinematics_lyap_surrogate}. Our MPC works better than the demonstrator when using the incorrect model. The learned MPC transfers successfully and completes the task safely.




\begin{table*}[t]
\centering
\begin{minipage}{\textwidth}
    \centering
    \caption{\addition{{\bf Comparison with baselines.} We compare our NLMPC (with surrogate model for predictions) with baselines.  In the pendulum, our approach is second to the demonstrator for less than $1\%$ margin. In the car task, NLMPC performs better than all baselines and improves convergence from the demonstrator, while it is nearly on par with the latter on constraints.}}
    \label{tab:compare_baselines}
    \begin{center}
    \begin{small}
    \begin{sc}
    \begin{tabular}{l|c|c|c|c}
         \toprule
         \multirow{2}{*}{Algorithm} & \multicolumn{2}{c|}{Constrained Inverted Pendulum} & \multicolumn{2}{c}{Constrained Car Kinematics} \\
         \cline{2-5}
         & Stability (\%) & Safety (\%) & Stability (\%) & Safety (\%) \\
         \midrule
         PPO (v1) & 14.67 & 14.66 & 15.17 & 0.50 \\
         PPO (v2) & 26.33 & 26.33 & 8.16 & 0.83 \\
         SAC (v1) & 29.99 & 29.99 & 12.33 & 0.00\\
         SAC (v2) & 27.17 & 27.17 & 8.00 & 0.00 \\
         MBPO (v1) & 12.67 & 12.67 & 6.00 & 0.00 \\
         MBPO (v2) & 26.00 & 26.00 & 6.00 & 0.00 \\
         MPC (demo) & \bf{36.00} & \bf{36.00} & 81.33 & \bf{73.33} \\
         \bf{NLMPC} & 35.33 & 35.33 & \bf{86.00} & 72.67 \\
         \bottomrule
    \end{tabular}
    \end{sc}
    \end{small}
    \end{center}
    \end{minipage}
    \vspace{-8pt}
\end{table*}

\paragraph{Comparison to baselines}

Prior works such as constrained policy optimization (CPO)~\cite{achiam2017constrained} provide safety guarantees in terms of constraint satisfaction that hold in expectation. However, due to unavailability of a working implementation, we are unable to compare our approach against it. Instead to enforce safety constraints during training of the RL algorithms, we use two different strategies: \texttt{v1}) early episode termination; \texttt{v2}) reward shaping with a constraint penalty. The \texttt{v2} formulation is similar to the one used in~\cite{Ray2019}, which demonstrated its practical equivalence to CPO when tuned. We compare our approach against model-free and model-based baseline algorithms. For the model-free baselines, we consider the on-policy algorithm proximal policy optimization (PPO)~\cite{schulman2017proximal} and the off-policy algorithm soft actor-critic (SAC)~\cite{haarnoja2018soft}. For model-based baselines, we consider model-based policy optimization (MBPO)~\cite{janner_when_2019} and the demonstrator MPC. Further details about the reward shaping and learning curves are in Appendix.

We consider the performance of learned controllers in terms stability and safety. Stability performance is analogous to the success rate in performing the set-point tracking task. We consider a task is completed when $||x(T)||_2 < 0.2$ where $T$ is the final time of the trajectory. For the car, we exclude the orientation from this index. The safety performance combines the former with state constraints satisfaction over the entire trajectory. As shown in~\tabref{tab:compare_baselines}, for the inverted pendulum, all the policies lead to some safe trajectories. Note that the demonstrator (which has an LQR terminal cost) is an optimal controller and gives the maximum achievable performance. In terms of stability performance, our approach performs as good as the demonstrator MPC. The RL trained policies give sub-optimal behaviors, i.e. sometimes the system goes to the other equilibria. For the car, the demonstrator MPC is sub-optimal due to non-linearities. NLMPC improves upon it in performance and it is on par with it in terms of safety. NLMPC also significantly outperforms all of the considered RL baselines while using less samples for learning\footnote{For all our experiments, training datapoints: PPO: $4\times 10^6$, SAC: $4\times 10^6$, MBPO: $2.4\times 10^5$, NLMPC: $10^4$ (random) + $10^4$ (demonstrations).}. While the RL baselines have full access to the environments, it appears that our approach is better suited to non-linearities and constraints even by learning solely from offline data.

\section{Related Work} 

Stability and robustness of MPC and of discounted optimal control have been studied in several prior works~\cite{rawlings_mayne_paper,rawlingsMPC,Limon2009,Limon2003StableCM,Rakovic2012,Gaitsgory2016StabilizationWD}. Numerical stability verification was studied in \cite{bobiti_samplingdriven_nodate, bobiti_sampling-based_2016} and, using neural network Lyapunov functions in \cite{berkenkamp_safe_2017,gallieri_safe_2019}. Neural Lyapunov controllers were also trained in \cite{chang_neural_2019}.  MPC solvers based on iterative LQR (iLQR) were introduced in \cite{Tassa2012}. Sequential Quadratic Program (SQP) was studied in \cite{nocedal2006numerical}. 
NNs with structural priors have been studied in \cite{quaglino_snode_2020, yildiz_ode2vae_2019, pozzoli_tustin_2019}. 
Value functions for planning were learned in \cite{lowrey_plan_2018, Learning_from_Value_Function_Intervals, NEURIPS2018_f02208a0}. 
\cite{gallieri_safe_2019} learned a NN Lyapunov function and an NN policy with an alternating descent method, initialized using a known stabilizing policy. We remove this assumption and use MPC.
\addition{
Suboptimality was analysed in~\cite{grune2008infrhc} for MPC and in \cite{janner_when_2019} for policies. AWR \cite{peng_advantage-weighted_2019} seeks positive advantage only during the policy update, not for the critic.  Gaussian processes models have been studied in~\cite{koller2018learningsafeexp, Hewing_safe}.
}
\section{Conclusions}
We presented Neural Lyapunov MPC, a framework to train a stabilizing non-linear MPC based on learned neural network terminal cost and surrogate model. After extending existing theoretical results for MPC and value-based reinforcement learning, we have demonstrated that the proposed framework can incrementally increase the stability region of the MPC through offline RL and then safely transfer on simulated constrained non-linear control scenarios. Through comparison of our approach with existing RL baselines, we showed how NNs can be leveraged to achieve policies that outperform  these methods on safety and stability. 

Future work could address the reduction of the proposed sub-optimality bound, for instance through the integration of value learning with Lyapunov function learning as well as the optimal selection of the MPC prediction horizon. A broader class of stage costs and rewards could also be investigated.



\bibliographystyle{IEEEtran}
\bibliography{bibtex.bib}
\onecolumn
\newpage
\rule[0pt]{\columnwidth}{3pt}
\begin{center}
\huge{Neural Lyapunov Model Predictive Control  \\
\emph{Supplementary Material}}
\end{center}
\vspace*{3mm}
\rule[0pt]{\columnwidth}{1pt}

We provide proofs of the Theorems 1, 2 and 3, introduced in the paper, in Appendix \ref{sec:proofs}.  In Appendix B, we describe the formulation of the model predictive controller as a sequential quadratic program (SQP). In Appendix C, we discuss the experimental setup that comprises of the implementation specifics, details about baseline controllers, parameters for the experiments, and more description of the control problem setting. In Appendix D, we provide additional plots and results for the inverted pendulum and car kinematics examples. Finally, in Appendix E we discuss an algorithm for probabilistic safety verification. 

\section{Proof of the Theorems}\label{sec:proofs}

Here we provide the proofs of Theorems stated in the paper. We write the proof for Theorem~\ref{Theorem:performance} before Theorem~\ref{Theorem:stability} since it is simpler and helps in proving the latter.

\subsection*{Proof of Theorem \ref{lemma:adv}}
\label{appendix:TheoremAdv}

\begin{proof}
We denote the advantage of action $u$ with respect to an initial policy $K_0(x)$, namely  $\mathcal{A}^{K_0}(x,u)$, as  $\mathcal{A}^{\star}(x,u)$. We assume that there is a new policy $K(x)$ such that: $$u=K(x)\Rightarrow\mathcal{A}^{\star}(x,u)\geq0,$$ for all initial states $x$ such that  $x\in\mathbb{X}_s\subseteq\mathbb{X}$. Recall that~\cite{sutton1998a}, $\mathcal{A}^\star(x,u) =\mathcal{Q}^\star(x,u) - \mathcal{V}^\star(x)$, where the right-hand terms are, respectively, the action-value function (where we assume $K_0(x)$ is used for future actions) and the value of policy $K_0(x)$, namely, $\mathcal{V}^\star(x)=\mathcal{V}^{K_0}(x)$. Then, from their definition and by the above assumption:
\begin{equation}
u=K(x)\Rightarrow0\leq \mathcal{A}^\star(x,u) =\mathcal{Q}^\star(x,u) - \mathcal{V}^\star(x)= \mathbf{E}[r(x,u)+\gamma \mathcal{V}^\star(x^{+})] -\mathcal{V}^\star(x),\ \forall x \in\mathbb{X}_s. 
\end{equation}
Since, by assumption, $r(x,u)=-\ell(x,u)$ and $\ell(x,u)\geq l_\ell \|x\|_2^2$, is positive definite ($l_\ell>0$) and zero at the origin (the target), then we have that the value function also satisfies $\mathcal{V}^\star(x)=-V^\star(x)$, where $V^\star$ satisfies a condition similar to (\ref{eq:lyap1}) provided that either $\ell(x,K_0(x))\rightarrow 0, \forall x \in\mathbb{X}_s$ ($K_0$ is stabilizing) or that $\gamma<1$. In other words, under these assumptions there is a positive constant $L_{V^\star} $ such that:
\begin{equation}\label{eq:bound_opt}
    l_\ell \|x\|_2^2  \leq  V^\star(x) \leq L_{V^\star} \|x \|_2^2.  
  \end{equation}
From the above considerations, it follows that $\forall x \in\mathbb{X}_s$: 
 \begin{equation}
u=K(x)\Rightarrow 0\leq \mathcal{A}^\star(x,u) = \mathbf{E}[-\ell(x,u) -\gamma V^\star(x^{+})]+V^\star(x) = -\mathbf{E}[\ell(x,u)] -\gamma\mathbf{E}[ V^\star(x^{+})]+V^\star(x),  
\end{equation} 
in other words:
\begin{equation}
     \gamma\mathbf{E}_{x^{+}}[ V^\star(x^{+})]-V^\star(x)+\mathbf{E}[\ell(x,u)]=\gamma\mathbf{E}[ V^\star(x^{+})|x,u]-V^\star(x)+\ell(x,u)\leq 0, 
\end{equation}
where, in the last step, we have used the fact that the policy $K(x)$ and the loss $\ell$ are  deterministic. 

Following \cite{Qin_2020}, the stability of a stochastic system origin is determined with high probability by means of (\ref{eq:bound_opt}) as well as:
\begin{equation}
    \mathbf{E}[ V^\star(x^{+})|x,u]-V^\star(x)+\ell(x,u)\leq 0, 
\end{equation}
that is, when the discount factor is $\gamma=1$. The same applies to deterministic systems, and with probability 1. Recall that in our setting $\gamma=1$ yields a finite value only when trajectories terminate at $\ell(x,u)=0$ under $K_0$. 

For systems that are deterministic, we generalize this result to show that there exist a discount factor less than 1 for which the set $\mathbb{X}_s$ is invariant, namely, the trajectories starting in $\mathbb{X}_s$ stay in $\mathbb{X}_s$. Recall that we have assumed that $\mathbb{X}_s$ satisfying the non-negative advantage condition is a level set of $V^\star$, namely, $$\mathbb{X}_s=\{x\in\mathbb{R}^{n_x}: V^\star(x)\leq c\},\  c>0.$$
 In order to prove the claim, it is sufficient show that, there exists a $\bar{\gamma}\in(0,1)$ and a positive scalar $c_1<c$ such that trajectories starting in $\mathbb{X}_s$ terminate in  $\mathbb{X}_{\bar\gamma}=\{x\in\mathbb{R}^{n_x}: V^\star(x)\leq c_1\}$. 
 In other words, we wish to show that, given $\bar{\gamma}$, \begin{equation}
     c_1\leq V^\star(x)\leq c \Rightarrow V^\star(x^{+})< V^\star(x). 
 \end{equation} To proceed, notice that in $\mathbb{X}_s$ we have that:

\begin{equation}
    \gamma V^\star(x^{+})\leq V^\star(x)-\ell(x,u)= \gamma V^\star(x) -\ell(x,u) + (1-\gamma) V^\star(x).
\end{equation} 
Rearranging the terms, we wish the following quantity to be \emph{less than zero}:
\begin{equation}\label{eq:inequality1}
    \gamma V^\star(x^{+})-\gamma V^\star(x)= -\ell(x,u) + (1-\gamma) V^\star(x). 
\end{equation}
In other words, we wish that:
\begin{equation}
     (1-\gamma) <\frac{\ell(x,u)}{V^\star(x)},\ \quad \text{or more specifically, } \quad \gamma > 1-\frac{\ell(x,u)}{V^\star(x)}. 
\end{equation}
The final step consists of using the describing function inequalities, leading to the sufficient condition:
\begin{equation}
     \gamma > 1-\frac{l_\ell\|x\|_2^2}{L_V\|x\|_2^2}=1-\frac{l_\ell}{L_V}=1-\epsilon,
\end{equation}
where we defined $\epsilon=\frac{l_\ell}{L_V}$, with $\epsilon\in(0,1)$ and hence the claim follows with $c_1=0$. 

Further lowering the discount factor can be allowed by considering that constraints are bounded and so is $\mathbb{X}_s$. Hence (\ref{eq:inequality1}) satisfies:
\begin{equation}
\gamma V^\star(x^{+})-\gamma V^\star(x)= -\ell(x,u) + (1-\gamma) V^\star(x) < -\ell(x,u) + (1-\gamma)c. 
\end{equation}
For which, since the actions are bounded and hence $\beta V^\star(x)\geq  \ell(x,K(x))$, for some $\beta>0$ ($\beta=1 \Leftarrow K_0=K=K^\star $), then it is sufficient to have: 
\begin{equation} 
V^\star(x) > \frac{(1-\bar{\gamma})}{\beta}c=c_1.
\end{equation}
Hence, the set $\mathbb{X}_{\bar{\gamma}}$ is defined and $\bar{\gamma}$ determines whether it is inside $\mathbb{X}_s$ which would make $\mathbb{X}_s$ invariant. 
\end{proof}

\subsubsection*{Possible extensions of  Theorem \ref{lemma:adv}} 
This paper considered the case of a quadratic (or positive definite) cost with a single minimum at a given target (e.g. the origin). Reinforcement learning and economic MPC often deal with different classes of rewards/costs, possibly leading to complex equilibria. We make some considerations on how future work could aim to address more general cases. 

An extension of the Theorem \ref{lemma:adv} could be considered for the case when the reward is upper bounded by $r_{\max}$, e.g. if we can express the reward as:  
$$
r(x,u)=r_{\max}-\ell(x,u),
$$ 
with $\min_u\ell(x,u)>0$ outside some target set, $x\not\in\mathbb{T}$, and zero at the target set $\mathbb{T}$.   Then the value function satisfies 
$$
\mathcal{V}^\star(x)=-V^\star(x)+\sum_{t=0}^\infty \gamma^t r_{\max},
$$
where $V^\star$ satisfies a condition similar to (\ref{eq:lyap1}). In particular, the constant reward terms in the infinite sums cancel out from the  stability condition:
$$
\mathcal{A}^\star(x,u)=r(x,u) +\gamma\mathcal{V}^\star(x^{+})-\mathcal{V}^\star(x)= -\gamma V^\star(x^{+}) + V^\star(x) - \ell(x,u)>0,\ \forall x\not \in \mathbb{T}.  
$$
This hints to the fact that the advantage function could play a crucial role in stability and convergence under RL policies. 

\subsubsection*{Limitations of  Theorem \ref{lemma:adv}}
 Theorem \ref{lemma:adv} requires the  value function $\mathcal{V}^\star(x)$ to be known exactly, which is generally hard for continuous control problems with function approximation. Our subsequent results relax this assumption to be that the candidate function $V$ is a valid Lyapunov function rather than the exact value function. This for instance means that scaling errors, local sub-optimality and potentially even a large bias in the Bellman error could be tolerated without  affecting the task success.   

\subsection*{Proof of Theorem \ref{Theorem:performance}}
\label{appendix:Theorem2}

\begin{proof}
To prove the result, we first write: 
\begin{align}
    \label{eq:proof2}
    & \mathbf{E}_{x\in\mathcal{D}}[J^\star_{\text{MPC}}(x)] - \mathbf{E}_{x\in\mathcal{D}}[J^\star_{V^\star}(x)] = 
    \underbrace{\mathbf{E}_{x\in\mathcal{D}}[J^\star_{\text{MPC}}(x)]  - \mathbf{E}_{x\in\mathcal{D}}[J^\star_{\text{MPC},f}(x)]}_{I_1} ~+
    \underbrace{\mathbf{E}_{x\in\mathcal{D}}[J^\star_{\text{MPC},f}(x)] - \mathbf{E}_{x\in\mathcal{D}}[J^\star_{V^\star}(x)]}_{I_2} ,
\end{align}
where $\mathbf{E}_{x\in\mathcal{D}}[J^\star_{\text{MPC},f}(x)]$ denotes the MPC performance when a perfect model is used for predictions. 

For the term $I_2$, \cite{lowrey_plan_2018} provide a bound on the performance of the MPC policy. It is important to note that in their problem formulation, the MPC's objective is defined as a maximization over the cumulative discounted reward, while in our formulation~\eqref{problem} we consider a minimization over the cost. Consequently, compared to inequality presented by~\cite{lowrey_plan_2018}, there is a change in sign of the terms in the left-hand side of the inequality. This means: 
\begin{equation}
    \label{eq:lowrey_ineq}
    \mathbf{E}_{x\in\mathcal{D}}[J^\star_{\text{MPC},f}(x)]  - \mathbf{E}_{x\in\mathcal{D}}[J^\star_{V^\star}(x)]\leq \frac{2 \gamma^N\epsilon}{1-\gamma^N}.
\end{equation}


We now focus on the term $I_1$ in~\eqtref{eq:proof2}. Let us denote $x^\star(i)$ and $u^\star(i)$ as the optimal state and action predictions respectively, obtained by using the correct model, $f$, and the MPC policy at time $i$. By the principle of optimality, the optimal sequence for the MPC using the correct model, $\underbar{u}_f$, can be used to upper-bound the optimal cost for the MPC using the surrogate model:
\begin{align}
    &\mathbf{E}_{x\in\mathcal{D}}[J^\star_{\text{MPC}}(x)]  - \mathbf{E}_{x\in\mathcal{D}}[J^\star_{\text{MPC},f}(x)] \leq 
    \mathbf{E}_{x\in\mathcal{D}}[J_{\text{MPC}}(x, \underbar{u}_f)] - \mathbf{E}_{x\in\mathcal{D}}[J^\star_{\text{MPC},f}(x)]. \label{eq:Theorem_I1}
\end{align}

Since the input sequence is now the same for both terms in right-hand side of~\eqtref{eq:Theorem_I1}, the difference in the cost is driven by the different state trajectories cost (the cost on $x$ over the horizon which includes the state-constraint violation penalty, as defined in~\eqtref{eq:state_const_cost_ineq}) as well as the terminal cost. In form of equation, this means:
\begin{align}
    &\mathbf{E}_{x\in\mathcal{D}}[J_{\text{MPC}}(x, \underbar{u}_f)] - \mathbf{E}_{x\in\mathcal{D}}[J^\star_{\text{MPC},f}(x)] \nonumber \\
    & = \mathbf{E}_{x\in\mathcal{D}}[J_{\text{MPC}}(x, \underbar{u}_f) - J^\star_{\text{MPC},f}(x)]  \nonumber \\
    & = \mathbf{E}_{x(0)\in\mathcal{D}}\Bigg[\sum_{j=0}^{N-1}\gamma^j\Big\{\hat{x}(j)^T Q \hat{x}(j) - x^\star(j)^T Q x^\star(j)\Big\}
    +  \gamma^N \alpha \big\{V(\hat{x}(N)) - V(x^\star(N))\big\} \nonumber \\
    & \qquad\qquad\quad +  \sum_{j=0}^{N}\big\{\ell_{\mathbb{X}}(\hat{x}(j))-\ell_{\mathbb{X}}({x}^\star(j))\big\}\Bigg]. \label{eq:mpc_cost_diff_ineq}
\end{align}

Recall that we assume the surrogate model is Lipschitz with constant $L_{\hat{f}x}$. This means that $\forall \tilde{x}, x \in \mathbb{R}^{n_x}$ and the same input $u \in \mathbb{R}^{n_u}$, we have:
\begin{align*}
    & ||\hat{f}(\tilde{x}, u)-\hat{f}(x, u)||_2 \leq  L_{\hat{f}x}     ||\tilde{x} - x||_2,
\end{align*}

Further, from~\eqtref{eq:pred_error}, $\forall (x, u) \in \tilde{\mathbb{X}}\times \mathbb{U}$, we have:
\begin{align*}
    ||w(x, u)||_2 = ||f(x, u) - \hat{f}(x, u)||_2 \leq \mu.
\end{align*}

Under the optimal policy for the correct model, let us denote the deviation in the state prediction when the MPC input prediction is applied with a different model, $\hat{f}$, as $\hat{d}(j) := \hat{x}(j)-{x}^\star(j)$.

At step $j=1$:

\begin{align*}
   \|\hat{d}(1)\|_2 &= \|\hat{x}(1)-x^\star(1)\|_2 \\ 
   &= \|\hat{f}(x^\star(0), u^\star(0))- f(x^\star(0), u^\star(0))\|_2 \\
   &= \|w(x^\star(0), u^\star(0))\|_2 \\
   &\leq \mu.
\end{align*}

At step $j=2$:

\begin{align*}
  \|\hat{d}(2)\|_2 &= \|\hat{x}(2) - x^\star(2)\| \\
  &= \|\hat{f}(x^\star(1)+\hat{d}(1), u^\star(1))-f(x^\star(1), u^\star(1))\| \\ 
  & =~\|\hat{f}(x^\star(1) + \hat{d}(1), u^\star(1)) - \hat{f}(x^\star(1), u^\star(1)) +\hat{f}(x^\star(1), u^\star(1))-  f(x^\star(1), u^\star(1))\| \\
  & \leq \underbrace{\|\hat{f}(x^\star(1)+\hat{d}(1),u^\star(1))-\hat{f}(x^\star(1),u^\star(1))\|}_{\leq L_{\hat{f}x} \|\hat{d}(1)\|} + \underbrace{\|\hat{f}(x^\star(1),u^\star(1))- f(x^\star(1),u^\star(1))\|}_{ = \|w(x^\star(1), u^\star(1))\| \leq \mu} \\
  & \leq L_{\hat{f}x}\mu + \mu.
\end{align*}

By induction, it can be shown that:
\begin{align} 
  \|\hat{d}(j)\|_2=\|\hat{x}(j)-{x}^\star(j)\|\leq \sum_{i=0}^{j-1}L_{\hat{f}x}^i \mu. \label{eq:nominal_ineq}
\end{align}

Alternately, if we assume the correct system that is to be controlled is Lipschitz with constant $L_{fx}$, then proceeding as before:

At step $j=1$:
\begin{align*}
   \|\hat{d}(1)\|_2 &= \|\hat{x}(1)-x^\star(1)\|_2 \\ 
   &\leq \mu.
\end{align*}

At step $j=2$:
\begin{align*}
  \|\hat{d}(2)\|_2 &= \|\hat{x}(2) - x^\star(2)\| \\
  &= \|\hat{f}(\hat{x}(1), u^\star(1))-f(\hat{x}(1) - \hat{d}(1), u^\star(1))\| \\ 
  & \leq \underbrace{\|\hat{f}(\hat{x}(1), u^\star(1))-f(\hat{x}(1),u^\star(1))\|}_{= \|w(\hat{x}(1), u^\star(1))\| \leq \mu} + \underbrace{\|f(\hat{x}(1), u^\star(1))- f(\hat{x}(1) - \hat{d}(1), u^\star(1))\|}_{\leq L_{f x} \|\hat{d}(1)\| } \\
  & \leq \mu + L_{f x}\mu
\end{align*}

By induction, again we have:
\begin{align}
  \|\hat{d}(j)\|=\|\hat{x}(j)-{x}^\star(j)\|\leq \sum_{i=0}^{j-1}L_{{f}x}^i\mu. \label{eq:surrogate_ineq}
\end{align}

Combining equations~\eqref{eq:nominal_ineq} and~\eqref{eq:surrogate_ineq} and by letting  $\bar{L}_{{f}}^i = \min (L_{\hat{f}x}^i, L_{{f}x}^i)$, we obtain:
\begin{equation}
 \|\hat{d}(j)\|=\|\hat{x}(j)-{x}^\star(j)\|\leq \sum_{i=0}^{j-1}\bar{L}_{{f}}^i\mu. \label{eq:model_state_ineq}   
\end{equation}

The following identity is used; $\forall \delta>0$:
\begin{equation}
    \|a+b\|_2^2\leq \left(1+\frac{1}{\delta}\right)\|a\|_2^2 + (1+\delta)\|b\|_2^2 \nonumber. 
\end{equation}
Hence, we can write the cost over the predicted state as:
\begin{align}
    &\hat{x}(j)^T Q\hat{x}(j) \nonumber \\
    &= \|Q^{1/2} \hat{x}(j)\|_2^2 = \|Q^{1/2}(x^\star(j)+\hat{d}(j))\|_2^2 \nonumber \\
    & \leq (1+\delta)\|Q^{1/2}x^\star(j)\|_2^2 + \left(1+\frac{1}{\delta}\right)\|Q^{1/2}\hat{d}(j)\|_2^2 \nonumber \\
    & \leq (1+\delta)x^\star(j)^T Q x^\star(j)  + \left(1+\frac{1}{\delta}\right)\|Q\|_2 \|\hat{d}(j)\|_2^2. \label{eq:predicted_cost}
\end{align}

From equations~\eqref{eq:model_state_ineq} and~\eqref{eq:predicted_cost}, $\forall j \in \{0, 1, ..., N-1\}$, we obtain:
\begin{align}
    & \hat{x}(j)^T Q\hat{x}(j) - x^\star(j)^T Q x^\star(j) \nonumber \\ 
    & \leq  \delta \underbrace{x^\star(j)^T Q x^\star(j)}_{=\ell(x^\star(j),0)}  + \left(1+\frac{1}{\delta}\right)\|Q\|_2 \left(\sum_{i=0}^{j-1}\bar{L}_{f}^i\mu\right)^2 \nonumber \\
    & \leq \delta\ \ell(x^\star(i),u^\star(i)) + \left(1+\frac{1}{\delta}\right)\|Q\|_2 \left(\sum_{i=0}^{j-1}\bar{L}_{f}^i\mu\right)^2. \label{eq:x_stage_cost_ineq}
\end{align}

Recall from~\eqtref{eq:lyap1} that $V(x)\leq L_V\|x\|_2^2$. Proceeding as before, we can write the following for the terminal cost:
\begin{align}
    & V(\hat{x}(N))-V(x^\star(N)) \leq \delta~ V(x^\star(N)) + \left(1+\frac{1}{\delta}\right)L_V \left(\sum_{i=0}^{N-1}\bar{L}_{f}^i\mu\right)^2. \label{eq:terminal_cost_ineq}
\end{align}

The final part of the proof concerns the constraints cost term. Let the state constraints be defined as a set of inequalities:
$$ \mathbb{X} = \left\{x\in\mathbb{R}^n: g(x)\leq 1\right\},$$
where $g$ is a convex function. For the optimal solution, $x^\star$, the violation of the constraint is represented through the slack variable:
$$ s^\star = s(x^\star) = \frac{(g(x^\star) - 1) + |g(x^\star) - 1|}{2}.$$
Since the constraints are convex and compact, and they contain the origin, then at the optimal solution,  $x^\star$, we have that there exists a $\mathcal{K}_\infty$-function, $\bar{\eta}(r)$, such that:
\begin{align}
    & \left|\ell_{\mathbb{X}}(s(x^\star+\hat{d}))-\ell_{\mathbb{X}}\left(s\left(x^\star \right)\right)\right|\leq  \bar{\eta}(\|\hat{d}\|). \nonumber
\end{align} 

Using the above inequality and~\eqtref{eq:model_state_ineq}, it follows that, $\forall j \in \{0, 1, ..., N\}$:  
\begin{align}
  \ell_{\mathbb{X}}(\hat{x}(j))-\ell_{\mathbb{X}}({x}^\star(j)) \leq \bar{\eta}\left(\sum_{i=0}^{j-1}\bar{L}_{f}^i \mu\right)=\bar{\eta}_j. \label{eq:state_const_cost_ineq}  
\end{align}

By combining equations~\eqref{eq:proof2},~\eqref{eq:lowrey_ineq},~\eqref{eq:Theorem_I1},~\eqref{eq:mpc_cost_diff_ineq},~\eqref{eq:x_stage_cost_ineq},~\eqref{eq:terminal_cost_ineq} and~\eqref{eq:state_const_cost_ineq}, we obtain the bound stated in the Theorem:
\begin{align*}
    &\mathbf{E}_{x\in\mathcal{D}}[J^\star_{\text{MPC}}(x)] - \mathbf{E}_{x\in\mathcal{D}}[J^\star_{V^\star}(x)]  \nonumber \\
    &\leq  \frac{2 \gamma^N\epsilon}{1-\gamma^N} + \left(1+\frac{1}{\delta}\right)\|Q\|_2\sum_{i=0}^{N-1}\gamma^i \left( \sum_{j=0}^{i-1} \bar{L}_{f}^j\right)^2 \mu^2 \nonumber \\
    &\quad + \left(1+\frac{1}{\delta}\right)\gamma^N \alpha L_V \left(\sum_{i=0}^{N-1}\bar{L}_{f}^i\right)^2 \mu^2 + \underbrace{\sum_{j=0}^{N}\bar{\eta}_j}_{=:\bar{\psi}(\mu)} \nonumber \\
    &\quad +  \delta\ \mathbf{E}_{x\in\mathcal{D}}\left[J_{\text{MPC}}^\star(x; f)\right].\nonumber
\end{align*}

\end{proof}

\subsection*{Proof of Theorem \ref{Theorem:stability}}
\label{appendix:Theorem1}
We prove the following, extended version of the theorem.
\begin{theorem}{\bf Stability and robustness} 
Assume that $V(x)$ satisfies (\ref{eq:lyap2_lqr}), with  $\lambda\in[0,1)$, $\mathbb{X}_T=\{0\}$. Then, for any horizon length $N\geq1$ there exist a constant $\bar{\alpha}\geq0$, a minimum discount factor  $\bar{\gamma}\in(0,1]$, and a model error bound $\bar{\mu}$ such that, if $\alpha\geq\bar{\alpha}$, $\mu\leq\bar{\mu}$ and $\gamma\geq\bar{\gamma}$, then, $\forall x(0)\in\mathcal{C}(\mathbb{X}_s)$
 \vspace{-0.cm}
 \begin{enumerate}
    \item If $N=1$,  $\mu=0$, then the system is asymptotically stable for any $\gamma>0$, $\forall x(0)\in\Upsilon_{N,\gamma,\alpha}$.
    \item If $N>1$, $\mu=0$, then the system reaches a set $\mathbb{B}_{\gamma}$ that is included in $\mathbb{X}_s$. This set increases monotonically with decreasing discount factors, $\gamma$, $\forall x(0)\in\Upsilon_{N,\gamma,\alpha}$. $\gamma=1\Rightarrow \mathbb{B}_{\gamma}=\{0\}$. 
    \item If $N>1$, $\mu=0$, and once in $\mathbb{X}_s$ we switch to the expert policy, then the system is asymptotically stable, $\forall x(0)\in\Upsilon_{N,\gamma,\alpha}$.
    \item If $\alpha V(x)$ is the optimal value function for the discounted problem, $\mu=0$, and if $\mathcal{C}(\mathbb{X}_s)=\mathbb{X}_s$, then the system is asymptotically stable, $\forall x(0)\in\Upsilon_{N,\gamma,\alpha}$. 
    \item If $\alpha V(x)$ is the optimal value function in $\mathbb{X}_s$ for the problem, $\mu=0$, and if $\mathcal{C}(\mathbb{X}_s)\neq\mathbb{X}_s$, then the system is asymptotically stable, $\forall x(0)\in\Upsilon_{N,\gamma,\alpha}$. 
    \item The MPC has a stability margin. If the MPC uses a surrogate model satisfying (\ref{eq:pred_error}), with one-step error bound $\|w\|_2^2<\bar{\mu}^2=\frac{1-\lambda}{L_VL_{\hat{f}x}^{2N}}l_s$, then the system is Input-to-State (practically) Stable (ISpS) and there exists a set $\mathbb{B}_{N, {\gamma},{\mu}}:\ x(t)\rightarrow \mathbb{B}_{N, {\gamma}, {\mu}}$, $\forall x(0)\in\beta\Upsilon_{N,\gamma,\alpha}$, $\beta\leq1$.
    \item If $\mu=0$, then $\alpha\geq\bar{\alpha}$ implies that $\alpha V(x)\geq V^\star(x), \forall x\in\mathbb{X}_s$, where $V^\star$ is the optimal value function for the infinite horizon problem with cost (\ref{eq:stage_cost}) and subject to (\ref{eq:constraints}). 
\end{enumerate}
\end{theorem}

\begin{proof}

In order to prove point 1 in the theorem, we first use the standard arguments for the MPC without terminal constraint \cite{rawlings_mayne_paper,Limon2003StableCM} in the undiscounted case. We then extend the results to the discounted case.

\paragraph{Nominal stability} First, when an invariant set terminal constraint is used, which in our case corresponds to the condition $V(x(N))\leq l_s$ with $\mathbb{X}_s\subseteq\mathbb{X}$, then \cite{rawlings_mayne_paper} have provided conditions to prove stability by demonstrating that $J^\star_{\text{MPC}}(x)$ is a Lyapunov function. These require the terminal cost to be a Lyapunov function that satisfies~\eqtref{eq:lyap2_lqr}. Hence, we start by looking for values of $\alpha$ such that $\alpha V(x)$ satisfies~\eqtref{eq:lyap2_lqr}. In other words, we wish to find an $\bar{\alpha}_1\geq 1$ such that, for all $\alpha \geq \bar{\alpha}_1$ and for some policy $K_0$ (in our case, the demonstrator for $V$), the following condition holds:
\begin{equation}\label{eq:lqr_decrease_poor}
    \alpha V(f(x,K_0(x)) - \alpha V(x) \leq -\ell(x, K_0(x)).
\end{equation}
Let us denote $x^{+}=f(x,K_0(x))$ for brevity. We have, by assumption, that:
\begin{equation}
\alpha (V(x^{+}) - \lambda V(x)) \leq 0.
\end{equation}
This implies that: 
\begin{align}
    &\alpha V(x^{+}) -\alpha V(x) + \alpha V(x) - \alpha \lambda V(x) \leq 0, \nonumber\\
    \Rightarrow~ &\alpha V(x^{+}) -\alpha V(x) \leq -\alpha(1-\lambda)V(x). \label{eq:alpha_ineq_terminal}
\end{align}
Recall that the loss function satisfies $l_\ell\|x\|_2^2\leq\ell(x,u)$. Since the MPC is solved using a sequence of convex quadratic programs, it is also Lipschitz~\cite{Bemporad2000}. Similarly, if $K_0$ is Lipschitz or (uniformly) continuous over the closed and bounded set $\mathbb{U}$, then since $\mathbb{X}$ is also closed and bounded, there also exists a local upper bound for the loss function on this policy, namely, $\ell(x,K_0(x))\leq L_\ell\|x\|_2^2$. 

Further, recall from~\eqtref{eq:lyap1} that $l_\ell\|x\|_2^2\leq V(x)$. Using the above notions, we have:
\begin{align}
    \alpha V(x^{+}) -\alpha V(x) &\leq -\alpha(1-\lambda)l_\ell\|x\|_2^2, \nonumber \\
    &\quad = -\alpha(1-\lambda)l_\ell\frac{L_\ell}{L_\ell}\|x\|_2^2 \nonumber \\
    &\leq 
    -\frac{\alpha (1-\lambda)l_\ell}{L_\ell}\ell(x, K_0(x))\nonumber \\
    &\quad=-\beta \ell(x, K_0(x)). \label{eq:beta_ineq_terminal}
\end{align}
To satisfy the above condition, solving for a $\beta\geq 1$ is sufficient. From Equations~\eqref{eq:alpha_ineq_terminal} and~\eqref{eq:beta_ineq_terminal}, it implies that:
\begin{equation}\label{eq:bound_alpha1}
    \alpha\geq \frac{L_\ell}{l_\ell(1-\lambda)}=\bar{\alpha}_1\geq 1.
\end{equation}

Now, the function $\alpha V(x)$ satisfies all the sufficient conditions stated by~\cite{rawlings_mayne_paper} for the stability of an MPC under the terminal constraint $\hat{x}(N)\in\mathbb{X}_s$ which is equivalent to $V(\hat{x}(N))\leq l_s$, without discount (with $\gamma=1$).

Since we do not wish to have such a terminal constraint, we wish for another lower bound $\hat{\alpha}_2\geq 1$ such that, if $\alpha\geq \bar{\alpha}_2$, then $V(\hat{x}(N))\leq l_s$ at the optimal solution. The computation of this $\bar{\alpha}_2$ has been outlined by~\cite{Limon2009} for the undiscounted case. Since our constraints are closed, bounded and they contain the origin, our model and the MPC control law are both Lipschitz, we directly use the result from \cite{Limon2009} to compute $\bar{\alpha}_2$:
 \begin{equation}
      \bar{\alpha}_2 = \frac{\sum_{i=0}^{N-1}\ell(\tilde{x}(i),\tilde{u}(i)) - N\ d}{(1-\rho)l_s} 
 \end{equation}
 where $\tilde{x}(i),\tilde{u}(i)$ represent a sub-optimal state-action sequence for which $V(\tilde{x}(N))\leq \rho l_s$ with $\rho\in[0,1)$, and $d$ is a lower bound for the stage loss $\ell$ for all x outside $\mathbb{X}_s$ and all $u$ in $\mathbb{U}$.  
 
 Then, one can take: 
 \begin{equation}\label{eq:alpha_bound}
    \alpha\geq \max\left(\bar{\alpha}_1,\ \bar{\alpha}_2\right) = \bar{\alpha} 
 \end{equation}
to guarantee stability when $\gamma=1$. 

When the discount factor ($\gamma<1$) is used, condition~\eqref{eq:lqr_decrease_poor} is still respected by the same range of $\alpha$ since 
\begin{equation}
    \gamma V(x^{+}) - V(x) \leq  V(x^{+}) - V(x).
\end{equation}
However, from the discussion in \cite{Gaitsgory2016StabilizationWD}, for infinite horizon optimal control, it appears that~\eqtref{eq:lqr_decrease_poor} is not sufficient for $J^\star_{\text{MPC}}(x)$ to be a Lyapunov function, even when a terminal constraint is used. 

We wish to find a lower-bound $\bar{\gamma}$ such that, given $\alpha$ satisfying  \eqtref{eq:alpha_bound}, the MPC is stable for $\gamma\geq\bar{\gamma}$. 
For infinite-horizon optimal control, this was done by \cite{Gaitsgory2016StabilizationWD}. First, recall that:
\begin{align}
    \alpha V(x)\leq \alpha L_V\|x\|_2^2 &\leq \alpha \frac{L_V}{l_\ell}\ell(x,0)\nonumber \\ 
    &= C \inf_{u\in \mathbb{U}}\ell(x,u). 
\end{align}
In \cite{Gaitsgory2016StabilizationWD}, it shown that $1\leq C<1/(1-\gamma)$ is sufficient for stability of an infinite-horizon discounted optimal control problem, when $\alpha V(x)$ is its value function. This means that:
\begin{equation}
    \frac{\alpha L_V}{l_\ell}<\frac{1}{1-\gamma},
\end{equation}
which implies that:
\begin{equation}
    \gamma>1-\frac{l_\ell}{\alpha L_V}=\bar{\gamma}_1\in[0,1).
\end{equation}
For MPC, we will instead present an additional condition to the above one that leads to at least convergence to the safe set. This results in a bounded and safe solution. Exact convergence to the origin will be then confirmed when $V$ is the actual value function, as in \cite{Gaitsgory2016StabilizationWD}, or if we switch to the demonstrating policy, $K_0$, once in the terminal set. Finally, we will remove the terminal constraint as done for the undiscounted case with a final bound on $\alpha$ and $\gamma$. 

Recall that condition (\ref{eq:lqr_decrease_poor}) applies. 
If the terminal constraint was met at time $t$, namely, if $V(x^\star(N))\leq l_s$,  then at the next time step, $t+1$ we have that $u(N+1)=K_0(x^\star(N))$ is feasible. Hence, the optimal MPC solution can be upper-bounded by the shifted solution at the previous time $t$, with the $K_0$ policy appended at the end of the horizon \cite{rawlings_mayne_paper}. Denote this policy as $\tilde{\underbar{u}}$ and $\tilde{x}$ as the predictions. We have that:
\begin{align}
    \Delta J^\star_{MPC}(x) & = J^\star_{MPC} (x^{+})  - J^\star_{MPC}(x) \nonumber \\ & \leq J_{MPC}(x^{+},\tilde{\underbar{u}}) - J^\star_{MPC}(x). \nonumber
\end{align}
Hence,
\begin{align}
   \Delta J^\star_{MPC}(x) &\leq ~ \sum_{i=1}^N \gamma^{i-1} \ell(\tilde{x}(i),\tilde{u}(i))  + \gamma^{N} \alpha V(\tilde{x}^{+}(N)) -\ell(x,\tilde{u}(0)) - \sum_{i=1}^{N-1} \gamma^i \ell(\tilde{x}(i),\tilde{u}(i)) -\gamma^{N}\alpha V(\tilde{x}(N))  \nonumber \\
    &\quad = (1-\gamma) L_{N-1}(x) -\ell(x,\tilde{u}(0)) + \underbrace{\gamma^{N-1}\left(\gamma\alpha V(\tilde{x}^{+}(N))  - \gamma\alpha V(\tilde{x}(N)) 
    + \ell(\tilde{x}(N), K_0(x))) \right)}_{\leq 0\ \Leftarrow\ \gamma\alpha\geq\bar{\alpha}_1}  \nonumber \\
    &\leq (1-\gamma) {L}_{N-1}(x) -\ell(x,\tilde{u}(0)), \nonumber 
\end{align}
where 
$L_{N-1}(x) = \sum_{i=1}^{N-1} \gamma^{i-1} \ell(\tilde{x}(i),\tilde{u}(i))$ and we have taken $\alpha$ such that $\gamma\alpha\geq\bar{\alpha}_1$.  

Now, for $\gamma=1$, the effect of $L_{N-1}$ disappears and the MPC optimal cost is a Lyapunov function as in the standard MPC stability result from \cite{rawlings_mayne_paper}. By inspection of $L_{N-1}$, since the cost is bounded over bounded sets, also a small enough $\gamma$ could be found such that $L_{N-1}(x)<\ell(x,\tilde{u}(0))$. This $\gamma$, however, depends on $x$. Consider $x\not\in\mathbb{X}_s$, for which there exist a feasible solution, namely a solution providing $\tilde{x}(N)\in\mathbb{X}_s$. Then, since $\ell$ is strictly increasing, $\ell(0,0)=0$, $\mathbb{X}_s$ contains the origin and the constraints are bounded, we have that there exist a $\upsilon\geq 1$ such that for any feasbile $x$:
$$\bar{L}_{N-1}=\upsilon (N-1)\inf_{x \not\in\mathbb{X}_s} \ell({x},0),$$ 
is an upper bound for $L_{N-1}(x)$. For instance,
$$\upsilon = \frac{\sup_{(x,u) \in \epsilon\mathbb{X}\times\mathbb{U}} \ell({x},{u})}{\inf_{x \not\in\mathbb{X}_s} \ell({x},0)},$$ is sufficient for any closed set of initial conditions $x(0)\in\epsilon\mathbb{X}\supset\mathbb{X}_s$, with $\epsilon>0$. 
 In order to have stability, it suffices to have $(1-\gamma) \bar{L}_{N-1} -\ell(x,\tilde{u}(0))\leq 0$ which requires:
\begin{eqnarray}
    \gamma \geq 1 -\frac{\ell(x,\tilde{u}(0))}{\bar{L}_{N-1}}=\bar{\gamma}(x).
\end{eqnarray}
In the above condition $\bar{\gamma}(x)$ can be less than $1$ only outside a neighborhood of origin. Consider again 
\begin{equation}
    d = \inf_{x\not\in\mathbb{X}_s}{\ell(x,0)}.
\end{equation}
Then taking 
\begin{eqnarray}
    \gamma \geq 1 -\frac{d}{\bar{L}_{N-1}}=\bar{\gamma}_2\in(0,1), 
\end{eqnarray}
provides that the system trajectory will enter the safe set $\mathbb{X}_s$, hence $\mathbb{B}_\gamma\subseteq \mathbb{X}_s$. 
Finally, once $x\in\mathbb{X}_s$, we that
the policy $K_0(x)$ is feasible and:
\begin{eqnarray}
    \ell(x,\ K_0(x)) \leq \alpha V(x) - \alpha V(x^{+})\leq \alpha V(x). \nonumber 
\end{eqnarray}
Hence, we can use this policy to upper bound the MPC cost: 
\begin{eqnarray}
    J^\star_{MPC}(x) \leq 
    \sum_{i=0}^{N-1} \gamma^{i} \alpha V(\tilde{x}(i))  + \gamma^{N} \alpha V(\tilde{x}(N)). \nonumber 
\end{eqnarray}
If the above is true with equality, then we can proceed as in Theorem 3.1 of \cite{Gaitsgory2016StabilizationWD}, with $\gamma>\bar{\gamma}_1$. This would require $\alpha V(x)$ to be also a value function for the discounted problem. 

From the above considerations, we can conclude that that:
\begin{enumerate}
    \item If $N=1$, then $\bar{L}_{N-1}=0$ and the system is asymptotically stable for any $\gamma>0$.
    \item If $N>1$, $\gamma \geq \bar{\gamma_2}$, then the system reaches an bound $\mathbb{B}_\gamma$ that is included in $\mathbb{X}_s$. 
    \item If $N>1$  $\gamma \geq \bar{\gamma_2}$ and once in $\mathbb{X}_s$ we switch to the policy $K_0(x)$ then the system is asymptotically stable.
    \item If $\alpha V(x)$ is the global value function for the discounted problem and if $\mathcal{R}(\mathbb{X}_s)=\mathbb{X}_s$, then $\gamma>\bar{\gamma}_1$ provides that the system is Asymptotically stable. 
    \item If $\alpha V(x)$ is only the value function in $\mathbb{X}_s$ for the discounted problem, and if $\mathcal{R}(\mathbb{X}_s)\neq\mathbb{X}_s$, then $\gamma>\max(\bar{\gamma}_1,\bar{\gamma_2})$ provides that the system is Asymptotically stable. 
\end{enumerate}
Finally, following Theorem 3 from \cite{Limon2003StableCM}, the terminal constraint can be removed for all points $x\in\mathcal{R}^N(\rho\mathbb{X}_s)$, with $\rho\in[0,1)$, by setting:
 \begin{align}\label{eq:alpha_terminal_cost}
      \alpha\geq \bar{\alpha} = \max\left(\bar{\alpha}_1/\gamma,\bar{\alpha}_3\right), \\
      \bar{\alpha_3}=\frac{\bar{L}_{N} - \frac{1-\gamma^{N}}{1-\gamma}\ d}{(1-\rho)\gamma^N l_s}>0.  
 \end{align}
In fact, by the same argument of \cite{Limon2003StableCM}, for any states for which it exists a feasible sequence $\tilde{\underbar{u}}$ taking $\hat{x}(N)$ to $\rho\mathbb{X}_s$ we have that:
\begin{align}
    J_{\text{MPC}}^\star(x)\leq J_{\text{MPC}}(x,\tilde{\underbar{u}})=L_N(x)+\rho \ \alpha \gamma^N\ 
    l_s. 
\end{align}
If $\alpha$ satisfies (\ref{eq:alpha_terminal_cost}), then we also have that:
\begin{align}
    (1-\rho)\alpha\ \gamma^N\ 
    l_s\geq {L}_{N}(x) - \frac{1-\gamma^{N}}{1-\gamma}\ d,
\end{align}
from which we can directly verify that the set defined in (\ref{eq:ROA}) is a ROA (for either asymptotic or practical stability): 
\begin{equation}\label{eq:ROA2}
     \Upsilon_{N,\gamma,\alpha} = \left\{x\in\mathbb{R}^{n_x}: J_{\text{MPC}}^\star(x)\leq \frac{1-\gamma^{N}}{1-\gamma}\ d +\gamma^N \alpha l_s\right\}.   \nonumber
 \end{equation}

\paragraph{Robustness} For point 6, the stability margins of nominal MPC have been studied in \cite{Limon2009}. In particular, in a setup without terminal constraint, under nominal stabilising conditions, with a uniformly continuous model (in our case even Lipschitz), cost functions being also uniformly continuous, then the optimal MPC cost is also uniformly continuous \cite[Proposition 1]{Limon2009}. In other words, from \cite[Theorem 1]{Limon2009}, there is a $\mathcal{K}_\infty$-function, $\sigma$, such that at the optimal solution, $\underbar{u}^\star$, we have:
\begin{equation}
    |J^\star_{MPC}(x+w)- J^\star_{MPC}(x)|\leq \sigma(\|w\|). 
\end{equation}
Using the above identity, one wish to bound the increase in the MPC cost due to uncertainty. At the same time, we wish the MPC to remain feasible and perform a contraction, namely, to have a stability margin.  Since we are using soft constraints, then the MPC remains always feasible, however, we need the predictions at the end of the horizon to be in an invariant set $\mathbb{X}_s$ even under the effect of uncertainty. In particular, we will use $V(x)$ and its contraction factor $\lambda$ to compute a smaller level set $\zeta\mathbb{X}_s$, for some $\zeta\in(0,1)$ which is invariant under the uncertainty. Once this is found then we can compute a new $\alpha$ for this set according to (\ref{eq:alpha_terminal_cost}). 
In particular, under the policy $K_0$, we have that:
\begin{align}
    V(x^{+}+w)- V(x) &\leq V(x^{+}) -V(x) + L_V\|w\|^2_2 \nonumber \\ 
    & \leq  (\lambda-1) V(x) + L_V\|w\|^2_2 \nonumber.
\end{align}
We wish this quantity to be non-positive for $x\not\in\zeta \mathbb{X}_s$, which means that $V(x)\geq \zeta l_s$. For this it is sufficient to have:
\begin{align}
     \|w\|^2_2\leq \frac{1-\lambda}{L_V}  \zeta l_s\leq \frac{1-\lambda}{L_V} V(x) 
\end{align}
Since we apply the function $V$ to the prediction at time $N$, and at the next step the measured state (prediction at index $0$) differs from the previous MPC  prediction of a quantity $w$, for the MPC this condition translates into:
\begin{align}
\|w\|_2^2\leq\frac{1-\lambda}{L_V L_{\hat{f}x}^{2N}}\zeta l_s <\bar{\mu}^2 = \frac{1-\lambda}{L_V L_{\hat{f}x}^{2N}} l_s , 
\end{align}

Therefore, given the model error $w$, if the MPC remains feasible and if $\alpha$ and $\gamma$ exceed their lower bounds given the restricted set $\zeta\mathbb{X}_s$, we have that  $\|w\|_2<\bar{\mu}$ implies that: 
\begin{align}
    \Delta J^\star_{MPC}(x) & = J^\star_{MPC}(x^+ +w) - J^\star_{MPC}(x)  + J^\star_{MPC}(x^+) - J^\star_{MPC}(x^+) \nonumber\\
    & \leq  J_{MPC}(x^{+},\tilde{\underbar{u}}) - J^\star_{MPC}(x) + \sigma(\|w\|) \nonumber \\
     & \leq ~ (1-\gamma) \bar{L}_{N-1} -\ell(x,\tilde{u}(0)) + \sigma(\|w\|) \nonumber  \\
    & \leq  -l_\ell\|x\|_2^2 + \sigma(\|w\|) + \bar{d}(N), \nonumber
\end{align}
which is the definition of Input-to-State practical Stability \cite{Khalil_book,Limon2009}, where we have defined $\bar{d}(N)=(1-\gamma)\bar{L}_{N-1}$. The trajectory of the system is therefore bounded by the  level-set of $J^\star_{MPC}(x)$ outside which $ \sigma(\mu) + \bar{d}(N) \leq l_\ell\|x\|_2^2.$
Since $\sigma$ is strictly increasing and  $\bar{d}$ is strictly decreasing, we can also conclude that the size of this bound increases with increasing model error $\mu$ and with the horizon length $N$. Note that the term $\bar{d}$ vanishes if $\gamma=1$ but the term $\sigma$ will also increase with $N$ if $\bar{L}_f>1$. From the restriction of the terminal set, it also follows that the ROA defined in \eqtref{eq:ROA} will also be restricted unless we recompute a larger $\alpha$ for this new set.    
\paragraph{Value function upper bound} Denote $\hat{V}(x)=\alpha V(x)$. Recall that, if $\alpha\geq\alpha_1$, as defined in \eqtref{eq:bound_alpha1}, and if $\|w(t)\|=0\ \forall t$, then we have that, $\forall x(t)\in\mathbb{X}_s$:
\begin{align}
   \hat{V}(x(t))=\alpha V(x(t))\geq \ell(x(t),K_0(t)) + \alpha V(x(t+1)),
\end{align}
which in turn implies that, by means of $\gamma\leq1$ and of induction: 
\begin{align}
   \hat{V}(x(t)) & \geq \ell(x(t),K_0(t)) + \gamma   \hat{V}(x(t+1)) \\
   & \geq \ell(x(t),K_0(t)) + \gamma  \left(\ell(x(t+1),K_0(t+1)) + \gamma \hat{V}(x(t+2))\right) \\
  & \geq \sum^{\infty}_{i=0} \gamma^{i}\ell(x(t+i),K_0(t+i))=V_{K_0}(x(t)), 
\end{align}
which is the value of the policy $K_0$. This is, by definition, an upper bound of the optimal value function, $V^\star(x)$. Hence $\alpha V(x)\geq V^\star(x), \forall x\in\mathbb{X}_s$.

\end{proof}

In practice, the ISS bound $\sigma(\mu)$ from Theorem \ref{Theorem:stability} has a form similar to the one discussed for the constraint penalty in the proof of Theorem \ref{Theorem:performance}, see~\eqtref{eq:state_const_cost_ineq}. Its explicit computation is omitted for brevity; however, in general, we can expect the bound to become worse for systems that are open-loop unstable as the horizon length increases. 

\section{MPC as SQP Formulation}
\label{appendix:mpc_sqp}

We solve the MPC problem through iterative linearisations of the forward model, which gives the state and input matrices: 
\begin{align}\label{eq:linearize}
    A(i)=\left.\frac{\partial \hat{f}}{\partial x}\right|_{\bar{x}(i)},\ B(i)=\left.\frac{\partial \hat{f}}{\partial u}\right|_{\bar{u}(i)}.
\end{align}     
These are evaluated around a reference trajectory:
\begin{align}\label{eq:sequences}
    \underbar{x}^\star&=\{\bar{x}^\star(i),\ i=0,\dots.N\}, \\
    \underbar{u}^\star&=\{\bar{u}^\star(i),\ i=0,\dots.N-1\}.
\end{align} 
This is initialised by running the forward model with zero inputs, and then updated at each iteration by simulating the forward model on the current optimal soulution.   

The Lyapunov function $V$ is expanded to a second order term by using Taylor expansion and is evaluated around the same trajectory. The Jacobian and Hessian matrices, respectively, $\Gamma$ and $H$, are:
\begin{equation}\label{eq:hessians}
    \Gamma=\left.\frac{\partial V}{\partial x}\right|_{\bar{x}(N)},\quad H=\frac{1}{2}\left.\frac{\partial^2 V}{\partial^2 x}\right|_{\bar{x}(N)}. 
\end{equation}

All the quantities in~\eqtref{eq:linearize} and ~\eqref{eq:hessians} are computed using automatic differentiation. Using these matrices, we solve the convex optimization problem:
\begin{align}\label{problem_convex} 
	& \delta \underbar{u}^\star = \argmin \quad   \gamma^N\alpha \left(\|H^{1/2}\ \delta\hat{x}(N)\|_2^2+\Gamma^T\delta\hat{x}(N)\right)  + \sum_{i=0}^{N-1} \gamma^i \ell(\hat{x}(i), \hat{u}(i)) \\ 
	\mathrm{s.t.} &~~ \hat{x}(i+1) = A(i)\delta\hat{x}(i) + B(i) \delta\hat{u}(i)+\hat{f}(\bar{x}(i)), \nonumber\\
	& \qquad \hat{x}(i)-\delta\hat{x}(i)=\bar{x}(i) \nonumber \\ 
	& \qquad \hat{u}(i)-\delta\hat{u}(i)=\bar{u}(i) \nonumber \\ 
	& \qquad \hat{x}(i)\in\mathbb{X},\ \forall i\in[0, N], \nonumber\\
	& \qquad \hat{u}(i)\in\mathbb{U},\ \forall i\in[0, N-1], \nonumber\\
	& \qquad \hat{x}(0) = x(t), \nonumber \\
	& \qquad \|\delta\hat{u}(i)\|_\infty \leq r_{\text{trust}},\ \forall i\in[0, N-1], \nonumber
\end{align}
where the state constraints are again softened and the last inequality constraint is used to impose a trust region with a fixed radius, $r_{\text{trust}}$. This notably improves the search for an optimal solution, as shown for the inverted pendulum case in Figure \ref{fig:trust_region}.

Once problem~\eqref{problem_convex} is solved, we obtain the delta sequence $\delta\underbar{u}^*$. The new optimal solution is then computed according to the update rule:
\begin{align}
    \underbar{u}^\star \leftarrow \underbar{u}^\star + lr\ \delta\underbar{u}^\star, \nonumber
\end{align}
where $lr<1$ is a learning rate, which is annealed after each iteration. 
Finally, the reference trajectories used for the next linearization are $\underbar{u}^\star$ and the state series resulting from simulating the forward model on $\underbar{u}^\star$, namely 
$$\underbar{x}^\star = \hat{f} \circ \underbar{u}^\star.$$ 
This is summarised in Algorithm \ref{alg:mpc}. Interested readers can find more details on SQP in the work by~\cite{nocedal2006numerical, hadfield-menell_sequential_2016}, where adaptive schemes for computing the trust radius are discussed. 

        \begin{algorithm}[H]
            \DontPrintSemicolon
            \KwInput{$x(t)$, $\hat{f}$, $\alpha$, $V$, $N_{\text{steps}}$, $\epsilon_{lr}\in[0,1)$, $r_{\text{trust}}>0$, $lr$}
            \KwOutput{$u(t)$}
            \small
            \caption{Neural Lyapunov MPC solver}
            \vspace{-0.1cm} \hrulefill \vspace{0.1cm} \par
            \label{alg:mpc}
            $\underbar{x}^\star \gets \{x(t)\}^{N+1}, \quad \underbar{u}^\star \gets \{0\}^{N}$ \par 
            \For{$j=0,..., N_{\text{steps}}$}{ \;   
            $\{A(i)\},\ \{B(i)\}  \gets$ linearization of $\hat{f}$ using~\eqtref{eq:linearize} \\ \;
            $(\Gamma,\ H)  \gets$ Taylor expansion of $V$ using~\eqtref{eq:hessians} \\ \;
            $\delta\underbar{u}^\star  \gets$ solution of optimization problem~\eqref{problem_convex}\\ \;
            $\underbar{u}^\star \leftarrow \underbar{u}^\star + lr~\delta\underbar{u}^\star$ \\\;
            $\underbar{x}^\star \leftarrow \hat{f} \circ \underbar{u}^\star$ \\ \;
            $lr \gets (1-\epsilon_{lr})\ lr$ \;
            }
        $u(t) \gets \underbar{u}^\star(0)$ \;     
        \end{algorithm}

\section{Experimental Setup}
\label{appendix:experiments}

We demonstrate our approach on an inverted pendulum and vehicle kinematics control problem. We also provide additional results and plots in Appendix D. 
\footnote{Our code, including the world models, is implemented using PyTorch~\cite{paszke2019pytorch}. This allows us to exploit automatic differentiation to linearise the dynamics for the SQP-based MPC.}

\begin{table*}[t]
\centering
\caption{\textbf{Configuration for experiments in main paper.} We specify the parameters used for the simulation of the system dynamics, the demonstrator, the Neural Lyapunov MPC as well as for the alternate learning algorithm.}
\label{tab:environment_config}
\begin{center}
\begin{small}
\begin{sc}
\begin{tabular}{l | c | c | c }
\toprule
\textbf{Parameter} & \textbf{Symbol} & \multicolumn{2}{c}{\textbf{Value}} \\
& & Car Kinematics & Pendulum \\
\midrule

\multicolumn{3}{l}{\textbf{General}}\\
\quad mass    & $m$ & - & 0.2 \si{kg}\\
\quad length  & $l$ & - & 0.5 \si{m} \\
\quad rotational friction & $\lambda_F$ & - & 0.1 \si{N.m.s.rad^{-1}}\\
\quad gravity & $g$ & - & 9.81 \si{m.s^{-2}}  \\
\quad sampling time & $dt$ & 0.2 \si{s} & 0.01 \si{s} \\
\quad state constraint & $\mathbb{X}$ & $[-1, 1]^2 \times [-\pi, \pi]$ & $[-\pi, \pi] \times [-2\pi, 2\pi]$   \\
\quad input constraint & $\mathbb{U}$ & $[-10, 10] \times [-2\pi, 2\pi]$ & $[-0.64, 0.64]$\\
\quad state penalty & $Q$ & $\textrm{diag}(1, 1, 0.001\pi)$ & $\textrm{diag}(0.1, 0.1)$\\
\quad input penalty & $R$ & $\textrm{diag}(100, 20\pi)$ & $0.1$ \\
\quad discount factor & $\gamma$ & $1$ & $1$ \\
\midrule

\multicolumn{3}{l}{\textbf{Demonstrator}}\\
\quad Type & $K_0$ & MPC & MPC  \\
\quad Horizon & $N$ & 5 & 5 \\
\quad \# of linearisations & $N_{steps}$ & 3 & 10  \\
\quad trust region & $r_{trust}$ & $\infty$ & 2.5 \\
\quad learning rate & $lr$ & 0.9 & 0.9 \\
\quad decay rate & $\epsilon_{lr}$ & 0.2 & 0.2 \\
\quad terminal penalty & $P$ & $400\ Q$ & $500 P_{\text{LQR}}$ \\
\midrule

\multicolumn{3}{l}{\textbf{Neural Lyapunov MPC}}\\
\quad Horizon & $N$ & 1 & 1 \\
\quad \# of linearisations & $N_{steps}$ & 18 & 18 \\
\quad trust region & $r_{trust}$ & 0.005 & 0.5 \\
\quad learning rate & $lr$ & 0.9 & 0.9\\
\quad scaling of $V$ & $\alpha$ & \tabref{tab:alpha_car_nominal}  & \tabref{tab:alpha_pendulum}\\
\quad decay rate & $\epsilon_{lr}$ & 0.02 & 0.02 \\
\quad $V_{{net}}$ architecture & MLP & $(128, 128, 128)$ & $(64, 64, 64)$ \\
\quad $V_{{net}}$ output & $n_V \times n_x$ & $400 \times 3$ & $100 \times 2$\\

\midrule
\multicolumn{3}{l}{\textbf{Alternate learning}}\\
\quad Outer iterations & $N_{ext}$ & 3 & 2 \\
\quad Enlargement factor & $\epsilon_{ext}$ & 0.1 & 0.1 \\
\quad MPC line search & $\alpha_{list}$ & $\{1, 6, ..., 36\}$ & $\{0.2, 0.4, ..., 2\}$\\
\quad Lyapunov epochs & $N_V$ & 500 & 200 \\
\quad Loss~\eqtref{eq:lyapunov_loss} & $\rho$ & 0.0001 & 0.0001 \\
\quad contraction factor & $\lambda$ & 0.99 & 0.99 \\
\quad Lyapunov learning rate & $lr$ & 0.001 & 0.001 \\
\quad Lyapunov weight decay & $wd$ & 0 & 0.002 \\
\bottomrule
\end{tabular}
\end{sc}
\end{small}
\end{center}
\vskip -0.1in
\end{table*}

\subsection{Implementation Specifics}

A practical consideration for implementing~\algref{alg:alternateDescent} is tuning the MPC terminal cost scaling, $\alpha$, via grid-search. The MPC needs to run over the entire dataset of initial points, $\mathcal{X}_0 = \{x_m(0)\}_{m=1}^M$, with different configurations. In order to speed up the search for $\alpha$, we run the MPC only on a sample of $20\%$ of the initial dataset. Once the optimal $\alpha$ is found, only then we run the MPC over the entire dataset and use this data to compute the next $V(x)$.

Additionally to what presented in the main text, we parameterize $V(x)$ with a trainable scaling factor, $\beta$, as:
\begin{equation}\label{eq:lyap_definition2}
  V(x) = \text{softplus}(\beta) x^T\left(l_\ell I + V_{net}(x)^T V_{net}(x)\right)x, 
\end{equation}
where $\text{softplus}(x) = \log(1+\exp(x))$. The parameter $\beta$ is initialized to $1$ in all experiments except for the inverted pendulum without LQR loss, i.e. for results in Figure \ref{fig:contraction_factor}.

The full set of parameters for the experiments can be found in Table~\ref{tab:environment_config}.

\subsection{Baseline Controllers}

Our Neural Lyapunov MPC has a single-step horizon and uses the learned Lyapunov function as the terminal cost. To compare its performance, we consider two other MPCs and RL baselines:
\begin{itemize}[noitemsep,topsep=0pt]
    \item \textbf{Long-horizon MPC (demonstrator/demo)}: An MPC with a longer horizon and a quadratic terminal cost $x^TPx$. This MPC is also used to generate the initial demonstrations for alternate learning.
    \item \textbf{Short-horizon MPC}: An MPC with a single-step horizon and the same terminal cost as the long-horizon MPC. All other parameters are the same as the Neural Lyapunov MPC except $\alpha$, which is tuned manually.
    \item \textbf{Model-free RL}: Neural-network parameterized policy trained using on-policy algorithm, PPO~\cite{schulman2017proximal}, and off-policy algorithm, SAC~\cite{haarnoja2018soft}.
    \item \textbf{Model-based RL}: Neural-network parameterized policy and ensemble of dynamic models trained using the algorithm, MBPO~\cite{janner_when_2019}.
\end{itemize}

In order to train the RL baselines, we consider two variations of our proposed models. In the first version (marked \texttt{v1}), the environment considers the stage cost used in demonstrator MDP as the penalty and terminates the episodes whenever the state constraints are violated. In the second version (marked \texttt{v2}), there is no environment termination. Instead an additional penalty term is added for violating the constraint. This is further described in~\tabref{tab:rewards}. We show the results of training in these variants of the environments for SAC and PPO in~\figref{fig:learning_curves}.

\begin{table}[H]
    \centering
    \begin{sc}
    \caption{\textbf{Reward and termination conditions for RL MDPs.} We train RL baselines on two variants of the inverted pendulum and car kinematics models. Note, for inverted pendulum: $\beta_1 = 10, \beta_2 = 200$, while car kinematics: $\beta_1 = 1, \beta_2 = 200$. The symbol $\mathbb{I}_{(.)}$ is the indicator function.}
    \label{tab:rewards}
    \begin{tabular}{c|c|c}
        \toprule
         Version & Reward Term & Termination Condition \\
         \midrule
          \texttt{v1} & $-\eta^T Q \eta - u^T R u$ & $\eta \notin \mathbb{X}$ \\
          \texttt{v2} & $-\beta_1(\eta^T Q \eta - u^T R u) - \beta_2 \mathbb{I}_{\eta \notin \mathbb{X}}$ & - \\
          \bottomrule
    \end{tabular}
    \end{sc}
\end{table}

\begin{figure*}[!ht]
    \captionsetup[subfigure]{justification=centering}
    \centering
    \begin{subfigure}[b]{0.48\textwidth}
        \centering
        \includegraphics[width=\textwidth]{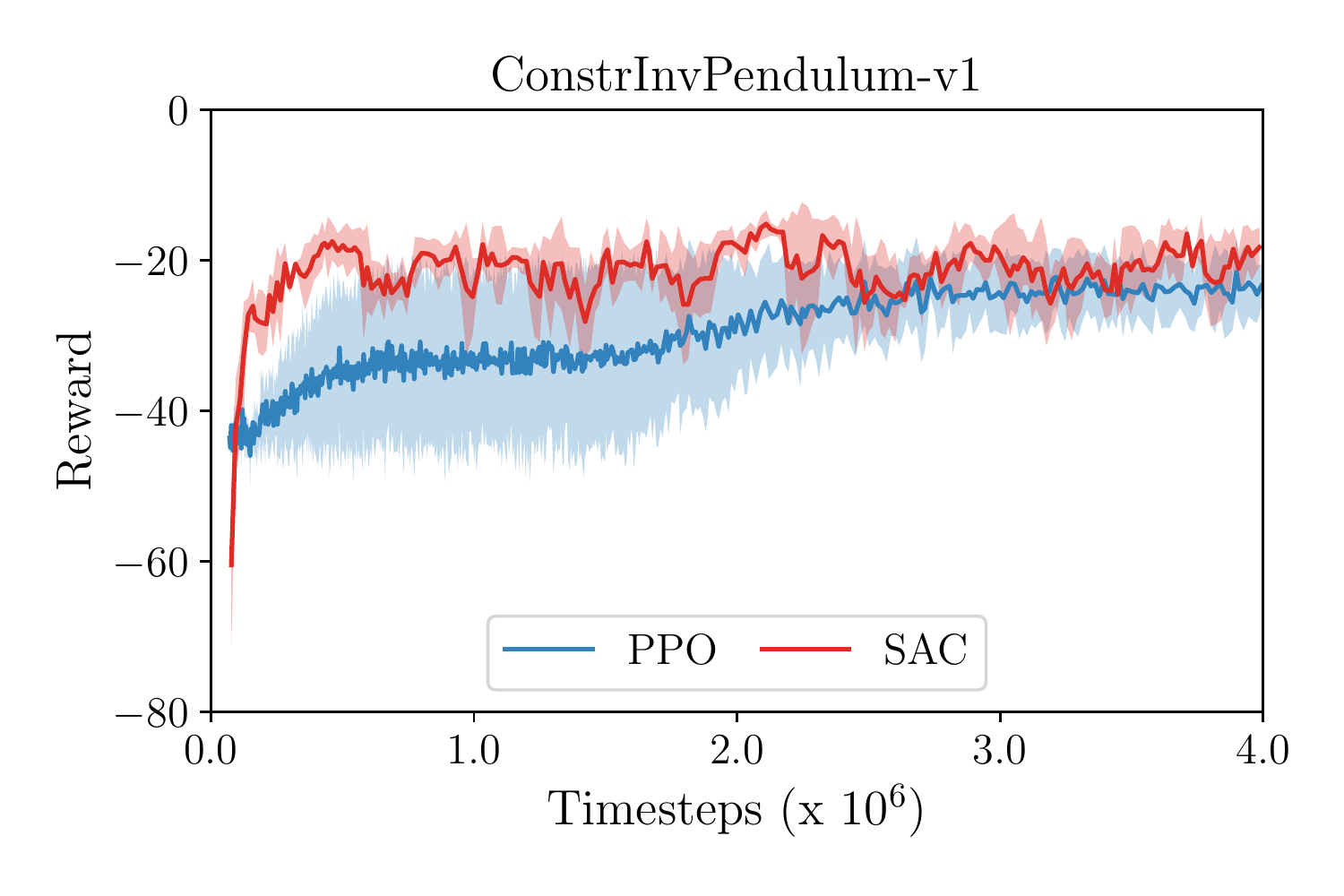}
        \vspace{-0.1cm}
    \end{subfigure}
    \begin{subfigure}[b]{0.48\textwidth}
        \centering
        \includegraphics[width=\textwidth]{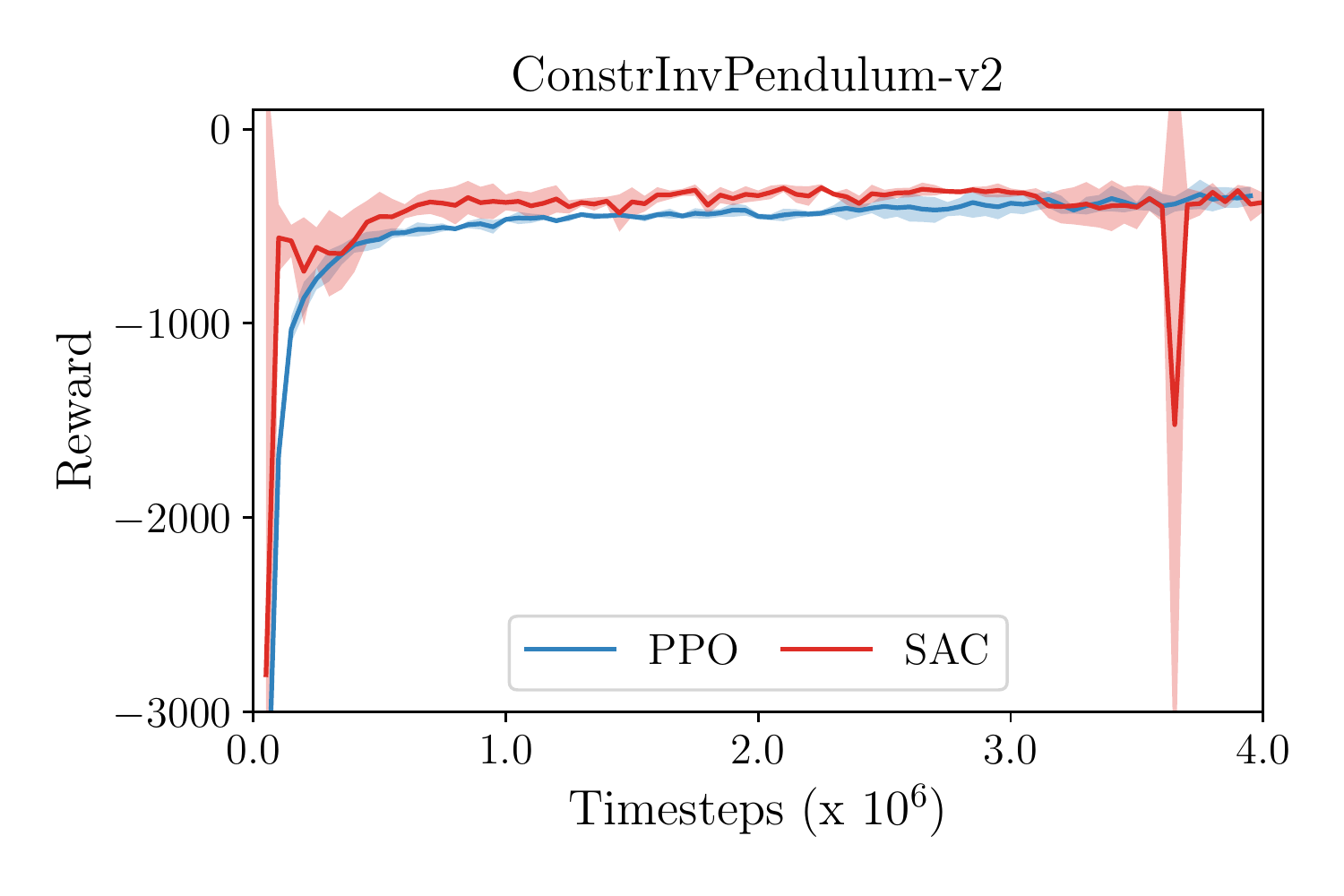}
        \vspace{-0.1cm}
    \end{subfigure}
    \begin{subfigure}[b]{0.48\textwidth}
        \centering
        \includegraphics[width=\textwidth]{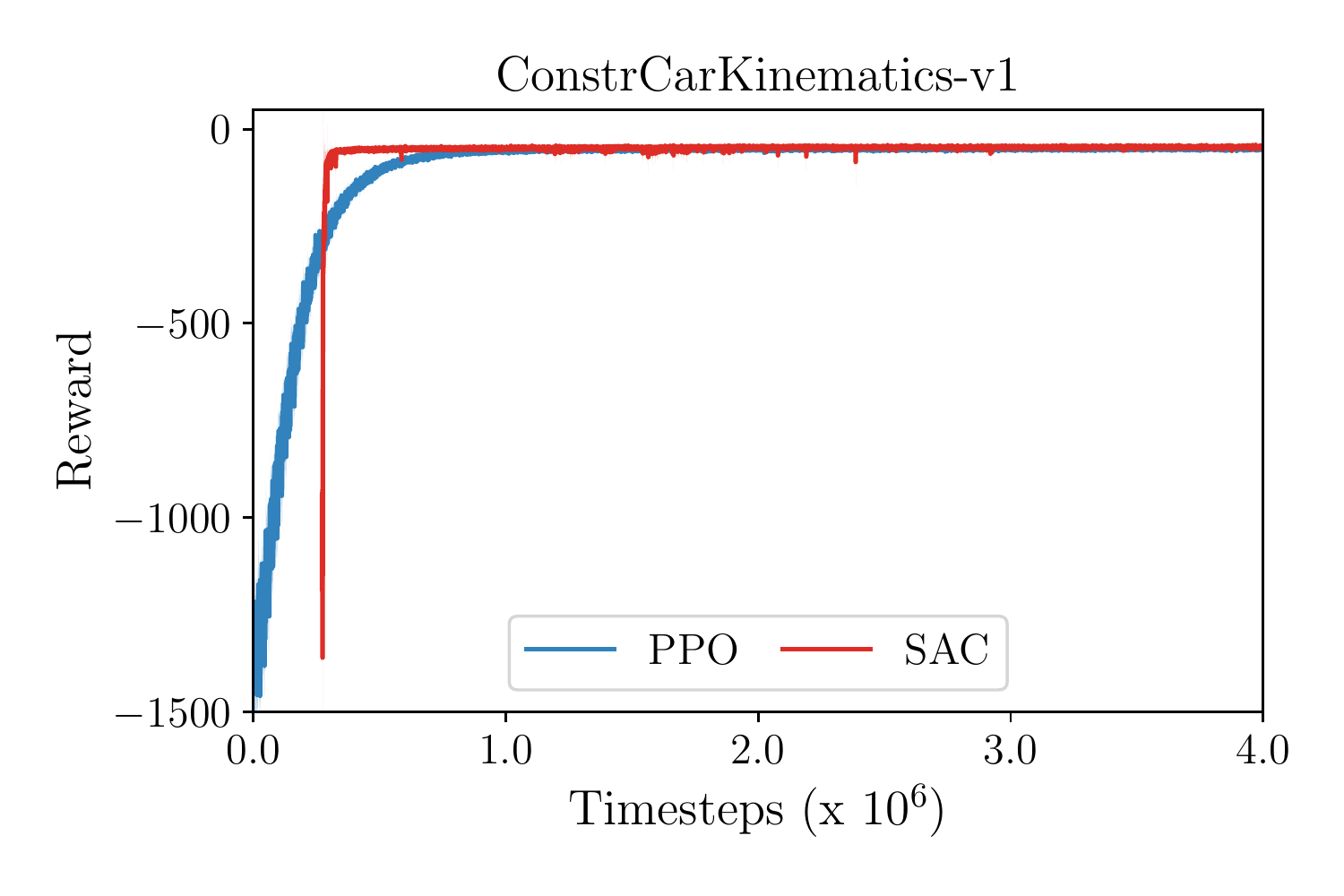}
        \vspace{-0.1cm}
    \end{subfigure}
    \begin{subfigure}[b]{0.48\textwidth}
        \centering
        \includegraphics[width=\textwidth]{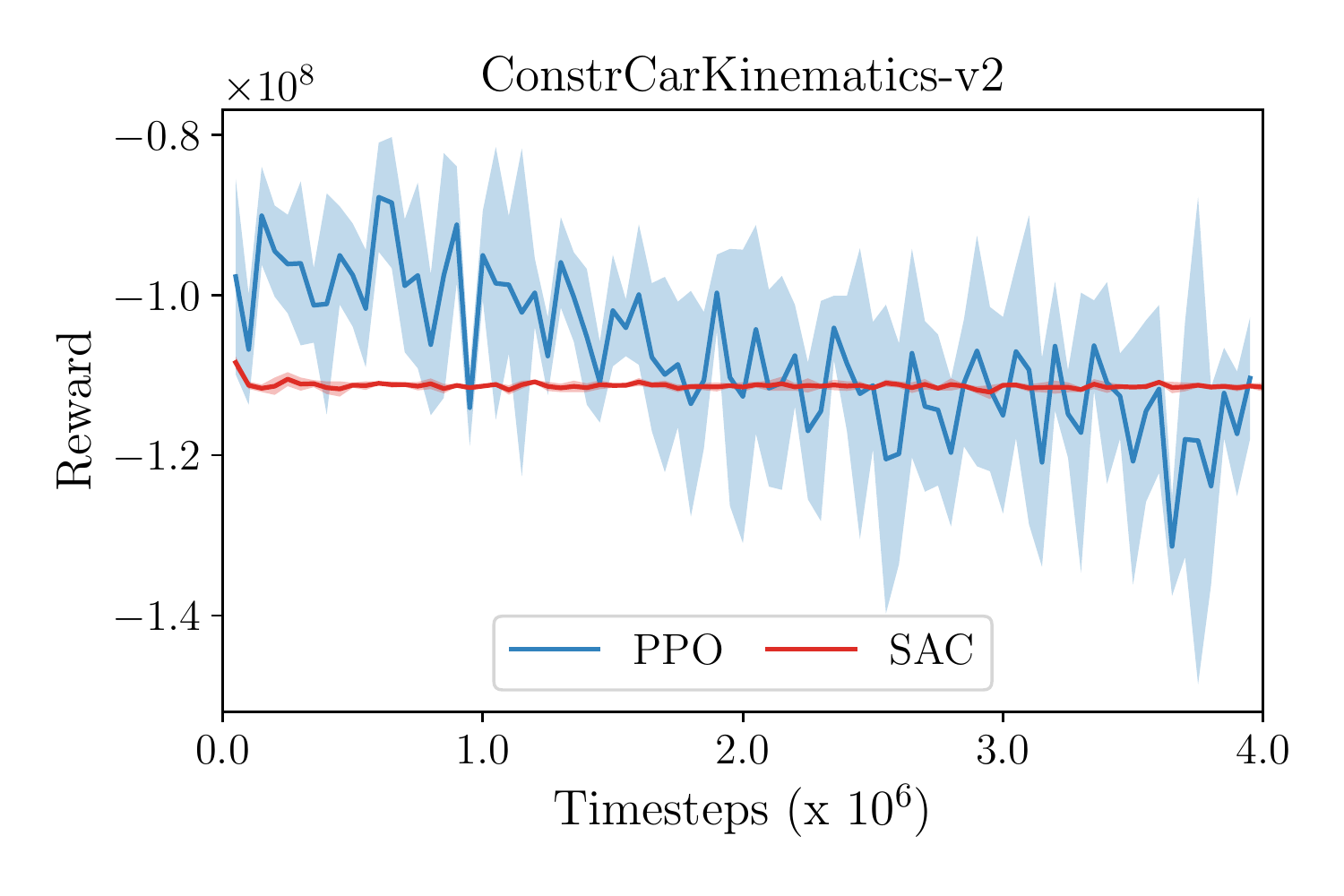}
        \vspace{-0.1cm}
    \end{subfigure}
    
    \caption{\textbf{Training of RL agents.} The plots show the learning curves for PPO and SAC in the two variants each of inverted pendulum and car kinematics environments. We train the agent on 4 different seeds for $4 \times 10^6$ timesteps. The solid line and filled region indicates the mean and one standard deviation reward across the seeds.}
    \label{fig:learning_curves}
\end{figure*}

\begin{figure*}[!ht]
    \captionsetup[subfigure]{justification=centering}
    \centering
    \begin{subfigure}[b]{0.48\textwidth}
        \centering
        \includegraphics[scale=0.55]{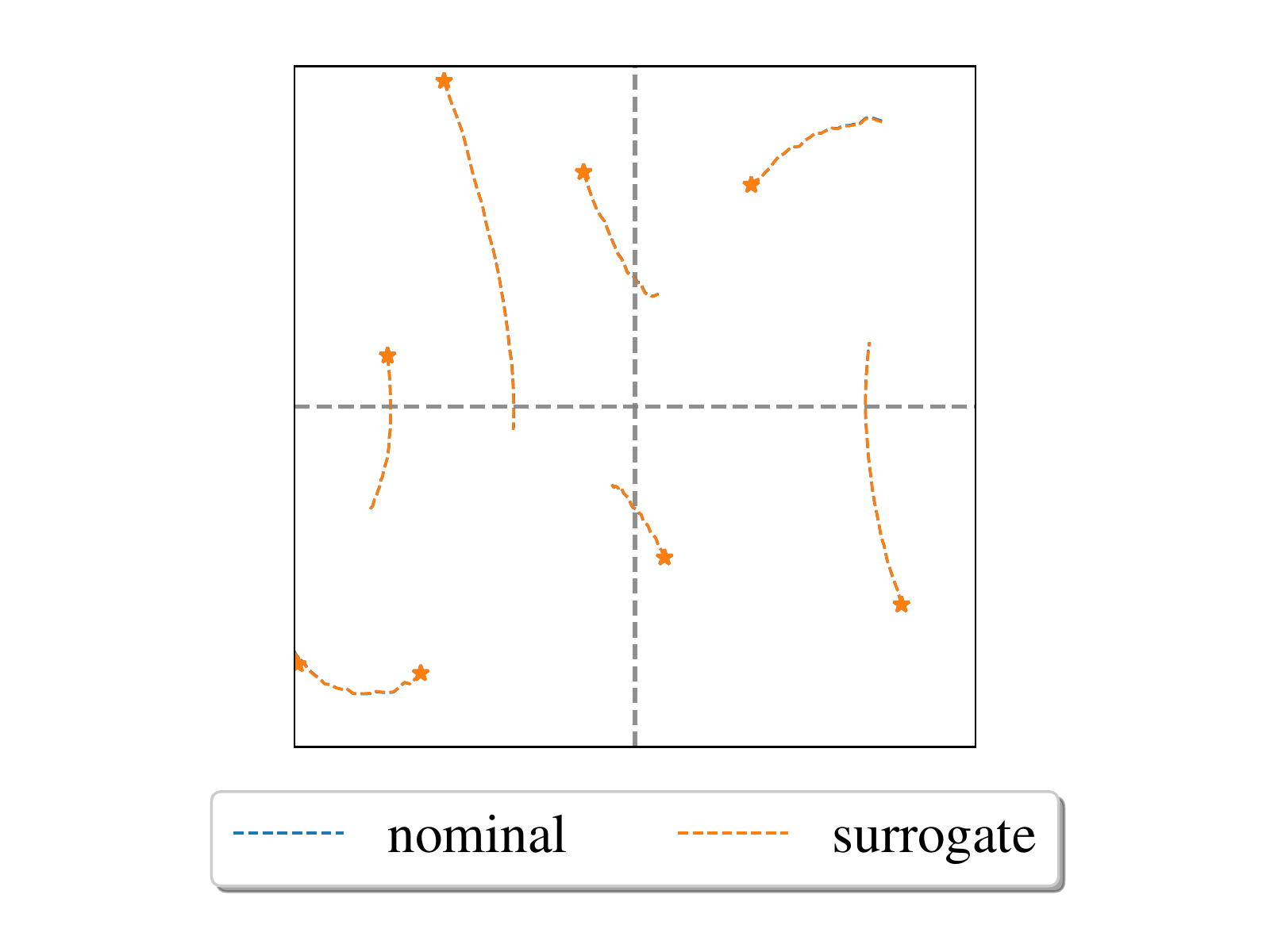}\vspace{-0.1cm}
        \caption{Test on Inverted Pendulum}
    \end{subfigure}
    \begin{subfigure}[b]{0.48\textwidth}
    \centering
        \includegraphics[scale=0.55]{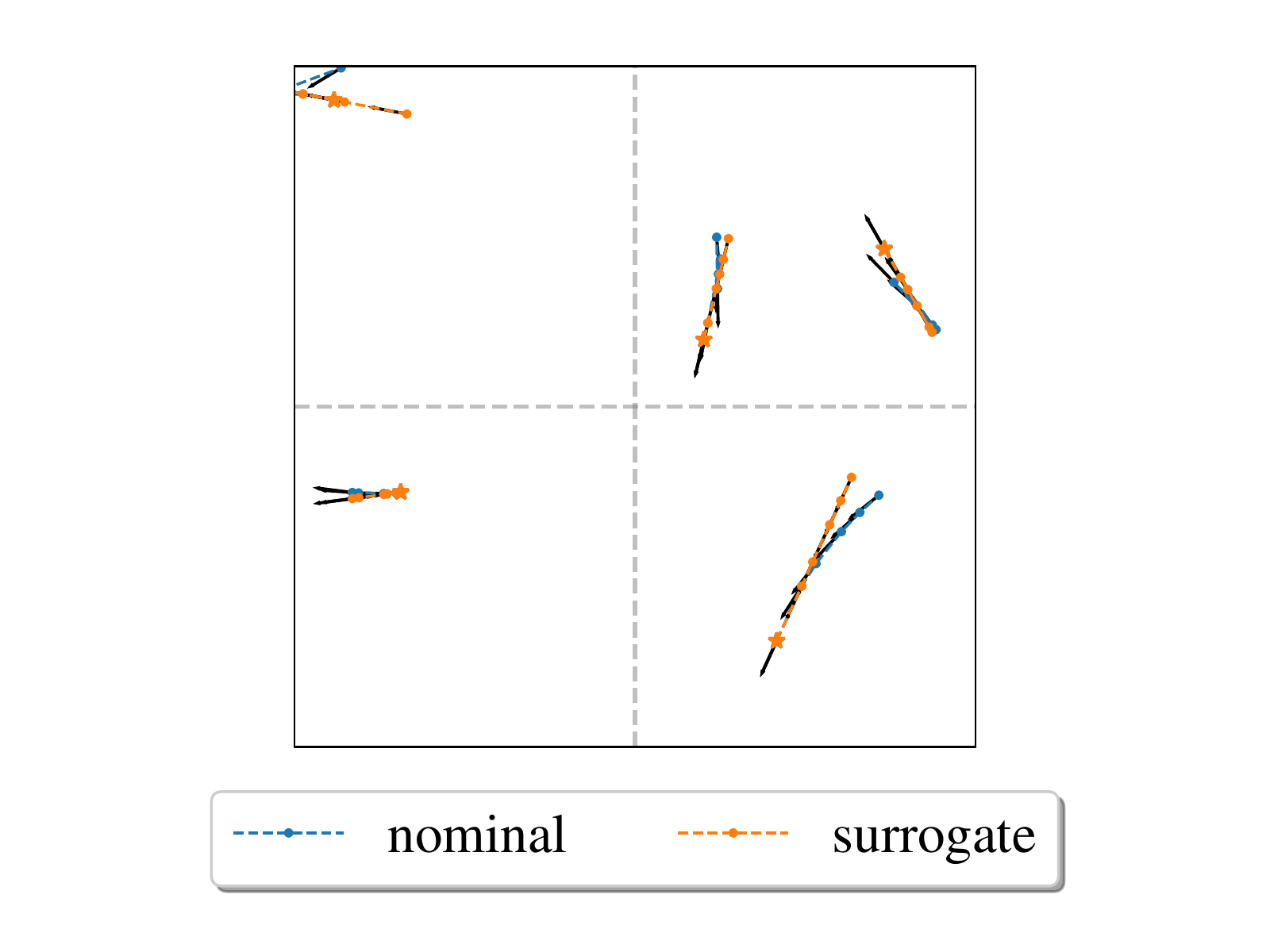}\vspace{-0.1cm}
        \caption{Test on Vehicle Kinematics}
    \end{subfigure}
    
    \caption{\textbf{Testing the surrogate models.} We generate each trajectory by propagating an initial state $\eta(0)$ with input sequence $\{u(t)\}_{t=0}^{T-1}$ through the nominal system and the learned surrogate model. (a) For the inverted pendulum: $u(t) \sim \mathcal{U}([-0.4, 0.4])$ and simulation is performed for $T=25$. (b) For the vehicle kinematics: $u(t) \sim \mathcal{U}([-0.1, 0.1] \times [-0.1, 0.1])$ and simulation is performed for $T=5$.}
    \label{fig:surrogate_test}
\end{figure*}

\subsection{Forward models}

We use an Euler forward model for the environments. Consider $dt$ as the sampling time, then the state transition is:
\begin{align}
	\eta(t+1) &= \eta(t) + dt\ f_u(\eta(t), u(t)),  \label{eq:3dof_model}
\end{align}
where $\eta(t)$ is the state, $u(t)$ is the input and $f_u(\eta(t), u(t))$ is the time-invariant, deterministic dynamical system. 

\subsubsection{Vehicle Kinematics}

\paragraph{World model} For the non-holonomic vehicle, $\eta=(x,y,\phi) \in \mathbb{R}^3$ is the state, respectively, the coordinates in the plane and the vehicle orientation, and $u = (v, \omega)\in \mathbb{R}^2$  is the control input, respectively, the linear and angular velocity in the body frame. $f_u(\eta, u)$ encodes the coordinate transformation from the body to the world frame \cite{fossen2011handbook}:
\begin{align}
     f_u (\eta(t), u(t)) &= \begin{pmatrix} \dot{x} \\ \dot{y} \\ \dot{\phi}\end{pmatrix} = \begin{pmatrix} v(t) \cos{\phi(t)} \\ v(t) \sin{\phi(t)} \\ \omega(t) \end{pmatrix} = \underbrace{\begin{pmatrix} \cos{\phi(t)}  & 0 \\ \sin{\phi(t)} &  0 \\ 0 & 1 \end{pmatrix}}_{J(\eta)} \begin{pmatrix} v(t) \\ \omega(t) \end{pmatrix}
\end{align}

\paragraph{Surrogate model} We build a gray-box model using a neural network to model $J(\eta)$, similar to the work by~\cite{quaglino_snode_2020}. The input feature to the network is $\sin{\phi}$ and $\cos{\phi}$, where $\phi$ is the vehicle's orientation. The network consists of a hidden layer with $20$ hidden units and $\tanh$ activation function, and an output layer without any activation function. The weights in the network are initialized using Xavier initialization~\cite{glorot2010initialization} and biases are initialized from a standard normal distribution.

\paragraph{Training the surrogate model} We generate a dataset of $10K$ sequences, each of length $T=1$. For each sequence, the initial state $\eta(0)$ is sampled uniformly from $\mathbb{X}$, while the input $u(0)$ is sampled uniformly from $\mathbb{U}$. We use a training and validation split of $7:3$. Training is performed for $300$ epochs using the Adam optimizer~\cite{kingma_adam:_2014} and the mean-squared error (MSE) loss over the predicted states. The learning rate is $0.01$ and the batch size is $700$.

\subsubsection{Inverted Pendulum}

\paragraph{World model} Inverted pendulum is one of the most standard non-linear systems for testing control methods. We consider the following model \cite{berkenkamp_safe_2017}:
\begin{align}
    ml^2 \Ddot{\theta} = mgl \sin{\theta} - \lambda_F \Dot{\theta} + u
\end{align}
where $\theta \in \mathbb{R}$ is the angle, $m$ and $l$ are the mass and pole length,  respectively, $g$ is gravitational acceleration, $\lambda_F$ is the rotational friction coefficient and $u\in \mathbb{R}$ is the control input.
 We denote the state of the system as $\eta=(\theta, \dot{\theta}) \in \mathbb{X} \subset \mathbb{R}^2$ and input as $u \in \mathbb{U} \subset \mathbb{R}$. We use an Euler discretization, as in \eqtref{eq:3dof_model}, to solve the initial-value problem (IVP) associated with the following equation:
\begin{align}
    f_u (\eta(t), u(t))=\begin{pmatrix} \dot{\theta} \\ \ddot{\theta}\end{pmatrix} = \begin{pmatrix} \dot{\theta}(t) \\ \frac{mgl \sin{\theta(t)} + u(t) - \lambda_F \dot{\theta}(t)}{ml^2} \end{pmatrix}
\end{align}

\paragraph{Surrogate model} We use a neural network to predict the acceleration of the pendulum, $\ddot{\theta}(t)$. The input to the network is the state $\eta(t)$ and action $u(t)$ at the current time-step $t$. The network is a three layer feedforward network with $64$ hidden units and $\tanh$ activation in each hidden layer. The output layer is linear. All the layers have no biases and their weights are initialized as in \cite{glorot2010initialization}.

\paragraph{Training the surrogate model} To train the surrogate model, we generate a dataset of $10K$ sequences, each of length $T=1$. We use MSE loss and Adam optimizer~\cite{kingma_adam:_2014} to train the model. The rest of the parameters are kept the same as those used for the vehicle kinematics model training.

\section{Additional Results}\label{appendix:results}

Here we provide additional results and plots related to the experiments specified in the paper. 

\subsection{Vehicle Kinematics}

The vehicle model results are run over different machines. The resulting losses and $\alpha$ are discussed. In all experiments, the parameters of $V$ are reinitialised after every outer epoch. The Lyapunov loss does not make use of the LQR loss $\ell$.

\paragraph{Choosing number of outer iterations $N_{ext}$} We run our algorithm for more iterations on a machine with different operating system. This leads to a slight difference from the results mentioned in the paper. As observed from~\tabref{tab:vehicle_model_metric_new}, the third iteration leads to the best performance in terms of verified and not verified percentages. We set $N_{ext}=3$ in all experiments.

\begin{table}[H]
    \caption{\textbf{Car Kinematics: Choosing number of outer iterations.} Lyapunov Function learning performed on a different machine (results are slightly different from the ones in the paper). }
    \label{tab:vehicle_model_metric_new}
    \begin{center}
    \begin{small}
    \begin{sc}
    \begin{tabular}{P{1cm} P{1.5cm} P{1.5cm} P{2.0cm}}
    \toprule
    Iter. & Loss ($\log(1+x)$) & Verified (\%) & Not Verified (\%) \\
    \midrule
    1 & 1.42 & 92.83 & 3.95 \\
    2 & 0.91 & 93.08 & 4.71 \\
    3 & 0.62 & 94.26 & 3.89 \\
    4 & 0.46 & 93.65 & 4.38 \\
    5 & 0.53 & 92.18 & 5.63 \\
    \bottomrule
    \end{tabular}
    \end{sc}
    \end{small}
    \end{center}
    \vskip -0.1in
\end{table}

\paragraph{Alternate learning on nominal model} 

We train the Neural Lyapunov MPC while using a nominal model for internal dynamics. In~\figref{fig:car_alternate_learning_nominal_complete}, we plot the training curves, the resulting Lyapunov function, and line-search for MPC in each outer epoch. The results are also shown in~\tabref{tab:alpha_car_nominal}. As can be seen, tuning the MPC parameter $\alpha$ helps in minimizing the loss~(\protect\ref{eq:lyapunov_loss}) further. Points near the origin don't always verify. The MPC obtained after the third iteration achieves the best performance. This can further be validated from~\figref{fig:car_alternate_learning_nominal_trajectories}, where we plot trajectories obtained by using the Neural Lyapunov MPC from each iteration.

\paragraph{Alternate learning on surrogate model}

In order to test the transfer capability of the approach, we perform the training of Neural Lyapunov MPC using an inaccurate surrogate model for the internal dynamics. This model is also used for calculating the Lyapunov loss~(\protect\ref{eq:lyapunov_loss}) and evaluating the performance of the Lyapunov function. We plot the training curves, the resulting Lyapunov function, and line-search for MPC in each outer epoch in~\figref{fig:car_alternate_learning_surrogate_complete}. The results of the training procedure are presented in~\tabref{tab:alpha_car_surrogate}. The MPC obtained from the second iteration achieves the best performance. In the third iteration, the Lyapunov loss increases and number of verified and not verified points becomes worse. The poor performance also reflects in the trajectories obtained by using the Lyapunov MPC from the third iteration, as shown in~\figref{fig:car_alternate_learning_surrogate_trajectories}.

\begin{table}[H]
    \captionsetup[subfigure]{justification=centering}
    \centering
    \caption{\textbf{Car Kinematics: Learning on nominal model.} Results for training Neural Lyapunov MPC while using a nominal model for internal dynamics. We use the Lyapunov loss~(\protect\ref{eq:lyapunov_loss}) for both learning the Lyapunov function and tuning the MPC. This is specified in the $\log(1+x)$ scale.}
    
    \label{tab:alpha_car_nominal}
    
    \begin{subtable}{0.5\textwidth}
        \caption{Lyapunov Function Learning} \vspace{-0.3cm}
        \begin{center}
        \begin{small}
        \begin{sc}
        \begin{tabular}{P{1cm} P{1.5cm} P{1.5cm} P{2.0cm}}
            \toprule
            Iter. & Loss ($\log(1+x)$) & Verified (\%) & Not Verified (\%) \\
            \midrule
            1 & 1.55 & 92.20 & 4.42 \\
            2 & 0.87 & 93.17 & 4.89 \\
            3 & 0.48 & 94.87 & 3.89 \\
            \bottomrule
        \end{tabular}
        \end{sc}
        \end{small}
        \end{center}
    \end{subtable}
    
    \medskip \medskip
    
    \begin{subtable}{0.5\textwidth}
        \caption{MPC Parameter Tuning} \vspace{-0.3cm}
        \begin{center}
        \begin{small}
        \begin{sc}
        \begin{tabular}{P{1cm} P{1.5cm} P{1.5cm} P{2.0cm}}
            \toprule
            Iter. & \multicolumn{2}{c}{Loss ($\log(1+x)$)} & Parameter \\
            & before & after &  $\alpha^\star$\\
            \midrule
            1 & 1.55 & 1.07 & 26.00 \\
            2 & 0.87 & 0.71 & 31.00 \\
            3 & 0.48 & 0.52 & 36.00 \\
            \bottomrule
        \end{tabular}
        \end{sc}
        \end{small}
        \end{center}
    \end{subtable}
    
    \vskip -0.1in
\end{table}

\begin{table}[H]
    \captionsetup[subfigure]{justification=centering}
    \centering
    \caption{\textbf{Car Kinematics: Learning on surrogate model.} Results for training Neural Lyapunov MPC while using the surrogate model for internal dynamics as well as in Lyapunov training. We use the Lyapunov loss~(\protect\ref{eq:lyapunov_loss}) for both learning the Lyapunov function and tuning the MPC. This is specified in the $\log(1+x)$ scale.}
    \label{tab:alpha_car_surrogate}
    \begin{subtable}{0.5\textwidth}
        \caption{Lyapunov Function Learning} \vspace{-0.3cm}
        \begin{center}
        \begin{small}
        \begin{sc}
        \begin{tabular}{P{1cm} P{1.5cm} P{1.5cm} P{2.0cm}}
            \toprule
            Iter. & Loss ($\log(1+x)$) & Verified (\%) & Not Verified (\%) \\
            \midrule
            1 & 1.84 & 91.74 & 8.26 \\
            2 & 1.43 & 92.26 & 7.74 \\
            3 & 1.65 & 91.61 & 8.39 \\
            \bottomrule
        \end{tabular}
        \end{sc}
        \end{small}
        \end{center}
    \end{subtable}
    
    \medskip \medskip
    
    \begin{subtable}{0.5\textwidth}
        \caption{MPC Parameter Tuning} \vspace{-0.3cm}
        \begin{center}
        \begin{small}
        \begin{sc}
        \begin{tabular}{P{1cm} P{1.5cm} P{1.5cm} P{2.0cm}}
            \toprule
            Iter. & \multicolumn{2}{c}{Loss ($\log(1+x)$)} & Parameter \\
            & before & after &  $\alpha^\star$\\
            \midrule
            1 & 1.84 & 1.41 & 36.00 \\
            2 & 1.43 & 1.05 & 36.00 \\
            3 & 1.65 & 1.30 & 36.00 \\
            \bottomrule
        \end{tabular}
        \end{sc}
        \end{small}
        \end{center}
    \end{subtable}
\end{table}

\subsection{Inverted Pendulum}

Differently from the car kinematics, for the inverted pendulum task, the parameters of $V$ are not re-initialized after every outer epoch, and the Lyapunov loss makes use of the LQR loss, $\ell$, for all the experiments except for the results in~\figref{fig:contraction_factor}. In this section, we discuss the results obtained from the alternate learning on the nominal model. We also provide an ablation study on: 1) a solution based solely on a contraction factor, and 2) the effect of having an imperfect solver, in particular the instability resulting from wrongly tuning the trust-region radius.

\paragraph{Alternate learning} 
Since the trained surrogate model has a high accuracy, we only consider the scenario for alternate learning with the nominal model. The main results for this scenario are in~\figref{fig:pendulum_alternate_learning} and \tabref{tab:alpha_pendulum}. We notice a slight improvement in the MPC's performance in the second iteration of the training procedure. In~\figref{fig:pendulum_alternate_learning}, it can be noticed that a small $\alpha$ needs to be used, which contradicts the ideal theoretical result. In practice, a very large value of this parameter results in bad conditioning for the QPs used by the MPC and causes the solver to fail.\footnote{When the solver fails, we simply set the solution to zero.} Since the pendulum is open-loop unstable, an increase of the Lyapunov loss can be noticed for larger values of $\alpha$. This demonstrates that it is necessary to perform a search over the parameter and that we cannot simply set it to a very large value. 

In~\figref{fig:pendulum_alternate_learning_nominal_trajectories}, we show the trajectories obtained by running the Neural Lyapunov MPC obtained from the first and second iterations. The initial states are sampled inside the safe level-set by using rejection sampling. The trajectories obtained from both the iterations are similar even though the Lyapunov function is different. The Lyapunov function from second iteration has a larger ROA.

We also compare the Neural Lyapunov MPC from the second iteration with the baseline MPCs in~\figref{fig:pendulum_lyap_big}. The baselines controllers behave quite similarly in this problem, although they have different prediction horizons. This is because, for both of them, the LQR terminal cost is a dominating term in the optimization's objective function. The Neural Lyapunov MPC achieves a slightly slower decrease rate compared to the demonstrator; however, it still stabilizes the system. The transfer from nominal to surrogate model is very successful for all the controllers, though in this case, the surrogate model is particularly accurate. 

It should be kept in mind that in order to produce these results, it was necessary to impose in the Lyapunov loss~(\protect\ref{eq:lyapunov_loss}) that the decrease rate of $V(x)$ needs to be upper bounded by the LQR stage loss, as in Equation \ref{eq:lyap2_lqr}. This resulted in the most effective learning of the function $V(x)$, contrarily to the vehicle kinematics example.   

\begin{table}[t]
    \captionsetup[subfigure]{justification=centering}
    \centering
    \caption{\textbf{Inverted Pendulum: Learning on nominal model.} Results for training Neural Lyapunov MPC while using a nominal model for internal dynamics. We use the Lyapunov loss~(\protect\ref{eq:lyapunov_loss}) for both learning the Lyapunov function and tuning the MPC. This is specified in the $\log(1+x)$ scale.}
    \label{tab:alpha_pendulum}
    \begin{subtable}{0.5\textwidth}
        \caption{Lyapunov Function Learning} \vspace{-0.3cm}
        \begin{center}
        \begin{small}
        \begin{sc}
        \begin{tabular}{P{1cm} P{1.5cm} P{1.5cm} P{2.0cm}}
            \toprule
            Iter. & Loss ($\log(1+x)$) & Verified (\%) & Not Verified (\%) \\
            \midrule
            1 & 3.21 & 13.25 & 0.00 \\
            2 & 1.08 & 13.54 & 0.00 \\
            \bottomrule
        \end{tabular}
        \end{sc}
        \end{small}
        \end{center}
    \end{subtable}
    
    \medskip \medskip
    
    \begin{subtable}{0.5\textwidth}
        \caption{MPC Parameter Tuning} \vspace{-0.3cm}
        \begin{center}
        \begin{small}
        \begin{sc}
        \begin{tabular}{P{1cm} P{1.5cm} P{1.5cm} P{2.0cm}}
            \toprule
            Iter. & \multicolumn{2}{c}{Loss ($\log(1+x)$)} & Parameter \\
            & before & after &  $\alpha^\star$\\
            \midrule
            1 & 3.21 & 2.47 & 1.40 \\
            2 & 1.08 & 1.28 & 1.00 \\
            \bottomrule
        \end{tabular}
        \end{sc}
        \end{small}
        \end{center}
    \end{subtable}
    
\end{table}

\paragraph{Alternate learning without LQR loss}
We now consider the case when the learning is performed while using a contraction factor of $\lambda=0.9$ and without the LQR loss term in the Lyapunov loss (i.e., $v=0$). The results are depicted in~\figref{fig:contraction_factor}. In order to obtain these results, the Lyapunov NN scaling $\beta$, in Equation (\ref{eq:lyap_definition2}), was initialized with:
\begin{equation*}
    \beta_0=\text{softplus}^{-1}(25),
\end{equation*}
according to a rough estimate of the minimum scaling $\alpha$ from~\eqtref{eq:bound_alpha1}. This was able to produce a Lyapunov function that makes the MPC safe with $\alpha=12$. However, the learning becomes more difficult, and it results in a smaller safe region estimate with a slower convergence rate for the system trajectories.  

\paragraph{Effects of the trust region}
In~\figref{fig:trust_region}, we show the result of varying the trust radius of the SQP solver on the inverted pendulum. While a larger value can result in further approximation, given the limited number of iterations, in this case a small value of the radius results in a weaker control signal and local instability near the origin.

\paragraph{Formal verification:} These procedures are expensive but necessary, as pathological cases could result in the training loss (\ref{eq:lyapunov_loss}) for which the safe set could be  converging to a local minima with a very small set.  Results can also vary where different random seeds are used. For these reasons, during training we also an informal verification over a fixed validation set and use this to select the best $V$ from all training epochs. In practical applications we reccommend to use formal methods. 


\begin{figure*}[t]
    \captionsetup[subfigure]{justification=centering}
    \centering
    
    \begin{subfigure}[b]{0.32\textwidth}
    \centering
        \includegraphics[scale=0.3]{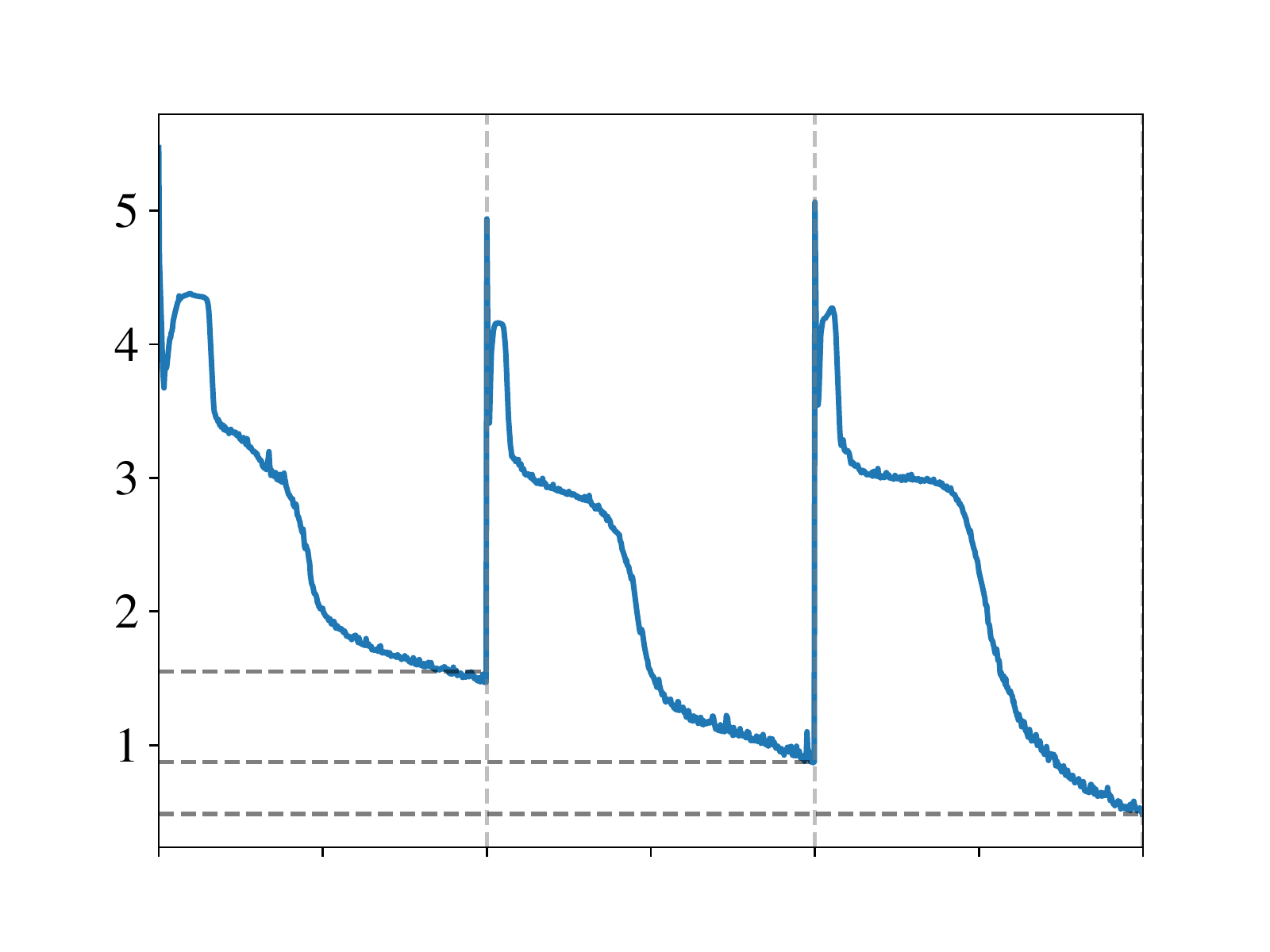}\vspace{-0.1cm}
        \caption{Lyapunov Loss ($\log(1+x)$)}
    \end{subfigure}
    \begin{subfigure}[b]{0.32\textwidth}
        \centering
        \includegraphics[scale=0.3]{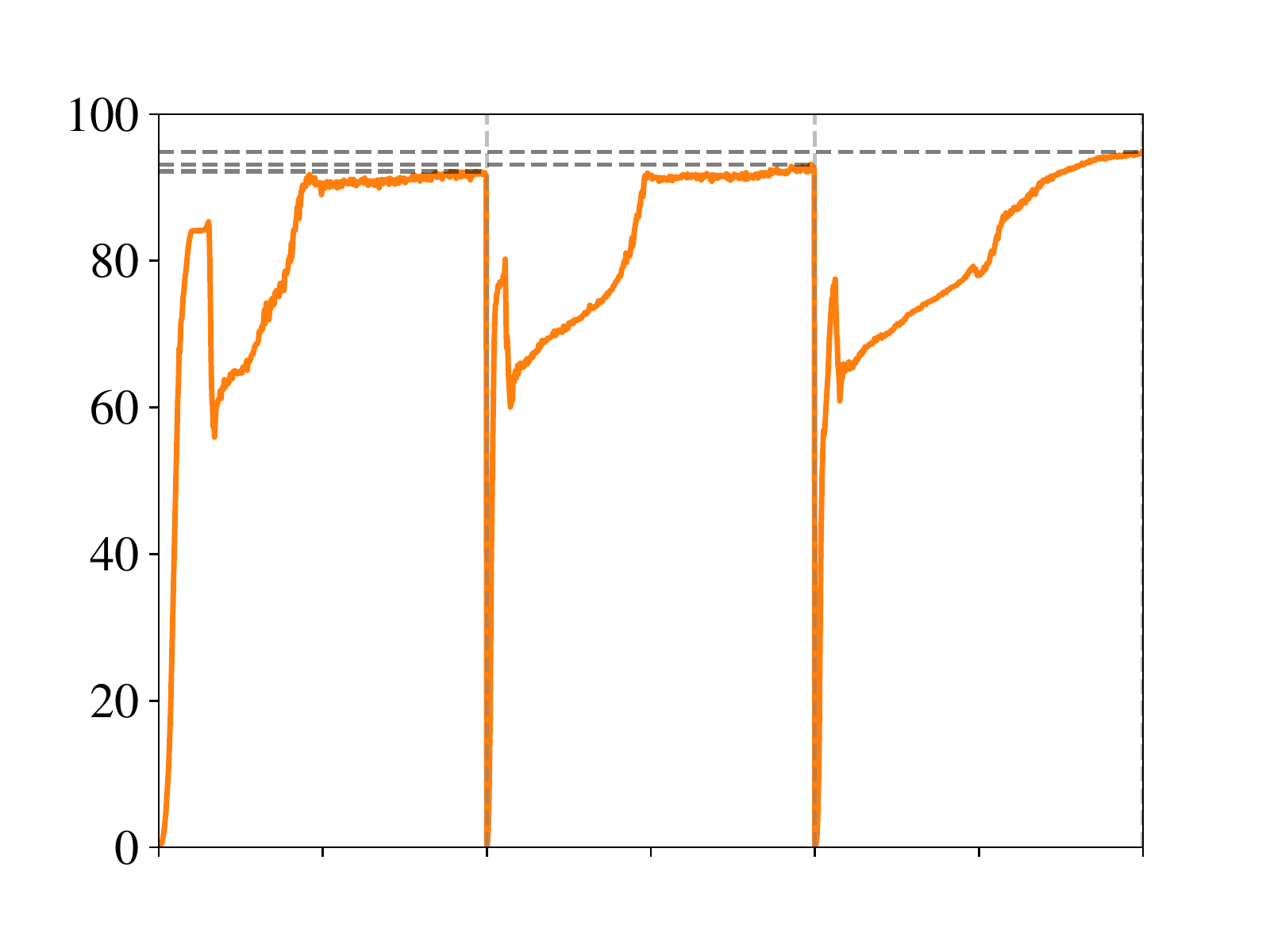}\vspace{-0.1cm}
        \caption{Verified Points (\%)}
    \end{subfigure}
    \begin{subfigure}[b]{0.32\textwidth}
        \centering
        \includegraphics[scale=0.3]{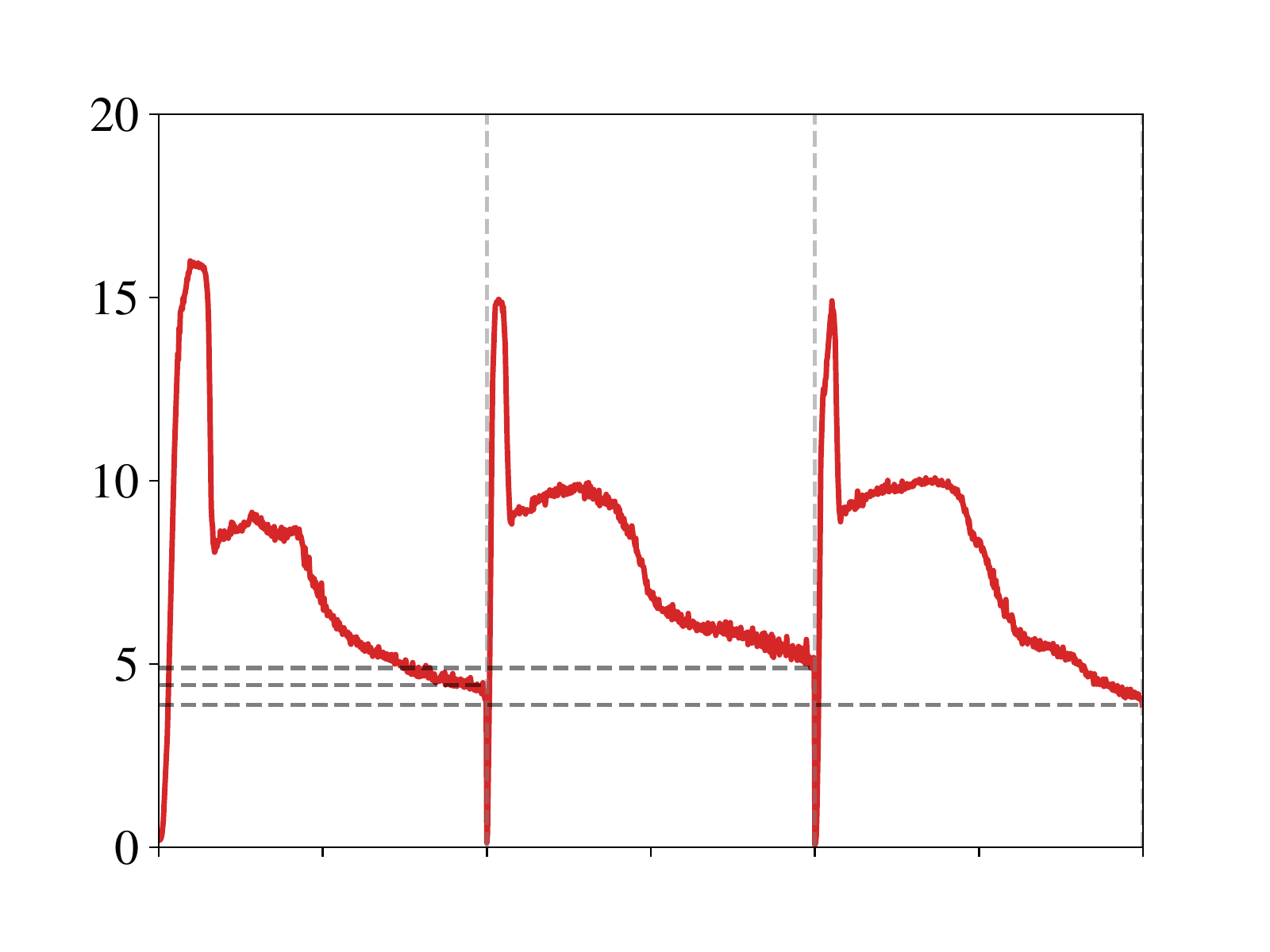}\vspace{-0.1cm}
        \caption{Not Verified Points (\%)}
    \end{subfigure}
    
    \medskip    \medskip \medskip

    \begin{subfigure}[b]{0.32\textwidth}
        \centering
        \includegraphics[trim={52 35 60 40}, clip, scale=0.32]{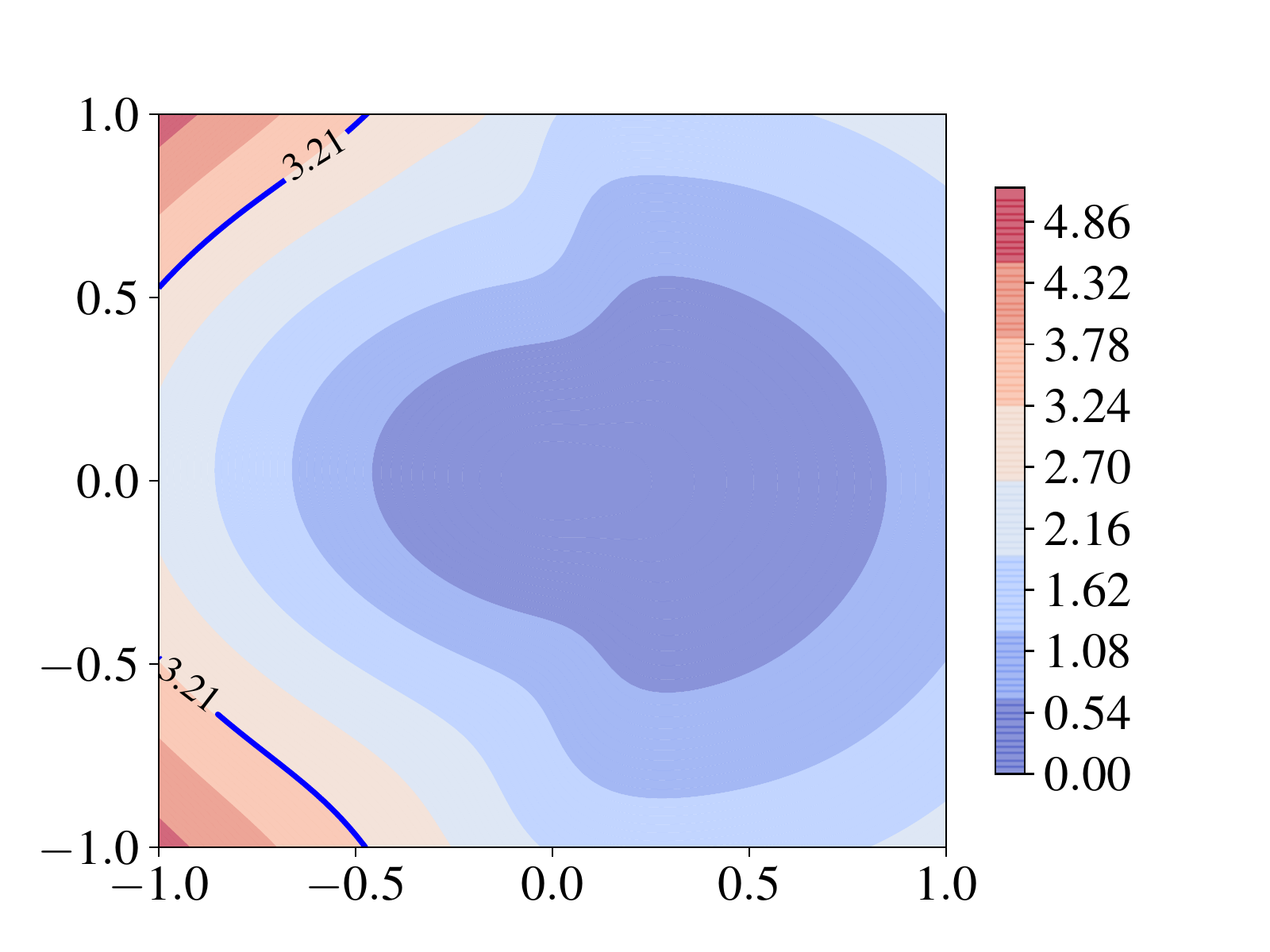}\vspace{-0.1cm}
    \end{subfigure}
    \begin{subfigure}[b]{0.32\textwidth}
    \centering
        \includegraphics[trim={52 35 60 40}, clip, scale=0.32]{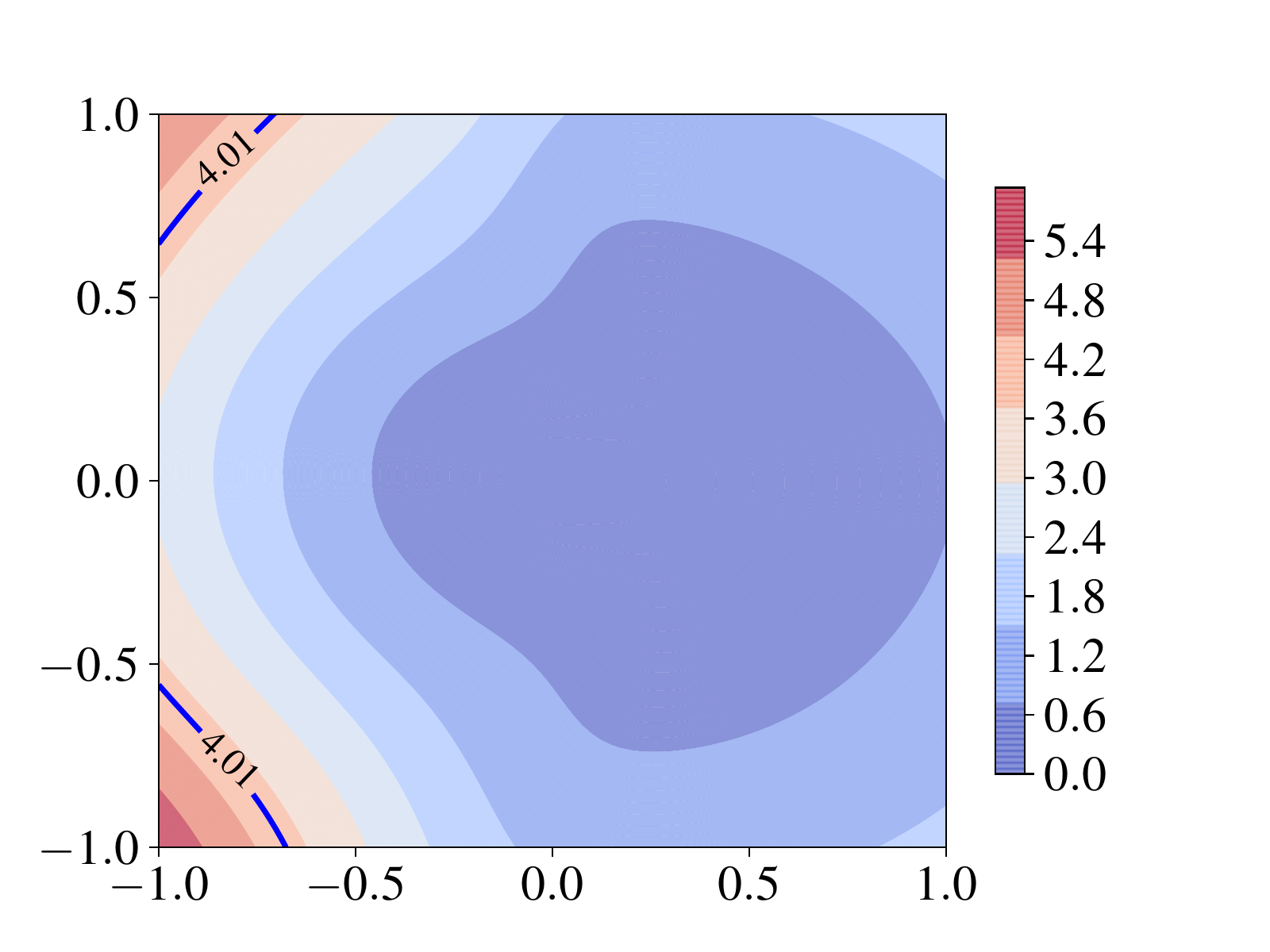}\vspace{-0.1cm}
    \end{subfigure}
    \begin{subfigure}[b]{0.32\textwidth}
    \centering
        \includegraphics[trim={52 35 60 40}, clip, scale=0.32]{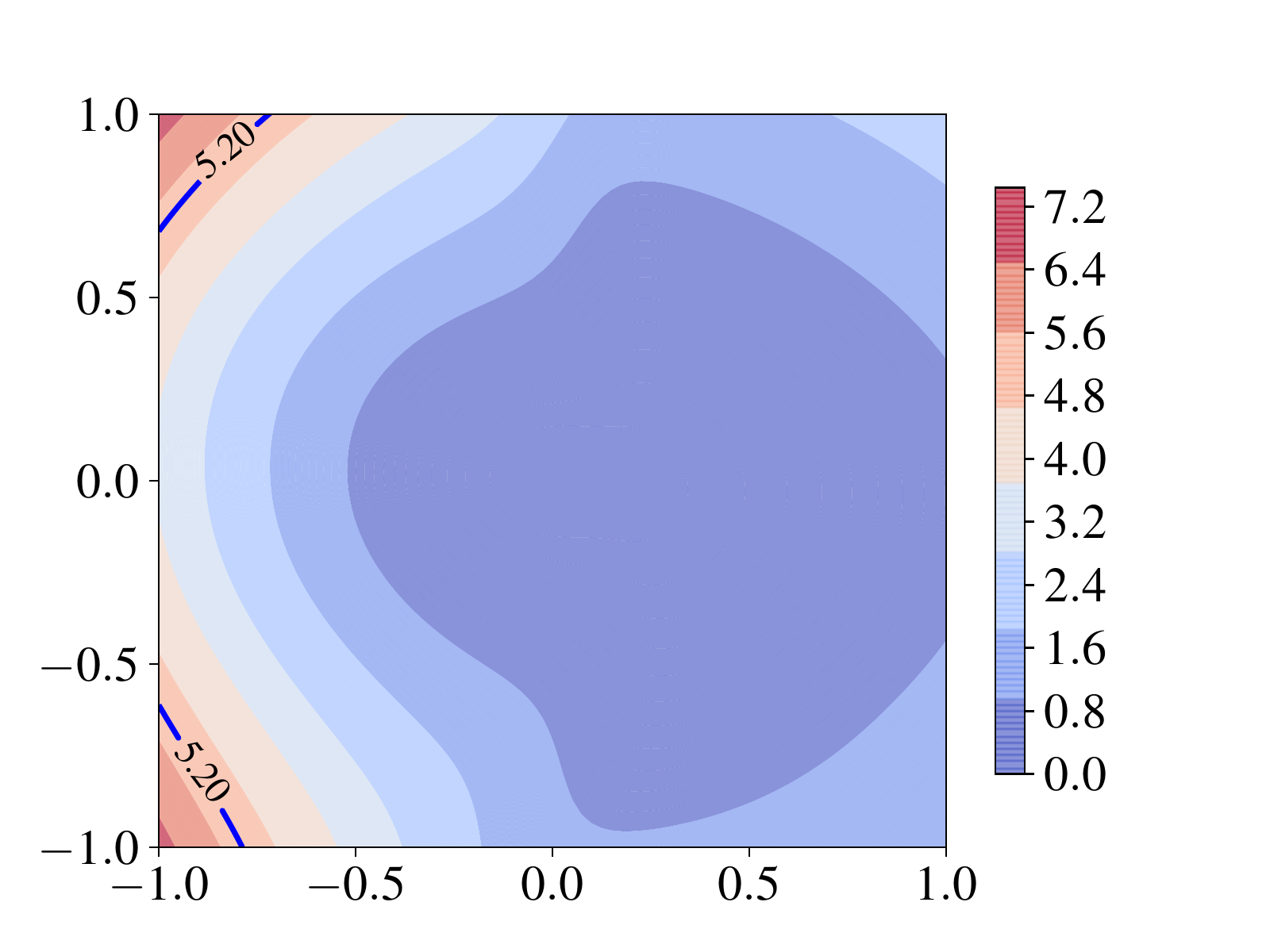}\vspace{-0.1cm}
    \end{subfigure}
    
    \medskip
    
     \begin{subfigure}[b]{0.32\textwidth}
        \centering
        \includegraphics[trim={15 10 10 30}, clip, scale=0.3]{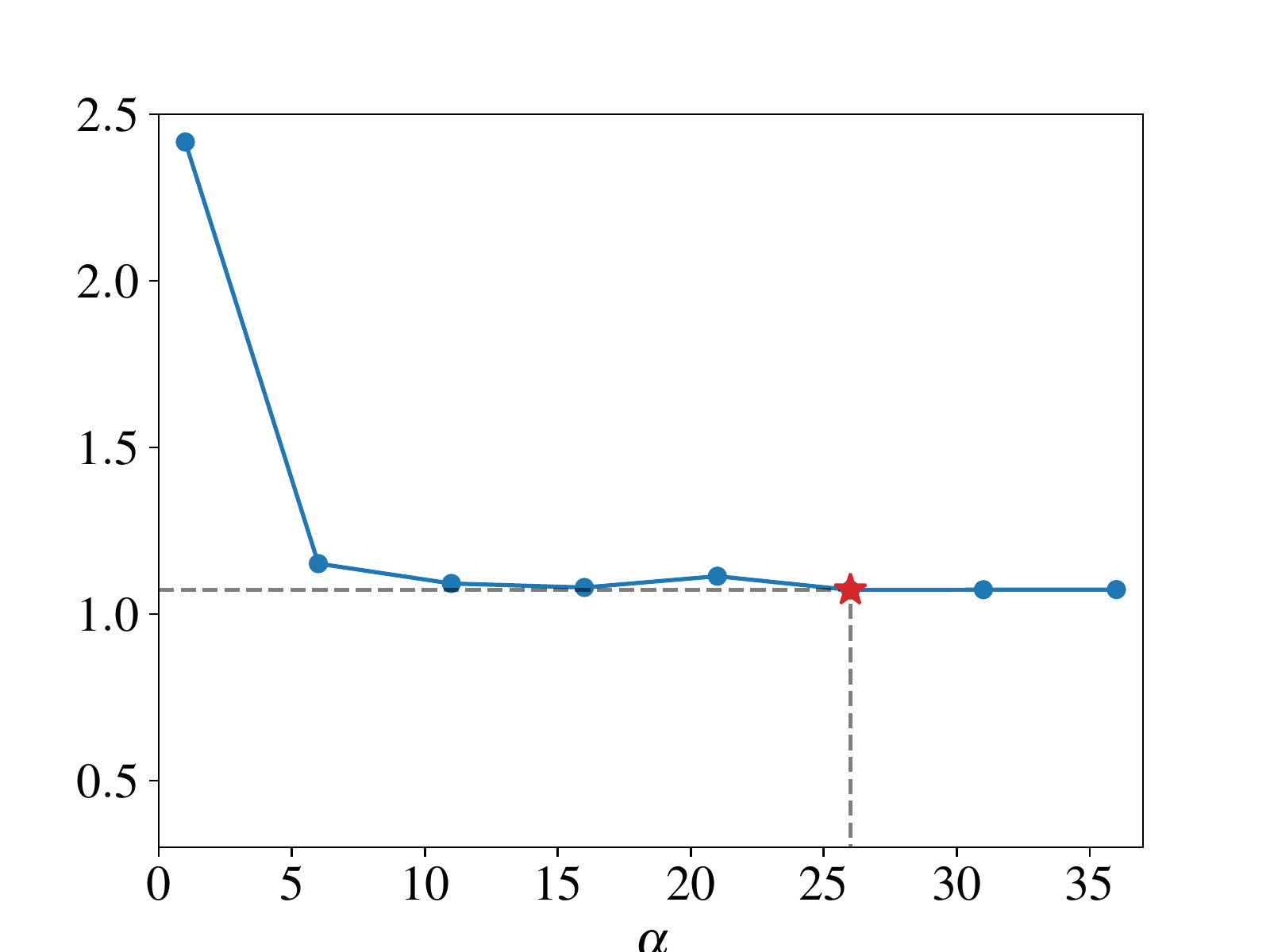}\vspace{-0.1cm}
        \caption*{Iteration 1}
    \end{subfigure}
    \begin{subfigure}[b]{0.32\textwidth}
    \centering
        \includegraphics[trim={52 10 10 30}, clip, scale=0.3]{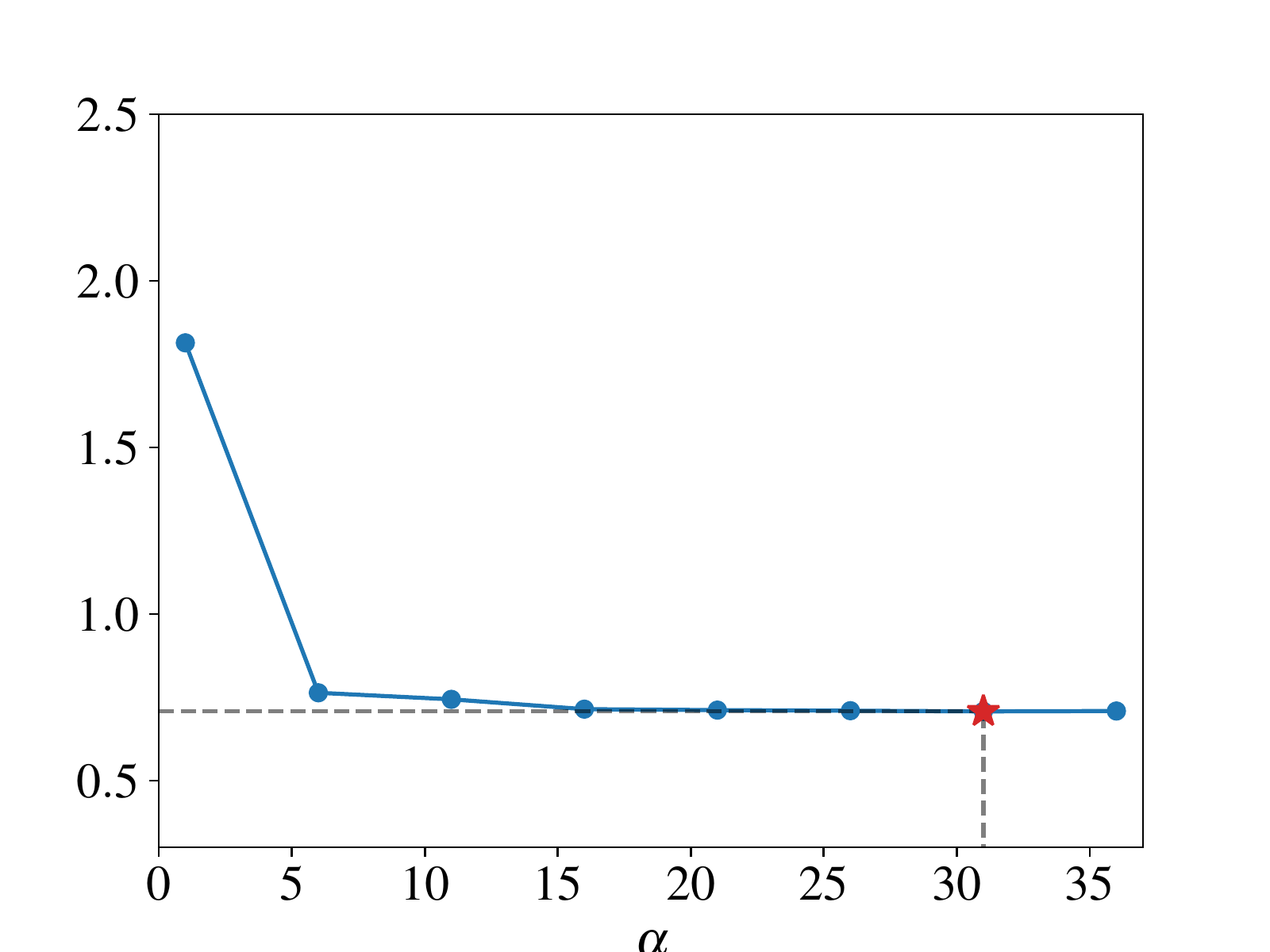}\vspace{-0.1cm}
        \caption*{Iteration 2}
    \end{subfigure}
    \begin{subfigure}[b]{0.32\textwidth}
    \centering
        \includegraphics[trim={52 10 10 30}, clip, scale=0.3]{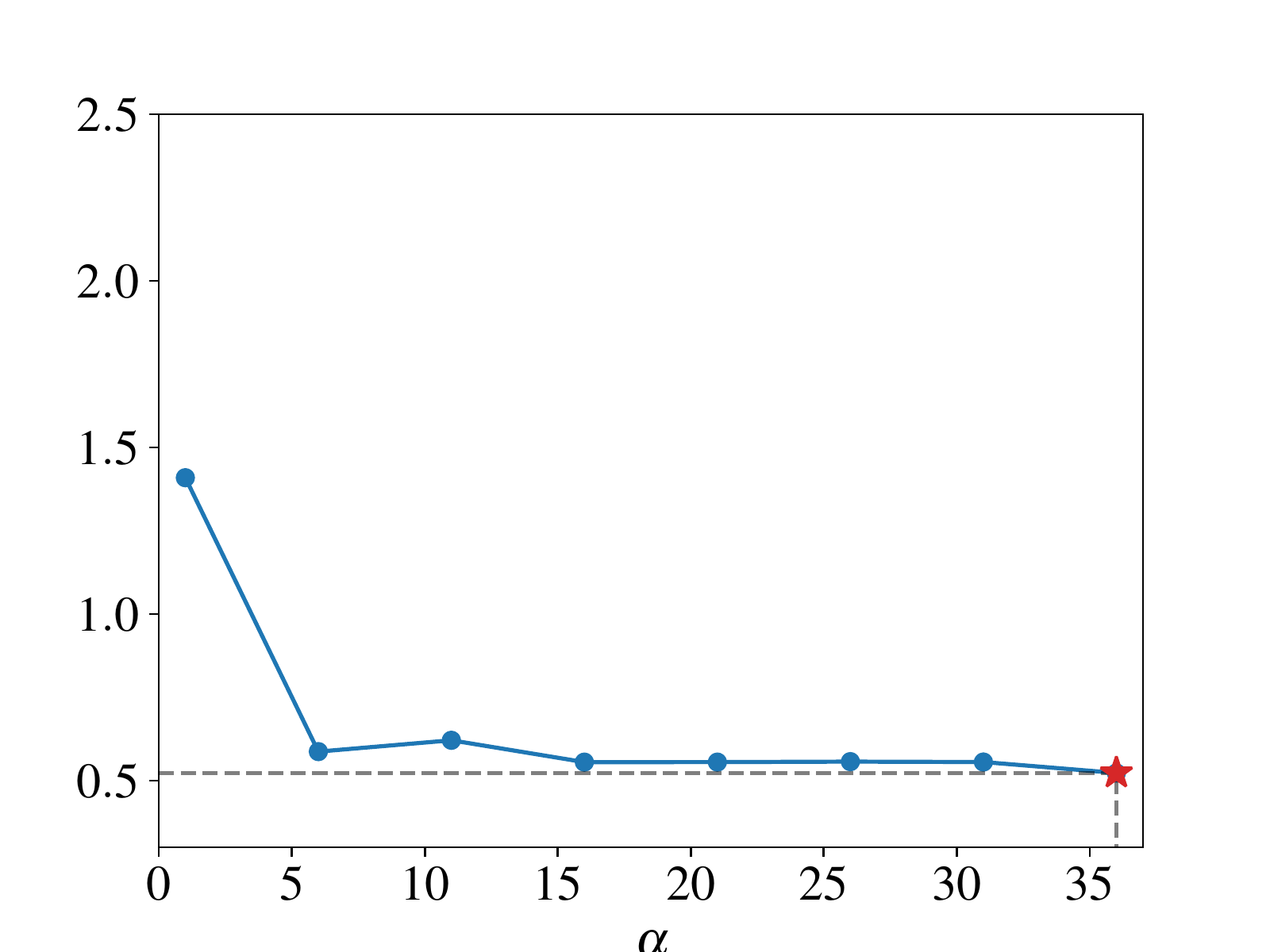}\vspace{-0.1cm}
        \caption*{Iteration 3 (best)}
    \end{subfigure}
    
    \caption{\textbf{Car kinematics: Alternate learning on nominal model.} After every $N_V=500$ epochs of Lyapunov learning, the learned Lyapunov function is used to tune the MPC parameters. \textbf{Top:} The training curves for Lyapunov function. Vertical lines separate iterations.  \textbf{Middle:}  The resulting Lyapunov function $V$ at $\phi=0$ with the best performance. \textbf{Bottom:} Line-search for the MPC parameter $\alpha$ to minimize the Lyapunov loss~(\protect\ref{eq:lyapunov_loss}) with $V$ as terminal cost. The loss is plotted on the y-axis in a $\log(1+x)$ scale. The point marked in red is the parameter which minimizes the loss.}
    \label{fig:car_alternate_learning_nominal_complete}
\end{figure*}

\begin{figure*}[t]
    \captionsetup[subfigure]{justification=centering}
    \centering
    
    \begin{subfigure}[b]{0.32\textwidth}
        \centering
        \includegraphics[trim={15 35 60 40}, clip, scale=0.32]{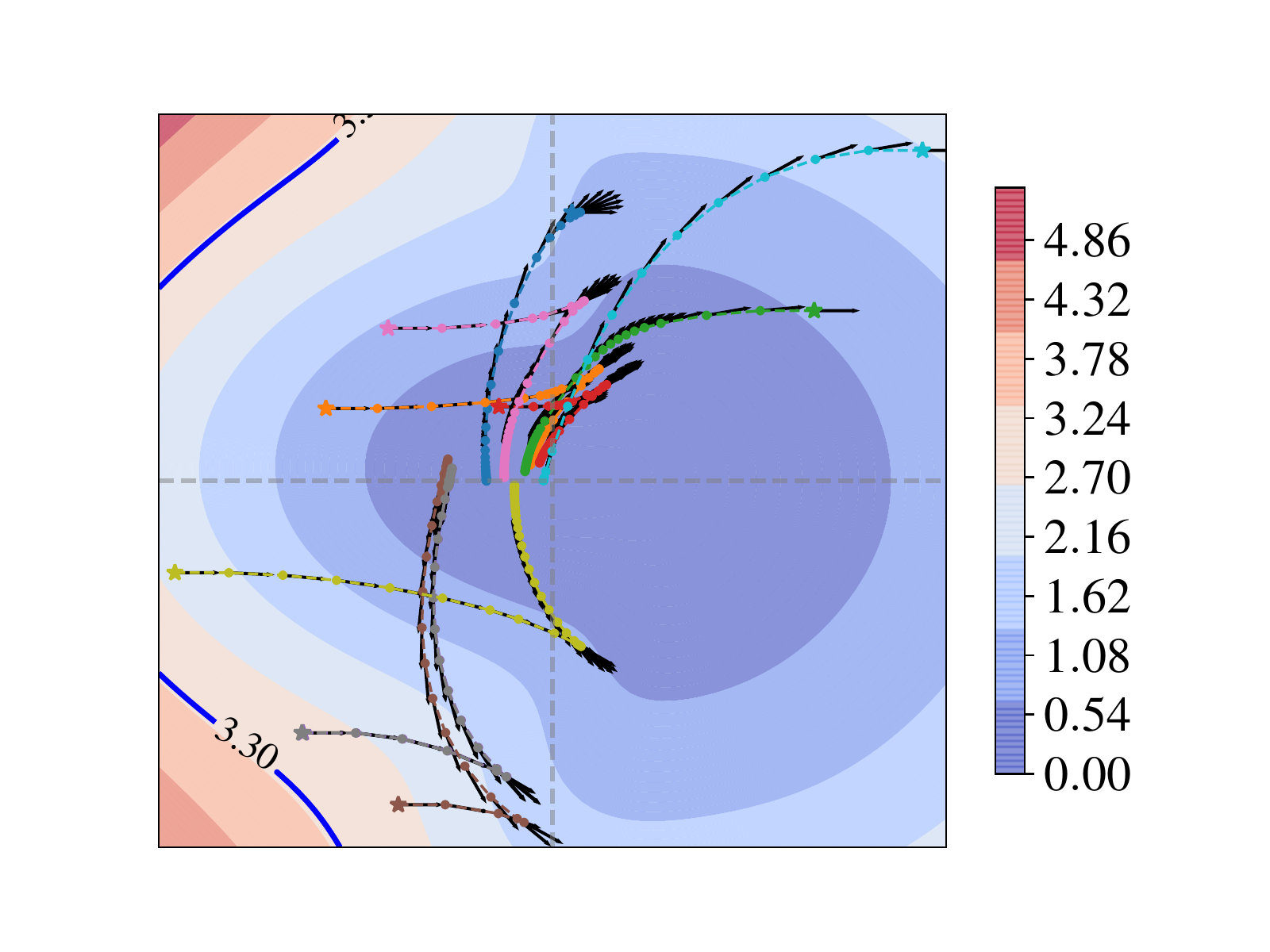}\vspace{-0.1cm}
    \end{subfigure}
    \begin{subfigure}[b]{0.32\textwidth}
    \centering
        \includegraphics[trim={15 35 60 40}, clip, scale=0.32]{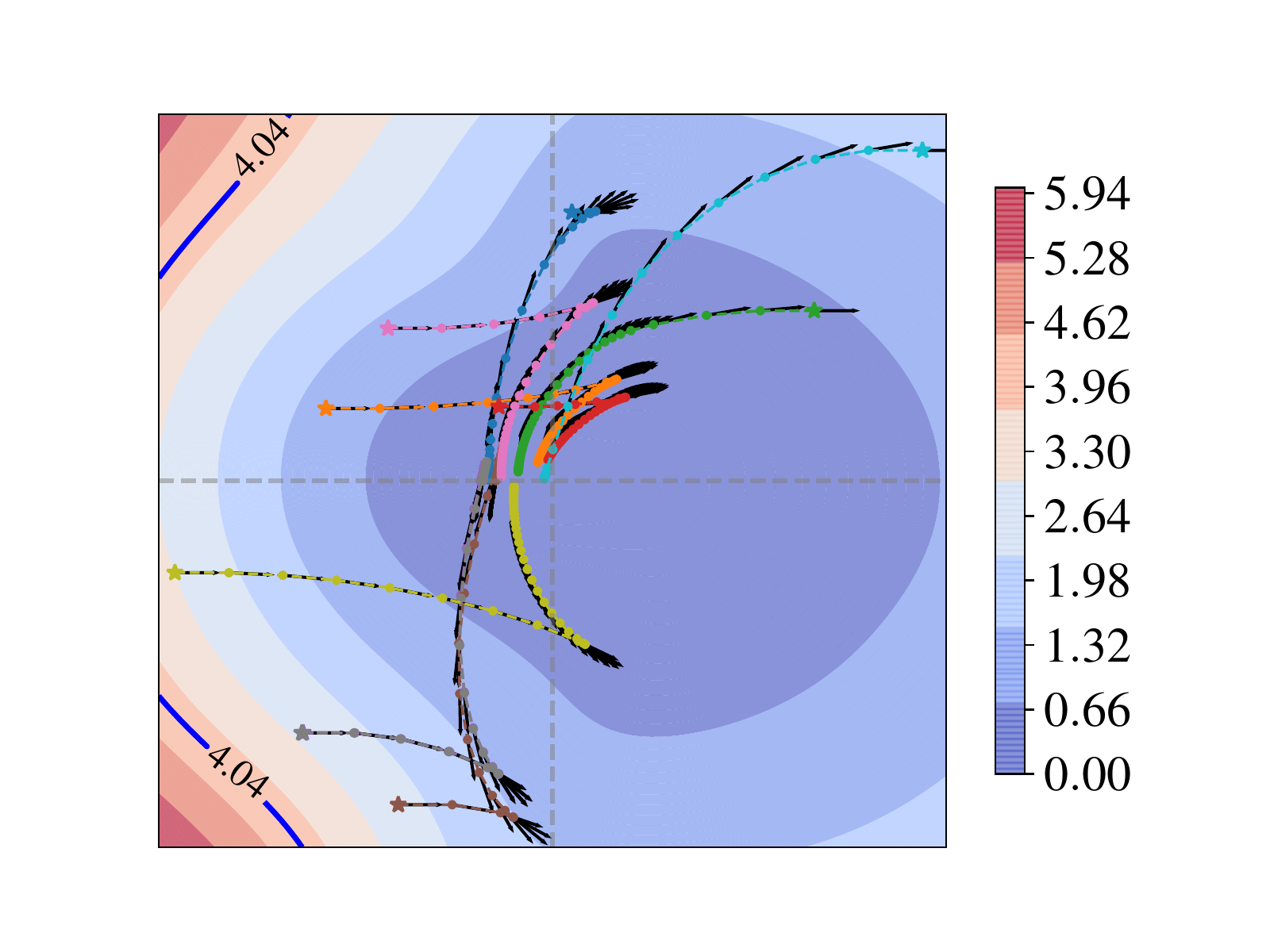}\vspace{-0.1cm}
    \end{subfigure}
    \begin{subfigure}[b]{0.32\textwidth}
    \centering
        \includegraphics[trim={15 35 60 40}, clip, scale=0.32]{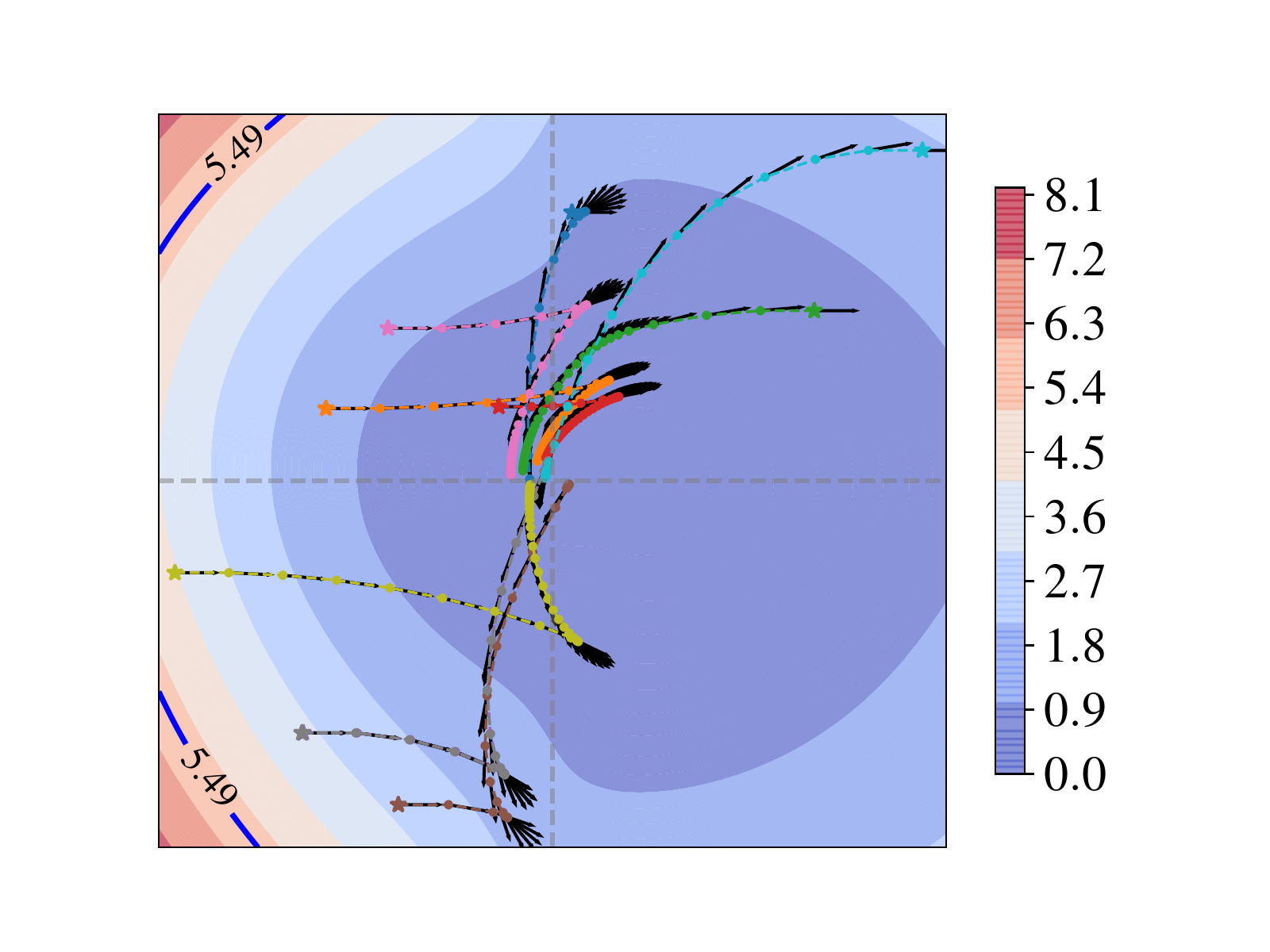}\vspace{-0.1cm}
    \end{subfigure}
    
    \medskip
    
    \begin{subfigure}[b]{0.32\textwidth}
        \centering
        \includegraphics[trim={15 25 30 30}, clip, scale=0.32]{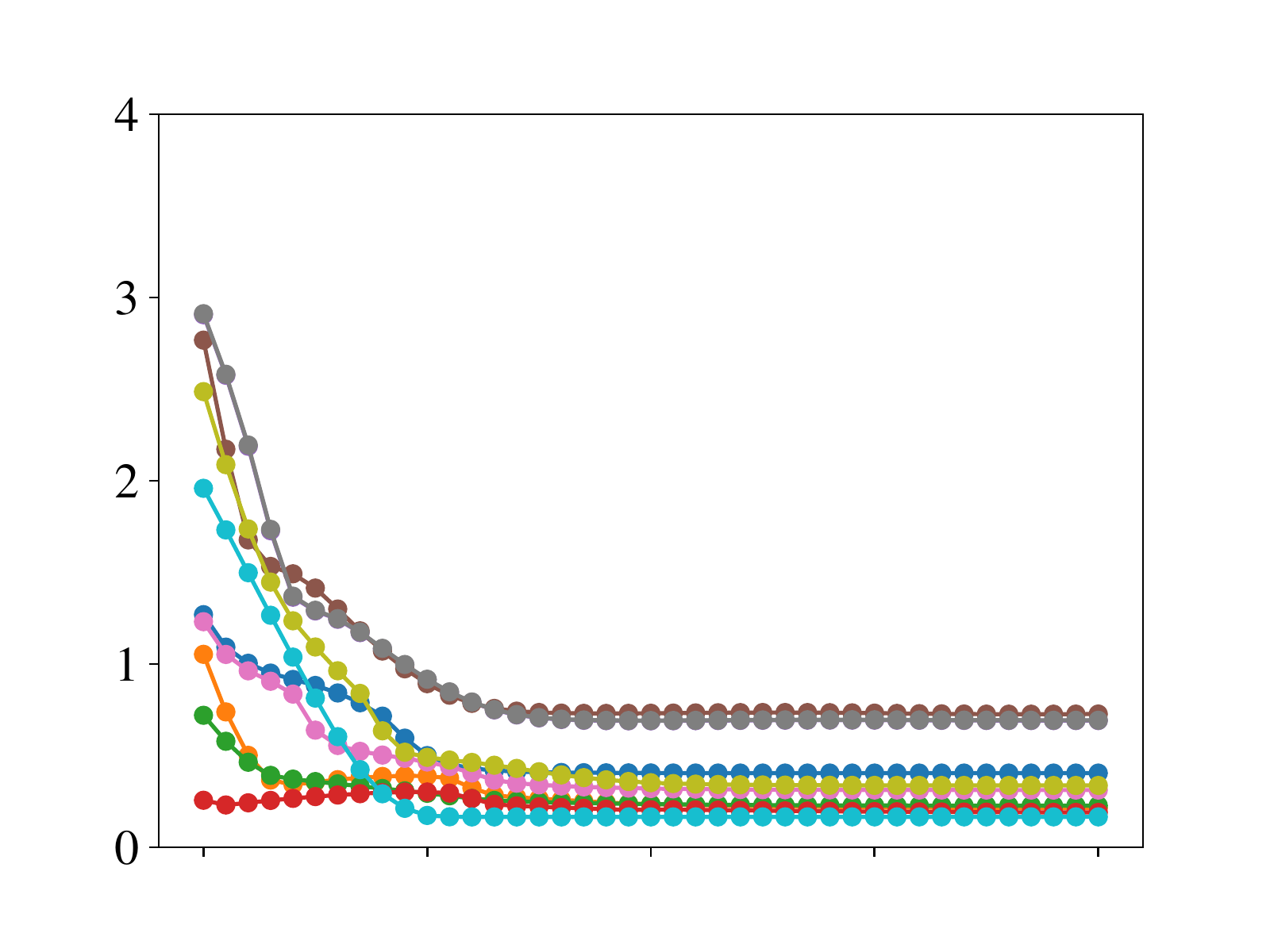}\vspace{-0.1cm}
    \end{subfigure}
    \begin{subfigure}[b]{0.32\textwidth}
    \centering
        \includegraphics[trim={52 25 30 30}, clip, scale=0.32]{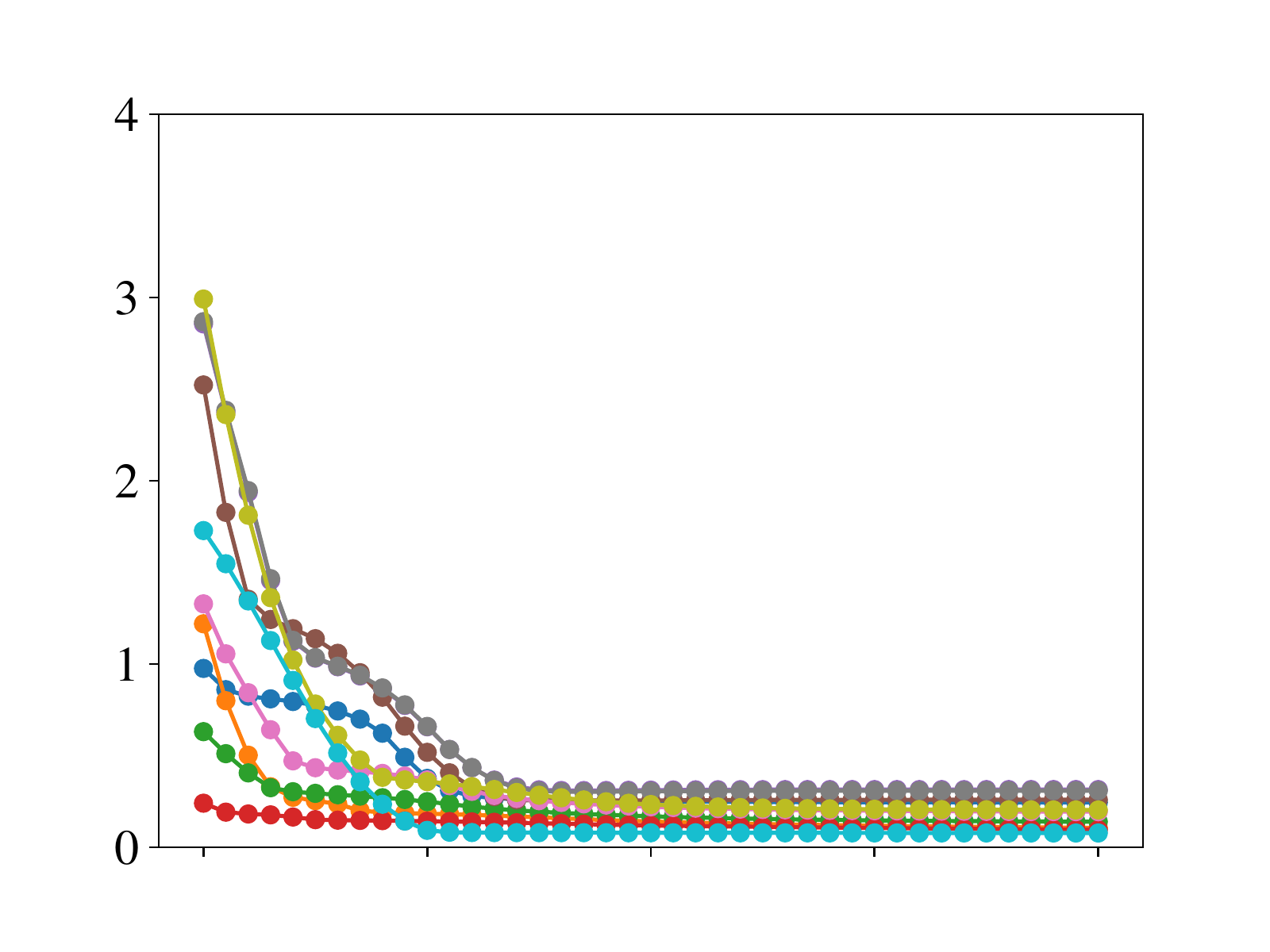}\vspace{-0.1cm}
    \end{subfigure}
    \begin{subfigure}[b]{0.32\textwidth}
    \centering
        \includegraphics[trim={52 25 30 30}, clip, scale=0.32]{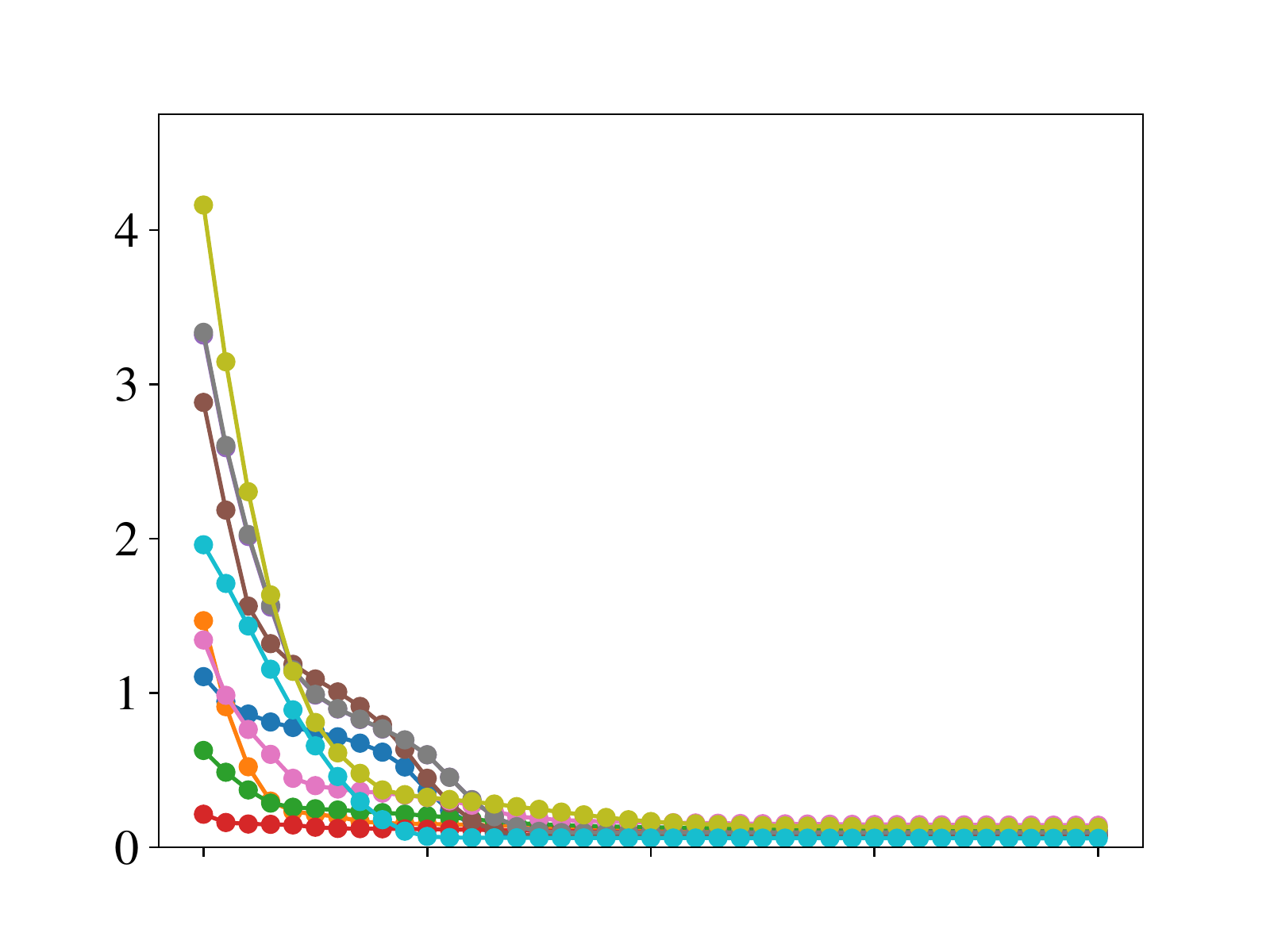}\vspace{-0.1cm}
    \end{subfigure}
    
    \medskip
    
     \begin{subfigure}[b]{0.32\textwidth}
    \centering
        \includegraphics[trim={15 25 30 30}, clip, scale=0.32]{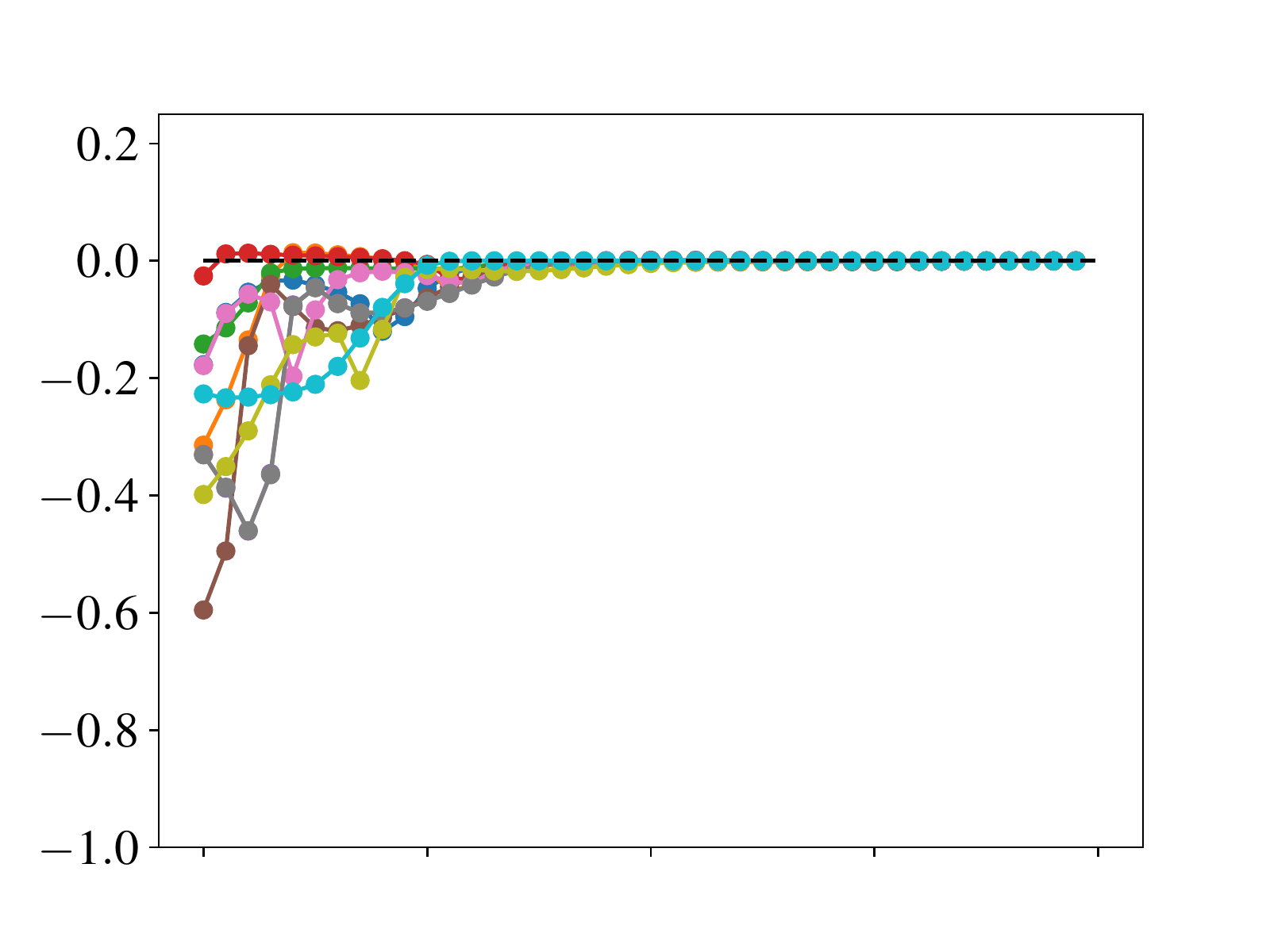}\vspace{-0.1cm}
        \caption*{Iteration 1}
    \end{subfigure}
    \begin{subfigure}[b]{0.32\textwidth}
    \centering
        \includegraphics[trim={52 25 30 30}, clip, scale=0.32]{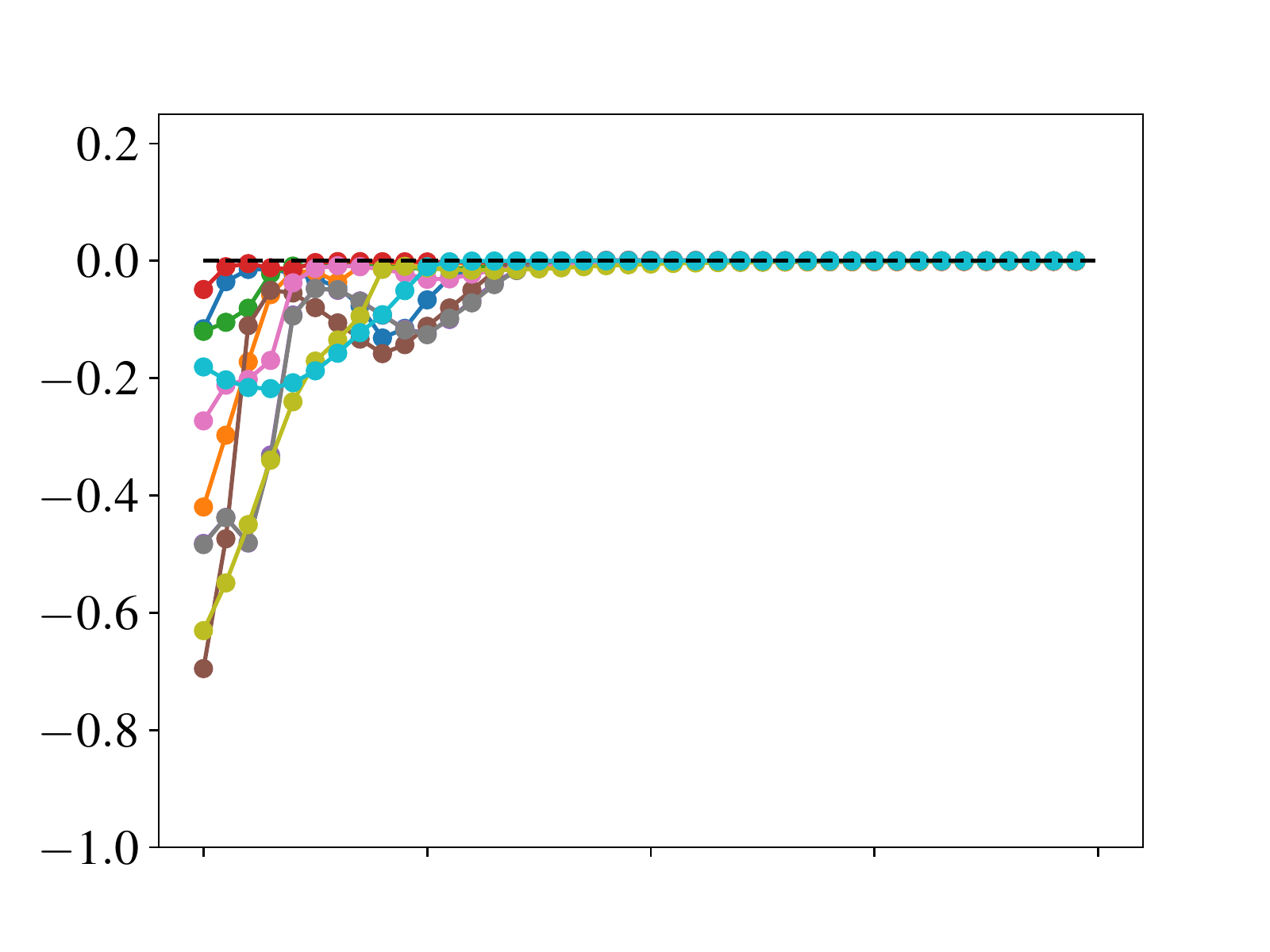}\vspace{-0.1cm}
        \caption*{Iteration 2}
    \end{subfigure}
    \begin{subfigure}[b]{0.32\textwidth}
        \centering
        \includegraphics[trim={52 25 30 30}, clip, scale=0.32]{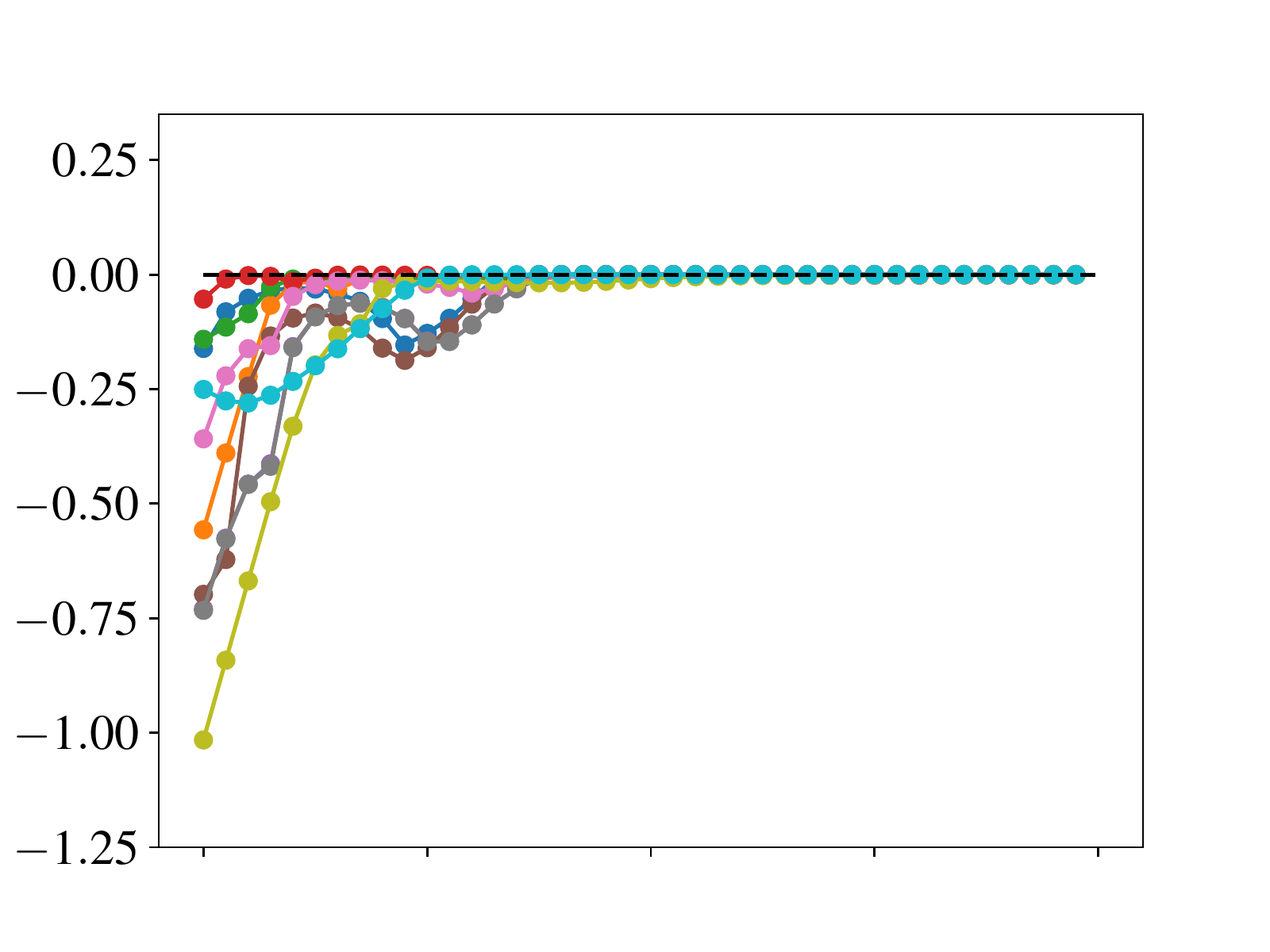}\vspace{-0.1cm}
        \caption*{Iteration 3 (best)}
    \end{subfigure}
    
    \caption{{\bf Car kinematics: Testing Neural Lyapunov MPC obtained from training on nominal model.} For each iteration, we show the trajectories obtained through our Neural Lyapunov MPC while using the resulting Lyapunov function and the MPC parameter selected from the line-search. {\bf Top}: The Lyapunov function at $\phi=0$ with trajectories for $40$ steps at each iteration. {\bf Middle}: The evaluated Lyapunov function. {\bf Bottom}: The Lyapunov function time difference.}
    \label{fig:car_alternate_learning_nominal_trajectories}
\end{figure*}

\begin{figure*}[t]
    \captionsetup[subfigure]{justification=centering}
    \centering
    
    \begin{subfigure}[b]{0.32\textwidth}
        \centering
        \includegraphics[trim={15 35 30 40}, clip, scale=0.32]{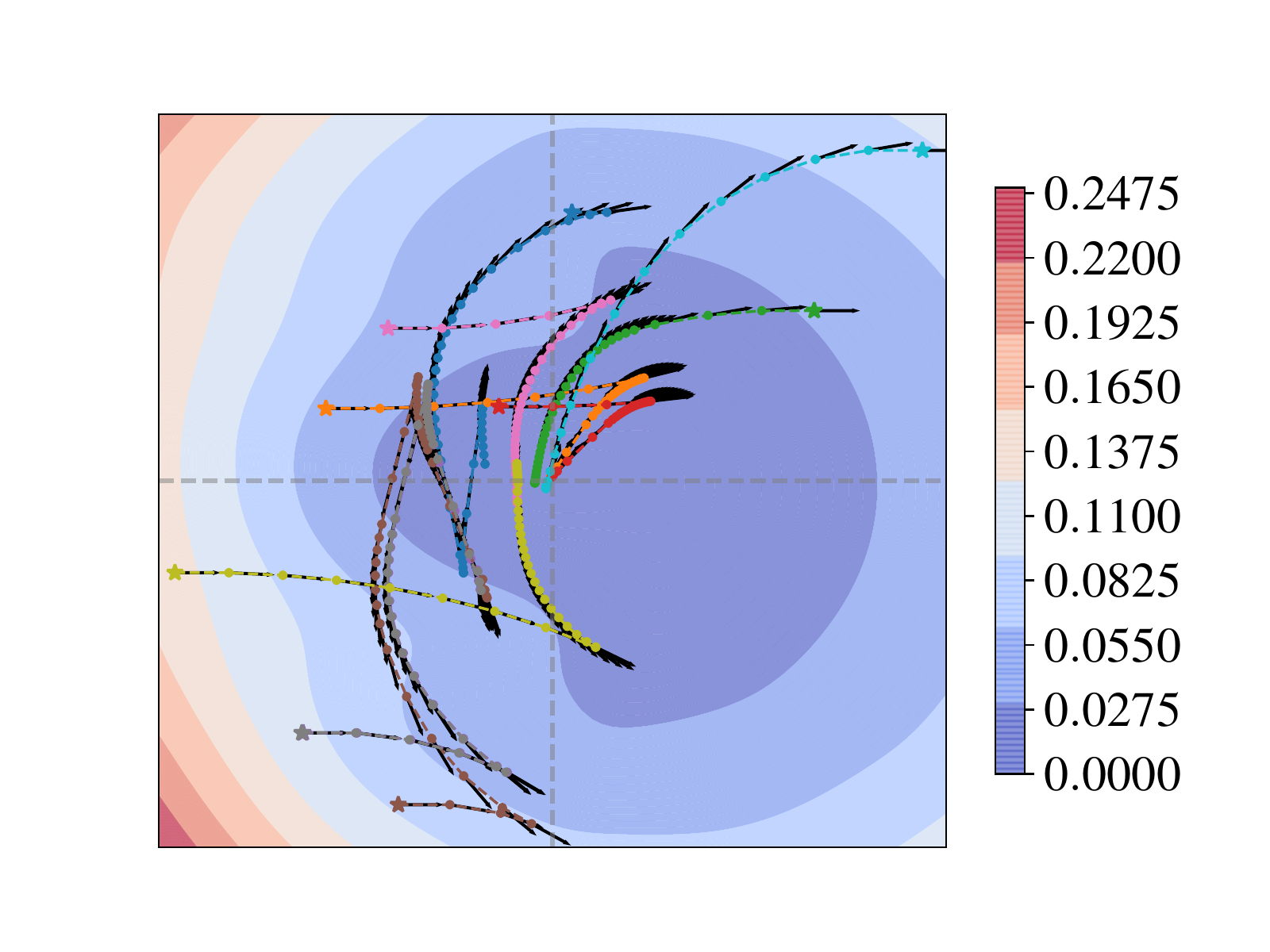}\vspace{-0.1cm}
    \end{subfigure}
    \begin{subfigure}[b]{0.32\textwidth}
    \centering
        \includegraphics[trim={15 35 30 40}, clip, scale=0.32]{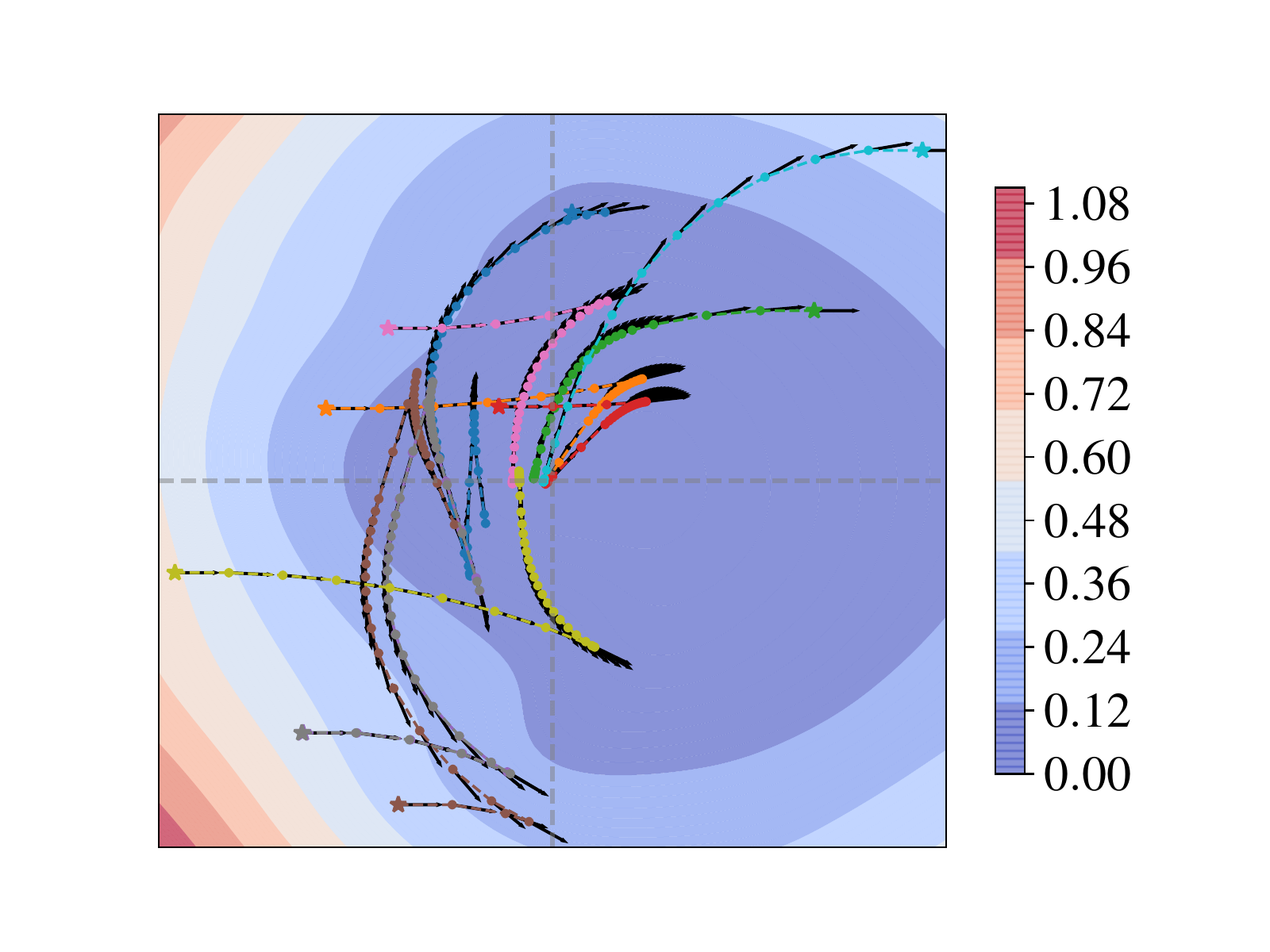}\vspace{-0.1cm}
    \end{subfigure}
    \begin{subfigure}[b]{0.32\textwidth}
    \centering
        \includegraphics[trim={15 35 30 40}, clip, scale=0.32]{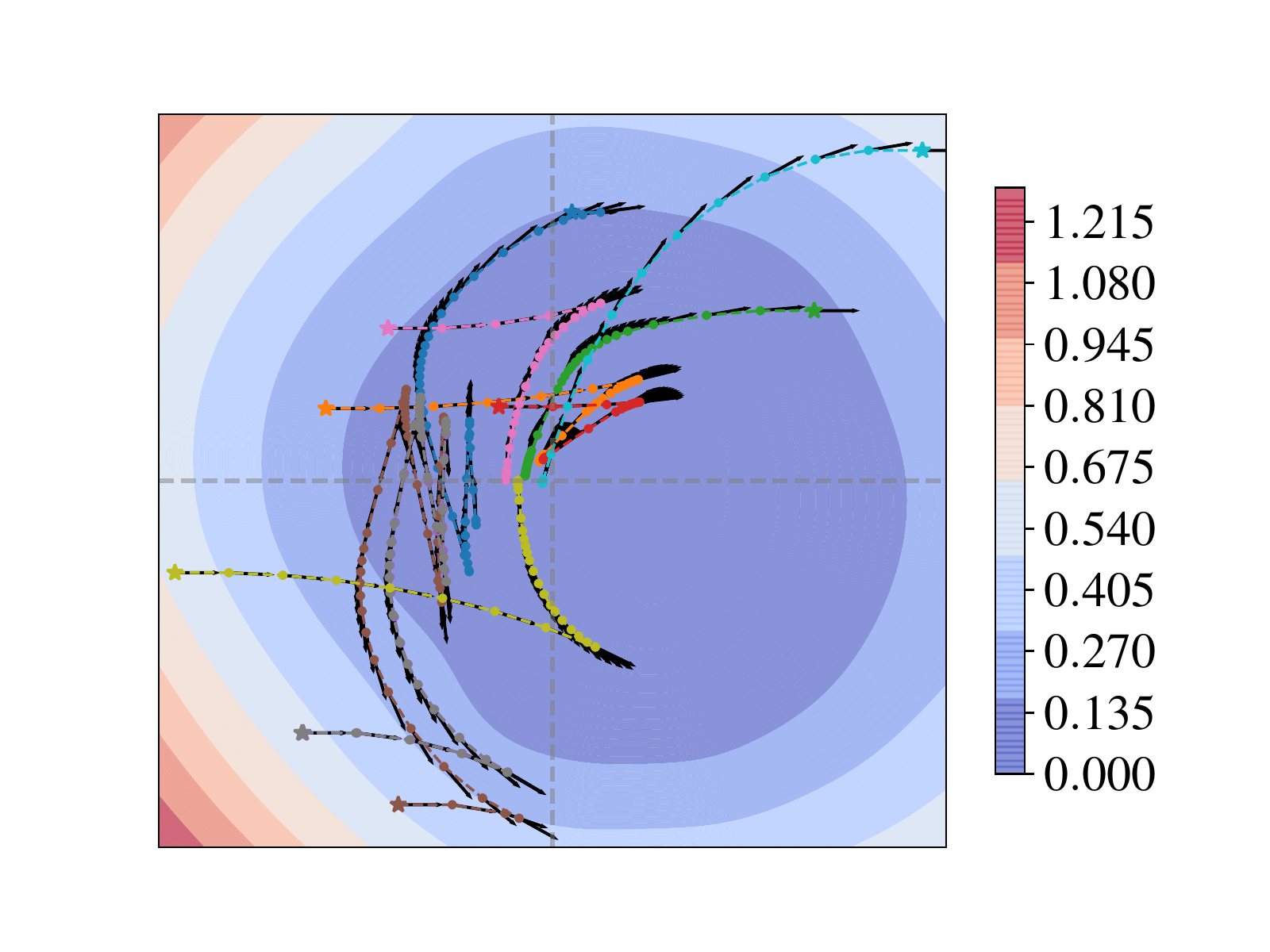}\vspace{-0.1cm}
    \end{subfigure}
    
    \medskip
    
    \begin{subfigure}[b]{0.32\textwidth}
        \centering
        \includegraphics[trim={15 25 30 30}, clip, scale=0.32]{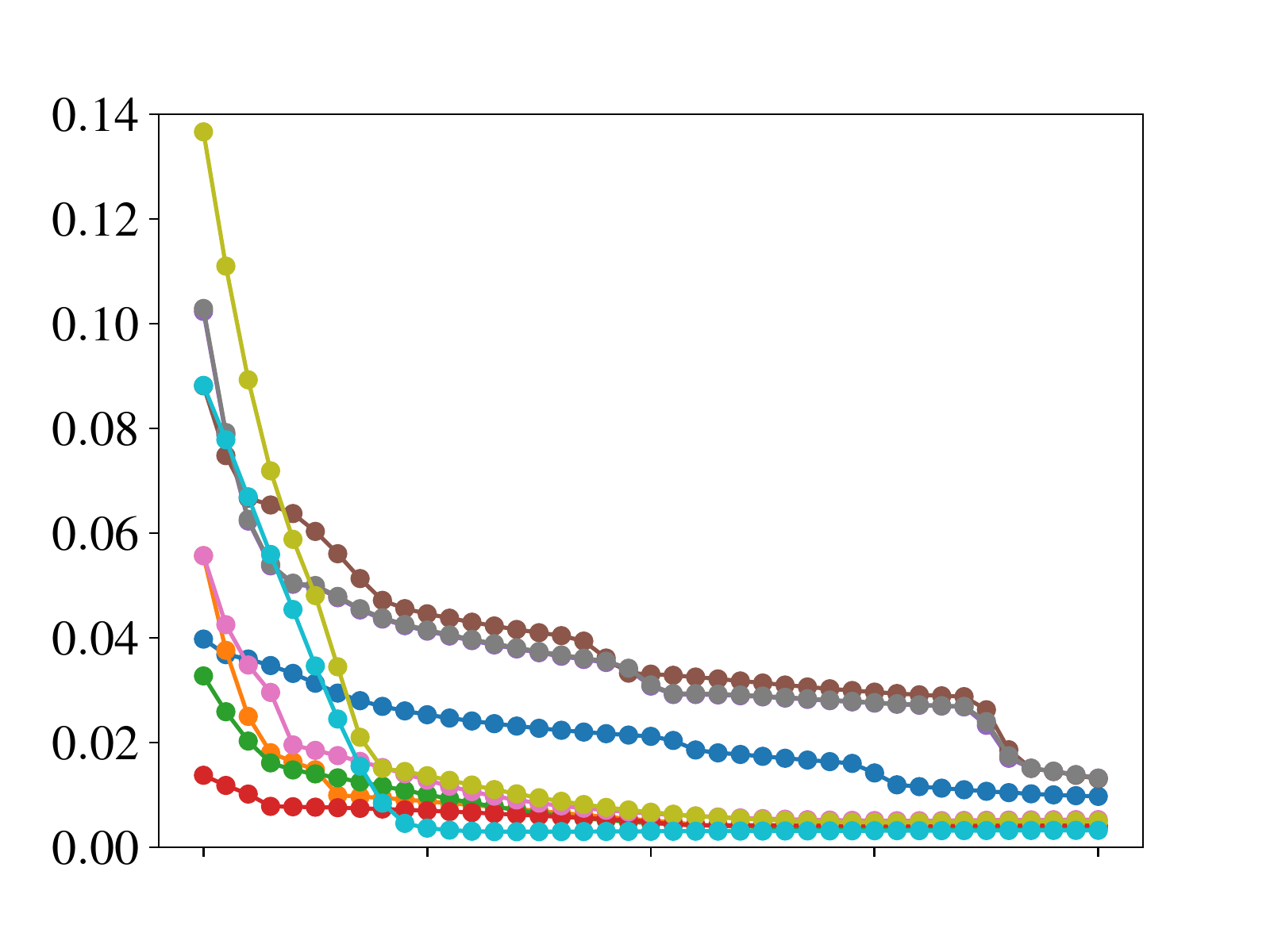}\vspace{-0.1cm}
    \end{subfigure}
    \begin{subfigure}[b]{0.32\textwidth}
    \centering
        \includegraphics[trim={52 25 30 30}, clip, scale=0.32]{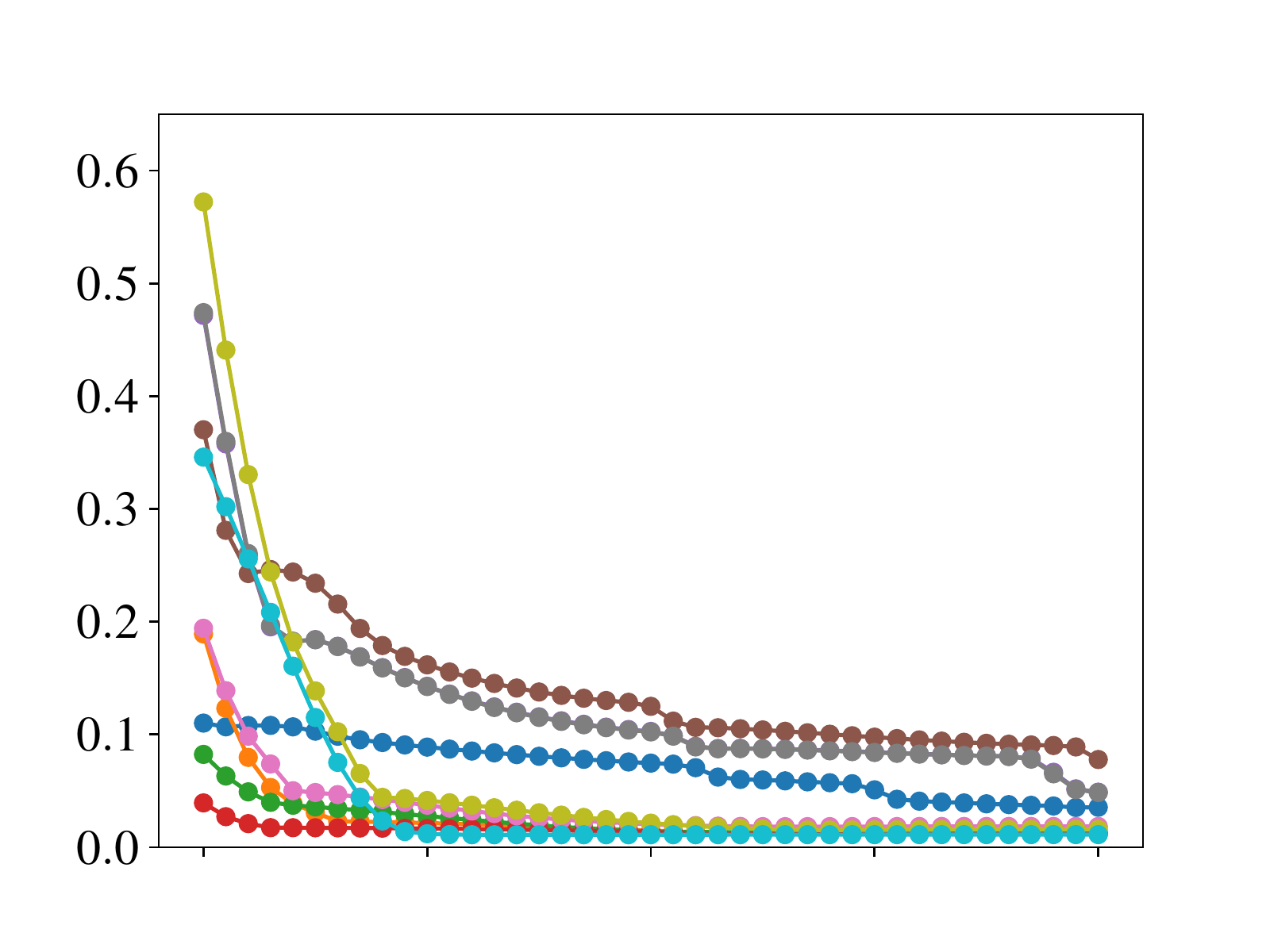}\vspace{-0.1cm}
    \end{subfigure}
    \begin{subfigure}[b]{0.32\textwidth}
    \centering
        \includegraphics[trim={52 25 30 30}, clip, scale=0.32]{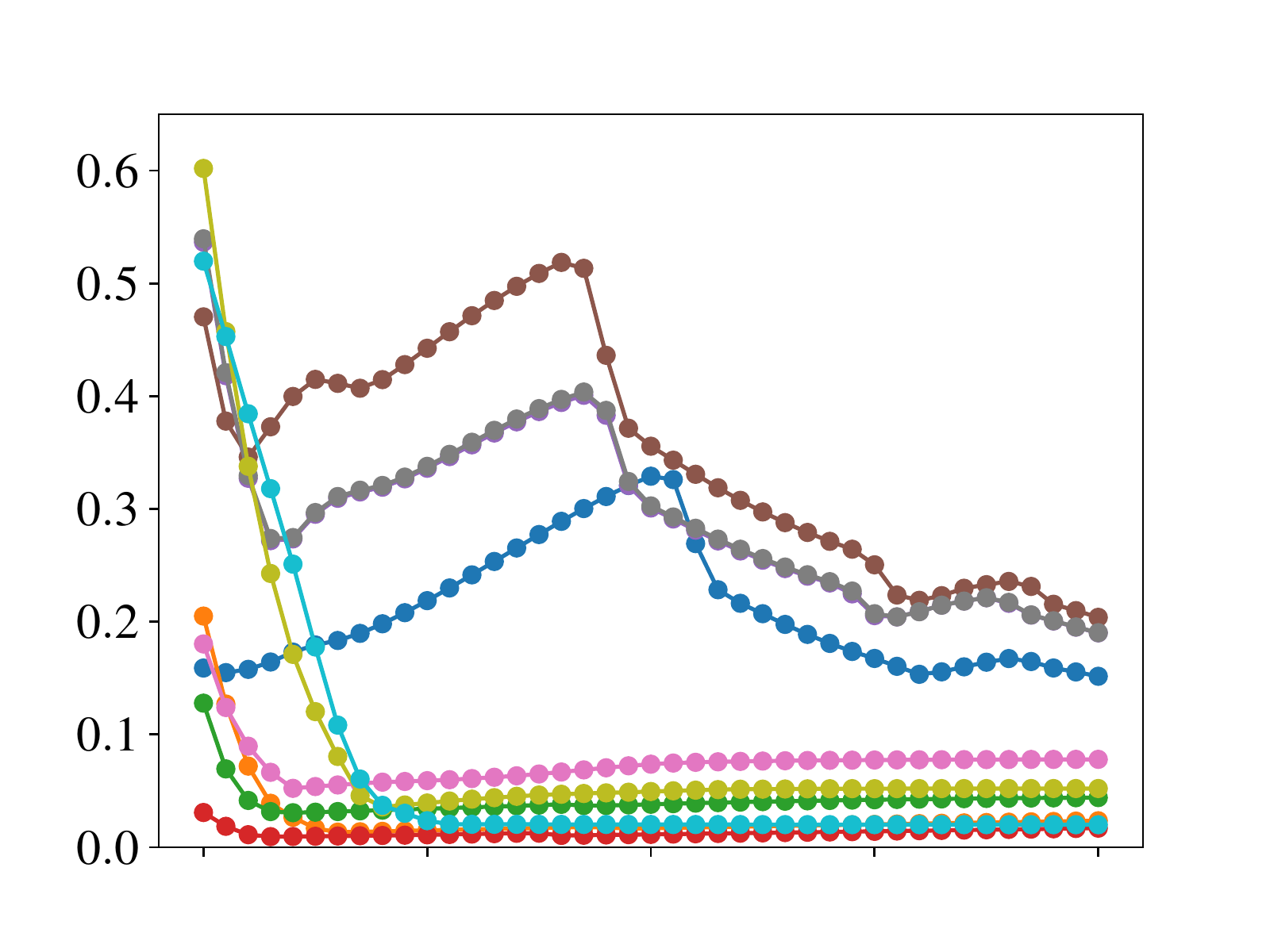}\vspace{-0.1cm}
    \end{subfigure}
    
    \medskip
    
     \begin{subfigure}[b]{0.32\textwidth}
    \centering
        \includegraphics[trim={15 25 30 30}, clip, scale=0.32]{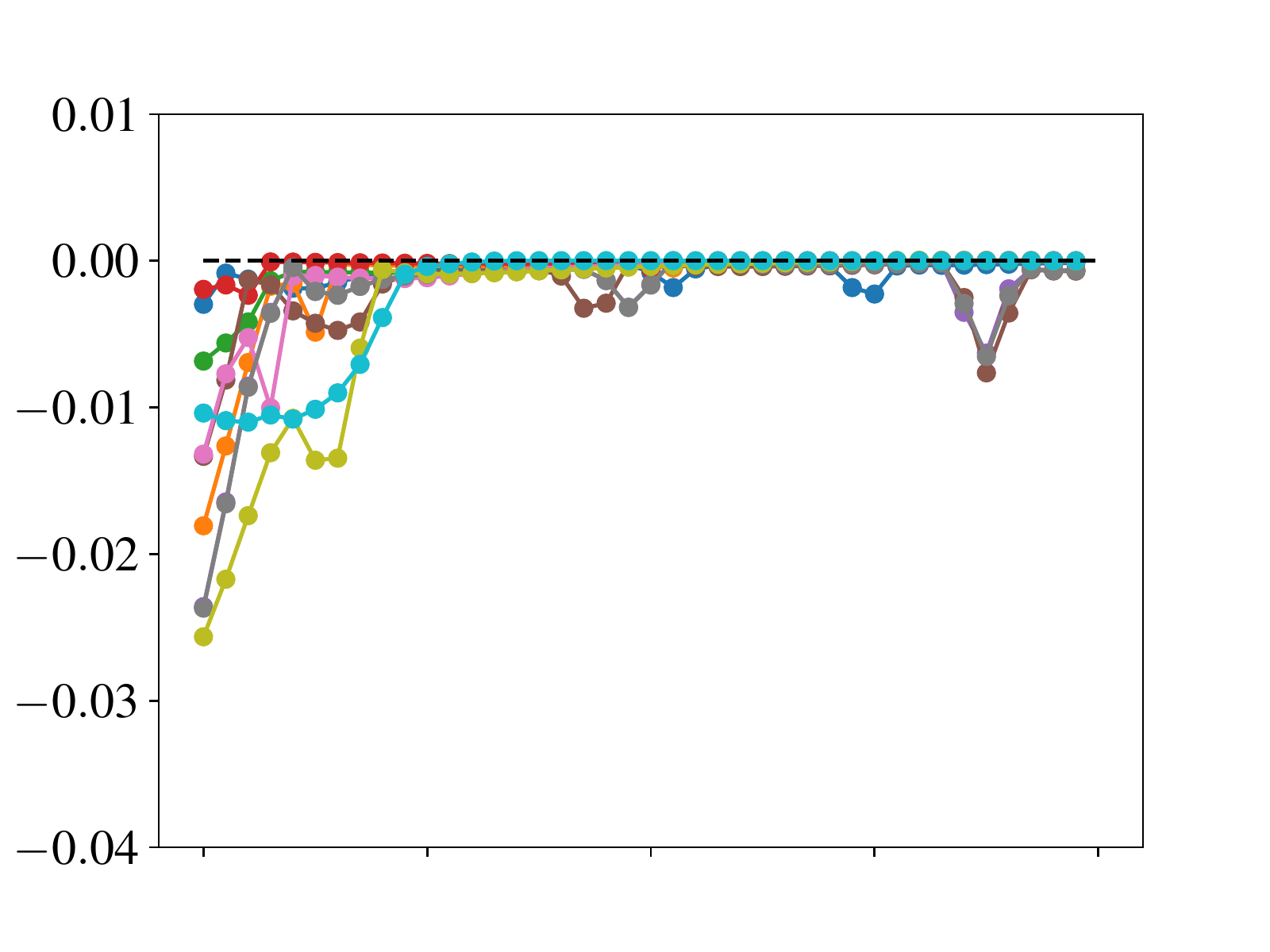}\vspace{-0.1cm}
        \caption*{Iteration 1}
    \end{subfigure}
    \begin{subfigure}[b]{0.32\textwidth}
    \centering
        \includegraphics[trim={52 25 30 30}, clip, scale=0.32]{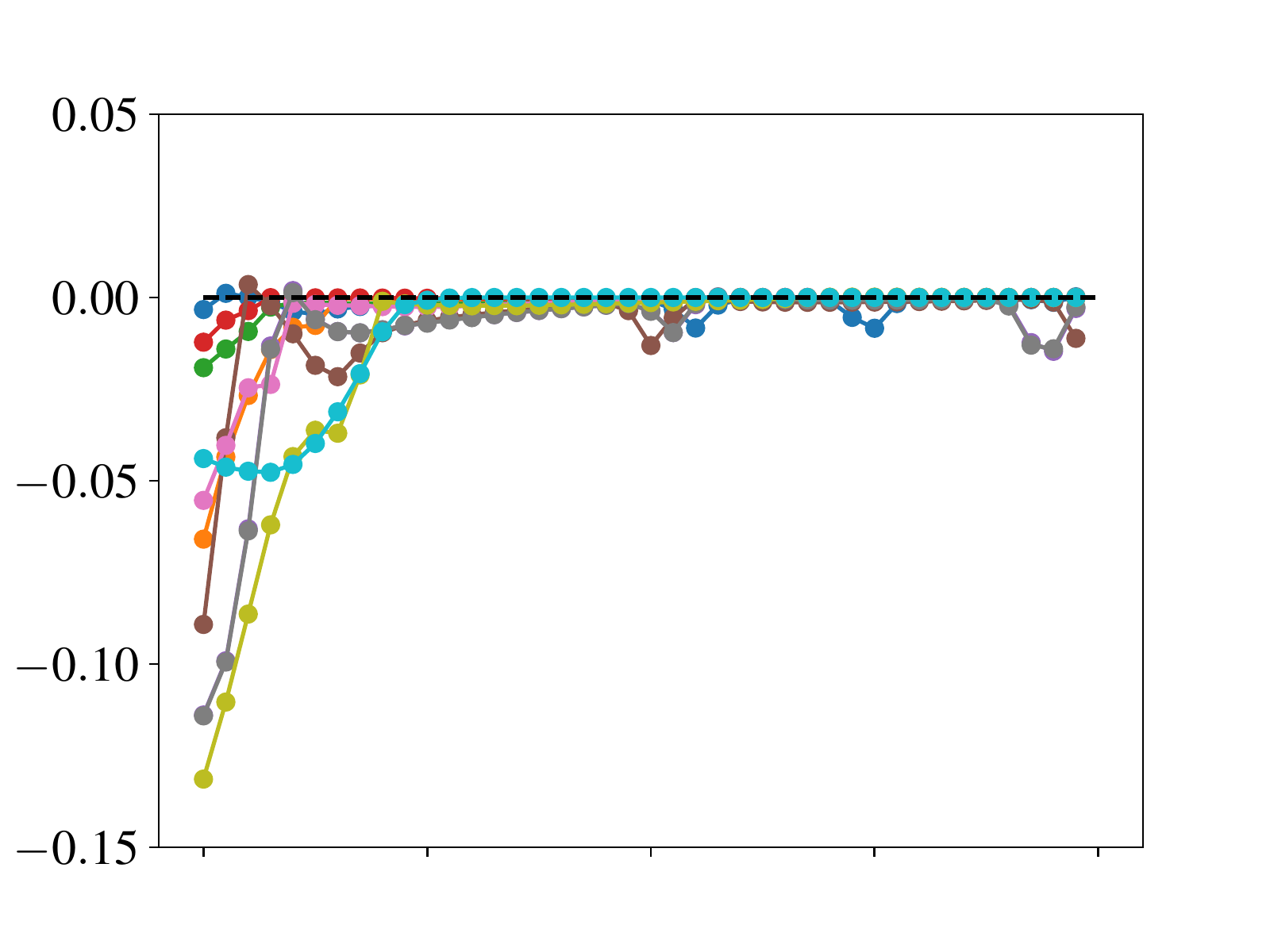}\vspace{-0.1cm}
        \caption*{Iteration 2 (best)}
    \end{subfigure}
    \begin{subfigure}[b]{0.32\textwidth}
        \centering
        \includegraphics[trim={52 25 30 30}, clip, scale=0.32]{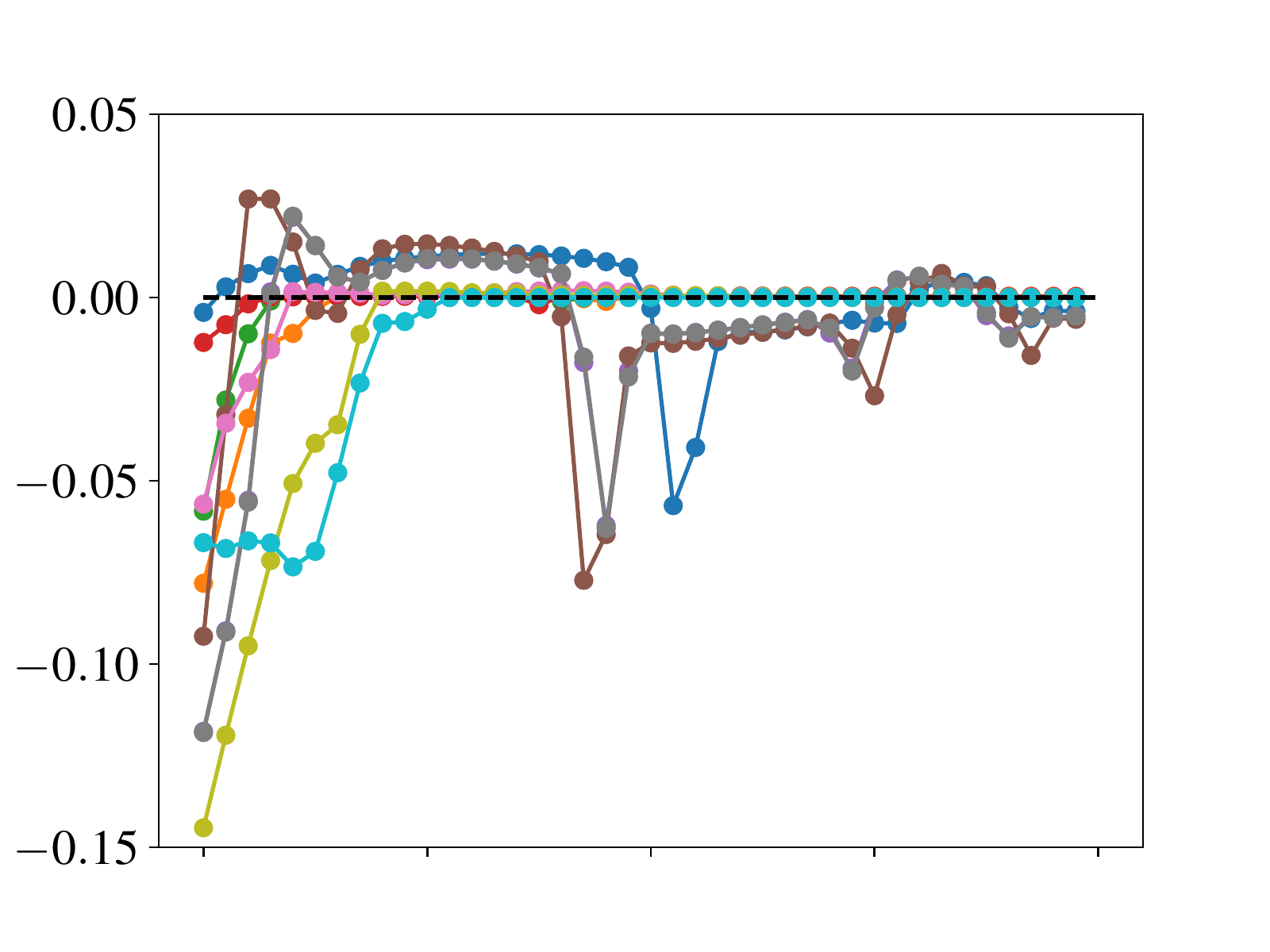}\vspace{-0.1cm}
        \caption*{Iteration 3}
    \end{subfigure}
    
    \caption{{\bf Car kinematics: Testing Neural Lyapunov MPC obtained from training on surrogate model.} For each iteration, we show the trajectories obtained through our Neural Lyapunov MPC while using the resulting Lyapunov function and the MPC parameter selected from the line-search. {\bf Top}: The Lyapunov function at $\phi=0$ with trajectories for $40$ steps at each iteration. {\bf Middle}: The evaluated Lyapunov function. {\bf Bottom}: The Lyapunov function time difference.}
    \label{fig:car_alternate_learning_surrogate_trajectories}
\end{figure*}

\begin{figure*}
    \captionsetup[subfigure]{justification=centering}
    \centering
    
    \begin{subfigure}[b]{0.32\textwidth}
    \centering
        \includegraphics[scale=0.3]{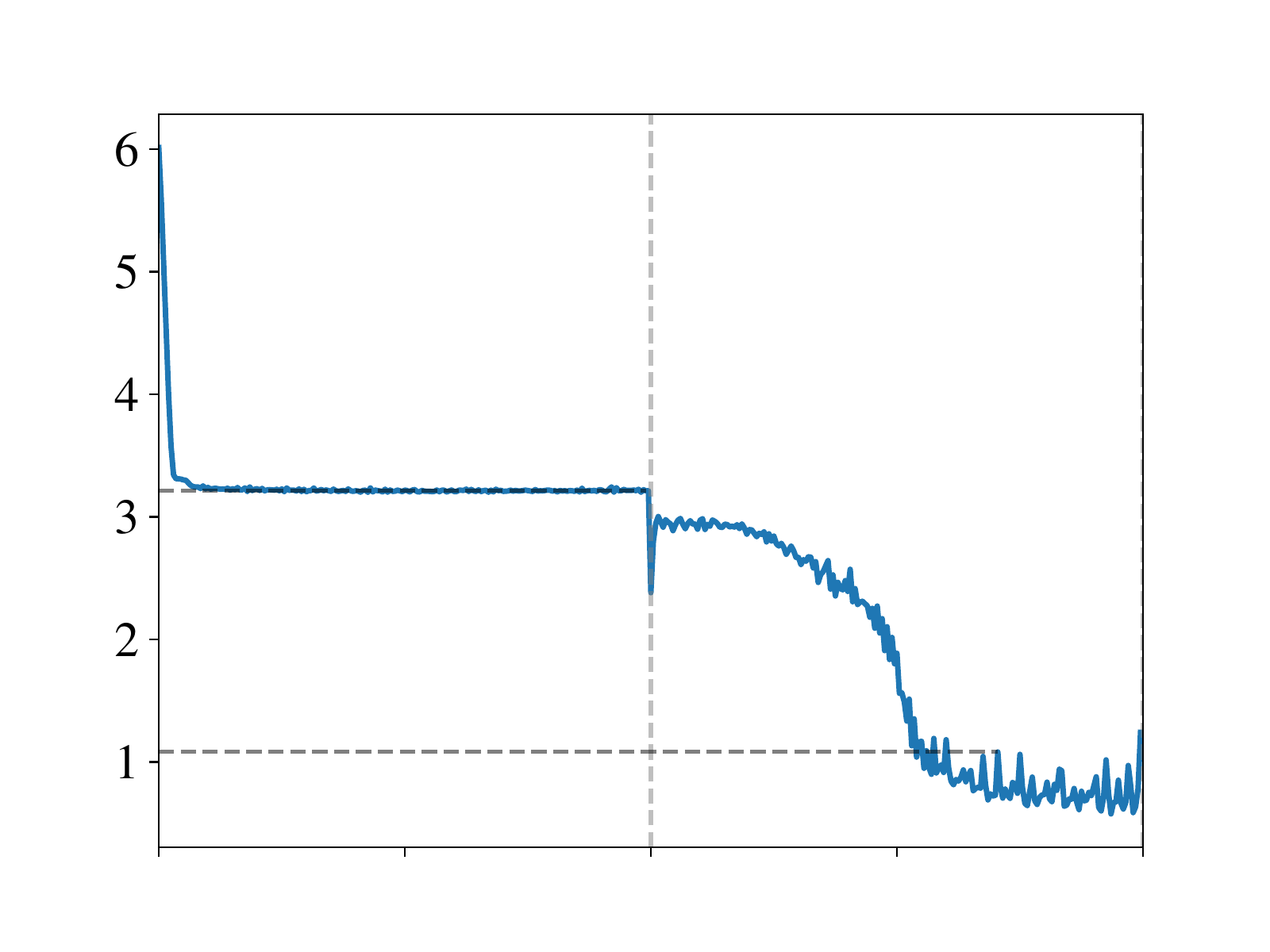}\vspace{-0.1cm}
        \caption{Lyapunov Loss \\($\log(1+x)$)}
    \end{subfigure}
    \begin{subfigure}[b]{0.32\textwidth}
        \centering
        \includegraphics[scale=0.3]{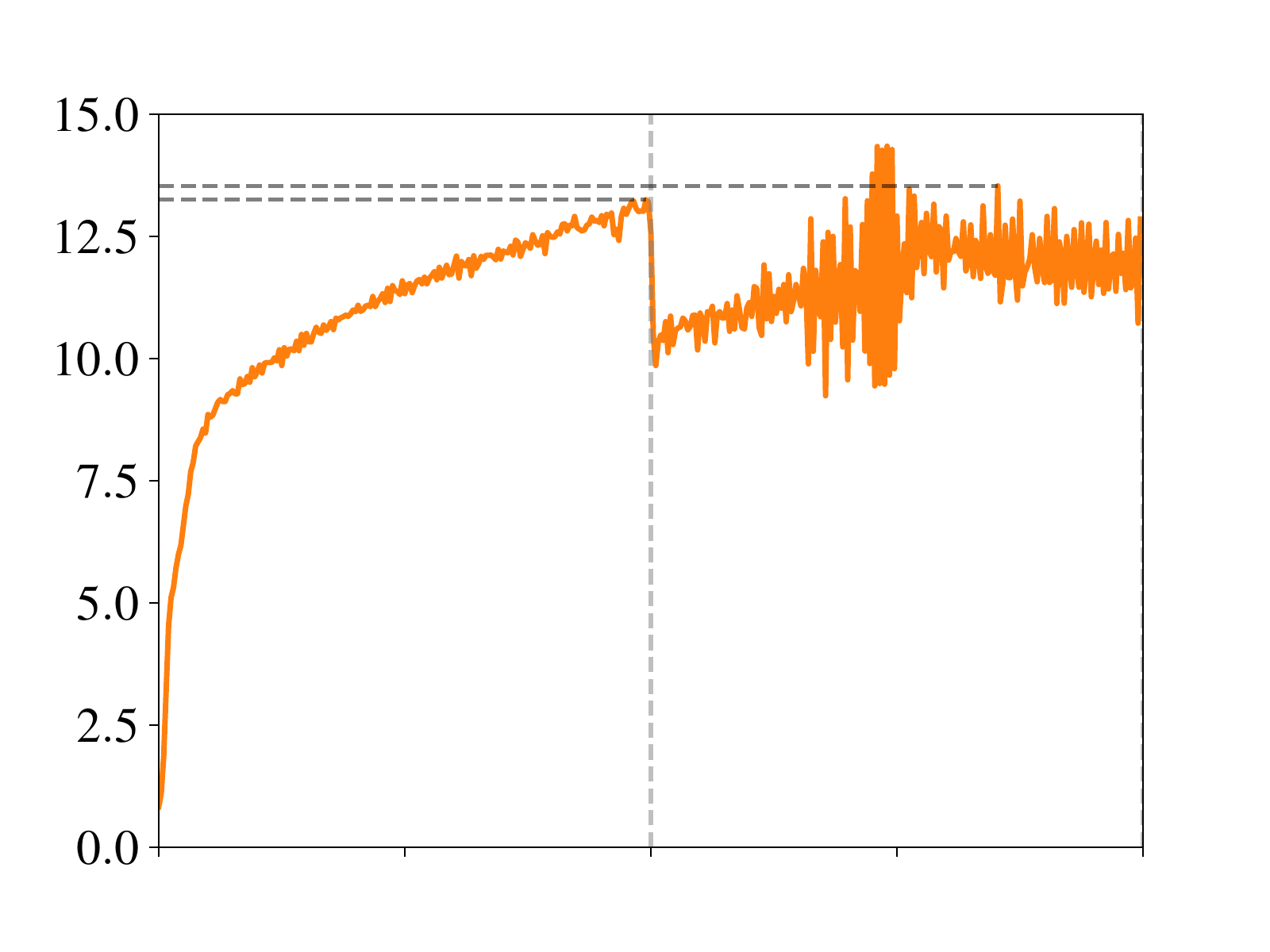}\vspace{-0.1cm}
        \caption{Verified Points \\(\%)}
    \end{subfigure}
    \begin{subfigure}[b]{0.32\textwidth}
        \centering
        \includegraphics[scale=0.3]{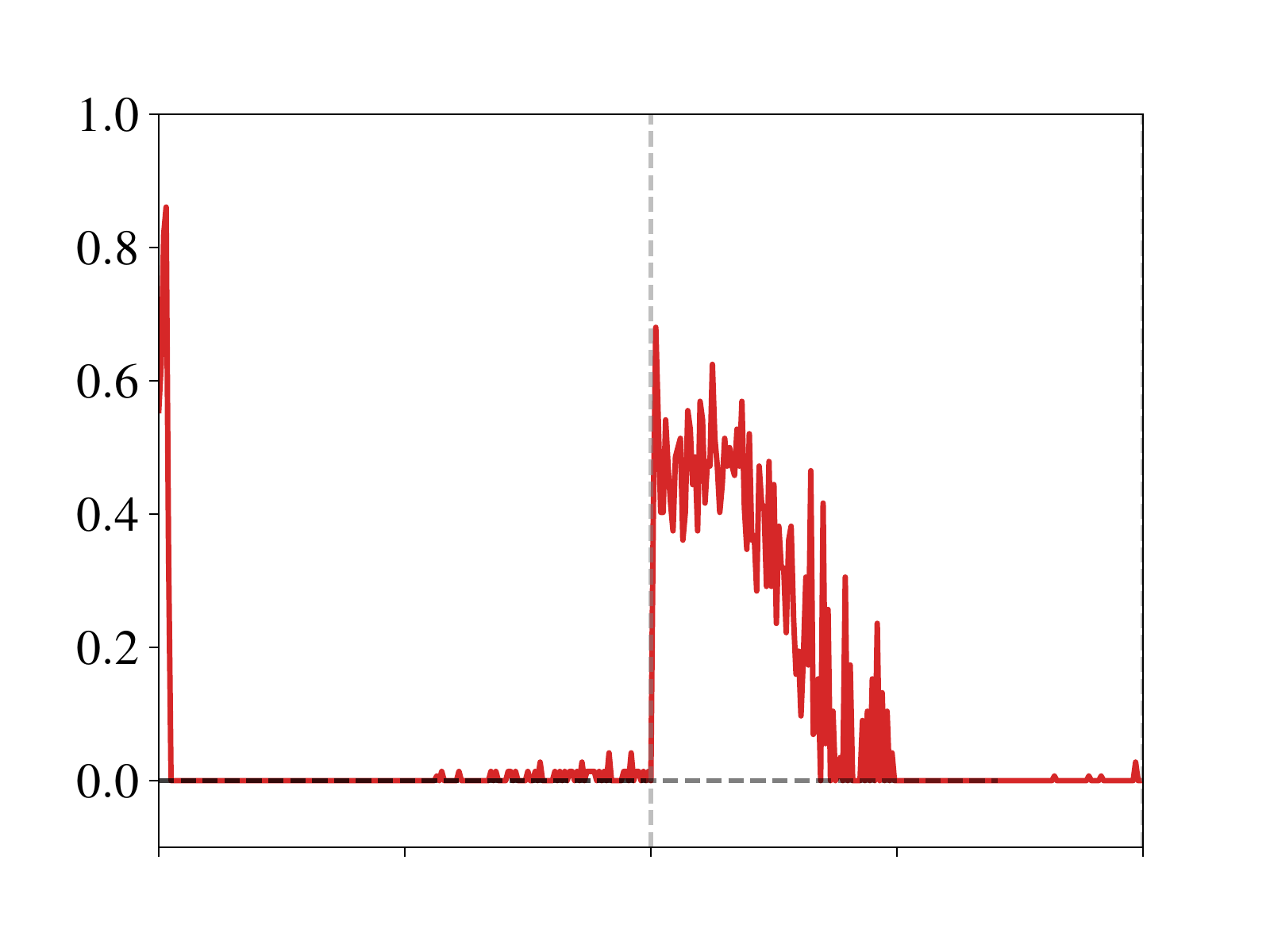}\vspace{-0.1cm}
        \caption{Not Verified Points \\(\%)}
    \end{subfigure}
    
    \medskip \medskip \medskip
    
    \begin{subfigure}[b]{0.45\textwidth}
        \centering
        \includegraphics[trim={52 35 40 40}, clip, scale=0.35]{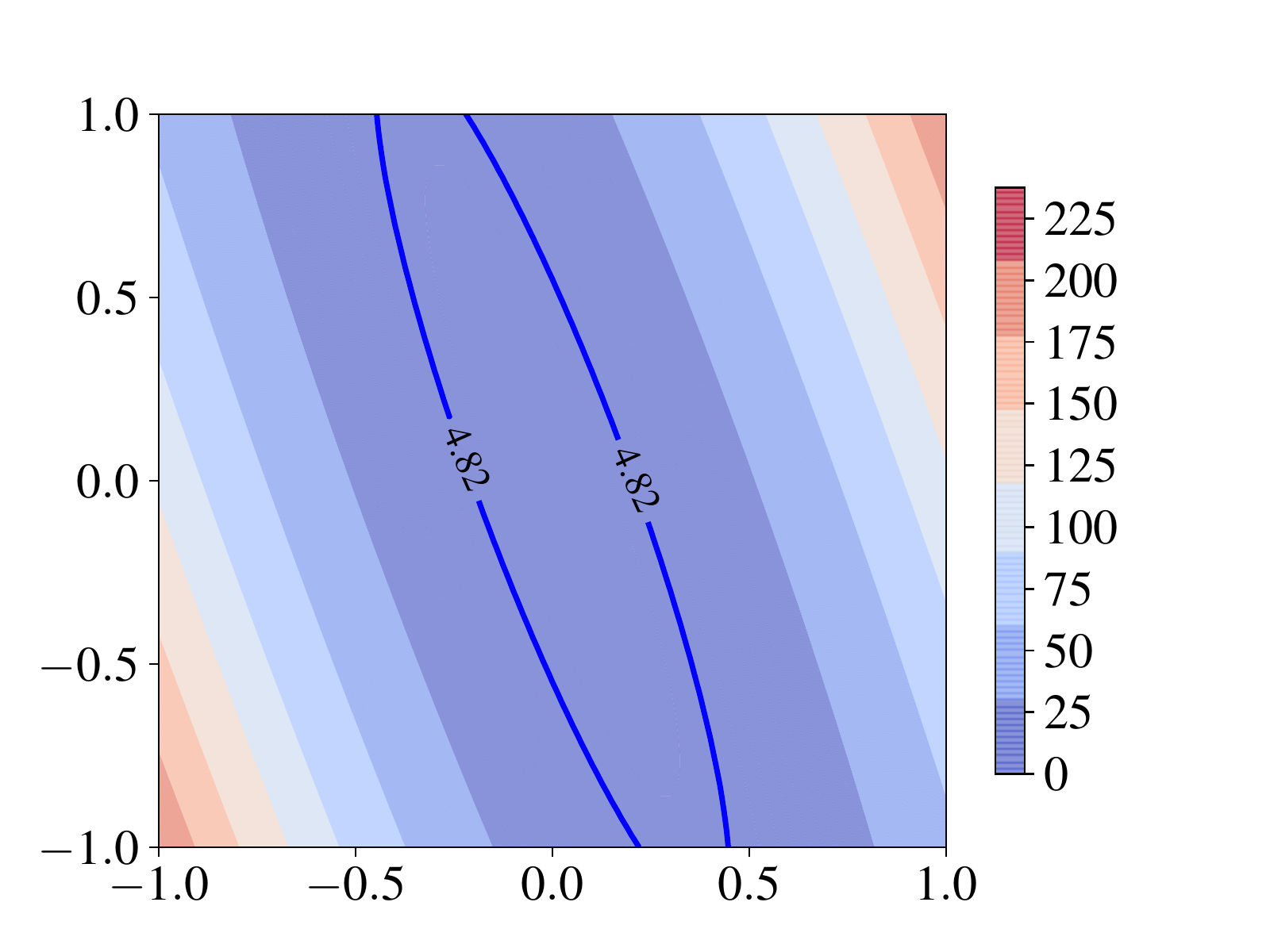}\vspace{-0.1cm}
    \end{subfigure}
    \hspace{-0.1cm}
    \begin{subfigure}[b]{0.35\textwidth}
    \centering
        \includegraphics[trim={52 35 40 40}, clip, scale=0.35]{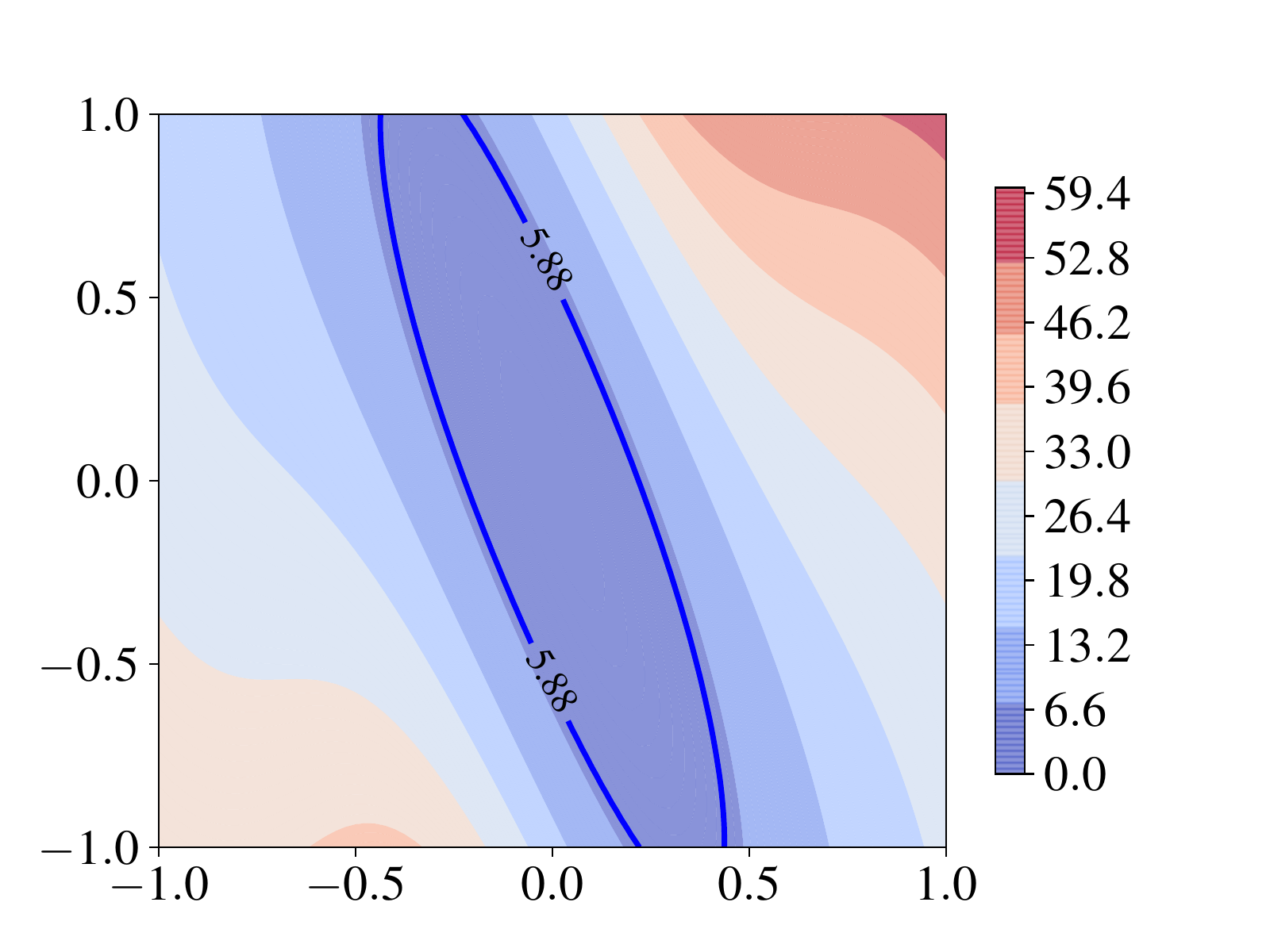}\vspace{-0.1cm}
    \end{subfigure}
    
    \medskip
    
    \begin{subfigure}[b]{0.45\textwidth}
        \centering
        \includegraphics[trim={0 40 0 0}, clip, scale=0.35]{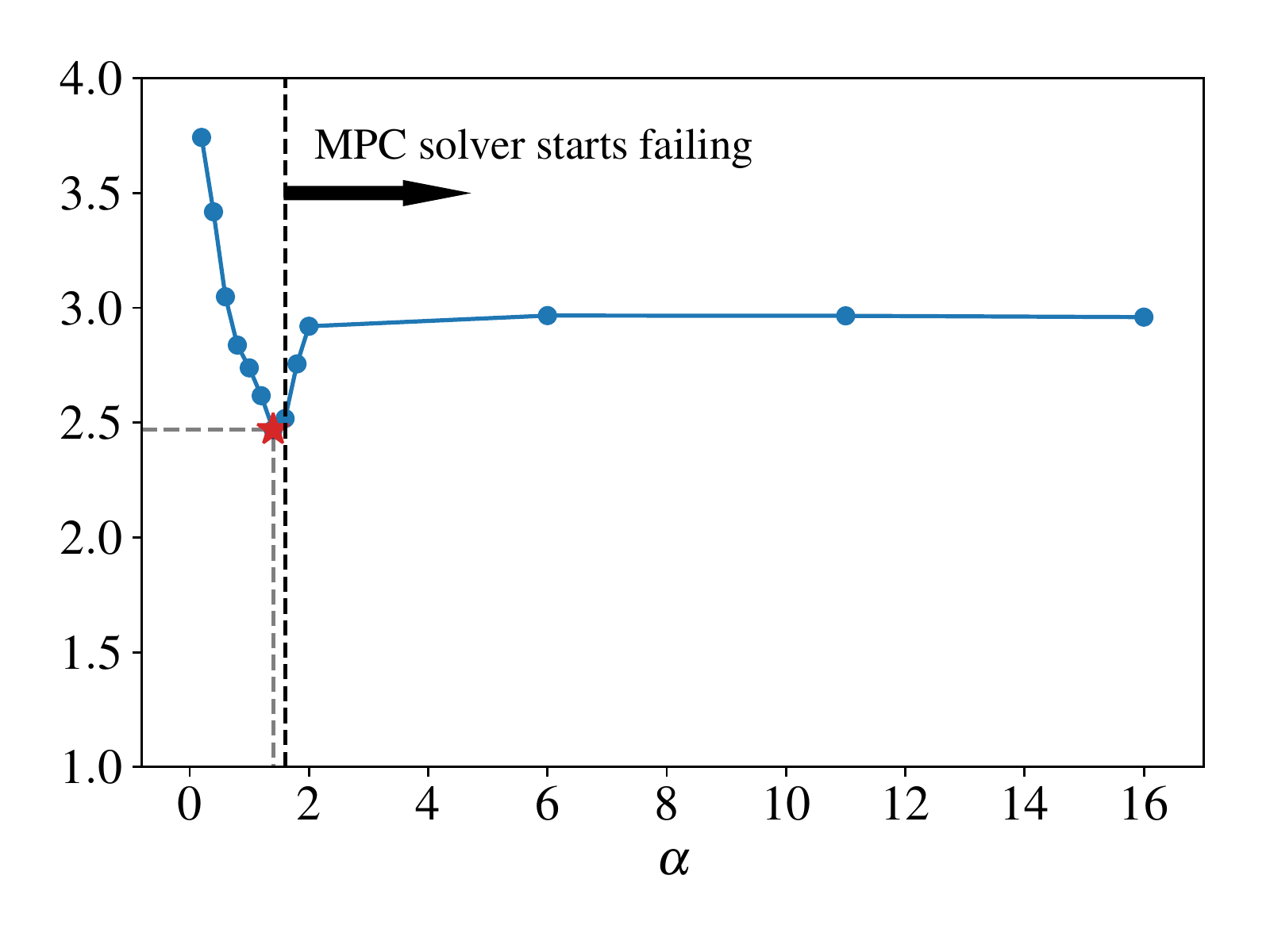}\vspace{-0.1cm}
        \caption*{Iteration 1}
    \end{subfigure}
     \begin{subfigure}[b]{0.35\textwidth}
        \centering
        \includegraphics[trim={45 40 0 0}, clip, scale=0.35]{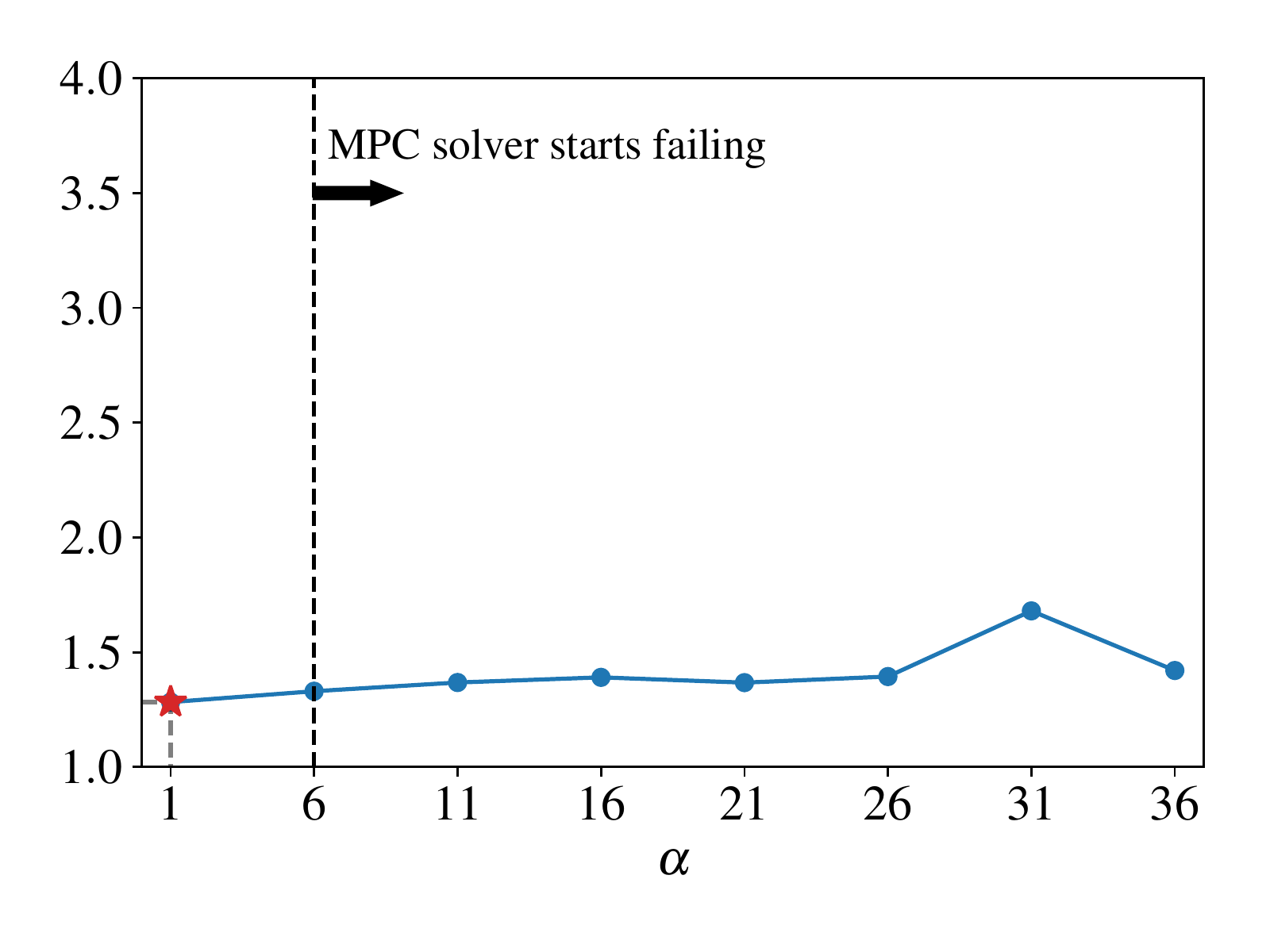}\vspace{-0.1cm}
        \caption*{Iteration 2 (best)}
    \end{subfigure}
    
    \caption{{\bf Inverted Pendulum: Alternate learning on surrogate model.} After every $N_V=200$ epochs of Lyapunov learning, the learned Lyapunov function is used to tune the MPC parameters. Unlike the vehicle kinematics example, we do not reinitialize $V$ between the iterations. \textbf{Top:} The training curves for Lyapunov function. Vertical lines separate iterations.  \textbf{Middle:}  The resulting Lyapunov function $V$ with the best performance. \textbf{Bottom:} Line-search for the MPC parameter $\alpha$ to minimize the Lyapunov loss~(\protect\ref{eq:lyapunov_loss}) with $V$ as terminal cost. The loss is plotted on the y-axis in a $\log(1+x)$ scale. The point marked in red is the parameter which minimizes the loss.}
    \label{fig:pendulum_alternate_learning}
\end{figure*}

\begin{figure*}[t]
    \captionsetup[subfigure]{justification=centering}
    \centering
    
    \begin{subfigure}[b]{0.4\textwidth}
        \centering
        \includegraphics[trim={15 35 30 40}, clip, scale=0.4]{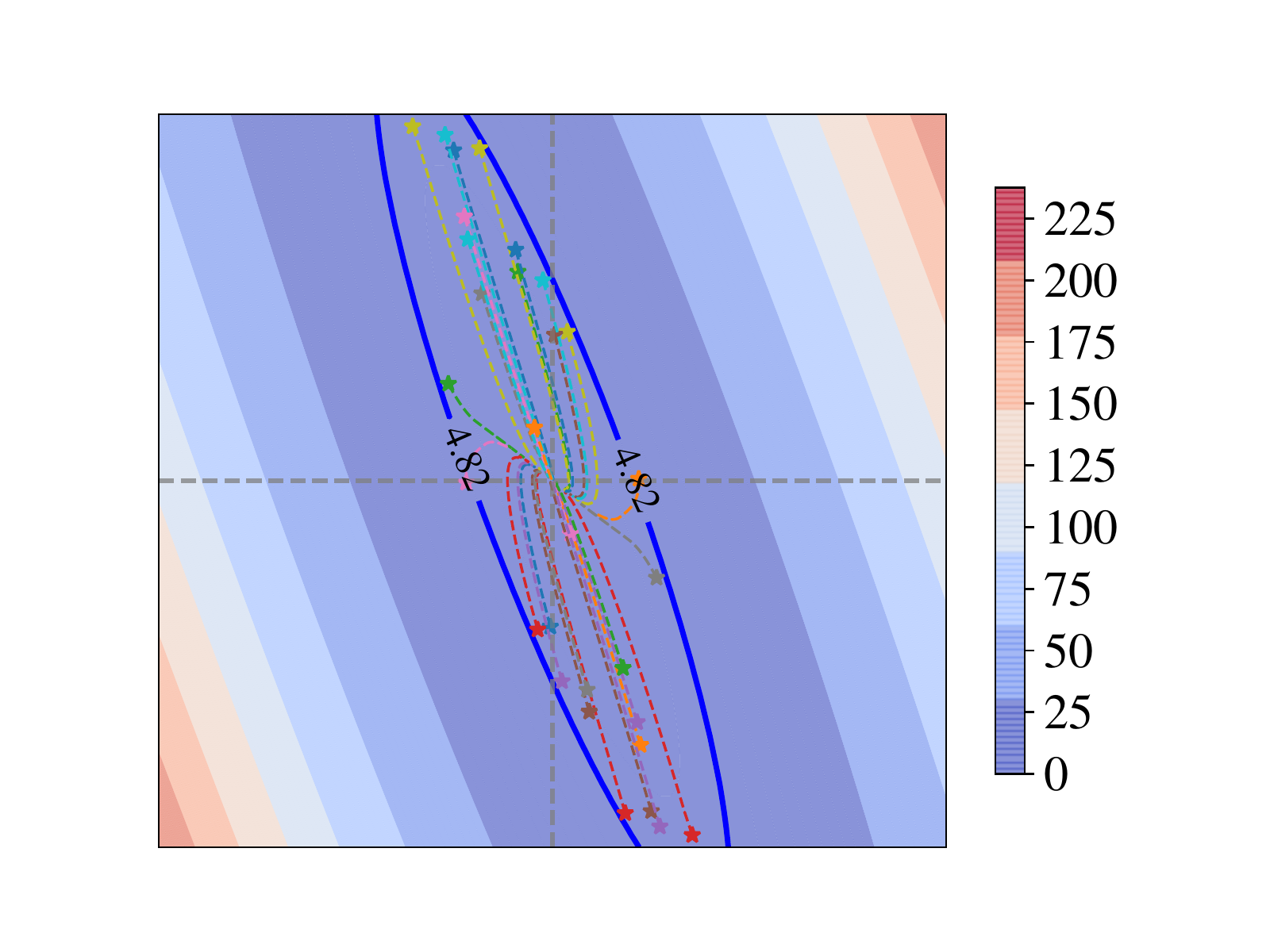}\vspace{-0.1cm}
    \end{subfigure}
    \begin{subfigure}[b]{0.4\textwidth}
    \centering
        \includegraphics[trim={52 35 30 40}, clip, scale=0.4]{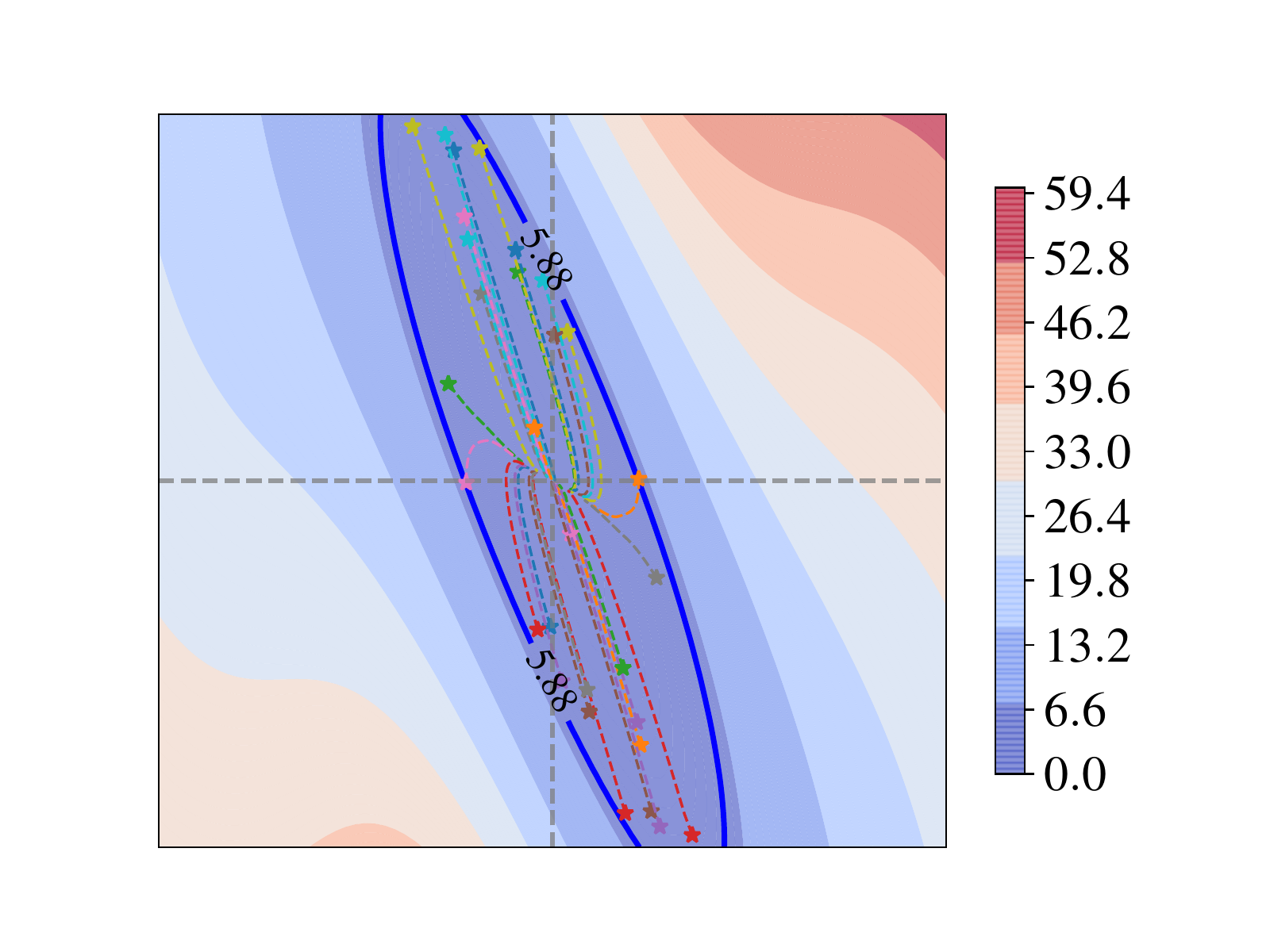}\vspace{-0.1cm}
    \end{subfigure}
    
    \medskip
    
    \begin{subfigure}[b]{0.4\textwidth}
        \centering
        \includegraphics[trim={15 25 30 30}, clip, scale=0.4]{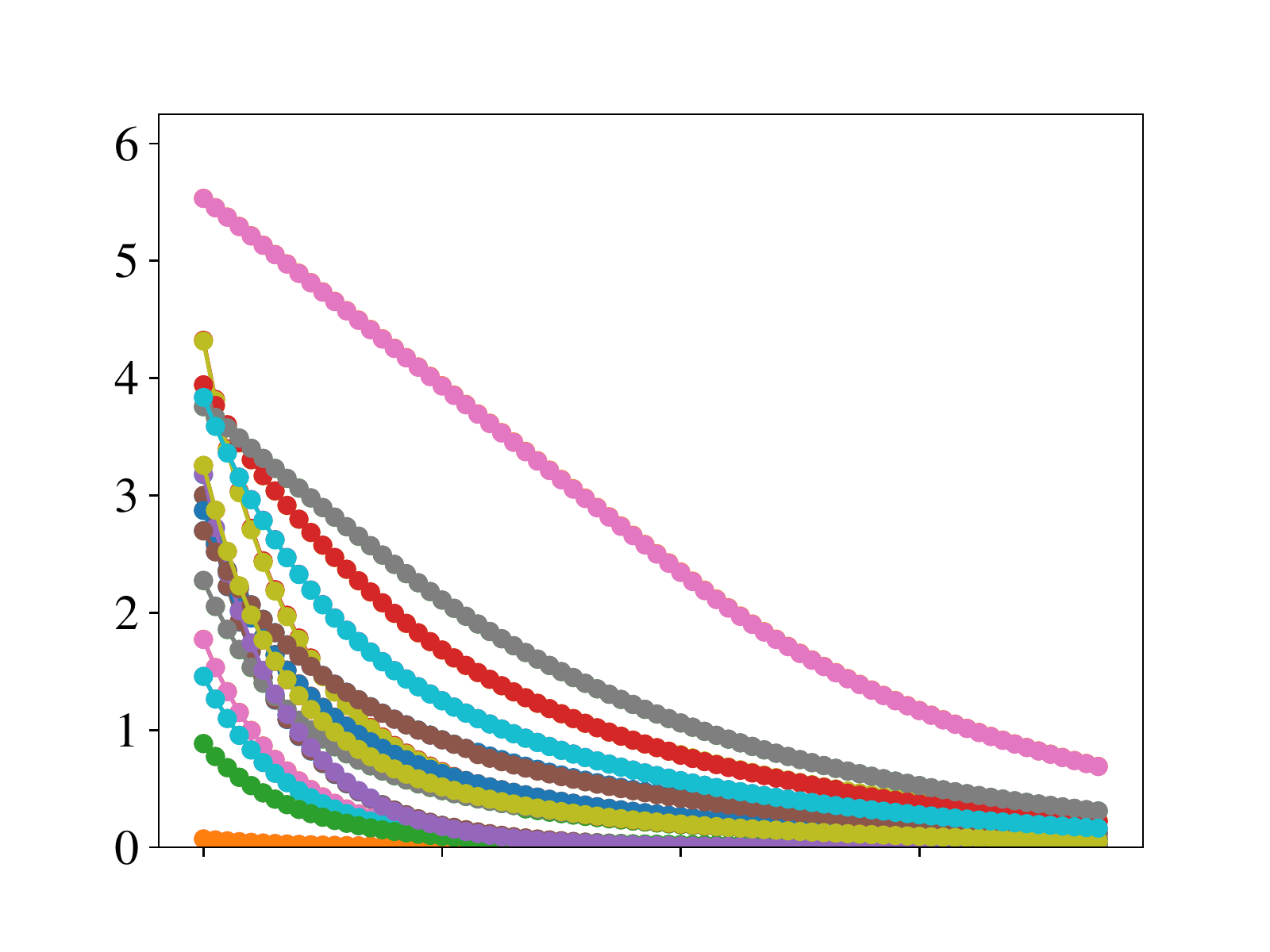}\vspace{-0.1cm}
    \end{subfigure}
    \begin{subfigure}[b]{0.4\textwidth}
    \centering
        \includegraphics[trim={52 25 30 30}, clip, scale=0.4]{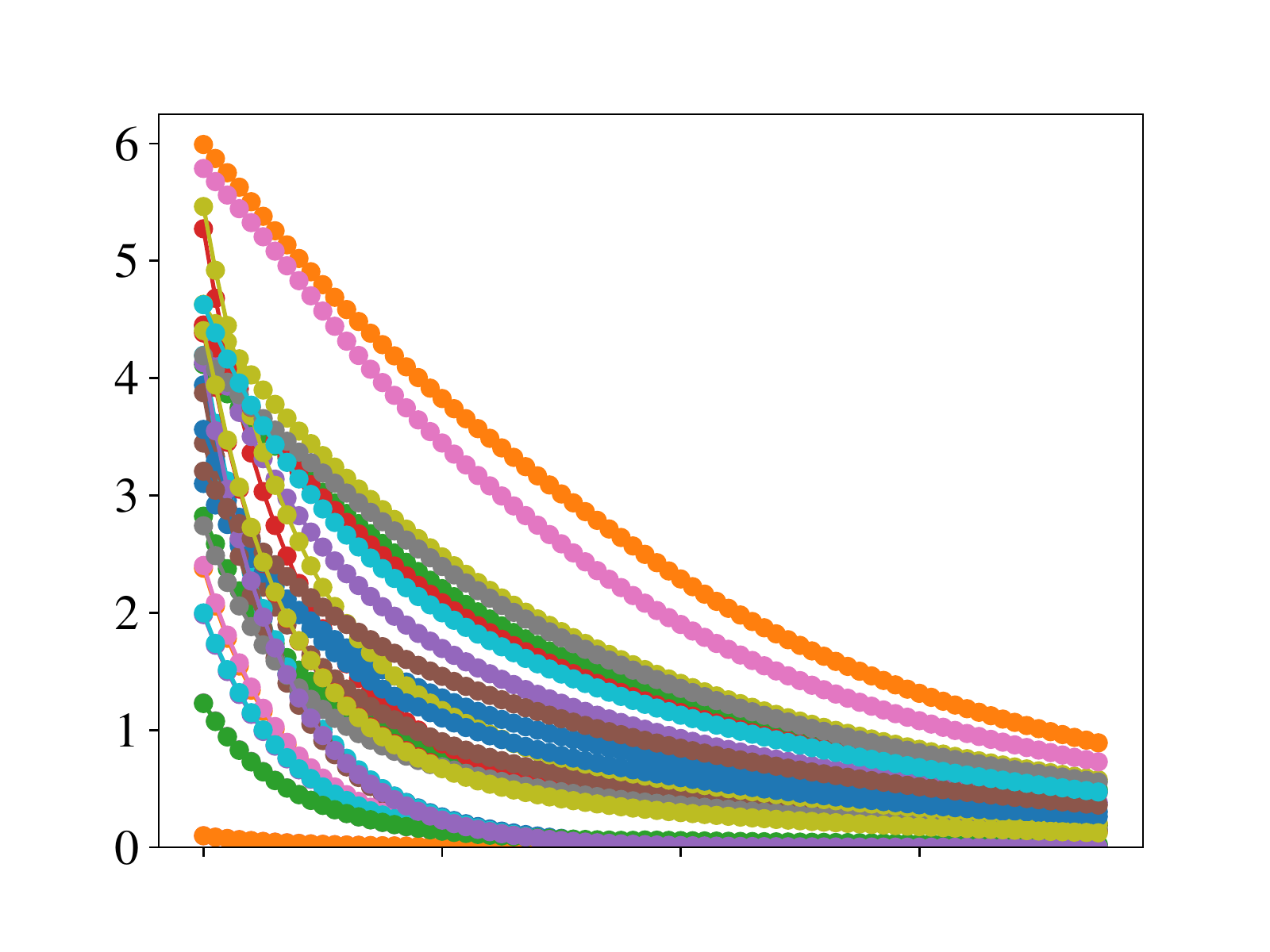}\vspace{-0.1cm}
    \end{subfigure}
    
    \medskip
    
     \begin{subfigure}[b]{0.4\textwidth}
    \centering
        \includegraphics[trim={15 25 30 30}, clip, scale=0.4]{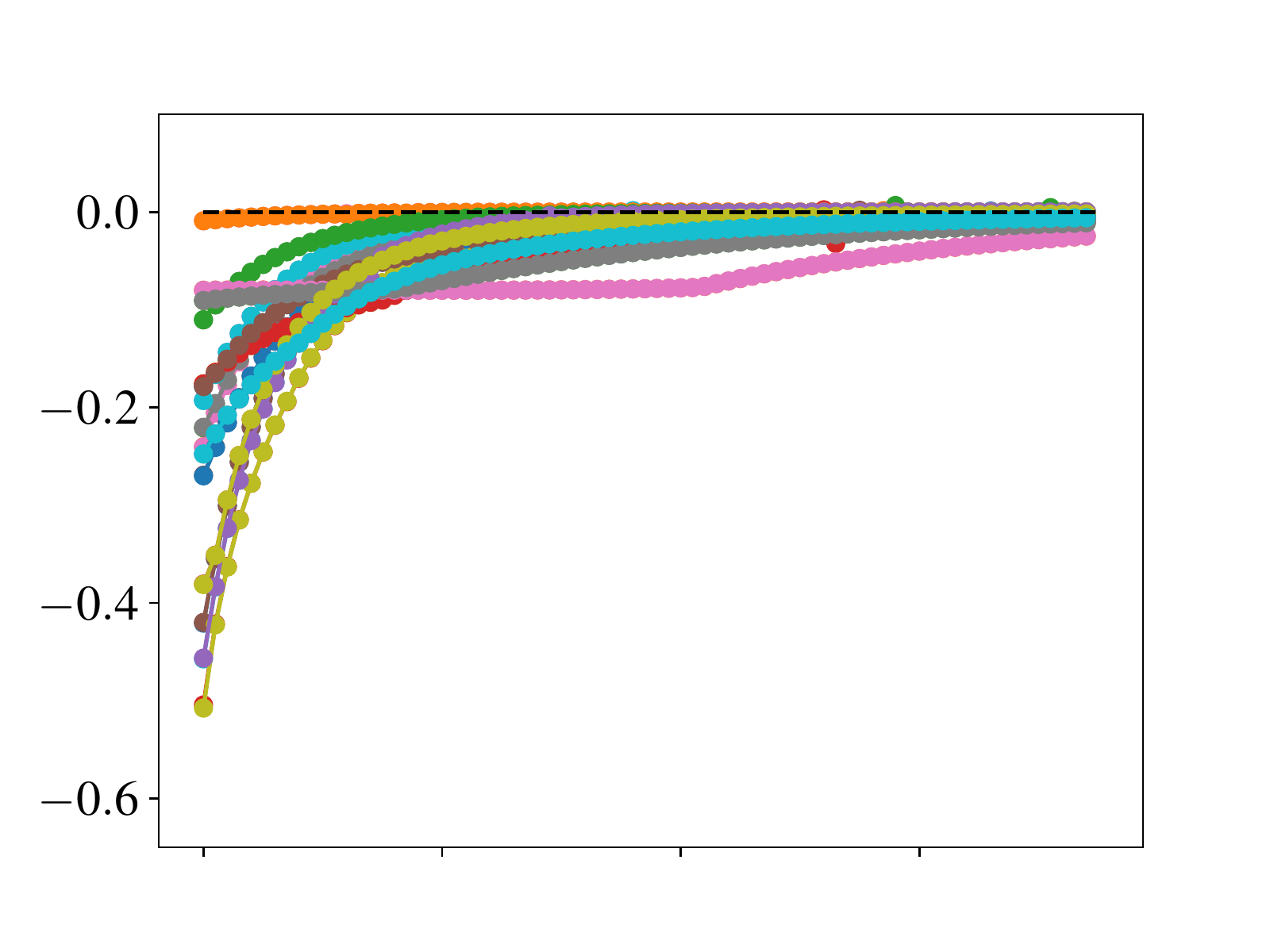}\vspace{-0.1cm}
        \caption*{Iteration 1}
    \end{subfigure}
    \begin{subfigure}[b]{0.4\textwidth}
    \centering
        \includegraphics[trim={52 25 30 30}, clip, scale=0.4]{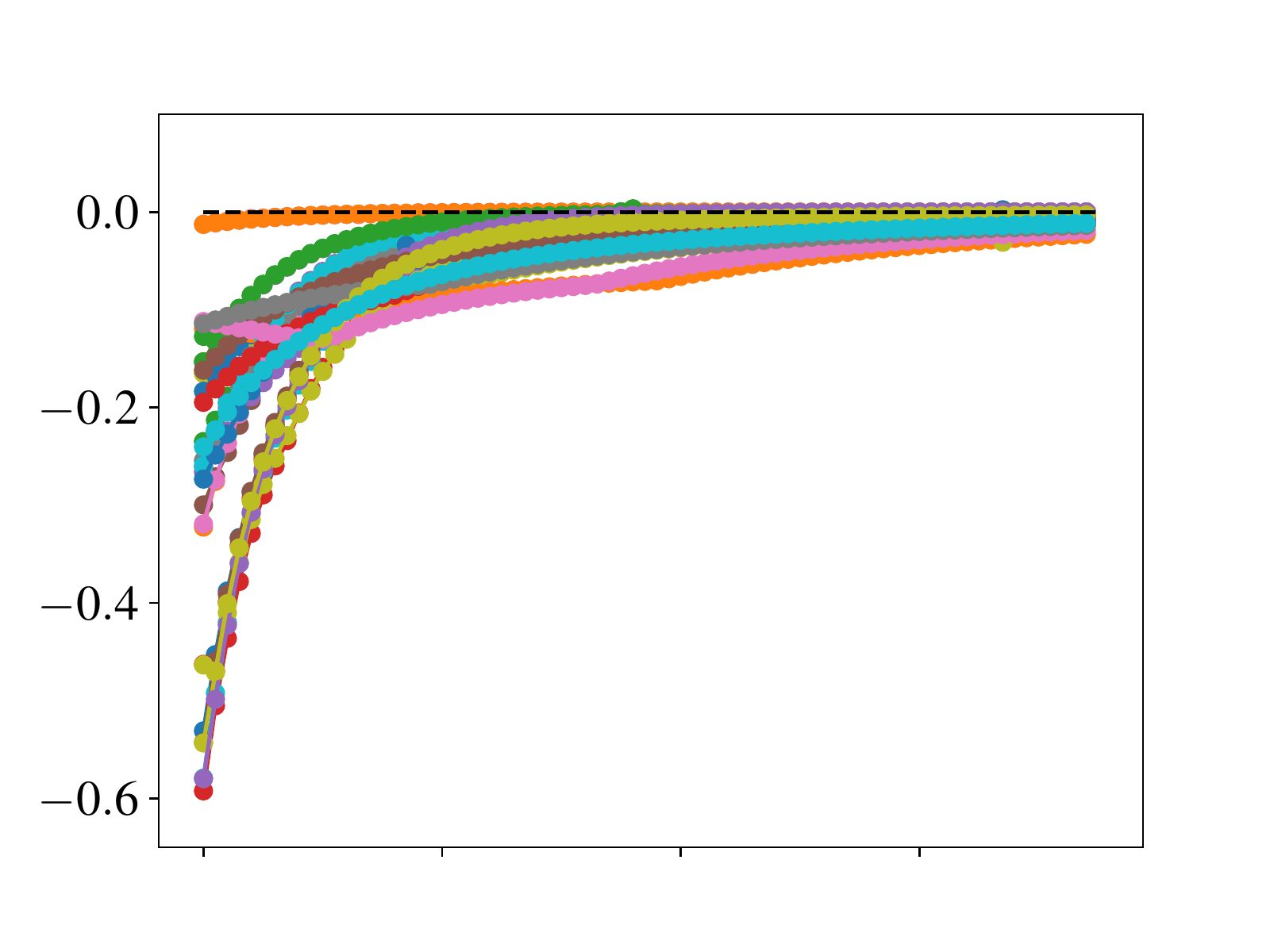}\vspace{-0.1cm}
        \caption*{Iteration 2 (best)}
    \end{subfigure}
    
    \caption{{\bf Inverted Pendulum: Testing Neural Lyapunov MPC obtained from training on nominal model over iterations.} For each iteration, we show the trajectories obtained through our Neural Lyapunov MPC while using the resulting Lyapunov function and the MPC parameter selected from the line-search. The initial states are sampled inside the safe level-set using rejection sampling. {\bf Top}: The Lyapunov function with trajectories for $80$ steps at each iteration. {\bf Middle}: The evaluated Lyapunov function. {\bf Bottom}: The Lyapunov function time difference.}
    \label{fig:pendulum_alternate_learning_nominal_trajectories}
\end{figure*}

\begin{figure*}[h!t]
    \captionsetup[subfigure]{justification=centering}
    \centering
    \begin{subfigure}[b]{0.19\textwidth}
        \centering
        \includegraphics[trim={10 35 105 40}, clip, scale=0.22]{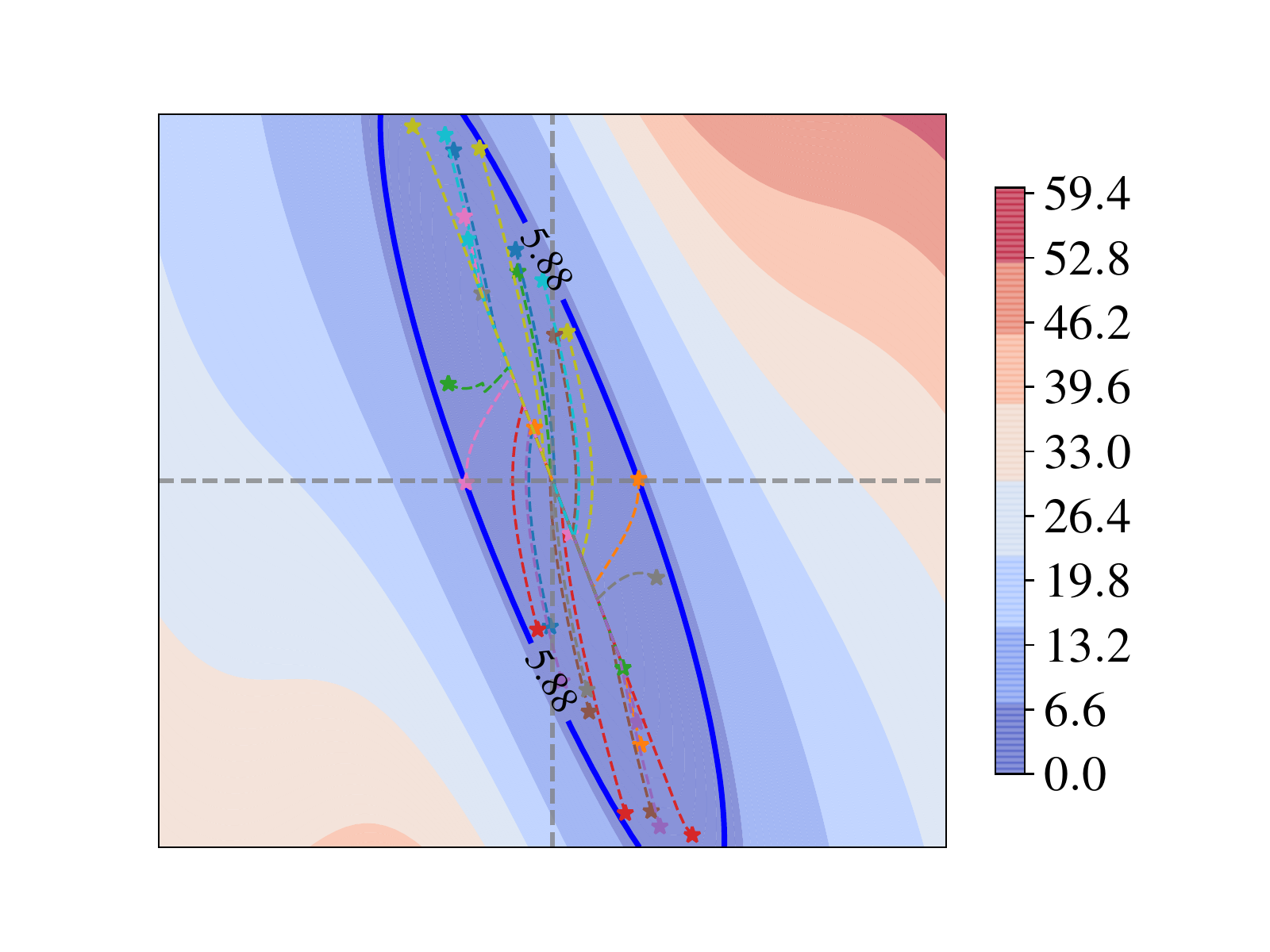}\vspace{-0.1cm}
    \end{subfigure}
    \begin{subfigure}[b]{0.19\textwidth}
        \centering
        \includegraphics[trim={52 35 105 40}, clip, scale=0.22]{img/inv_pendulum/model_nominal_mpc_long_horizon/lyapunov_function.pdf}\vspace{-0.1cm}
    \end{subfigure}
    \begin{subfigure}[b]{0.19\textwidth}
        \centering
        \includegraphics[trim={52 35 105 40}, clip, scale=0.22]{img/inv_pendulum/model_nominal_mpc_short_horizon_lyap/lyapunov_function.pdf}\vspace{-0.1cm}
    \end{subfigure}
    \begin{subfigure}[b]{0.19\textwidth}
        \centering
        \includegraphics[trim={52 35 105  40}, clip, scale=0.22]{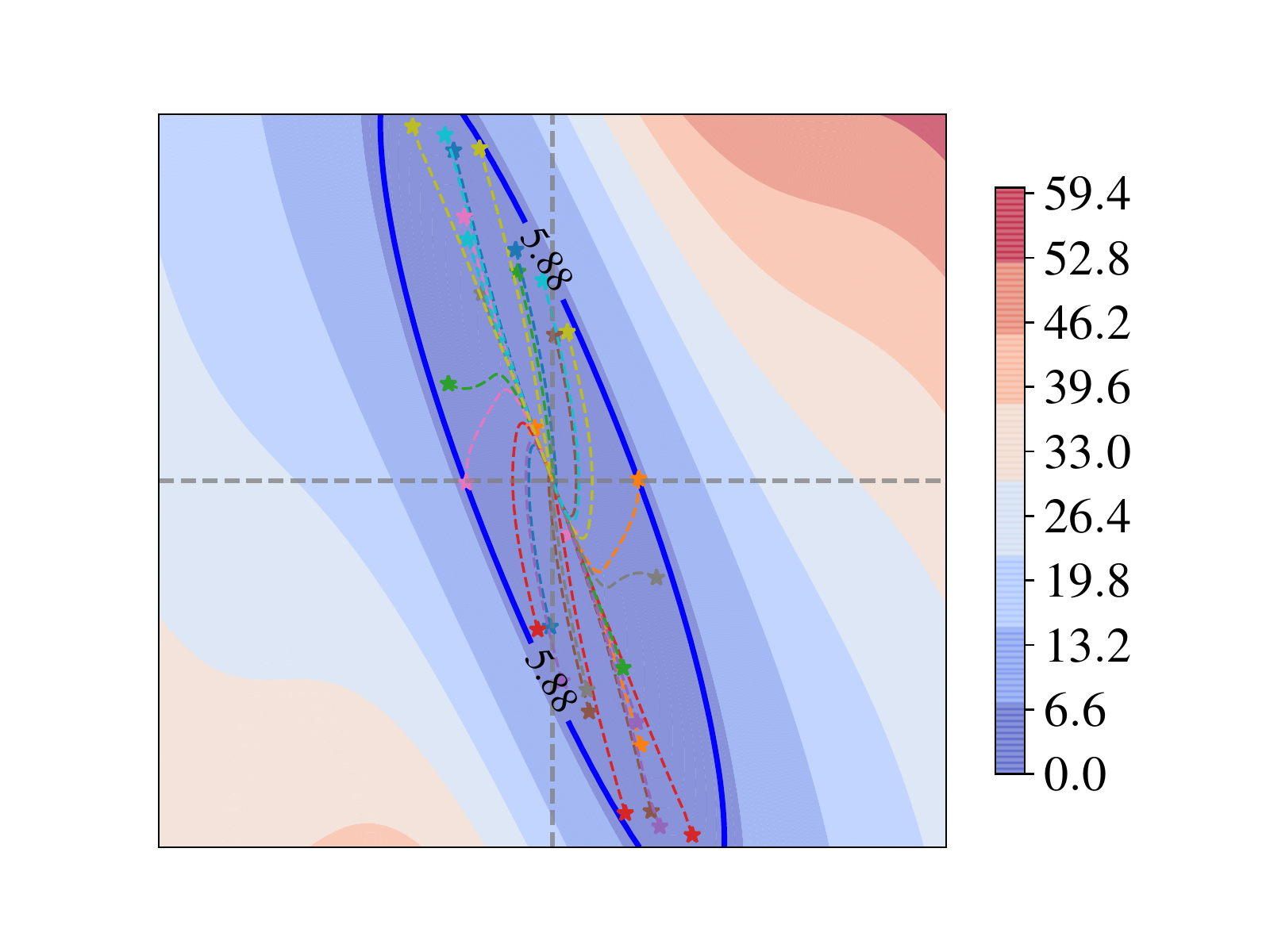}\vspace{-0.1cm}
    \end{subfigure}
        \begin{subfigure}[b]{0.19\textwidth}
        \centering
        \includegraphics[trim={52 35 50 40}, clip, scale=0.22]{img/inv_pendulum/model_surrogate_mpc_short_horizon_lyap/lyapunov_function.pdf}\vspace{-0.1cm}
    \end{subfigure}

    \medskip
    
    \begin{subfigure}[b]{0.2\textwidth}
    \centering
        \includegraphics[trim={15 25 30 40}, clip, scale=0.2]{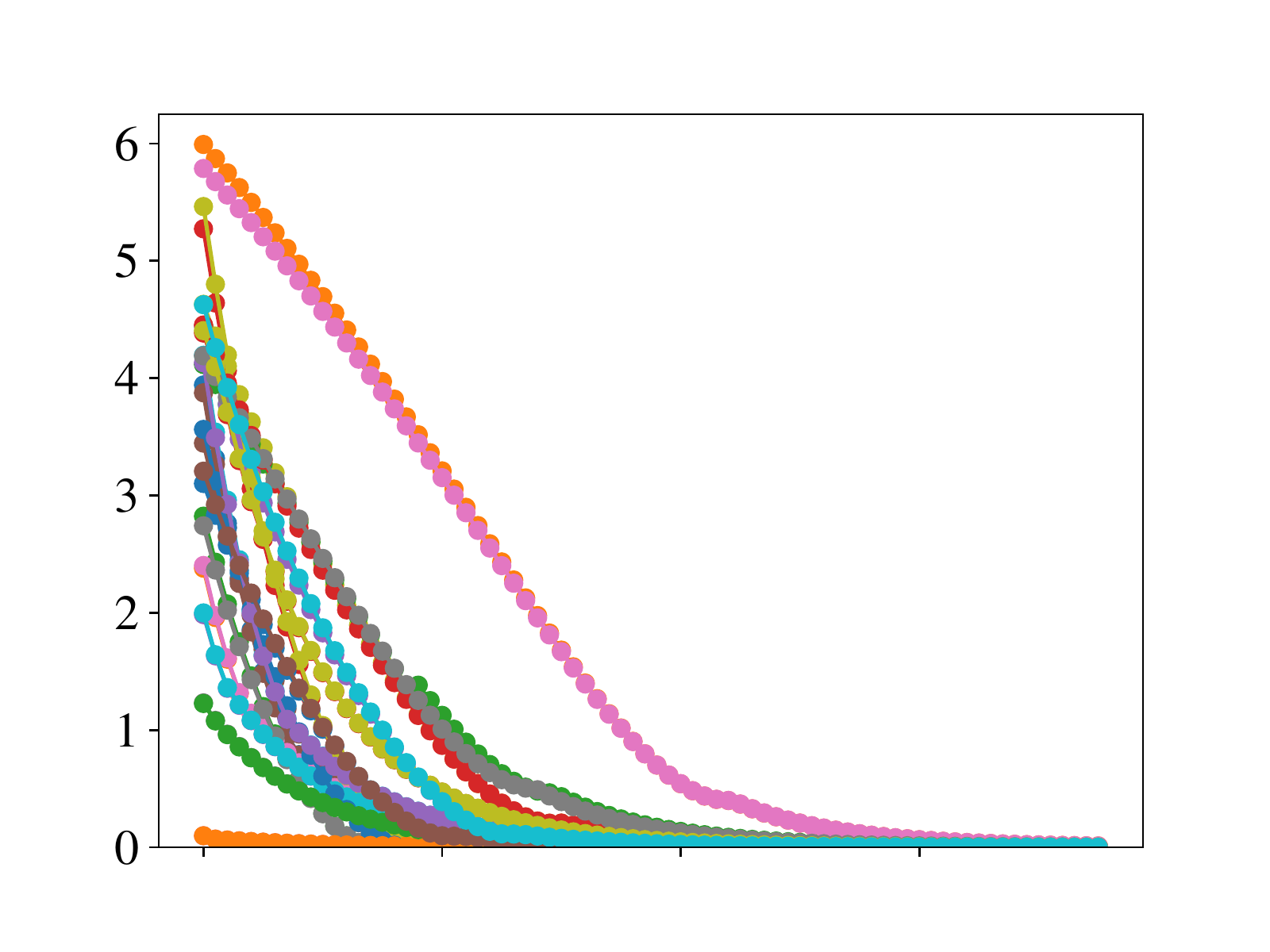}\vspace{-0.1cm}
    \end{subfigure}
    \begin{subfigure}[b]{0.19\textwidth}
    \centering
        \includegraphics[trim={52 25 30 40}, clip, scale=0.2]{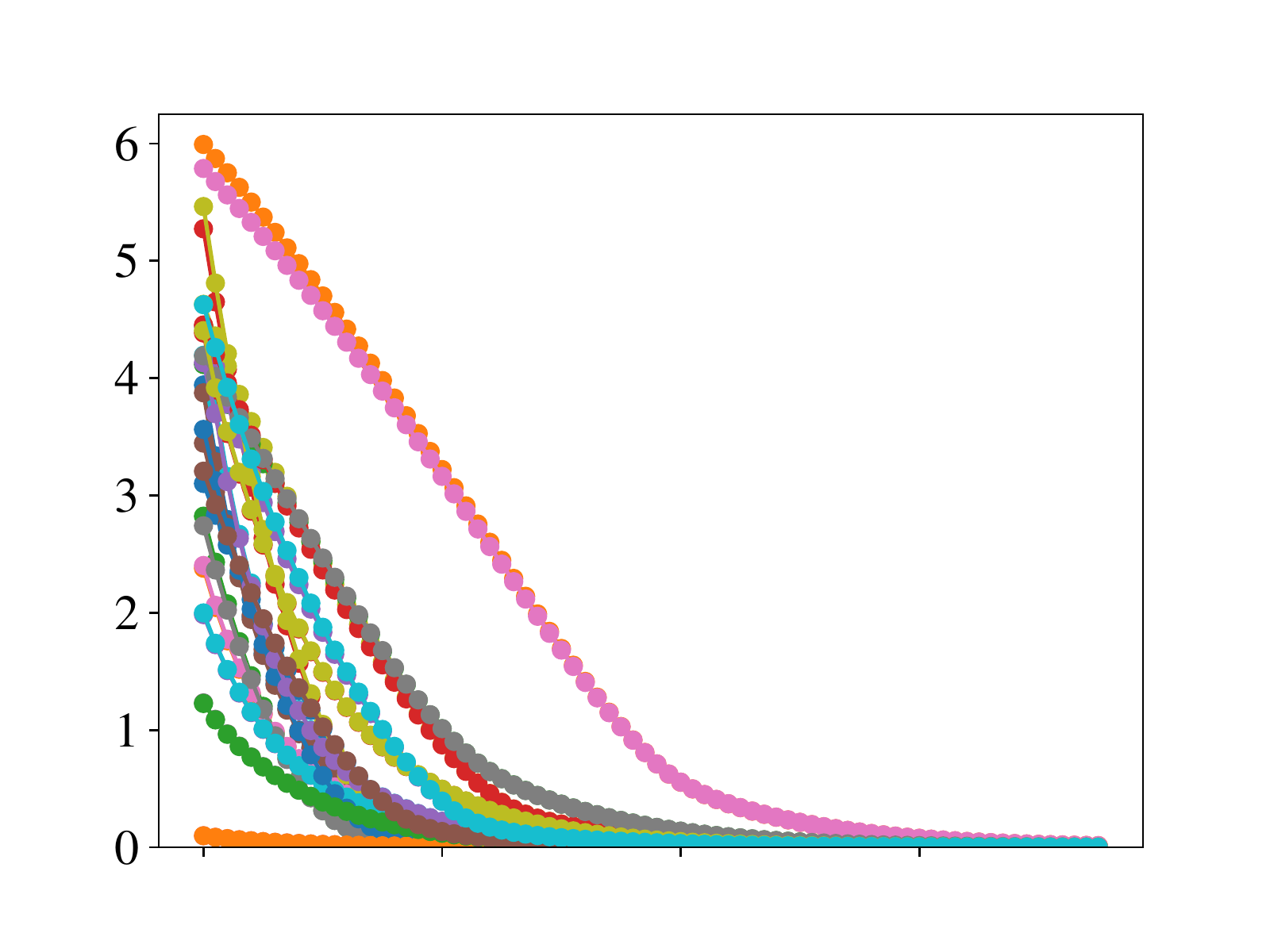}\vspace{-0.1cm}
    \end{subfigure}
    \begin{subfigure}[b]{0.19\textwidth}
    \centering
        \includegraphics[trim={52 25 30 40}, clip, scale=0.2]{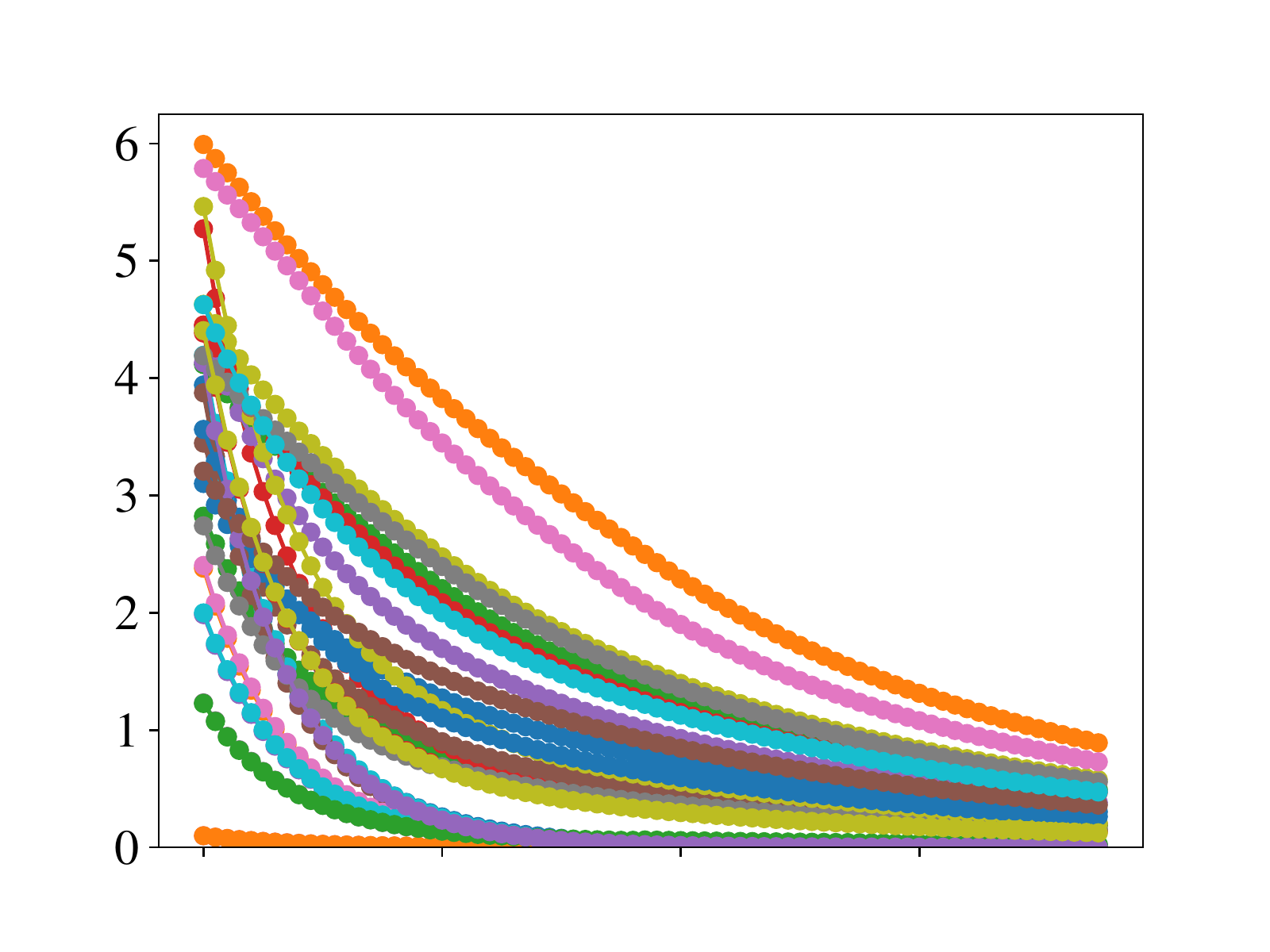}\vspace{-0.1cm}
    \end{subfigure}
    \begin{subfigure}[b]{0.19\textwidth}
    \centering
        \includegraphics[trim={52 25 30 40}, clip, scale=0.2]{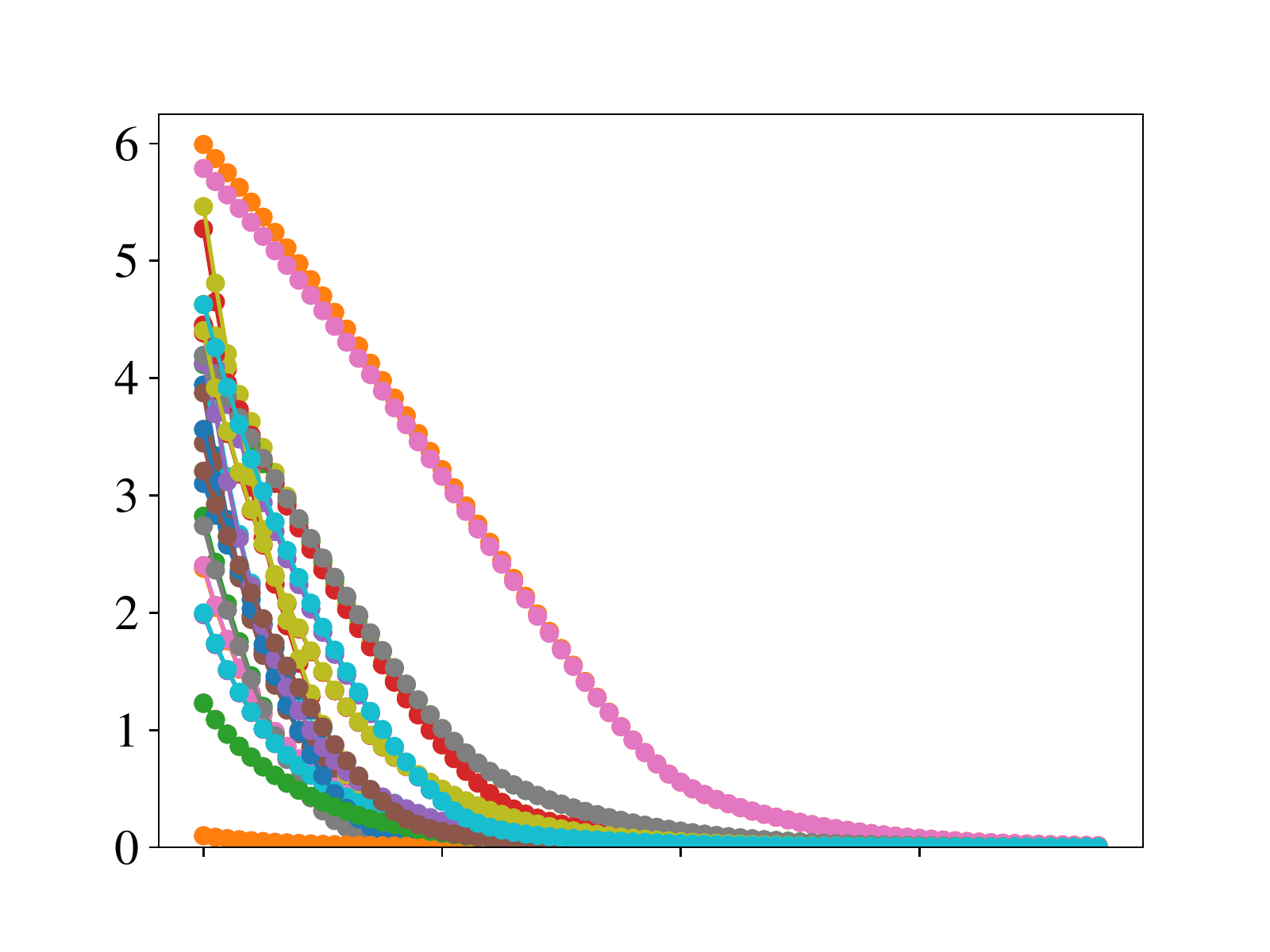}\vspace{-0.1cm}
    \end{subfigure}
        \begin{subfigure}[b]{0.19\textwidth}
    \centering
        \includegraphics[trim={52 25 0 40}, clip, scale=0.2]{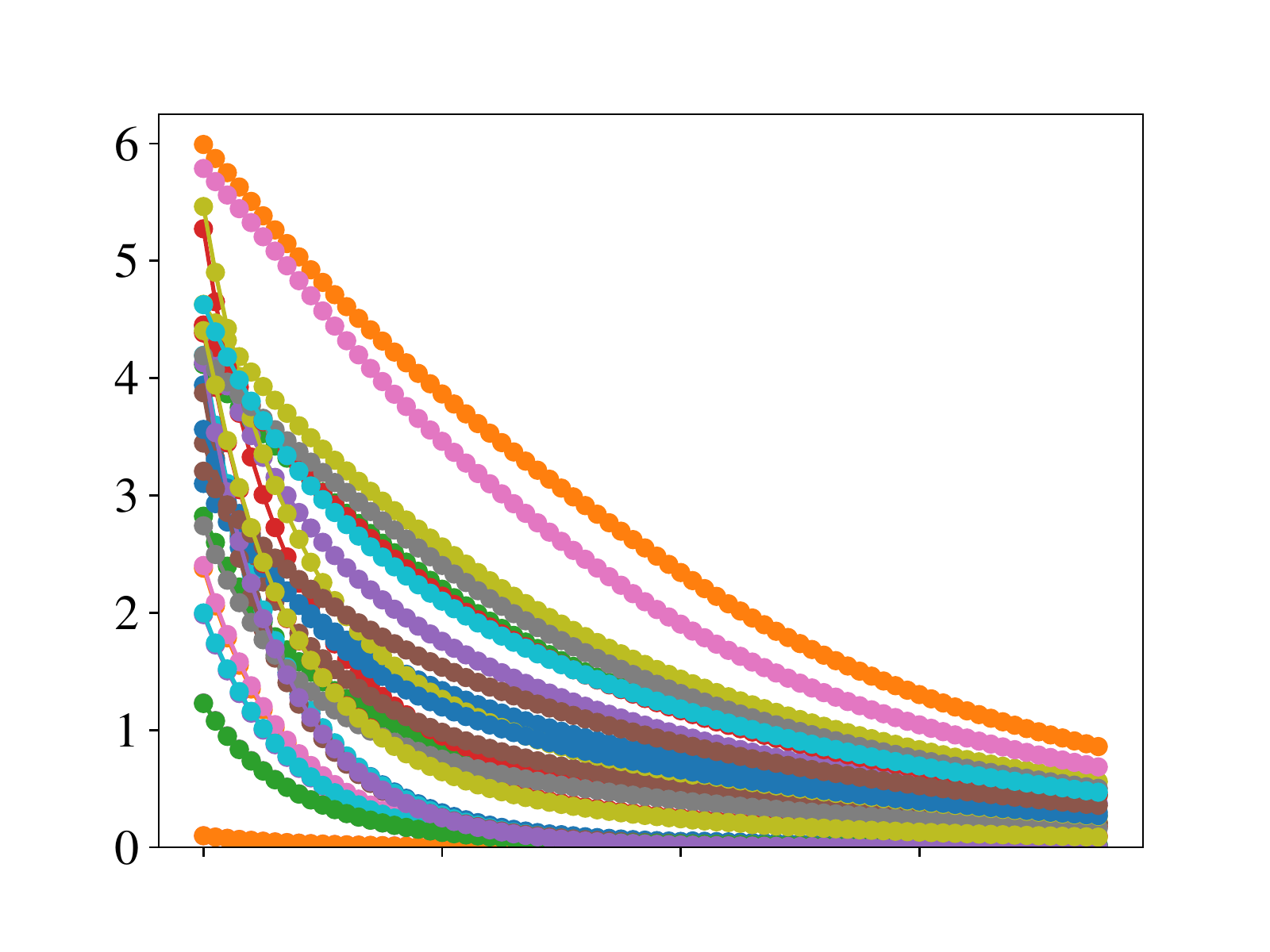}\vspace{-0.1cm}
    \end{subfigure}
    
    \medskip
    
    \begin{subfigure}[b]{0.2\textwidth}
    \centering
        \includegraphics[trim={15 25 30 30}, clip, scale=0.2]{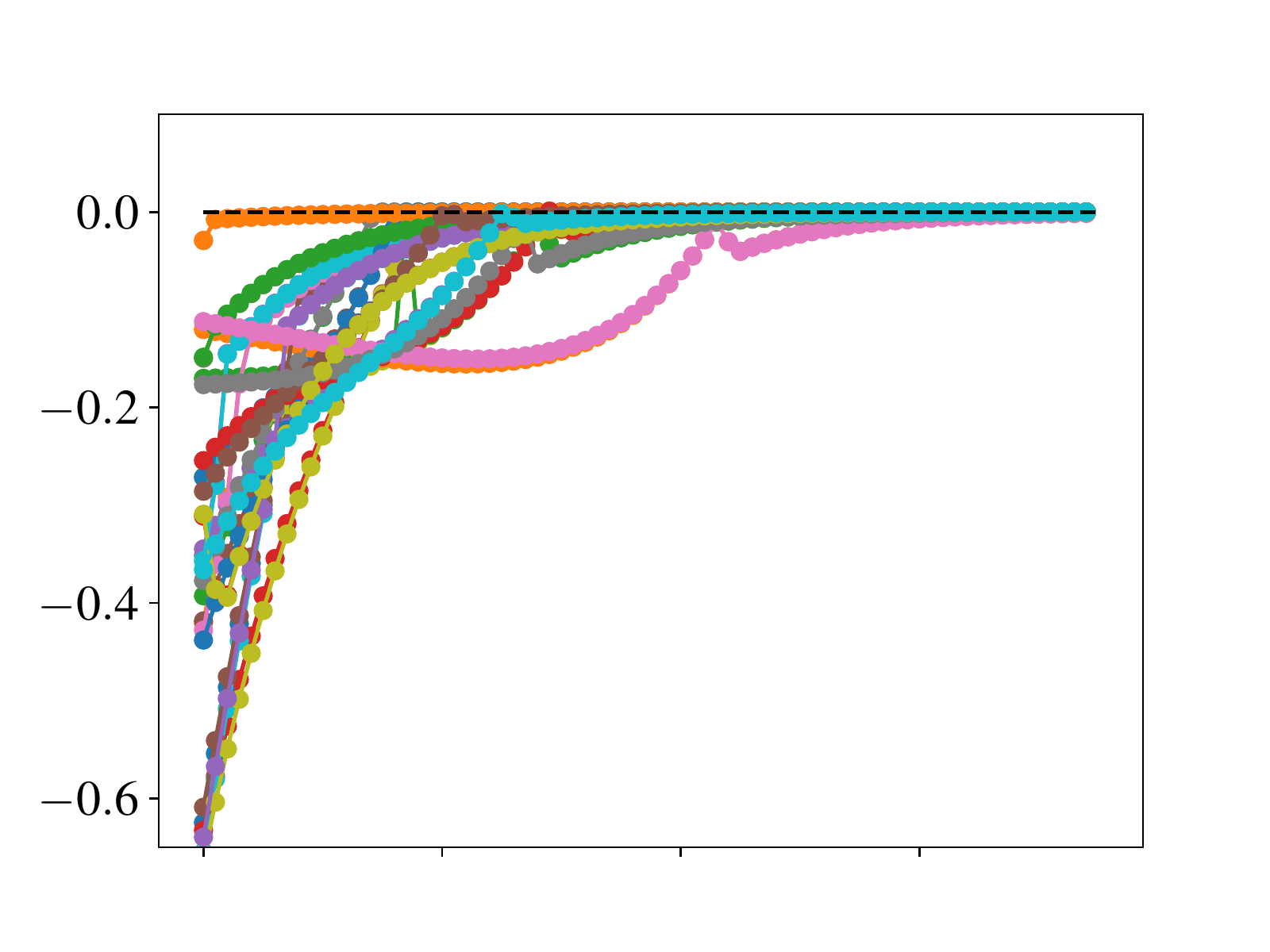}\vspace{-0.1cm}
        \caption{Short-horizon MPC with nominal system}
    \end{subfigure}
    \begin{subfigure}[b]{0.19\textwidth}
    \centering
        \includegraphics[trim={52 25 30 30}, clip, scale=0.2]{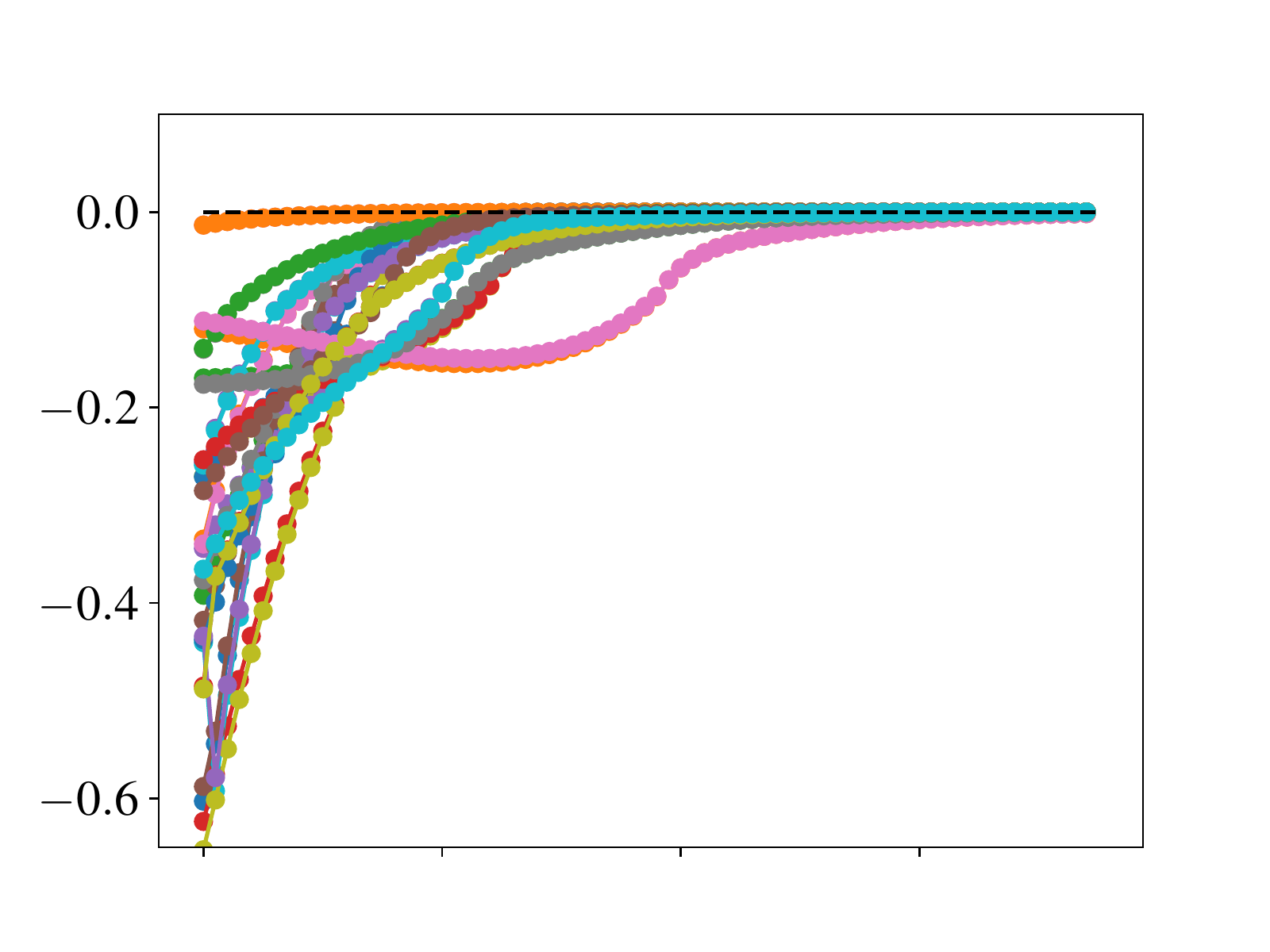}\vspace{-0.1cm}
        \caption{Long-horizon MPC with nominal model}
    \end{subfigure}
    \begin{subfigure}[b]{0.19\textwidth}
        \centering
        \includegraphics[trim={52 25 30 30}, clip, scale=0.2]{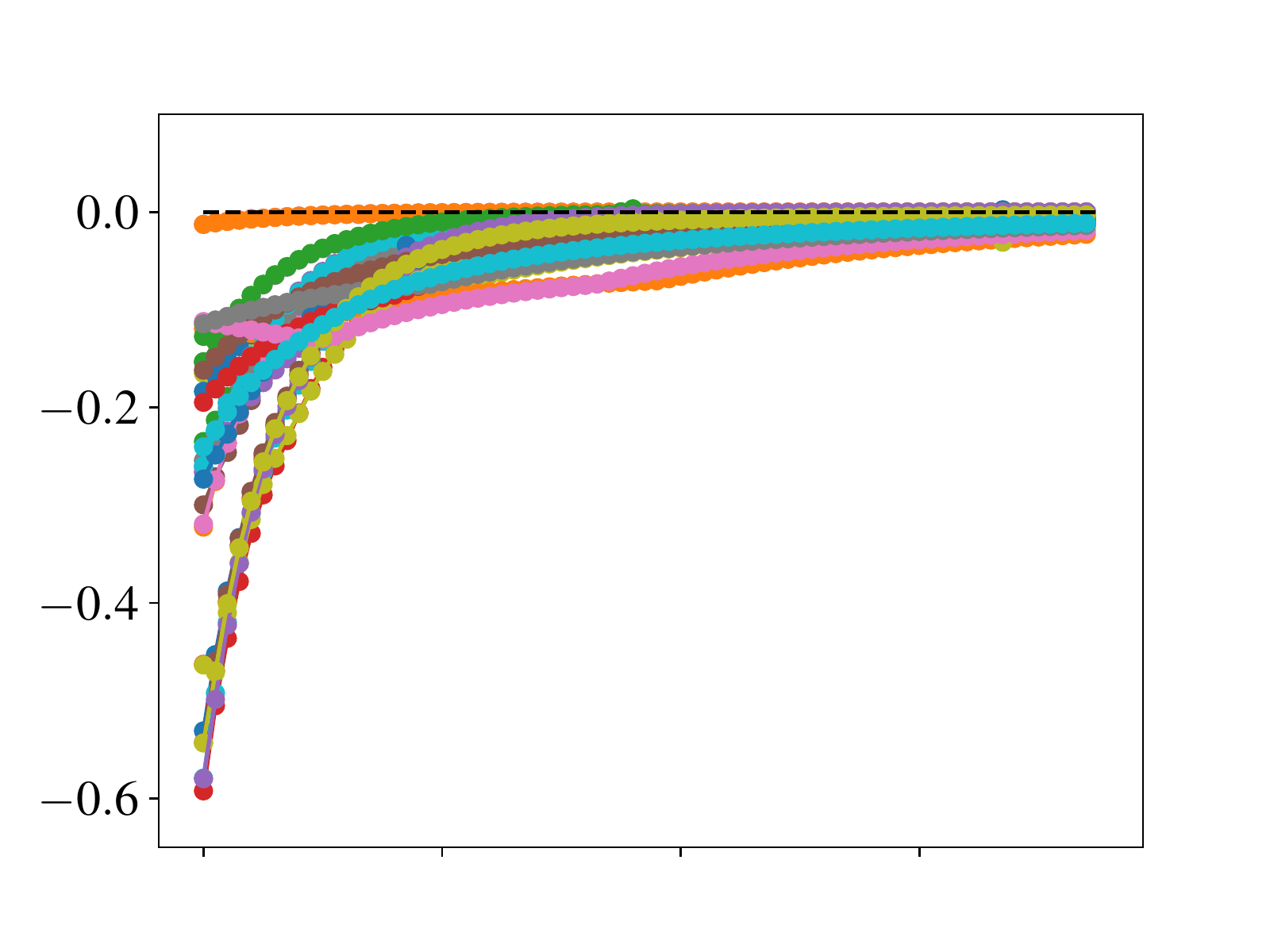}\vspace{-0.1cm}
        \caption{Lyapunov MPC with nominal model \\ \bf{(Ours)}}
    \end{subfigure}
    \begin{subfigure}[b]{0.19\textwidth}
        \centering
        \includegraphics[trim={52 25 30 30}, clip, scale=0.2]{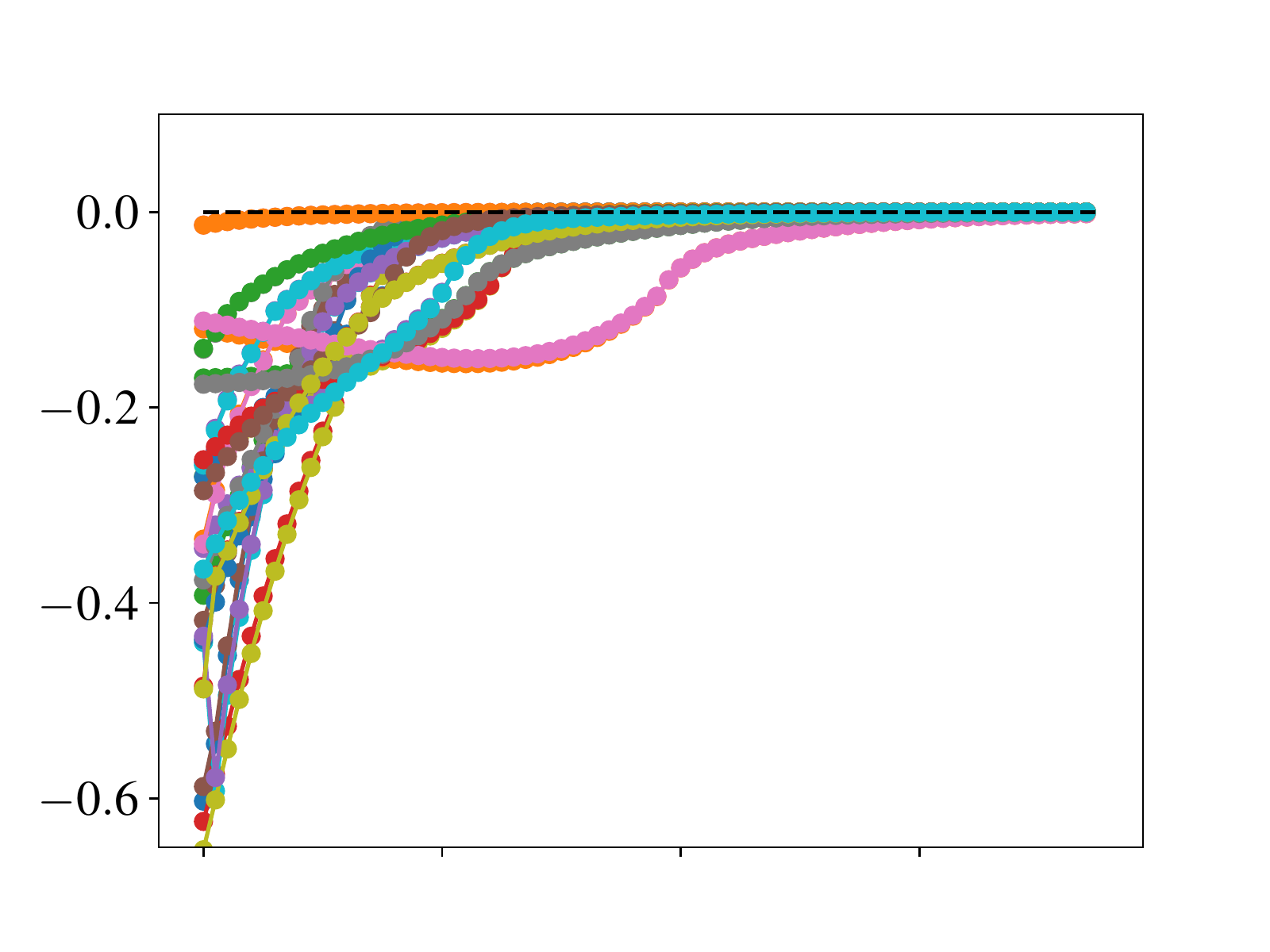}\vspace{-0.1cm}
        \caption{Long-horizon MPC with surrogate model}
    \end{subfigure}
    \begin{subfigure}[b]{0.19\textwidth}
        \centering
        \includegraphics[trim={52 25 0 40}, clip, scale=0.2]{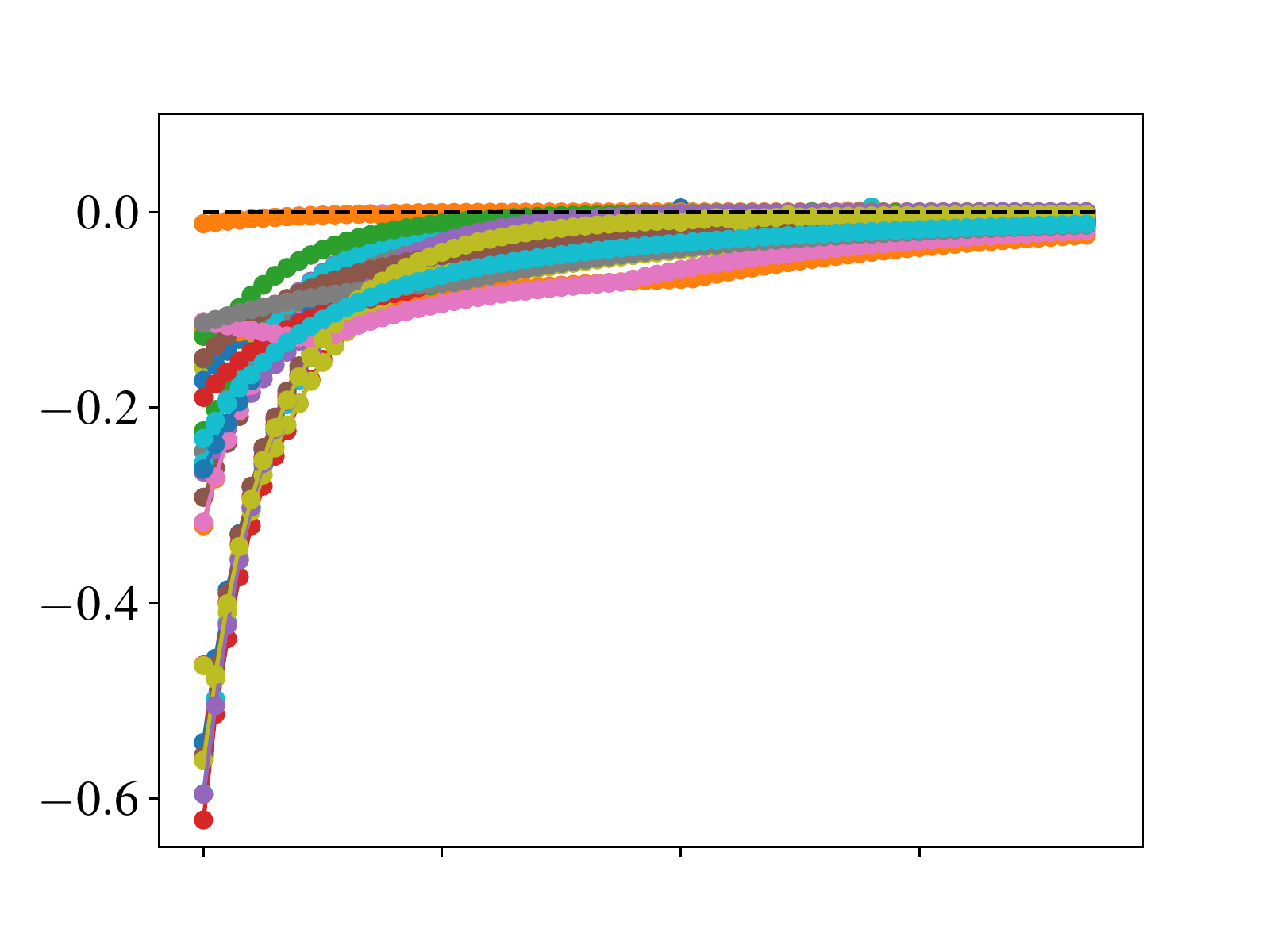}\vspace{-0.1cm}
        \caption{Lyapunov MPC with surrogate model \\ \bf{(Ours)}}
    \end{subfigure}
    \caption{{\bf Inverted Pendulum: Transfer from nominal to surrogate model.} {\bf Top}: The Lyapunov function with overlaid trajectories for $80$ timesteps. {\bf Middle}: The Lyapunov function evaluated along trajectories. {\bf Bottom}: The Lyapunov decrease evaluated along trajectories.}
    \label{fig:pendulum_lyap_big}
\end{figure*}

\begin{figure*}[t]
    \captionsetup[subfigure]{justification=centering}
    \centering
    \begin{subfigure}[b]{0.45\textwidth}
        \centering
        \includegraphics[trim={52 35 102 40}, clip, scale=0.35]{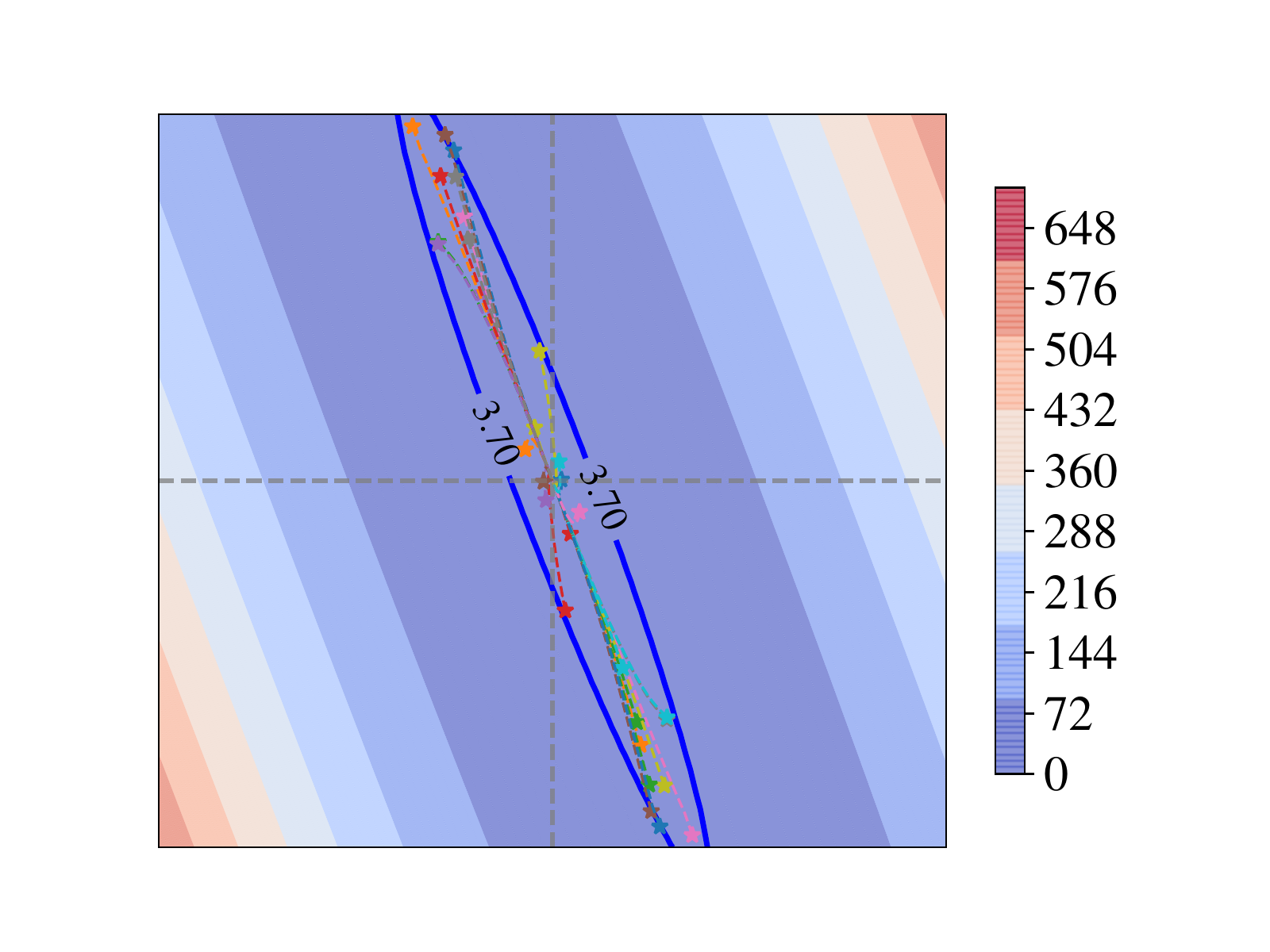}\vspace{-0.1cm}
        \caption{MPC Demonstrator with Nominal Model}
    \end{subfigure}
    \begin{subfigure}[b]{0.45\textwidth}
    \centering
        \includegraphics[trim={0 35 0 40}, clip, scale=0.35]{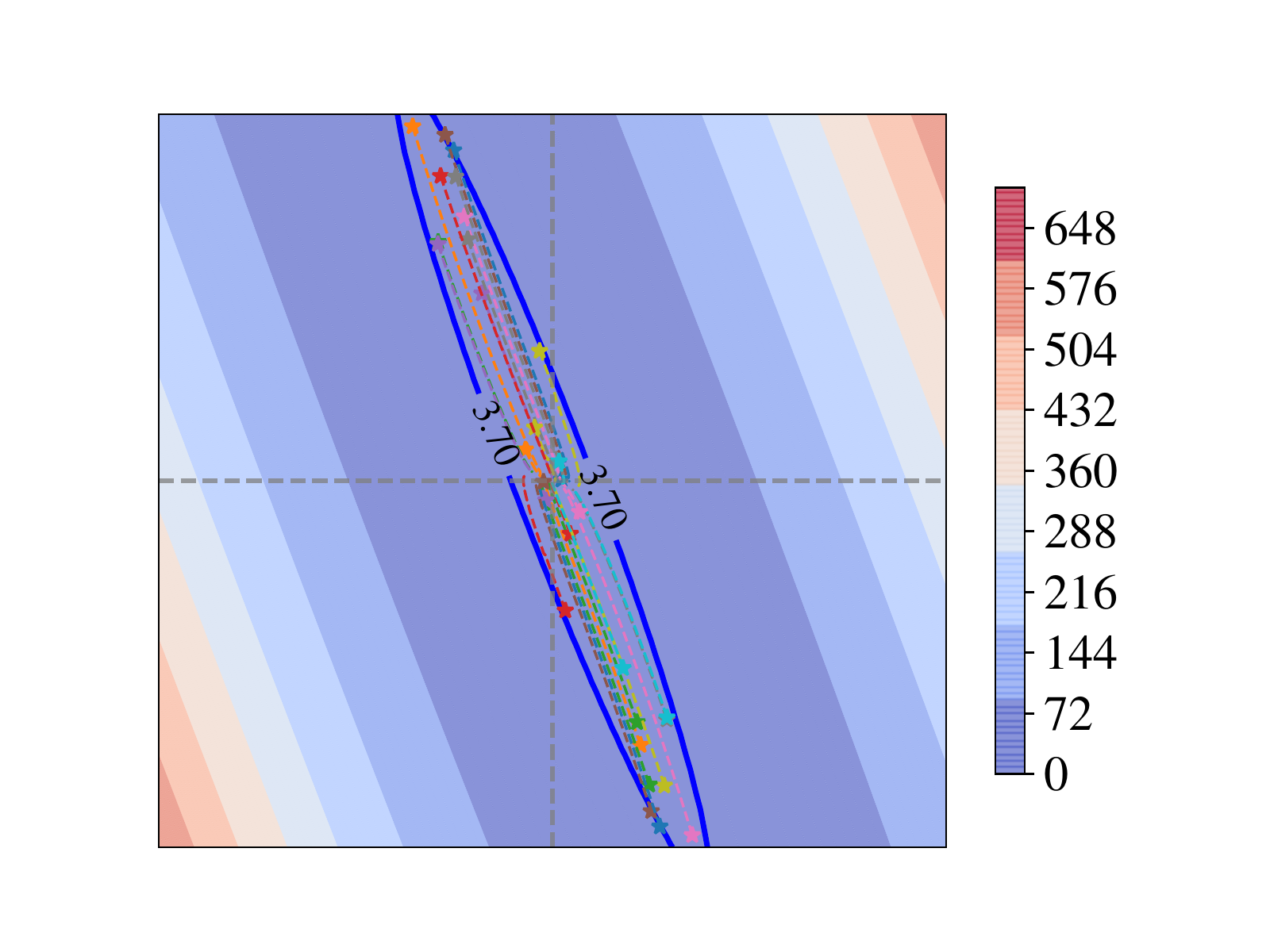}\vspace{-0.1cm}
        \caption{Lyapunov MPC with Nominal Model \\ \bf{(Ours)}}
    \end{subfigure}
    
    \caption{\textbf{Inverted Pendulum: Effect of using Lyapunov MPC with contractor factor and no LQR loss.} The Lyapunov function and safe-level set obtained from the first iteration of alternate learning with $\lambda=0.9, v=0$ in the Lyapunov loss~(\protect\ref{eq:lyapunov_loss}). This results in a smaller safe region estimate and slower closed-loop trajectories compared to the case when $\lambda=0.99, v=1$. Each trajectory is simulated for $T=80$ timesteps.}
    \label{fig:contraction_factor}
\end{figure*}

\begin{figure*}[t]
    \captionsetup[subfigure]{justification=centering}
    \centering
    \begin{subfigure}[b]{0.45\textwidth}
        \centering
        \includegraphics[trim={52 35 102 40}, clip, scale=0.35]{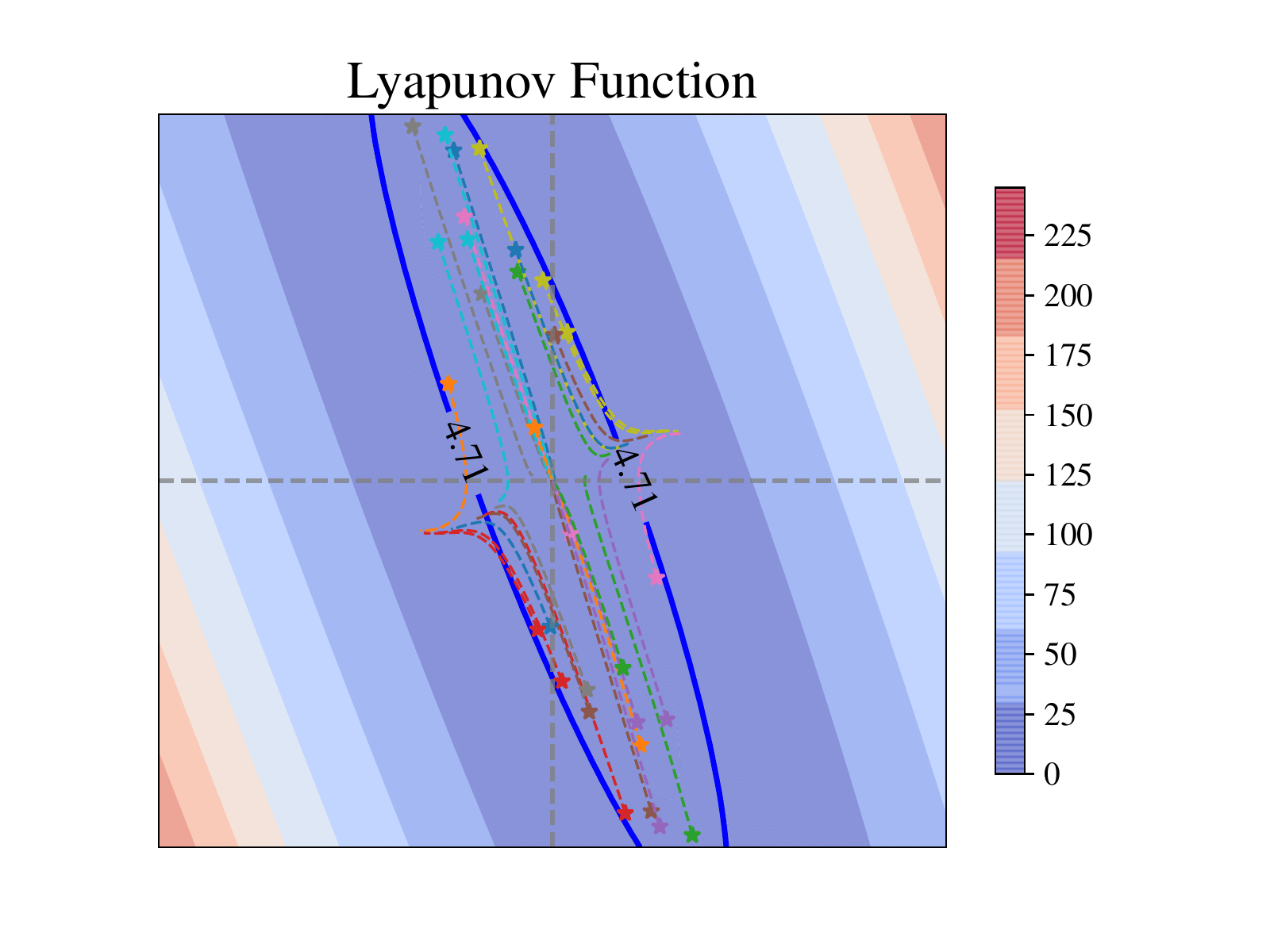}\vspace{-0.1cm}
        \caption{$r_\text{trust}=0.1$}
    \end{subfigure}
    \begin{subfigure}[b]{0.45\textwidth}
    \centering
        \includegraphics[trim={52 35 60 40}, clip, scale=0.35]{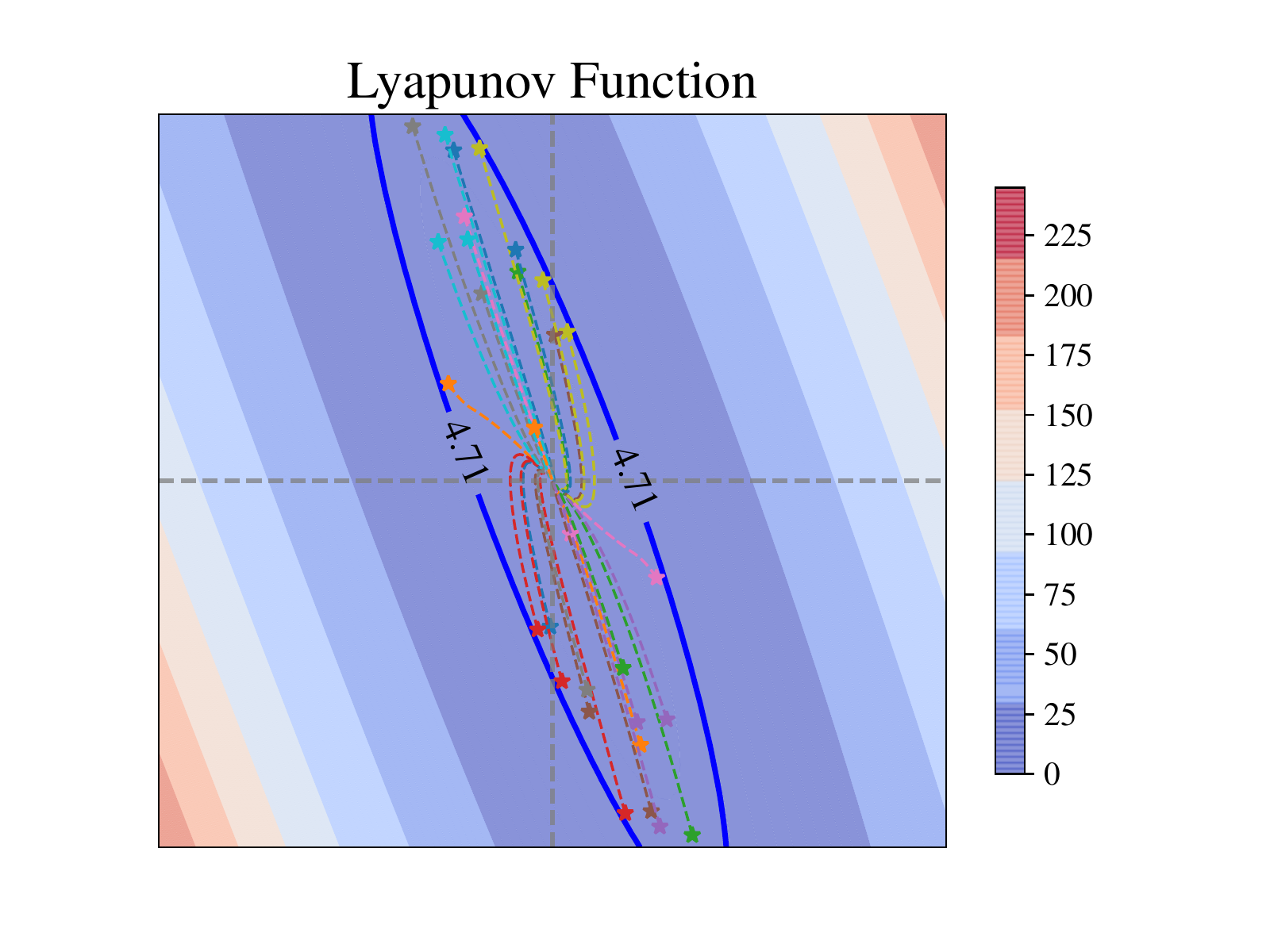}\vspace{-0.1cm}
        \caption{$r_\text{trust}=0.5$}
    \end{subfigure}
    
    \caption{\textbf{Inverted Pendulum: Effect of trust region on MPC.} The solver hyperparameter, $r_{trust}$, can also affect the stability of the MPC and was tuned manually at this stage. Given the limited amount of solver iterations, a small trust region results in weaker control signals and local instability. A larger trust radius can in this case stabilize the system, while being possibly more sub-optimal. The depicted Lyapunov function is obtained from first iteration of alternate learning and is used by the MPC as its terminal cost.}
    \label{fig:trust_region}
\end{figure*}

\end{document}